\documentclass{article}

\usepackage[nonatbib, preprint]{neurips_2023}

\usepackage[utf8]{inputenc} 
\usepackage[T1]{fontenc}    
\usepackage{url}            
\usepackage{booktabs}       
\usepackage{amsfonts}       
\usepackage{nicefrac}       
\usepackage{microtype}      
\usepackage{xcolor}         

\usepackage[pdftex,bookmarksnumbered,bookmarksopen,colorlinks,citecolor={blue!70!black},linkcolor={blue!70!black},urlcolor={blue!60!black}]{hyperref}

\usepackage{graphicx}
\usepackage[bf,small]{caption}
\usepackage{subcaption}
\usepackage{booktabs} 
\usepackage{tcolorbox}
\usepackage{wrapfig}

\usepackage[ruled,vlined]{algorithm2e}
\usepackage{algorithmic}
\usepackage{enumitem}

\usepackage[
backend=biber,
style=numeric-comp,
sorting=ynt,
doi=false,
hyperref=true,
isbn=false,
sortcites,
url=false,
maxbibnames=99,
maxalphanames=3,
minalphanames=3,
maxcitenames=2,
mincitenames=1,
backref=true
]{biblatex}
\renewbibmacro{in:}{}
\AtEveryBibitem{%
  \clearlist{language}%
}
\renewbibmacro*{pageref}{. \iflistundef{pageref}{}{{\printlist[pageref][-\value{listtotal}]{pageref}}}}

\definecolor{commentcolor}{RGB}{110,154,155}   
\newcommand{\PyComment}[1]{\footnotesize\ttfamily\textcolor{commentcolor}{\# #1}}  
\newcommand{\PyCode}[1]{\footnotesize\ttfamily\textcolor{black}{#1}} 

\usepackage{bm}
\usepackage{xparse}
\usepackage{amsmath,amsthm,amssymb}
\usepackage{mleftright}
\usepackage{cleveref}
\crefname{equation}{}{}
\Crefname{equation}{}{}

\usepackage{titlesec}

\NewDocumentCommand{\bp}{m}{\left(#1\right)}
\NewDocumentCommand{\rp}{m}{\mleft(#1\mright)}

\theoremstyle{plain}
\newtheorem{theorem}{Theorem}
\newtheorem{lemma}[theorem]{Lemma}
\newtheorem{proposition}[theorem]{Proposition}
\newtheorem{approximation}[theorem]{Approximation}

\theoremstyle{remark}
\newtheorem{remark}[theorem]{Remark}

\newcommand{\e}{\bm{e}}
\newcommand{\w}{\bm{w}}
\newcommand{\x}{\bm{x}}
\newcommand{\z}{\bm{z}}

\newcommand{\eps}{\varepsilon}
\newcommand{\bR}{\mathbb{R}}

\newcommand{\M}{\bm{M}}

\newcommand{\I}{\bm{I}}
\newcommand{\W}{\bm{W}}
\newcommand{\X}{\bm{X}}
\newcommand{\Z}{\bm{Z}}
\newcommand{\D}{\bm{D}}
\newcommand{\Q}{\bm{Q}}
\newcommand{\K}{\bm{K}}
\newcommand{\V}{\bm{V}}

\renewcommand{\P}{\bm{P}}

\NewDocumentCommand{\logdet}{m}{\operatorname{logdet}\rp{#1}}
\NewDocumentCommand{\mat}{m}{\begin{bmatrix}#1\end{bmatrix}}
\NewDocumentCommand{\softmax}{m}{\operatorname{softmax}\rp{#1}}
\NewDocumentCommand{\diaglr}{m}{\operatorname{diag}\rp{#1}}
\NewDocumentCommand{\explr}{m}{\operatorname{exp}\rp{#1}}

\NewDocumentCommand{\given}{}{\;\ifnumequal{\currentgrouptype}{16}{\middle|}{\mid}\;}

\newcommand{\ISTA}[1]{\operatorname{\texttt{ISTA}}(#1)}

\newcommand{\SSA}[1]{\operatorname{\texttt{SSA}}(#1)}
\newcommand{\MSSA}[1]{\operatorname{\texttt{MSSA}}(#1)}
\newcommand{\ours}{\textsc{crate}}

\input{macros.def}

\newcommand\blfootnote[1]{%
	\begingroup
	\renewcommand\thefootnote{}\footnote{#1}%
	\addtocounter{footnote}{-1}%
	\endgroup
}

\title{
White-Box Transformers via Sparse Rate Reduction 
}

\author{
    Yaodong Yu$^{1}$\, 
    Sam Buchanan$^{2}$\, 
    Druv Pai$^{1}$ \,
    Tianzhe Chu$^{1}$ \,
    Ziyang Wu$^{1}$ \,
    Shengbang Tong$^{1}$
}

\begin{document}
\allowdisplaybreaks

\maketitle

\vspace{-0.42in}
\begin{center}
    \textbf{Benjamin D. Haeffele}$^{3}$ \,
    \textbf{Yi Ma}$^{1}$
    \\
    \vspace{0.15in}
    $^{1}$University of California, Berkeley\quad
    $^{2}$TTIC\quad
    $^{3}$Johns Hopkins University
\end{center}

\blfootnote{\hspace*{-1.5mm}$^{1}$\texttt{\{yyu,yima\}@eecs.berkeley.edu}\texttt{,}\, \texttt{\{druvpai,chutzh,zywu,tsb\}@berkeley.edu}
\\ 
\hspace*{4.2mm}$^{2}$\,\texttt{sam@ttic.edu}
\\ 
\hspace*{4.2mm}$^{3}$\,\texttt{bhaeffele@jhu.edu}}

\begin{abstract} 
In this paper, we contend that the objective of representation learning is  to compress and transform the distribution of the data, say sets of tokens, towards a mixture of low-dimensional Gaussian distributions supported on incoherent subspaces. The quality of the final representation can be measured by a unified objective function called \textit{sparse rate reduction}. From this perspective, popular deep networks such as transformers can be naturally viewed as realizing iterative schemes to optimize this objective incrementally. Particularly, we show that the standard transformer block can be derived from alternating optimization on complementary parts of this objective: the multi-head self-attention operator can be viewed as a gradient descent step to compress the token sets by minimizing their lossy coding rate, and the subsequent multi-layer perceptron can be viewed as attempting to sparsify the representation of the tokens. This leads to a family of \textit{white-box} transformer-like deep network architectures which are mathematically fully interpretable. Despite their simplicity, experiments show that these networks indeed learn to optimize the designed objective: they compress and sparsify representations of large-scale real-world vision datasets such as ImageNet, and achieve performance very close to thoroughly engineered transformers such as ViT. 
Code is at \url{https://github.com/Ma-Lab-Berkeley/CRATE}.
\vspace{-2mm}
\end{abstract}

\setlength{\parskip}{0.4ex}
\setlength{\parindent}{0pt}

\vspace{2mm}
\section{Introduction}\label{sec:intro}
In recent years, deep learning has seen tremendous empirical success in processing
massive amounts of high-dimensional and multi-modal data. Much of this success
is owed to effective learning of the data distribution and then transforming
the distribution to a parsimonious, i.e. \textit{structured and compact},
representation \cite{Radford2021-ir,Chen2020-ha,He2021-lb,ma2022principles},
which facilitates many downstream tasks (e.g., in vision, classification
\cite{He2016-lc,dosovitskiy2020image}, recognition and segmentation
\cite{Carion2020-fm,He2017-be,Kirillov2023-pm}, and generation
\cite{Karras2018-si,rombach2022high,Saharia2022-na}). To this end, many models
and methods have been proposed and practiced, each with its own strengths and
limitations. Here, we give several popular methods a brief accounting as
context for a complete understanding and unification that we seek in this
work.

\paragraph{Transformer models and self-attention.} Transformers \cite{vaswani2017attention} are one of the latest popular models for learning a representation for high-dimensional structured data, such as text \cite{vaswani2017attention,devlin2018bert,brown2020language}, images \cite{dosovitskiy2020image,dehghani2023scaling}, and other types of signals \cite{gong2022contrastive,arnab2021vivit}. After the first block, which converts each data point (such as a text corpus or image) into a set or sequence of \textit{tokens}, further processing is performed on the token sets, in a medium-agnostic manner \cite{vaswani2017attention,dosovitskiy2020image}. A cornerstone of the transformer model is the so-called \textit{self-attention layer}, which exploits the statistical correlations among the sequence of tokens to refine the token representation. Transformers have been highly successful in learning compact representations that perform well on many downstream tasks. Yet the transformer network architecture is empirically designed and lacks a rigorous mathematical interpretation. In fact, the output of the attention layer itself has several competing interpretations \cite{vidal2022attention,li2023theoretical}. As a result, the statistical and geometric relationship between the data distribution and the final representation learned by a transformer largely remains a mysterious black box. 

\paragraph{Diffusion models and denoising.} Diffusion models
\cite{Sohl-Dickstein2015-kz,ho2020denoising,Song2019-ww,Song2020-xo,Song2020-hb}
have recently become a popular method for learning the data distribution,
particularly for generative tasks and natural image data which are highly
structured but notoriously difficult to effectively model
\cite{wakin2005multiscale,Donoho2005-ag}. The core concept of
diffusion models is to start with features sampled from a Gaussian noise
distribution (or some other standard template) and \textit{iteratively denoise}
and deform the feature distribution until it converges to the original data
distribution. This process is computationally intractable if
modeled in just one step \cite{Koehler2022-ed}, so it is typically broken into multiple
incremental steps. The key to each step is the so-called \textit{score
function}, or equivalently \cite{Efron2011-wn} an estimate for the ``optimal
denoising function''; in practice this function is modeled using a generic
black-box deep network. Diffusion models have shown effectiveness at learning
and sampling from the data distribution
\cite{karras2022elucidating,chen2022improved,rombach2022high}. However, despite
some recent efforts \cite{song2023consistency}, they generally do not establish
any clear correspondence between the initial features and data samples. Hence,
diffusion models themselves do not offer a parsimonious or interpretable
representation of the data distribution.

\paragraph{Structure-seeking models and rate reduction.} In both of the
previous two methods, the representations were constructed implicitly as a
byproduct of solving a downstream task (e.g., classification or generation/sampling)
using deep networks. However, one can also explicitly learn a representation of
the data distribution as a task in and of itself; this is most commonly done by
trying to identify and represent low-dimensional structures in the input data.
Classical examples of this paradigm include model-based approaches such as
sparse coding \cite{olshausen1997sparse,chen2018sparse} and dictionary learning \cite{Spielman2012-le,Gribonval2014-zr,zhai2020complete}, out of which grew
early attempts at designing and interpreting deep network architectures
\cite{Papyan2018-qc,Bruna2013-on}. More recent approaches build instead from a
model-free perspective, where one learns a representation through a
sufficiently-informative pretext task (such as compressing similar and
separating dissimilar data in contrastive
learning \cite{tian2020makes,wang2022rethinking,shwartz2023compress}, or
maximizing the information gain in the class of maximal coding rate reduction methods
\cite{ma2007segmentation,OriginalMCR2,chan2021redunet}). 
Compared to black-box deep learning approaches, both model-based and model-free
representation learning schemes have the advantage of being more interpretable: 
they allow users to explicitly design desired properties of the learned
representation \cite{OriginalMCR2,chan2021redunet,pai2022pursuit}. Furthermore,
they allow users to construct new white-box forward-constructed deep network
architectures
\cite{gregor2010learning,chan2021redunet,hinton2022forwardforward} by
\textit{unrolling the optimization strategy for the representation learning
objective}, such that each layer of the constructed network implements an
iteration of the optimization algorithm
\cite{gregor2010learning,chan2021redunet,tolooshams2021stable}. 
Unfortunately, in this paradigm, if the desired properties are narrowly defined, it may be difficult to achieve good practical performance on large  real-world datasets.

\begin{figure}[t!]
     \centering
     \begin{subfigure}[b]{0.99\textwidth}
         \centering
         \includegraphics[width=\textwidth]{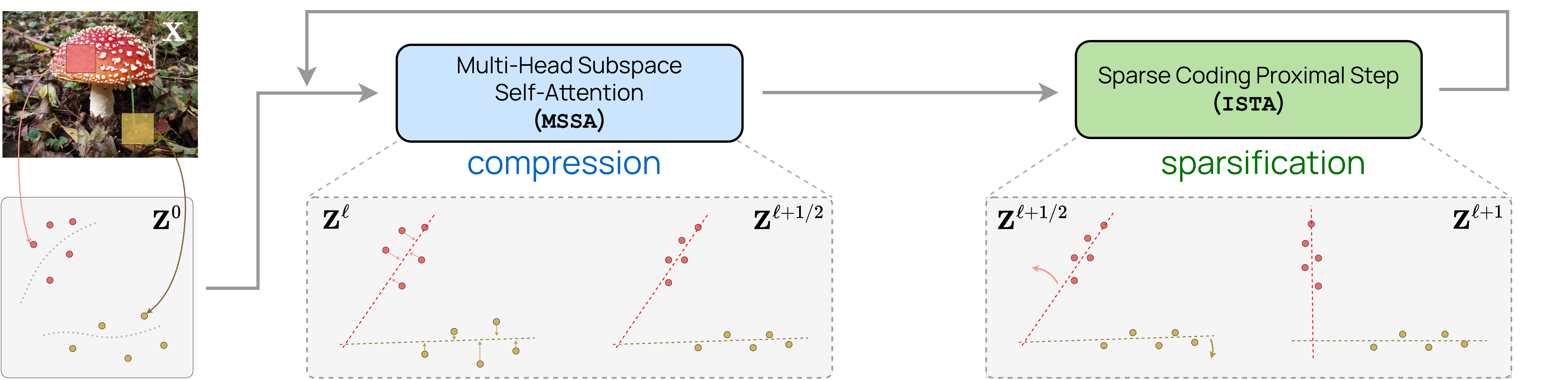}
     \end{subfigure}
     \caption{\textbf{The `main loop' of the \ours{} white-box deep network design.} After encoding input data $\vX$ as a sequence of tokens $\vZ^0$, \ours{} constructs a deep network that transforms the data to a canonical configuration of low-dimensional subspaces by successive {\color{blue!60!black}\bf\textit{compression}} against a local model for the distribution, generating $\vZ^{\ell+1/2}$, and {\color{green!60!black}\bf\textit{sparsification}} against a global dictionary, generating $\vZ^{\ell+1}$. Repeatedly stacking these blocks and training the model parameters via backpropagation yields a powerful and interpretable representation of the data.
        }
        \label{fig:diagram}
        \vspace{-0.1in}
\end{figure}

\paragraph{Our contributions, and outline of this work.} In this work, we aim
to remedy the limitations of these existing methods with a more unified framework for designing transformer-like network architectures that leads to both mathematical interpretability and good practical performance. 
To this end, we propose to learn a sequence of \textit{incremental mappings} to
obtain a most \textit{compressed and sparse} representation for the input data
(or their token sets) that optimizes 
{\em a unified objective function} known as the sparse rate reduction,
specified later in   \eqref{eq:sparse-rr}. The goal of the mapping is
illustrated in \Cref{fig:diagram}.
Within this framework, we unify the above three seemingly disparate approaches and show that \textit{transformer-like deep network layers can be naturally derived from unrolling iterative optimization schemes to incrementally optimize the sparse rate reduction objective.} In particular, our contributions and outline of the paper are as follows: 
\begin{itemize}[leftmargin=1.0cm]
    \item In \Cref{sub:denoising} we show, using an idealized model for the token distribution, that if one \textit{iteratively denoises} the tokens towards a family of low-dimensional subspaces, the associated score function assumes an explicit form similar to a self-attention operator seen in transformers. 
    \item In \Cref{sub:compression} we derive the multi-head self-attention layer as an unrolled gradient descent step to minimize  the lossy coding rate part of the rate reduction, showing another interpretation of the self-attention layer as compressing the token representation. 
    \item In \Cref{sub:sparse} we show that the multi-layer perceptron which immediately follows the multi-head self-attention in transformer blocks can be interpreted as (and replaced by) a layer which incrementally optimizes the remaining part of the sparse rate reduction objective by constructing a sparse coding of the token representations. 
    \item In \Cref{sub:architecture} we use this understanding to create a new white-box (fully mathematically interpretable) transformer architecture called \ours{} (i.e., Coding RAte reduction TransformEr), where each layer performs a \textit{single step} of an alternating minimization algorithm to optimize the sparse rate reduction objective. 
\end{itemize}
Hence, within our framework, the learning objective function, the deep learning architecture, and the final learned representation \textit{all become white boxes} that are fully mathematically interpretable. As the experiments in \Cref{sec:exp} show, the \ours{} networks, despite being simple, can already learn the desired compressed and sparse representations on large-scale real-world datasets and achieve performance on par with much more heavily engineered transformer networks (such as ViT) on a wide variety of tasks (e.g., classification and transfer learning).

\section{Technical Approach and Justification}\label{sec:approach}
\subsection{Objective and  Approach}  \label{sub:formulation}
We consider a general learning setup associated with real-world signals. We have some random variable \(\X = \mat{\x_{1}, \dots, \x_{N}} \in \bR^{D \times N}\) which is our data source; each \(\x_{i} \in \bR^{D}\) is interpreted as a \textit{token}\footnote{For language transformers, tokens roughly correspond to words~\cite{vaswani2017attention}, while for vision transformers, tokens correspond to image patches~\cite{dosovitskiy2020image}.}, and the \(\x_{i}\)'s may have arbitrary correlation structures. We use \(\Z = \mat{\z_{1}, \dots, \z_{N}} \in \bR^{d \times N}\) to denote the random variable which defines our representations. Each \(\z_{i} \in \bR^{d}\) is the representation of the corresponding token \(\x_i\). 
We are given \(B \geq 1\) i.i.d.~samples \(\X_{1}, \dots, \X_{B} \sim \X\), whose tokens are \(\x_{i, b}\). The representations of our samples are denoted \(\Z_{1}, \dots, \Z_{B} \sim \Z\), and those of our tokens are \(\z_{i, b}\). Finally, for a given network, we use \(\Z^{\ell}\) to denote the output of the first \(\ell\) layers when given \(\X\) as input. Correspondingly, the sample outputs are \(\Z_{i}^{\ell}\) and the token outputs are \(\z_{i, b}^{\ell}\).

\paragraph{Objective for learning a structured and compact representation.}
Following the framework of rate reduction \cite{chan2021redunet}, we contend
that the goal of representation learning is to find a feature mapping \(f
\colon \X \in \bR^{D \times N} \to \Z\in \bR^{d \times N}\) which transforms input
data \(\X \in \bR^{D \times N}\) with a potentially nonlinear and multi-modal
distribution to a (piecewise) \textit{linearized and compact} feature
representation \(\Z \in \bR^{d \times N}\). While the joint distribution of
tokens \((\z_{i})_{i = 1}^{N}\) in \(\Z\) may be sophisticated (and
task-specific), we further contend that it is reasonable and practical to
require that the target marginal distribution of individual tokens
\(\z_{i}\) should be highly compressed and structured, amenable for compact coding. Particularly, we require the distribution to be \textit{a mixture of low-dimensional (say \(K\)) Gaussian
distributions}, such that the \(k^{\mathrm{th}}\) Gaussian has mean \(\Zero \in \bR^{d}\), covariance \(\vSigma_{k} \succeq \Zero \in \bR^{d \times d}\), and support spanned by the orthonormal basis \(\vU_{k} \in \bR^{d \times p}\). We denote $\vU_{[K]} = (\vU_k)_{k=1}^K$ to be the set of bases of all Gaussians. 
Hence to maximize the \textit{information gain} \cite{ma2022principles} for the final token representation, we wish to maximize the rate reduction \cite{ma2007segmentation,OriginalMCR2} of the tokens, i.e., $\max_{\Z} \Delta R(\Z; \vU_{[K]}) = R(\Z) - R^c(\Z; \vU_{[K]})$, where \(R\) and \(R^{c}\) are estimates of lossy coding rates to be formally defined in \Cref{eq:coding_rate,eq:def-mcr-parts}. This also promotes token representations \(\z_{i}\) from different Gaussians to be \textit{incoherent} \cite{OriginalMCR2}.

Since rate reduction is an intrinsic measure of goodness for the representation, it is invariant to arbitrary rotations of the representations. Therefore, to ensure the final representations are amenable to more compact coding, we would like to transform the representations (and their supporting subspaces) so that they become \textit{sparse} with respect to the standard coordinates of the resulting representation space.\footnote{That is, having the fewest nonzero entries.}  
The combined rate reduction and sparsification process is illustrated in \Cref{fig:diagram}. Computationally, we may combine the above two goals into a unified objective for optimization:
\vspace{0.05in}
\begin{equation}
   \max_{f \in \mathcal{F}} \mathbb{E}_{\Z}\big[ \Delta R(\Z; \vU_{[K]}) - \lambda \|\bm{Z}\|_0 \big] = \max_{f \in \mathcal{F}} \mathbb{E}_{\Z}\big[ R(\Z) - {R}^c(\Z; \vU_{[K]})- \lambda \|\bm{Z}\|_0 \big]\ \text{s.t.}\ \Z = f(\X),
   \label{eq:sparse-rr}
\end{equation}

where the $\ell^0$ norm $\|\Z\|_0$ promotes the sparsity of the final token representations \(\Z = f(\X)\).\footnote{To simplify the notation, we will discuss the objective for one sample $\X$ at a time with the understanding that we always mean to optimize the expectation.} 
We call this objective ``\textit{sparse rate reduction.}''

\paragraph{White-box deep architecture as unrolled  incremental optimization.}
Although easy to state, each term of the above  objective can be computationally very challenging
to optimize \cite{Wright-Ma-2022,chan2021redunet}. Hence it is natural to take an approximation approach that realizes the global transformation $f$ optimizing 
\Cref{eq:sparse-rr} through a concatenation of multiple, say $L$, simple \textit{incremental and local} operations $f^\ell$ that push the representation distribution towards the desired parsimonious model distribution:
\begin{equation}
f\colon \X
\xrightarrow{\hspace{1mm} f^0 \hspace{1mm}} \Z^0 \rightarrow \cdots \rightarrow \Z^\ell \xrightarrow{\hspace{1mm} f^\ell \hspace{1mm}} \Z^{\ell+1} \rightarrow  \cdots \to \Z^L = \Z,
\label{eq:incremental}
\end{equation}

where $f^0: \bR^{D} \rightarrow \bR^{d}$ is the pre-processing mapping that transforms input tokens $\x_{i} \in \bR^{D}$ to their token representations $\z_{i}^{1} \in \bR^{d}$.  

Each incremental \textit{forward mapping} $\Z^{\ell + 1} = f^\ell(\Z^\ell)$, or a ``layer'', transforms the token distribution to \textit{optimize} the above sparse rate reduction objective \eqref{eq:sparse-rr}, conditioned on the distribution of its input tokens $\Z^\ell$. 
{In contrast to other unrolled optimization approaches such as the ReduNet \cite{chan2021redunet}}, we \textit{explicitly model} the distribution of $\Z^\ell$ at each layer, say as a mixture of linear subspaces or sparsely generated from a dictionary. The model parameters are learned from data (say via \textit{backward propagation} with end-to-end training). This separation of forward ``optimization'' and backward
``learning'' clarifies the mathematical role of each layer as an operator
transforming the distribution of its input, whereas the input distribution is
in turn modeled (and subsequently learned) by the parameters of the layer.

We show that we can derive these incremental, local operations through an unrolled optimization perspective to achieve \Cref{eq:sparse-rr} through \Cref{sub:compression,sub:sparse,sub:architecture}.  Once we decide on using an incremental approach to optimizing
\Cref{eq:sparse-rr}, there are a variety of possible choices to achieve the optimization. Given a model for $\Z^\ell$, say a mixture of subspaces $\vU_{[K]}$, we opt for a two-step \textit{alternating minimization} process with a strong conceptual basis: first in \Cref{sub:compression}, we \textit{compress} the tokens $\vZ^{\ell}$ via a gradient step to minimize the coding rate term $\min_{\Z} {R}^c(\vZ; \vU_{[K]})$; second, in \Cref{sub:sparse}, we \textit{sparsify} the compressed tokens, with a suitably-relaxed proximal gradient step on the difference of the sparsity penalty and the expansion term, i.e., $\min_{\Z}[\lambda \norm{\vZ}_{0} - R(\vZ)]$. Both actions are applied incrementally and repeatedly, as each $f^\ell$ in \Cref{eq:incremental} is instantiated with these two steps. 


\subsection{Self-Attention via Denoising Tokens Towards Multiple Subspaces}\label{sub:denoising}

There are many different ways to optimize the objective \eqref{eq:sparse-rr}
incrementally. In this work, we propose arguably \textit{the most basic}
scheme. To help clarify the intuition behind our derivation and approximation,
in this section (and Appendix \ref{app:proofs-denoising}) we study a largely
idealized model which nevertheless captures the essence of nearly the whole
process and particularly reveals the reason why  self-attention-like operators arise in many contexts. Assume that \(N = 1\), and the single token $\vx$ is drawn i.i.d.\ from an
unknown mixture of Gaussians \((\sN(\Zero, \vSigma_{k}))_{k = 1}^{K}\)
supported on low-dimensional subspaces with orthonormal bases \(\vU_{[K]} =
(\vU_{k})_{k = 1}^{K}\) and corrupted with additive Gaussian noise $\vw \sim
\sN(\Zero, \vI)$, i.e.,
\begin{equation}
    \vx = \vz + \sigma \vw,\label{eq:additive-gaussian-noise}
\end{equation}
where \(\z\) is distributed according to the mixture. Our goal is simply to transform the distribution of the noisy token \(\x\) to the mixture of low-dimensional Gaussians \(\z\). Towards incremental construction of a representation $f$ for this model following \Cref{eq:incremental}, we reason inductively: if $\vz^{\ell}$ is a noisy token \Cref{eq:additive-gaussian-noise} at noise level $\sigma^\ell$, it is natural to produce $\vz^{\ell+1}$ by denoising at the level $\sigma^\ell$. In the mean-square sense, the optimal estimate is $\E{\vz \given \vz^{\ell}}$, which has a variational characterization (e.g.\ \cite{Gyorfi2010-rn}):
\begin{equation}
   \E*{\vz \given \spcdot}
   =
   \argmin_{f 
   }\, \Esub*{\vz, \vw}{\norm*{
         f(\vz + \sigma^{\ell} \vw) - \vz
   }_2^2}.
   \label{eq:denoising-l2}
\end{equation}
Setting $\vz^{\ell+1} = \E{\vz \given \vz^{\ell}}$, \Cref{eq:denoising-l2} thus
characterizes the next stage of \Cref{eq:incremental} in terms of an
optimization objective based on a \textit{local signal model} for $\vz^{\ell}$. 
Moreover, letting $\vx \mapsto q^{\ell}(\vx)$ denote the density of
$\vz^{\ell}$,
Tweedie's formula \cite{Efron2011-wn} allows us to express the optimal
representation solving \Cref{eq:denoising-l2} in closed-form:
\begin{equation}
   \vz^{\ell+1} = \vz^{\ell} + (\sigma^{\ell})^2 \nabla_{\vx} \log
   q^{\ell}(\vz^\ell).
   \label{eq:tweedie}
\end{equation}
Tweedie's formula expresses the optimal representation in terms of
an additive correction (in general a nonlinear function of $\vz^{\ell}$) to the
noisy observations by the gradient of the \textit{log-likelihood} of the
distribution of the noisy observations, giving the optimal representation a
clear interpretation as an incremental perturbation to the current noisy
distribution $q^{\ell}$. 
This connection is well-known in the areas of estimation theory and inverse
problems
\cite{Efron2011-wn,Stein1981-jv,Raphan2011-by,Milanfar2013-js,Kadkhodaie2020-fc,Venkatakrishnan2013-rt,Romano2017-rp},
and more recently has found powerful applications in the training of generative
models for natural images
\cite{Hyvarinen2005-fi,Vincent2011-dr,Sohl-Dickstein2015-kz,Song2020-xo,Song2020-hb}.
Here, we can calculate a closed-form expression for
this \textit{score function} $\nabla_{\vx} \log q^{\ell}$, which, when combined with \Cref{eq:tweedie} and some technical assumptions\footnote{Such as \(\sigma\) being smaller than the nonzero eigenvalues of \(\vSigma_{k}\) and the normalization assumption \(\pi_{i}\det(\vSigma_{i} + \sigma^{2}\I)^{-1/2} = \pi_{j}\det(\vSigma_{j} + \sigma^{2}\I)^{-1/2}\) for all \(i, j \in [K]\), where \(\pi_{k}\) is the mixture proportion for the \(k^{\mathrm{th}}\) Gaussian.}, gives the following approximation (shown in \Cref{app:proofs-denoising}). Let \(\otimes\) denote the Kronecker product; then we have
\begin{equation}
    \vz^{\ell + 1} \approx \mat{\vU_{1}, \dots, \vU_{K}}\left[\diag\mathopen{}\left(\softmax{\frac{1}{2(\sigma^{\ell})^{2}}\mat{\norm{\vU_{1}\adj\z^{\ell}}_{2}^{2} \\ \vdots \\ \norm{\vU_{K}\adj\z^{\ell}}_{2}^{2}}}\right) \otimes \I_{p}\right]\mat{\vU_{1}\adj\z^{\ell} \\ \vdots \\ \vU_{K}\adj\z^{\ell}},
   \label{eq:mog-score-as-attention}
\end{equation}
This operation resembles a self-attention layer in a standard transformer architecture with \(K\) heads, sequence length \(N = 1\), the ``query-key-value'' constructs being replaced by a single linear projection \(\vU_{k}\adj\z^{\ell}\) of the token \(\z^{\ell}\), and the aggregation of head outputs (conventionally modeled by an MLP) done with the two leftmost matrices in \eqref{eq:mog-score-as-attention}. 
We thus derive the following useful interpretation, which we will exploit in
the sequel: \textit{Gaussian denoising against a mixture of subspaces model
leads to self-attention-type layers in the transformation $f$}.
Given an initial sample $\vx$ following the model
\Cref{eq:additive-gaussian-noise}, we can repeatedly apply local
transformations to the distribution with \Cref{eq:mog-score-as-attention} in order to realize the incremental mapping \(f \colon \vx \to \vz\) in
\Cref{eq:incremental}.\footnote{This statement can be made mathematically rigorous by
   exploiting a deep connection between neural ODEs and diffusion models,
following ideas in \textcite{Song2020-xo,Chen2023-an}.} 
These insights will guide us in the design of our white-box transformer
architecture in the upcoming subsections.

\subsection{Self-Attention via Compressing Token Sets through Optimizing Rate Reduction}\label{sub:compression}

In the last subsection, we have seen that the multi-head attention in a
transformer resembles the score-matching operator that aims to 
transform a token $\z^{\ell}$ towards a mixture of subspaces (or degenerate
Gaussians). 
Nevertheless, to carry out such an operation on any data, one needs
to first learn or estimate, typically from finite samples, the parameters of the
mixture of (degenerate) Gaussians, which is known to be a challenging task \cite{ma2007segmentation,Vidal:Springer16}.
This challenge is made even harder because in a typical learning setting, the given set
of tokens are {\em not} i.i.d.\ samples from the mixture of subspaces. The joint
distribution among these tokens can encode rich information about the
data---for example, co-occurrences between words or object parts in 
language and image data (resp.)---which we should also learn. 
Thus, we should compress\,/\,denoise\,/\,transform such a set of
tokens together. To this end, we need a measure of quality, i.e., compactness, for the resulting representation of the set of tokens.

A natural measure of the compactness of such a set of tokens
is the (lossy) coding rate to encode them up to a certain precision \(\eps > 0\) \cite{ma2007segmentation, OriginalMCR2}. For a
zero-mean Gaussian, this measure takes a closed form. If we view the tokens in
\(\Z \in \bR^{d\times N}\) as drawn from a single zero-mean Gaussian, an estimate of their
(lossy) coding rate, subject to quantization precision $\eps >0$, is given in 
\cite{ma2007segmentation} as:
\begin{equation}\label{eq:coding_rate}
    R(\Z) \doteq \frac{1}{2}\logdet{\I + \frac{d}{N\eps^{2}}\Z\adj\Z} = \frac{1}{2}\logdet{\I + \frac{d}{N\eps^{2}}\Z\Z\adj}.
\end{equation}
In practice, the data distribution is typically multi-modal, say an image
set consisting of many classes or a collection of image patches as in \Cref{fig:diagram}. It is more appropriate to require that the set
of tokens map to a mixture of, say $K$, subspaces (degenerate Gaussians)
\cite{chan2021redunet}. As before we  denote the (to be learned) bases of these subspaces as $\vU_{[K]} = (\vU_k)_{k=1}^{K}$, where $\vU_k \in
\mathbb{R}^{d \times p}$. Although the joint distribution of the tokens $\Z$ is
unknown, the desired marginal distribution of each token $\z_{i}$ is a mixture of
subspaces. So we may obtain an upper bound of the coding rate for the token set
$\Z$ by  projecting its tokens onto these subspaces and summing up the respective
coding rates:   
\begin{equation}\label{eq:def-mcr-parts}
    {R}^c(\Z; \vU_{[K]}) = \sum_{k=1}^{K}R(\vU_k\adj \Z) =
    \frac{1}{2}\sum_{k=1}^{K}\logdet{\I +
    \frac{p}{N\eps^{2}}(\vU_k\adj\Z)\adj(\vU_k\adj\Z)}.
\end{equation}

We would like to compress (or denoise) the set of tokens against these
subspaces by minimizing the coding rate. 
The gradient of ${R}^c(\vZ; \vU_{[K]})$ is
\begin{equation}
    \nabla_{\Z}{R}^c(\vZ; \vU_{[K]})
    = \frac{p}{N\eps^{2}}\sum_{k=1}^K \vU_k\vU_k\adj\Z\bp{\I +
    \frac{p}{N\eps^{2}}(\vU_k\adj\Z)\adj(\vU_k\adj\Z)}^{-1}.
    \label{eq:rate-gradient}
\end{equation}

The above expression approximates the residual of each projected token $\vU_k^* \z_{i}$ regressed by other tokens \(\vU_{k}\adj\z_{j}\) \cite{chan2021redunet}. 
But, differently from \cite{chan2021redunet}, not all tokens in $\Z$ are from the same subspace. Hence, to denoise each token with tokens from its own group, we can compute their similarity through an auto-correlation among the projected tokens as $(\vU_{k}\adj\Z)\adj(\vU_{k}\adj\Z)$ and convert it to a distribution of membership with a softmax, namely $\softmax{(\vU_{k}\adj\Z)\adj(\vU_{k}\adj\Z)}$. 
Then, as we show in \Cref{app:proofs-compression}, if we only use similar tokens to regress and denoise each other, then a gradient step on the coding rate with learning rate \(\kappa\) can be naturally approximated as follows:
\begin{equation}
    \Z^{\ell + 1/2} = \Z^{\ell} - \kappa\nabla_{\Z}R^{c}(\Z^{\ell}; \vU_{[K]}) \approx \left(1-\kappa\cdot\frac{p}{N\eps^{2}}\right) \Z^{\ell} + \kappa \cdot \frac{p}{N\eps^{2}}\cdot \MSSA{\Z^{\ell} \given \vU_{[K]}},  \label{eq:gd-mcr-parts} 
\end{equation}
where $\texttt{MSSA}$ is defined through an $\texttt{SSA}$ operator as:
\begin{align}
    \SSA{\Z \mid \vU_{k}} 
    &\doteq (\vU_{k}\adj \Z)\softmax{(\vU_{k}\adj\Z)\adj(\vU_{k}\adj\Z)}, \quad k \in [K], \label{eq:SSA} \\
    \MSSA{\Z \mid \vU_{[K]}} 
    &\doteq \frac{p}{N\eps^{2}} \cdot \mat{\vU_{1}, \dots, \vU_{K}}\mat{\SSA{\Z \mid \vU_{1}} \\ \vdots \\ \SSA{\Z \mid \vU_{K}}}.\label{eq:Multi-Head-SSA}
\end{align}

Here the \texttt{SSA} operator in \eqref{eq:SSA} resembles the {\em attention operator}  in a typical
transformer \cite{vaswani2017attention}, except that here the linear operators of value, key, and query  are
all set to be {\em the same} as the subspace basis, i.e., $\V = \K = \Q = \vU_k^*$.\footnote{We note a recent suggestion of    \textcite{hinton2021represent} that it is more sensible to set the
``value, key, and query''  projection matrices in a transformer to be equal. Our derivation in this section confirms this mathematically.} Hence, we name $\SSA{\spcdot|\vU_k}: \bR^{d\times N} \rightarrow \bR^{p\times N}$  the \textbf{S}ubspace \textbf{S}elf-\textbf{A}ttention (SSA) operator (more details and justification can be found in \Cref{eq:appendix-ssa-derivation} in \Cref{app:proofs-compression}). Then, the whole \texttt{MSSA} operator in \eqref{eq:Multi-Head-SSA}, formally defined as \(\MSSA{\spcdot | \vU_{[K]}} \colon \bR^{d \times N} \to \bR^{d \times N}\) and called the \textbf{M}ulti-Head \textbf{S}ubspace \textbf{S}elf-\textbf{A}ttention (MSSA) operator, aggregates the attention head outputs by averaging using model-dependent weights, similar in concept to the popular multi-head self-attention operator in existing transformer networks. The overall gradient step \eqref{eq:gd-mcr-parts} resembles the multi-head self-attention implemented with a skip connection in transformers. 

Notice that if we have \(N = 1\) tokens as well as take an aggressive gradient step (\(\kappa = 1\)) and tune the quantization error (\(\eps = \sqrt{p/N}\)), the multi-head subspace self-attention operator in \Cref{eq:Multi-Head-SSA} becomes the ideal denoiser defined in \Cref{eq:mog-score-as-attention}, with the one minor difference that the aggregation of the heads is done by a linear function here, while in \Cref{eq:mog-score-as-attention} it is done by a nonlinear mixture-of-experts type function.\footnote{This suggests that we could also consider such a mixture of expert type aggregation of the multiple attention heads. In this work, we use linear aggregation, and leave evaluation of more variants for future work.} This provides two very related interpretations of the multi-head self-attention operator, as denoising and compression against a mixture of low-dimensional subspaces.

\subsection{
MLP via Iterative Shrinkage-Thresholding Algorithms (ISTA) for Sparse Coding 
}
\label{sub:sparse}

In the previous subsection, we focused on how to compress a set of tokens against
a set of (learned) low-dimensional subspaces.  
Optimizing the remaining terms in the sparse rate reduction objective
\eqref{eq:sparse-rr}, including the non-smooth term, serves to sparsify
the compressed tokens, 
hence leading to a more compact and structured (i.e., \textit{parsimonious})
representation. From \Cref{eq:sparse-rr,eq:coding_rate}, this term is
\begin{equation}
   \max_{\Z}\, [ R(\Z) - \lambda \|\bm{Z}\|_0 ] = \min_{\Z} \left[\lambda \|\bm{Z}\|_0 - \frac{1}{2}\logdet{\I + \frac{d}{N\eps^{2}}\Z\adj\Z}\right],
   \label{eq:whole-sparse}
\end{equation}

where \(R(\Z)\) denotes the coding rate of the whole token set, as defined in \eqref{eq:coding_rate}. 
In addition to sparsification via the \(\norm{\Z}_{0}\) term, the expansion term $R(\vZ)$ in
\Cref{eq:whole-sparse} promotes diversity and non-collapse of the
representation, a highly desirable property. However, prior work has
struggled to realize this benefit on large-scale datasets due to poor
scalability of the gradient $\nabla_{\vZ} R(\vZ)$, which requires a matrix
inverse \cite{chan2021redunet}. 

To simplify things, we therefore take a different approach to trading off
between representational diversity and sparsification: we posit a
(complete) incoherent or orthogonal dictionary $\vD \in \bbR^{d \times d}$, and ask to sparsify the intermediate iterates $\vZ^{\ell+1/2}$ with respect to \(\vD\). That is, $\vZ^{\ell+1/2} = \vD \vZ^{\ell+1}$ where $\vZ^{\ell+1}$ is more sparse. The dictionary \(\vD\) is global, i.e., is used to sparsify all tokens simultaneously.

By the incoherence assumption, we have $\D^* \D \approx \I_{d}$; thus from \Cref{eq:coding_rate} we have $R(\Z^{\ell+1}) \approx R(\D\vZ^{\ell+1}) = R(\vZ^{\ell+1/2})$. Thus we approximately solve \Cref{eq:whole-sparse} with the following program:
\begin{equation}
   \Z^{\ell + 1} = \argmin_{\Z}  \|\bm{Z}\|_0 \quad \mbox{subject to} \quad \Z^{\ell+{1}/{2}} = \D \Z.
   \label{eq:sparse}
\end{equation}
The above sparse representation program is usually solved by relaxing it to an unconstrained convex program, known as LASSO: 
\begin{equation}
    \Z^{\ell + 1} = \argmin_{\Z} \Big[\lambda \|\bm{Z}\|_1 + \|\Z^{\ell+{1}/{2}} - \D \Z\|_F^2 \Big].
\end{equation}
In our implementation, motivated by \textcite{sun2018supervised, zarka2019deep}, we also add a non-negative constraint to $\Z^{\ell+1}$,
\begin{equation}
    \Z^{\ell+1} = \argmin_{\Z \geq \Zero}  \Big[\lambda\|\Z\|_1 + \|\Z^{\ell+1/2} - \D \Z\|_{F}^{2}\Big],
    \label{eq:sparse-nonnegative}
\end{equation}
which we then incrementally optimize by performing an unrolled proximal gradient descent step, known as an ISTA step \cite{beck2009fast}, to give the update:
\begin{equation}\label{eq:ista-block}
        \vZ^{\ell+1} = \operatorname{ReLU}(\Z^{\ell+1/2} + \eta \D\adj(\Z^{\ell+1/2} - \vD\Z^{\ell+1/2}) - \eta \lambda \bm{1}) \doteq \ISTA{\Z^{\ell+1/2} \mid \vD}.
\end{equation}
In \Cref{app:proofs-sparse}, we will show one can arrive at a similar operator to the above ISTA-like update for optimizing \eqref{eq:whole-sparse} by properly linearizing and approximating the rate term $R(\Z)$.

\subsection{The Overall White-Box \ours{} Architecture}
\label{sub:architecture}
\begin{figure}[t!]
     \centering
     \begin{subfigure}[b]{0.98\textwidth}
         \centering
         \includegraphics[width=\textwidth]{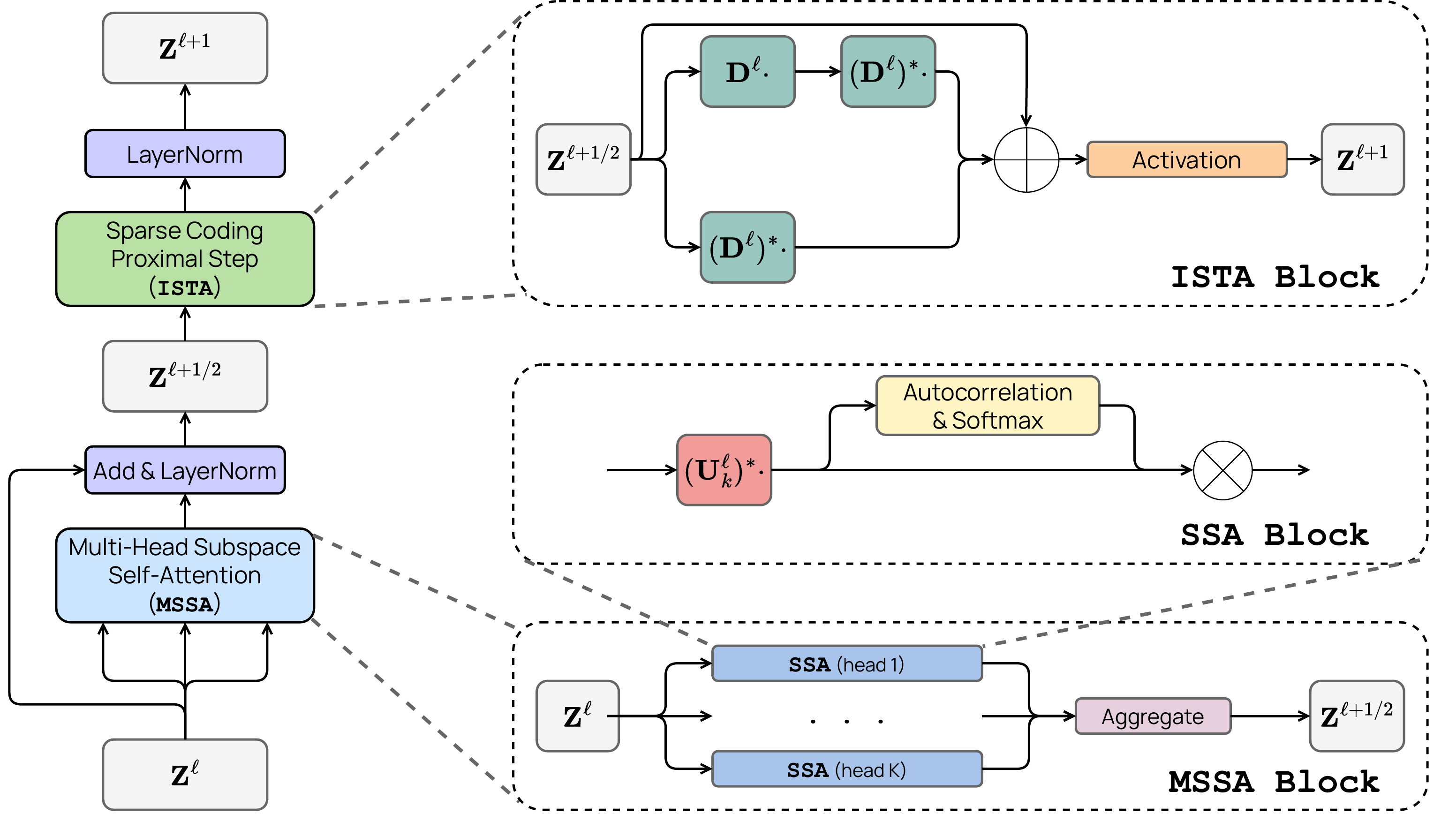}
     \end{subfigure}
        \caption{\small One layer of the CRATE architecture. The full architecture is simply a concatenation of such layers, with some initial tokenizer and final task-specific architecture (i.e., a classification head).
        }
        \label{fig:arch}
        \vspace{-0.1in}
\end{figure}

By combining the above two steps:
\begin{enumerate}[leftmargin=0.7cm]
    \item (\Cref{sub:denoising,sub:compression}) Local denoising and compression of tokens within a sample towards a mixture-of-subspace structure, leading to the multi-head subspace self-attention block -- \texttt{MSSA};
    \item (\Cref{sub:sparse}) Global compression and sparsification of token sets across all samples through sparse coding, leading to the sparsification block -- \texttt{ISTA};
\end{enumerate}
we can get the following rate-reduction-based transformer layer, illustrated in Figure \ref{fig:arch}, 
\begin{equation}
\Z^{\ell+1/2} \doteq \Z^{\ell} + \texttt{MSSA}(\Z^{\ell} \mid \vU_{[K]}^{\ell}), 
\qquad 
\Z^{\ell+1}\doteq \texttt{ISTA}(\Z^{\ell+1/2} \mid \D^\ell).
\end{equation}
Composing multiple such layers following the incremental construction of our representation in \Cref{eq:incremental}, we obtain a white-box transformer architecture that transforms the data tokens towards a compact and sparse union of incoherent subspaces.

This model has the parameters \((\vU_{[K]}^{\ell})_{\ell = 1}^{L}\) and \((\vD^{\ell})_{\ell = 1}^{L}\), which are learned from data via \textit{back-propagation}. Notably, in each layer \(\ell\), the learned \(\vU_{[K]}^{\ell}\) retain their interpretation as incoherent bases for supporting subspaces for the mixture-of-Gaussians model at layer \(\ell\), and the learned \(\vD^{\ell}\) retains its interpretation as a sparsifying dictionary at layer \(\ell\). {We emphasize that the parameters \(\vU_{[K]}^{\ell}\) and \(\vD^{\ell}\) are dependent on the layer \(\ell\) --- that is, we learn a different set of parameters at each layer. This is because at each layer we learn an approximate local parametric model for the input data distribution, then use that learned model to construct the layer operators that transform the distribution. Our procedure of parameterizing the data distribution at each layer distinguishes this work from previous works on unrolled optimization for neural networks such as the ReduNet  \cite{chan2021redunet}.} Our interpretation clarifies the roles of the network forward pass (given local signal models at each layer, denoise/compress/sparsify the input) and the backward pass (learn the local signal models from data via supervision). 

We note that in this work, at each stage of our construction, we have chosen arguably the \textit{simplest possible} construction to use. We can substitute each part of this construction, so long as the new part maintains the same conceptual role, and obtain another white-box architecture. Nevertheless, our such-constructed architecture, called \ours{} (i.e., Coding RAte TransformEr), connects to existing transformer models, obtains competitive results on real-world datasets, and is fully mathematically interpretable.

\section{Experiments}\label{sec:exp}

In this section, we conduct experiments to study the performance of our proposed white-box transformer \ours{} on real-world datasets and tasks. As the analysis in \Cref{sec:approach} suggests, either the compression or the sparsification step can be achieved through various alternative design choices or strategies. \ours{} arguably adopts the most basic choices and so our goal with the experiments is {\em not} simply to compete with other heavily engineered transformers while using such a rudimentary design. Rather, our goals are twofold. 
First, unlike any empirically designed black-box networks that are usually evaluated only on end-to-end performance, the white-box design of our network allows us to \textit{look inside} the deep architecture and verify if layers of the learned network indeed perform their design objective---say performing incremental optimization for the objective \Cref{eq:sparse-rr}. 
Second, despite their simplicity, our experiments will actually reveal the vast practical potential of our so-derived \ours{} architectures since, as we will show, they already achieve very strong performance on large-scale real-world datasets and tasks. 
In the remainder of this section we highlight a selection of results; additional experimental details and results can be found in \Cref{sec:appendix-exp}.

\paragraph{Model architecture.} 
We implement the architecture that is described in \Cref{sub:architecture}, with minor modifications that are described in \Cref{sub:appendix-exp-details}. We consider different model sizes of \ours{} by varying the token dimension $d$, number of heads $K$, and the number of layers $L$. 
We consider four model sizes in this work: \ours{-Tiny}, \ours{-Small}, \ours{-Base}, and \ours{-Large}. A PyTorch-style pseudocode can be found in \Cref{sub:appendix-exp-details}, which contains more implementation details.
For training using supervised classification, we first take the \texttt{CLS} token $\overline{\z}_{b} = \z_{1, b}^{L + 1}$ of for each sample, 
then apply a linear layer; the output of this linear layer \(\vu_{b} \doteq \vW\overline{\z}_{b}\) is used as input to the standard cross-entropy loss. The overall loss averages over all samples \(b \in [B]\).

\paragraph{Datasets and optimization.} We mainly consider ImageNet-1K~\cite{deng2009imagenet} as the testbed for our architecture. 
Specifically, we apply the Lion optimizer~\cite{chen2023symbolic} to train \ours{} models with different model sizes. 
Meanwhile, we also evaluate the transfer learning performance of \ours{}: by considering the models trained on ImageNet-1K as pre-trained models, we fine-tune \ours{} on several commonly used downstream datasets (CIFAR10/100, Oxford Flowers, Oxford-IIT-Pets). 
More details about the training and datasets can be found in \Cref{sub:appendix-exp-details}.

\subsection{In-depth Layer-wise Analysis of \ours{}}\label{subsec:exp-in-depth-analysis}
\textbf{Do layers of \ours{} achieve their design goals?} 
As described in \Cref{sub:compression} and \Cref{sub:sparse}, the \texttt{MSSA} block is designed to optimize the compression term  $R^{c}(\Z)$ and the \texttt{ISTA} block to sparsify the token representations (corresponding to the sparsification term $\|\Z\|_0$). 
To understand whether \ours{} indeed optimizes these terms, for each layer $\ell$, we measure (i) the compression term $R^{c}(\Z^{\ell+1/2})$ on the \texttt{MSSA} block outputs $\Z^{\ell+1/2}$; and (ii) sparsity $\|\Z^{\ell+1}\|_0$ on the \texttt{ISTA} block outputs $\Z^{\ell+1}$. 
Specifically, we evaluate these two terms by using training/validation samples from ImageNet-1K. Both terms are evaluated at the per-sample level and averaged over \(B = 10^3\) samples.

\Cref{fig:exp-rc-sparisty-small} shows the plots of these two key measures at all layers for the learned \ours{-small} model. We find that as the layer index $\ell$ increases, both the compression and the sparsification terms improve in most cases. The increase in the sparsity measure of the last layer is caused by the extra linear layer for classification.\footnote{ Note that the learned sparse (tokens) features need to be mixed in the last layer for predicting the class. The phenomenon of increase in the sparsity measure  at the last layer suggests that each class of objects may be associated with a number of features, and some of these features are likely to be shared across different classes.} 
These results suggest that \ours{} aligns well with the original design goals: once learned, it essentially learns to gradually compress and sparsity the representations through its layers. In addition, we also measure the compression and sparsification terms on \ours{} models with different model sizes as well as intermediate model checkpoints and the results are shown by plots in \Cref{fig:appendix-exp-rc-sparisty-all-model-size} of \Cref{subsec:appendix-exp-results}. The observations are very consistent across all different model sizes---both the compression and sparsification terms improve in most scenarios. Models with more layers tend to optimize the objectives more effectively, confirming our understanding of each layer's roles. 

To see the effect of learning, we present the evaluations on \ours{-Small} trained with different number of epochs in \Cref{fig:exp-rc-sparisty-small-epochs}. 
When the model is not trained enough (e.g.~untrained), the architecture does not optimize the objectives effectively. However, during training---learning better subspaces $\vU_{[K]}^{\ell}$ and dictionaries $\vD^{\ell}$---the designed blocks  start to optimize the objectives much more effectively.

\begin{figure}[t!]
     \centering
     \begin{subfigure}[b]{0.47\textwidth}
         \centering
    \includegraphics[width=\textwidth]{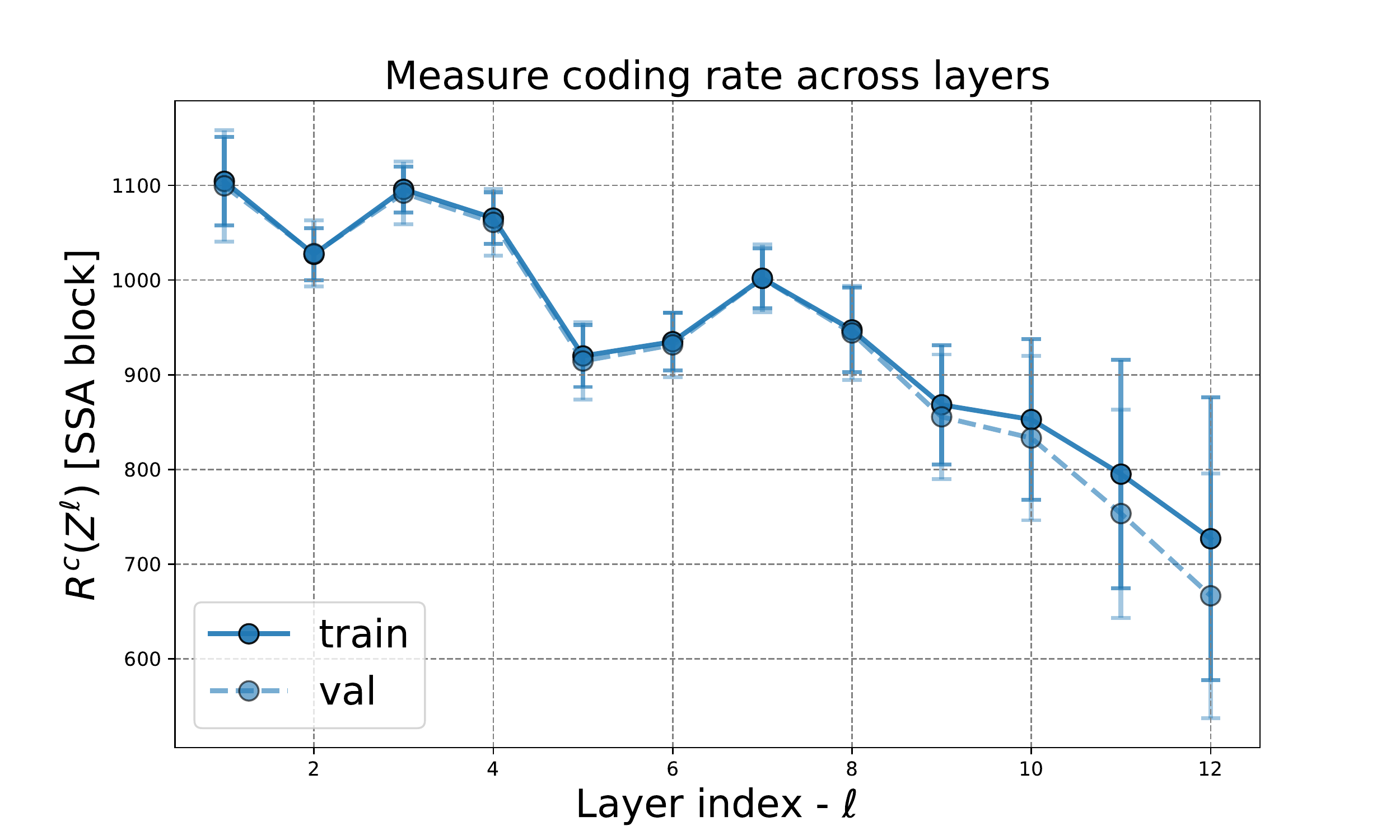}
     \end{subfigure}
     \begin{subfigure}[b]{0.482\textwidth}
         \centering
    \includegraphics[width=\textwidth]{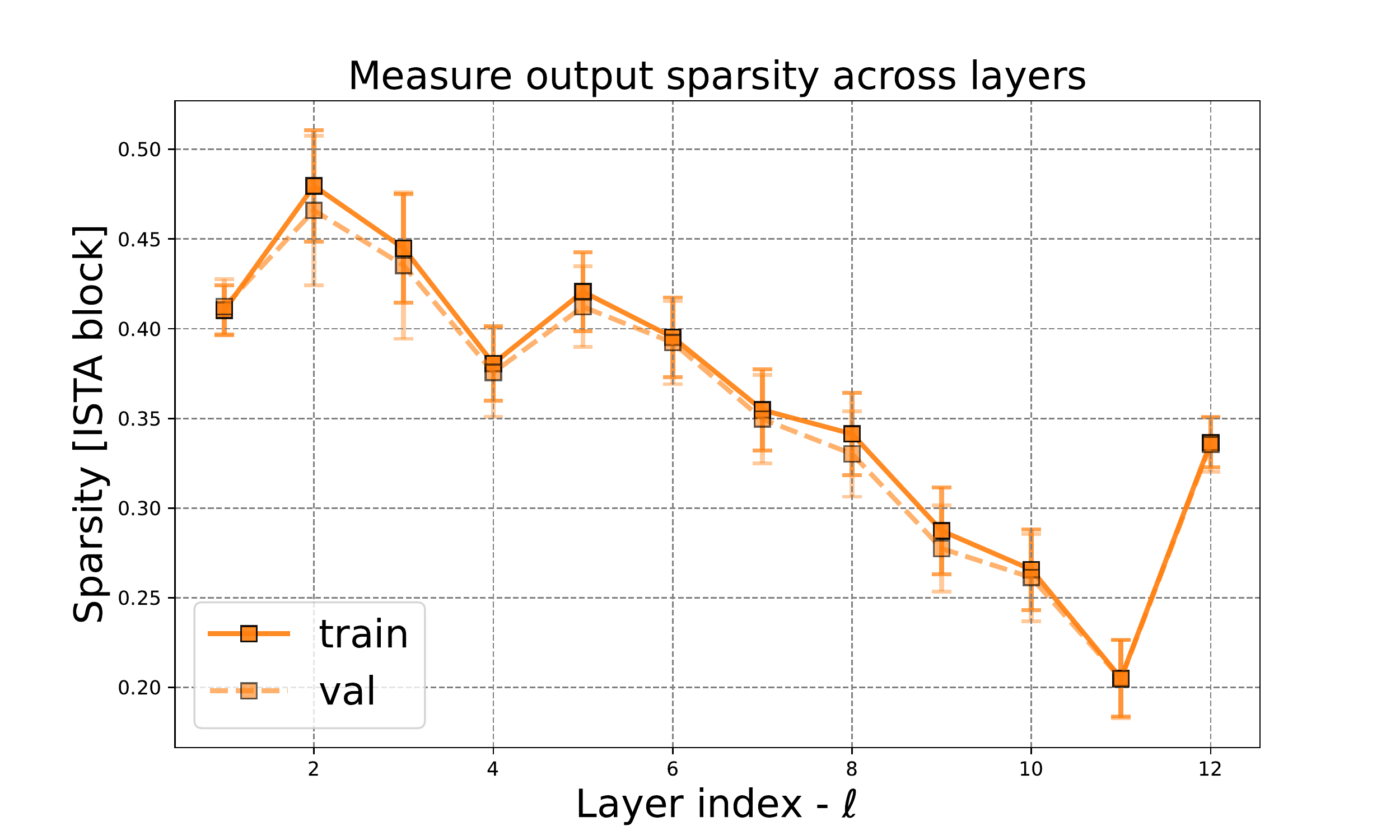}
     \end{subfigure}
     \vspace{-2mm}
        \caption{\small \textit{Left}: The compression term $R^{c}(\Z^{\ell+1/2})$ of the \texttt{MSSA} outputs at different layers. \textit{Right}: the sparsity of the \texttt{ISTA} output block, $\|\Z^{\ell+1}\|_0 / (d\cdot N)$, at different layers. 
        (Model: \ours{-Small}).}
        \label{fig:exp-rc-sparisty-small}
        \vspace{-0.15in}
\end{figure}

\begin{figure}[t!]
     \centering
     \begin{subfigure}[b]{0.49\textwidth}
         \centering
    \includegraphics[width=\textwidth]{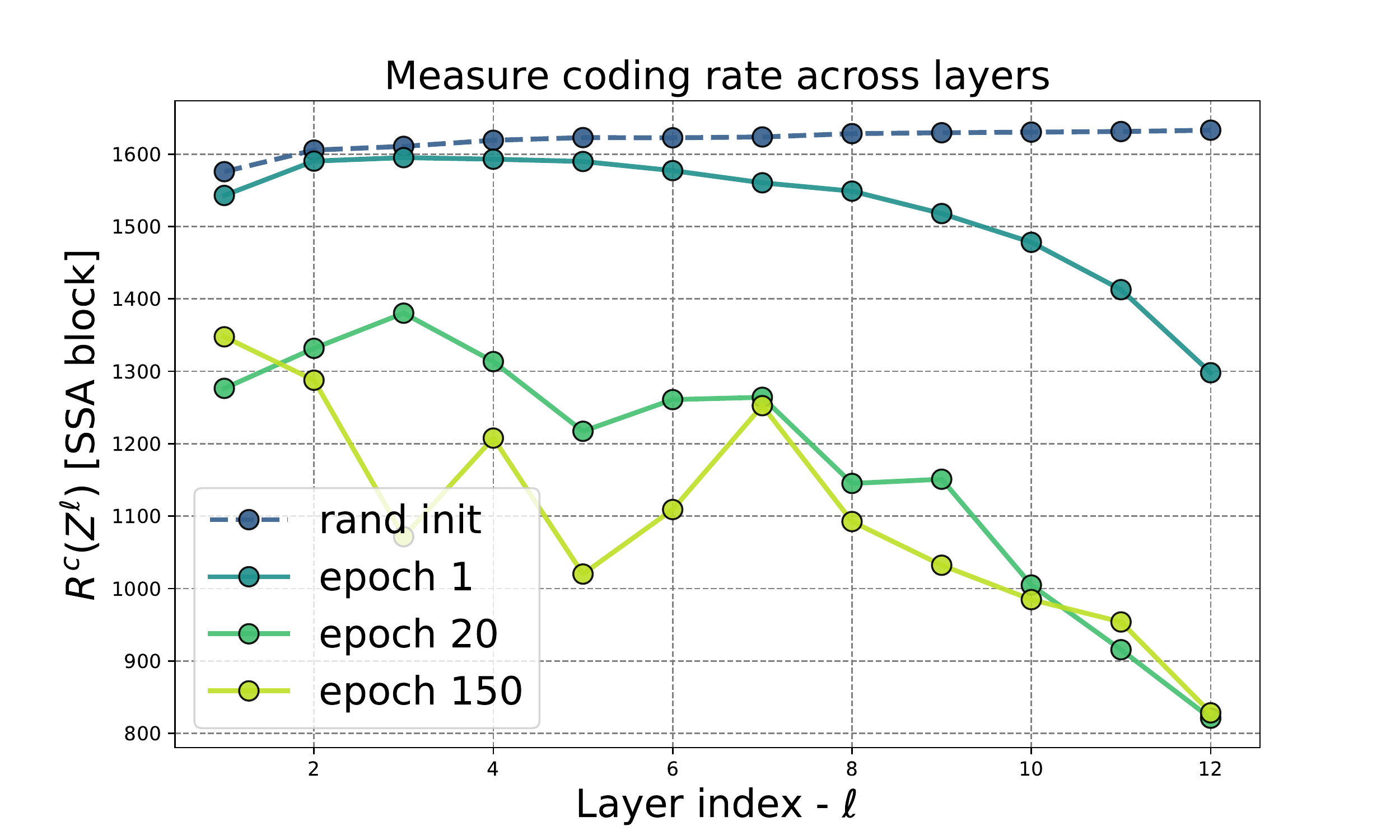}
     \end{subfigure}
     \begin{subfigure}[b]{0.49\textwidth}
         \centering
    \includegraphics[width=\textwidth]{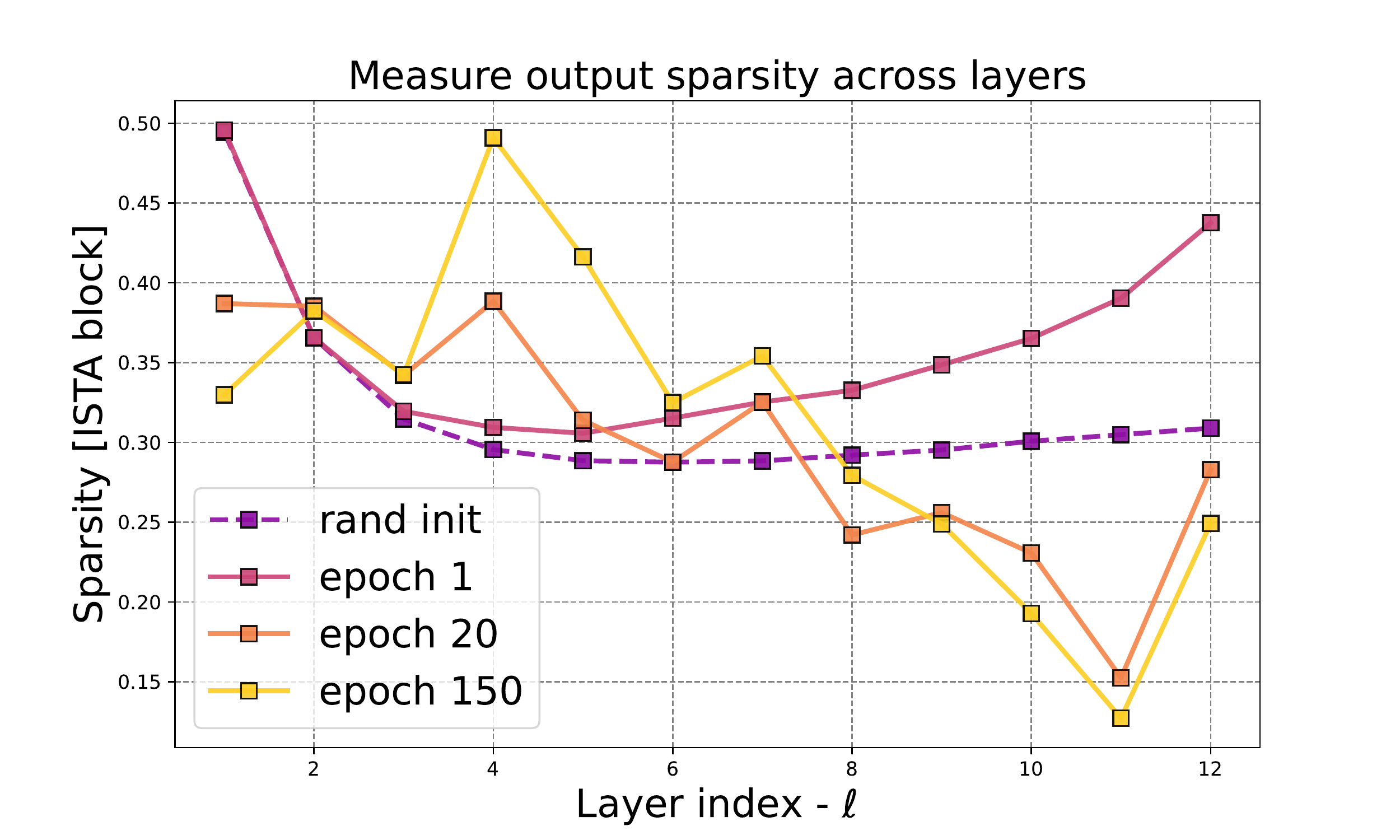}
     \end{subfigure}
     \vspace{-2mm}
        \caption{\small The compression term $R^c(\Z)$ (\textit{left}) and sparsification term $\|\Z\|_0 / (d \cdot N)$ (\textit{right}) across models trained with different numbers of epochs. (Model: \ours{-Base}).}
        \label{fig:exp-rc-sparisty-small-epochs}
        \vspace{-0.1in}
\end{figure}

\paragraph{Visualizing layer-wise token representations.} 
To gain a better understanding of the token representations of \ours{}, we visualize the output of each \texttt{ISTA} block at layer $\ell$ in Figure \ref{fig:appendix-exp-ista-sparsity-heatmap} of \Cref{subsec:appendix-exp-results}. 
Specifically, we visualize the $\Z^{\ell+1}$ via heatmap plots. 
We observe that the output $\Z^{\ell+1}$ becomes more sparse as the layer increases. 
Moreover, besides the sparsity, we also find that $\Z^{\ell+1}$ becomes more structured (i.e., low-rank), which indicates that the set of token representations become closer to linear subspaces, confirming our mental picture of the geometry of each layer (as in \Cref{fig:diagram}).

\paragraph{Visualizing layer-wise subspaces in multi-head self-attention.} 
We now visualize the $\vU_{[K]}^{\ell}$ matrices used in the \texttt{MSSA} block. 
In \Cref{sub:compression}, we assumed that $\vU_{[K]}^{\ell}$ were incoherent to capture different ``views'' of the set of tokens. 
In Fig.~\ref{fig:appendix-exp-visualize-UiUj} of \Cref{subsec:appendix-exp-results}, we first normalize the columns in each $\vU_k^{\ell}$, then we visualize the $[\vU_1^{\ell}, \dots, \vU_K^{\ell}]^{\adj}[\vU_1^{\ell}, \dots, \vU_K^{\ell}] \in \bR^{pK \times pK}$. 
The $(i, j)$-th block in each sub-figure corresponds to $(\vU_i^{\ell})\adj\vU_j^{\ell}$ for $i, j \in [K]$ at a particular layer $\ell$. 
We find that the learned $\vU_{[K]}^{\ell}$ are approximately incoherent, which aligns well with our assumptions. One interesting observation is that the $\vU_{[K]}^{\ell}$ becomes more incoherent when the layer index $\ell$ is larger, which suggests that the token representations are more separable. This mirrors the situation in other popular deep networks \cite{he2022law}.

\subsection{Evalutions of \ours{} on Large Real-World Datasets and Tasks} 
We now study the empirical performance of the proposed networks by measuring their top-1 accuracy on ImageNet-1K as well as transfer learning performance on several widely used downstream datasets. 
We summarize the results in \Cref{tab:Comparison_with_Sota}. 
As our designed architecture leverages parameter sharing in both the attention block (\texttt{MSSA}) and the MLP block (\texttt{ISTA}), our \ours{-Base} model (22.08 million) has a similar number of parameters to the ViT-Small (22.05 million).  

From \Cref{tab:Comparison_with_Sota}, we find that with a similar number of model parameters, our proposed network achieves similar ImageNet-1K and transfer learning performance as ViT, despite the simplicity and interpretability of our design. Moreover, with the same set of training hyperparameters, we observe promising scaling behavior in \ours{}---we consistently improve the performance by scaling up the model size. 
For comparison, directly scaling ViT on ImageNet-1K does not always lead to consistent performance improvement measured by top-1 accuracy~\cite{dosovitskiy2020image}. 
To summarize, we achieve promising performance on real-world large-scale datasets by directly implementing our principled architecture.

\begin{table*}[t!]
\centering
\caption{\small Top 1 accuracy of \ours{} on various datasets with different model scales 
when pre-trained on ImageNet. For ImageNet/ImageNetReaL, we directly evaluate the top-1 accuracy. For other datasets, we use models that are pre-trained on ImageNet as initialization and the evaluate the transfer learning performance via fine-tuning.}
\label{tab:Comparison_with_Sota}
\vspace{-1mm}
\small
    \setlength{\tabcolsep}{13.6pt}
\resizebox{0.98\textwidth}{!}{%
\begin{tabular}{@{}lcccc|cc@{}}
\toprule
\textbf{Datasets} & \ours{-T}  &  \ours{-S} & \ours{-B} & \ours{-L} & { \color{gray} ViT-T} &  { \color{gray}ViT-S } \\ 
\midrule
\midrule
 \# parameters & 6.09M & 13.12M & 22.80M & 77.64M & { \color{gray} 5.72M} & { \color{gray} 22.05M} \\
\midrule
 ImageNet & 66.7 & 69.2 & 70.8 & 71.3 & { \color{gray} 71.5} & { \color{gray} 72.4} \\
 ImageNet ReaL & 74.0 & 76.0 & 76.5 & 77.4 & { \color{gray} 78.3 } & { \color{gray} 78.4} \\
 \midrule
 CIFAR10 & 95.5 & 96.0 & 96.8 & 97.2 & { \color{gray} 96.6} & { \color{gray} 97.2} \\
 CIFAR100 & 78.9 & 81.0 & 82.7 & 83.6 & { \color{gray} 81.8} & { \color{gray} 83.2}\\
 Oxford Flowers-102 & 84.6 & 87.1 & 88.7 & 88.3 & { \color{gray} 85.1} & { \color{gray} 88.5}\\
 Oxford-IIIT-Pets & 81.4 & 84.9 & 85.3 & 87.4 & { \color{gray} 88.5} & { \color{gray} 88.6} \\
 \bottomrule
\end{tabular}%
}
\vspace{-0.1in}
\end{table*}

\section{Conclusion}\label{sec:conclusion}
In this paper, we propose a new theoretical framework that allows us to derive deep transformer-like network architectures as incremental optimization schemes to learn compressed and sparse representation of the input data (or token sets). The so derived and learned deep architectures are not only fully mathematically interpretable, but also consistent on a layer-by-layer level with their design objective. Despite being arguably the simplest among all possible designs, these networks already demonstrate  performance on large-scale real-world datasets and tasks close to seasoned transformers. We believe this work truly helps bridge the gap between theory and practice of deep neural networks as well as help unify seemingly separate approaches to learning and representing data distributions. Probably more importantly for practitioners, our framework provides theoretical guidelines to design and justify new, potentially more powerful, deep architectures for representation learning.

\newpage
\printbibliography

@article{phuong2022formal,
  title={Formal algorithms for transformers},
  author={Phuong, Mary and Hutter, Marcus},
  journal={arXiv preprint arXiv:2207.09238},
  year={2022}
}

@article{beyer2020we,
  title={Are we done with imagenet?},
  author={Beyer, Lucas and H{\'e}naff, Olivier J and Kolesnikov, Alexander and Zhai, Xiaohua and Oord, A{\"a}ron van den},
  journal={arXiv preprint arXiv:2006.07159},
  year={2020}
}

@article{krizhevsky2009learning,
  title={Learning multiple layers of features from tiny images},
  author={Krizhevsky, Alex and Hinton, Geoffrey and others},
  year={2009},
  publisher={Toronto, ON, Canada}
}

@inproceedings{nilsback2008automated,
  title={Automated flower classification over a large number of classes},
  author={Nilsback, Maria-Elena and Zisserman, Andrew},
  booktitle={2008 Sixth Indian Conference on Computer Vision, Graphics \& Image Processing},
  pages={722--729},
  year={2008},
  organization={IEEE}
}

@inproceedings{parkhi2012cats,
  title={Cats and dogs},
  author={Parkhi, Omkar M and Vedaldi, Andrea and Zisserman, Andrew and Jawahar, CV},
  booktitle={2012 IEEE conference on computer vision and pattern recognition},
  pages={3498--3505},
  year={2012},
  organization={IEEE}
}

@article{loshchilov2017decoupled,
  title={Decoupled weight decay regularization},
  author={Loshchilov, Ilya and Hutter, Frank},
  journal={arXiv preprint arXiv:1711.05101},
  year={2017}
}

@article{ma2022principles,
  title={On the principles of parsimony and self-consistency for the emergence of intelligence},
  author={Ma, Yi and Tsao, Doris and Shum, Heung-Yeung},
  journal={Frontiers of Information Technology \& Electronic Engineering},
  volume={23},
  number={9},
  pages={1298--1323},
  year={2022},
  publisher={Springer}
}

@article{pai2022pursuit,
  title={Pursuit of a discriminative representation for multiple subspaces via sequential games},
  author={Pai, Druv and Psenka, Michael and Chiu, Chih-Yuan and Wu, Manxi and Dobriban, Edgar and Ma, Yi},
  journal={arXiv preprint arXiv:2206.09120},
  year={2022}
}

@inproceedings{sun2018supervised,
  title={Supervised deep sparse coding networks},
  author={Sun, Xiaoxia and Nasrabadi, Nasser M and Tran, Trac D},
  booktitle={2018 25th IEEE International Conference on Image Processing (ICIP)},
  pages={346--350},
  year={2018},
  organization={IEEE}
}

@article{zarka2019deep,
  title={Deep network classification by scattering and homotopy dictionary learning},
  author={Zarka, John and Thiry, Louis and Angles, Tom{\'a}s and Mallat, St{\'e}phane},
  journal={arXiv preprint arXiv:1910.03561},
  year={2019}
}

@article{chan2021redunet,
  author  = {Kwan Ho Ryan Chan and Yaodong Yu and Chong You and Haozhi Qi and John Wright and Yi Ma},
  title   = {Redu{N}et: A White-box Deep Network from the Principle of Maximizing Rate Reduction},
  journal = {Journal of Machine Learning Research},
  year    = {2022},
  volume  = {23},
  number  = {114},
  pages   = {1-103},
  url     = {http://jmlr.org/papers/v23/21-0631.html}
}

@article{ma2007segmentation,
  title={Segmentation of multivariate mixed data via lossy data coding and compression},
  author={Ma, Yi and Derksen, Harm and Hong, Wei and Wright, John},
  journal={PAMI},
  year={2007},
}

@book{Vidal:Springer16,
	Author = {Ren\'{e} Vidal and Yi Ma and Shankar Sastry},
	Publisher = {Springer Verlag},
	Title = {{Generalized Principal Component Analysis}},
	Year = {2016}}

@article{OriginalMCR2,
 author = {Yu, Yaodong and Chan, Kwan Ho Ryan and You, Chong and Song, Chaobing and Ma, Yi},
 journal = {Advances in Neural Information Processing Systems},
 pages = {9422--9434},
 publisher = {Curran Associates, Inc.},
 title = {{Learning Diverse and Discriminative Representations via the Principle of Maximal Coding Rate Reduction}},
 url = {https://proceedings.neurips.cc/paper/2020/file/6ad4174eba19ecb5fed17411a34ff5e6-Paper.pdf},
 volume = {33},
 year = {2020}
}

@inproceedings{gregor2010learning,
  title={Learning fast approximations of sparse coding},
  author={Gregor, Karol and LeCun, Yann},
  booktitle={Proceedings of the 27th International Conference on International Conference on Machine Learning},
  pages={399--406},
  year={2010},
  organization={Omnipress}
}

@misc{hinton2022forwardforward,
      title={The Forward-Forward Algorithm: Some Preliminary Investigations}, 
      author={Geoffrey Hinton},
      year={2022},
      eprint={2212.13345},
      archivePrefix={arXiv},
      primaryClass={cs.LG}
}

@misc{hinton2021represent,
      title={How to represent part-whole hierarchies in a neural network}, 
      author={Geoffrey Hinton},
      year={2021},
      eprint={2102.12627},
      archivePrefix={arXiv},
      primaryClass={cs.CV}
}

@BOOK{Gyorfi2010-rn,
  title     = "A {Distribution-Free} Theory of Nonparametric Regression",
  author    = "Gy{\"o}rfi, L{\'a}szl{\'o} and Kohler, Michael and Krzyzak, Adam
               and Walk, Harro",
  publisher = "Springer New York",
  month     =  dec,
  year      =  2010,
  url       = "https://play.google.com/store/books/details?id=_RoFkgAACAAJ",
  language  = "en",
  isbn      = "9781441929983"
}

@ARTICLE{Efron2011-wn,
  title    = "Tweedie's Formula and Selection Bias",
  author   = "Efron, Bradley",
  journal  = "Journal of the American Statistical Association",
  volume   =  106,
  number   =  496,
  pages    = "1602--1614",
  year     =  2011,
  url      = "http://dx.doi.org/10.1198/jasa.2011.tm11181",
  language = "en",
  issn     = "0162-1459",
  pmid     = "22505788",
  doi      = "10.1198/jasa.2011.tm11181",
  pmc      = "PMC3325056"
}

@ARTICLE{Stein1981-jv,
  title     = "Estimation of the Mean of a Multivariate Normal Distribution",
  author    = "Stein, Charles M",
  journal   = "The Annals of Statistics",
  publisher = "Institute of Mathematical Statistics",
  volume    =  9,
  number    =  6,
  pages     = "1135--1151",
  month     =  nov,
  year      =  1981,
  url       = "https://projecteuclid.org/journals/annals-of-statistics/volume-9/issue-6/Estimation-of-the-Mean-of-a-Multivariate-Normal-Distribution/10.1214/aos/1176345632.full",
  language  = "en",
  issn      = "0090-5364, 2168-8966",
  doi       = "10.1214/aos/1176345632"
}

@ARTICLE{Raphan2011-by,
  title    = "Least squares estimation without priors or supervision",
  author   = "Raphan, Martin and Simoncelli, Eero P",
  journal  = "Neural computation",
  volume   =  23,
  number   =  2,
  pages    = "374--420",
  month    =  feb,
  year     =  2011,
  url      = "http://dx.doi.org/10.1162/NECO_a_00076",
  language = "en",
  issn     = "0899-7667, 1530-888X",
  pmid     = "21105827",
  doi      = "10.1162/NECO\_a\_00076",
  pmc      = "PMC3609403"
}

@ARTICLE{Milanfar2013-js,
  title    = "A Tour of Modern Image Filtering: New Insights and Methods, Both
              Practical and Theoretical",
  author   = "Milanfar, Peyman",
  journal  = "IEEE Signal Processing Magazine",
  volume   =  30,
  number   =  1,
  pages    = "106--128",
  month    =  jan,
  year     =  2013,
  url      = "http://dx.doi.org/10.1109/MSP.2011.2179329",
  issn     = "1558-0792",
  doi      = "10.1109/MSP.2011.2179329"
}

@ARTICLE{Kadkhodaie2020-fc,
  title         = "Solving Linear Inverse Problems Using the Prior Implicit in
                   a Denoiser",
  author        = "Kadkhodaie, Zahra and Simoncelli, Eero P",
  month         =  jul,
  year          =  2020,
  url           = "http://arxiv.org/abs/2007.13640",
  archivePrefix = "arXiv",
  eprint        = "2007.13640",
  primaryClass  = "cs.CV",
  arxivid       = "2007.13640"
}

@ARTICLE{Hyvarinen2005-fi,
  title    = "Estimation of {Non-Normalized} Statistical Models by Score
              Matching",
  author   = "Hyv{\"a}rinen, Aapo",
  journal  = "Journal of machine learning research: JMLR",
  volume   =  6,
  number   =  24,
  pages    = "695--709",
  year     =  2005,
  url      = "https://www.jmlr.org/papers/v6/hyvarinen05a.html",
  issn     = "1532-4435, 1533-7928"
}

@ARTICLE{Vincent2011-dr,
  title    = "A connection between score matching and denoising autoencoders",
  author   = "Vincent, Pascal",
  journal  = "Neural computation",
  volume   =  23,
  number   =  7,
  pages    = "1661--1674",
  month    =  jul,
  year     =  2011,
  url      = "http://dx.doi.org/10.1162/NECO_a_00142",
  language = "en",
  issn     = "0899-7667, 1530-888X",
  pmid     = "21492012",
  doi      = "10.1162/NECO\_a\_00142"
}

@ARTICLE{Sohl-Dickstein2015-kz,
  title         = "Deep Unsupervised Learning using Nonequilibrium
                   Thermodynamics",
  author        = "Sohl-Dickstein, Jascha and Weiss, Eric A and
                   Maheswaranathan, Niru and Ganguli, Surya",
  month         =  mar,
  year          =  2015,
  url           = "http://arxiv.org/abs/1503.03585",
  archivePrefix = "arXiv",
  eprint        = "1503.03585",
  primaryClass  = "cs.LG",
  arxivid       = "1503.03585"
}

@ARTICLE{Song2019-ww,
  title         = "Generative Modeling by Estimating Gradients of the Data
                   Distribution",
  author        = "Song, Yang and Ermon, Stefano",
  month         =  jul,
  year          =  2019,
  url           = "http://arxiv.org/abs/1907.05600",
  archivePrefix = "arXiv",
  eprint        = "1907.05600",
  primaryClass  = "cs.LG",
  arxivid       = "1907.05600"
}

@ARTICLE{Song2020-xo,
  title         = "{Score-Based} Generative Modeling through Stochastic
                   Differential Equations",
  author        = "Song, Yang and Sohl-Dickstein, Jascha and Kingma, Diederik P
                   and Kumar, Abhishek and Ermon, Stefano and Poole, Ben",
  month         =  nov,
  year          =  2020,
  url           = "http://arxiv.org/abs/2011.13456",
  archivePrefix = "arXiv",
  eprint        = "2011.13456",
  primaryClass  = "cs.LG",
  arxivid       = "2011.13456"
}

@ARTICLE{Song2020-hb,
  title         = "Denoising Diffusion Implicit Models",
  author        = "Song, Jiaming and Meng, Chenlin and Ermon, Stefano",
  month         =  oct,
  year          =  2020,
  url           = "http://arxiv.org/abs/2010.02502",
  archivePrefix = "arXiv",
  eprint        = "2010.02502",
  primaryClass  = "cs.LG",
  arxivid       = "2010.02502"
}

@ARTICLE{Chen2023-an,
  title         = "{Restoration-Degradation} Beyond Linear Diffusions: A
                   {Non-Asymptotic} Analysis For {DDIM-Type} Samplers",
  author        = "Chen, Sitan and Daras, Giannis and Dimakis, Alexandros G",
  month         =  mar,
  year          =  2023,
  url           = "http://arxiv.org/abs/2303.03384",
  archivePrefix = "arXiv",
  eprint        = "2303.03384",
  primaryClass  = "cs.LG",
  arxivid       = "2303.03384"
}

@article{vaswani2017attention,
  title={Attention is all you need},
  author={Vaswani, Ashish and Shazeer, Noam and Parmar, Niki and Uszkoreit, Jakob and Jones, Llion and Gomez, Aidan N and Kaiser, {\L}ukasz and Polosukhin, Illia},
  journal={Advances in neural information processing systems},
  volume={30},
  year={2017}
}

@article{devlin2018bert,
  title={Bert: Pre-training of deep bidirectional transformers for language understanding},
  author={Devlin, Jacob and Chang, Ming-Wei and Lee, Kenton and Toutanova, Kristina},
  journal={arXiv preprint arXiv:1810.04805},
  year={2018}
}

@article{brown2020language,
  title={Language models are few-shot learners},
  author={Brown, Tom and Mann, Benjamin and Ryder, Nick and Subbiah, Melanie and Kaplan, Jared D and Dhariwal, Prafulla and Neelakantan, Arvind and Shyam, Pranav and Sastry, Girish and Askell, Amanda and others},
  journal={Advances in neural information processing systems},
  volume={33},
  pages={1877--1901},
  year={2020}
}

@article{dosovitskiy2020image,
  title={An image is worth 16x16 words: Transformers for image recognition at scale},
  author={Dosovitskiy, Alexey and Beyer, Lucas and Kolesnikov, Alexander and Weissenborn, Dirk and Zhai, Xiaohua and Unterthiner, Thomas and Dehghani, Mostafa and Minderer, Matthias and Heigold, Georg and Gelly, Sylvain and others},
  journal={arXiv preprint arXiv:2010.11929},
  year={2020}
}

@article{dehghani2023scaling,
  title={Scaling vision transformers to 22 billion parameters},
  author={Dehghani, Mostafa and Djolonga, Josip and Mustafa, Basil and Padlewski, Piotr and Heek, Jonathan and Gilmer, Justin and Steiner, Andreas and Caron, Mathilde and Geirhos, Robert and Alabdulmohsin, Ibrahim and others},
  journal={arXiv preprint arXiv:2302.05442},
  year={2023}
}

@inproceedings{gong2022contrastive,
  title={Contrastive audio-visual masked autoencoder},
  author={Gong, Yuan and Rouditchenko, Andrew and Liu, Alexander H and Harwath, David and Karlinsky, Leonid and Kuehne, Hilde and Glass, James R},
  booktitle={The Eleventh International Conference on Learning Representations},
  year={2022}
}

@inproceedings{arnab2021vivit,
  title={Vivit: A video vision transformer},
  author={Arnab, Anurag and Dehghani, Mostafa and Heigold, Georg and Sun, Chen and Lu{\v{c}}i{\'c}, Mario and Schmid, Cordelia},
  booktitle={Proceedings of the IEEE/CVF international conference on computer vision},
  pages={6836--6846},
  year={2021}
}

@misc{
vidal2022attention,
title={Attention: Self-Expression Is All You Need},
author={Rene Vidal},
year={2022},
url={https://openreview.net/forum?id=MmujBClawFo},
note={Unpublished; available: \url{https://openreview.net/forum?id=MmujBClawFo}}
}

@article{li2023theoretical,
  title={A Theoretical Understanding of shallow Vision Transformers: Learning, Generalization, and Sample Complexity},
  author={Li, Hongkang and Wang, Meng and Liu, Sijia and Chen, Pin-Yu},
  journal={arXiv preprint arXiv:2302.06015},
  year={2023}
}

@article{ho2020denoising,
  title={Denoising diffusion probabilistic models},
  author={Ho, Jonathan and Jain, Ajay and Abbeel, Pieter},
  journal={Advances in Neural Information Processing Systems},
  volume={33},
  pages={6840--6851},
  year={2020}
}

@inproceedings{wakin2005multiscale,
  title={The multiscale structure of non-differentiable image manifolds},
  author={Wakin, Michael B and Donoho, David L and Choi, Hyeokho and Baraniuk, Richard G},
  booktitle={Wavelets XI},
  volume={5914},
  pages={413--429},
  year={2005},
  organization={SPIE}
}

@article{karras2022elucidating,
  title={Elucidating the design space of diffusion-based generative models},
  author={Karras, Tero and Aittala, Miika and Aila, Timo and Laine, Samuli},
  journal={arXiv preprint arXiv:2206.00364},
  year={2022}
}

@article{chen2022improved,
  title={Improved Analysis of Score-based Generative Modeling: User-Friendly Bounds under Minimal Smoothness Assumptions},
  author={Chen, Hongrui and Lee, Holden and Lu, Jianfeng},
  journal={arXiv preprint arXiv:2211.01916},
  year={2022}
}

@inproceedings{rombach2022high,
  title={High-resolution image synthesis with latent diffusion models},
  author={Rombach, Robin and Blattmann, Andreas and Lorenz, Dominik and Esser, Patrick and Ommer, Bj{\"o}rn},
  booktitle={Proceedings of the IEEE/CVF Conference on Computer Vision and Pattern Recognition},
  pages={10684--10695},
  year={2022}
}

@article{zhai2020complete,
  title={Complete dictionary learning via l 4-norm maximization over the orthogonal group},
  author={Zhai, Yuexiang and Yang, Zitong and Liao, Zhenyu and Wright, John and Ma, Yi},
  journal={The Journal of Machine Learning Research},
  volume={21},
  number={1},
  pages={6622--6689},
  year={2020},
  publisher={JMLRORG}
}

@article{olshausen1997sparse,
  title={Sparse coding with an overcomplete basis set: A strategy employed by V1?},
  author={Olshausen, Bruno A and Field, David J},
  journal={Vision research},
  volume={37},
  number={23},
  pages={3311--3325},
  year={1997},
  publisher={Elsevier}
}

@article{chen2018sparse,
  title={The sparse manifold transform},
  author={Chen, Yubei and Paiton, Dylan and Olshausen, Bruno},
  journal={Advances in neural information processing systems},
  volume={31},
  year={2018}
}

@article{tian2020makes,
  title={What makes for good views for contrastive learning?},
  author={Tian, Yonglong and Sun, Chen and Poole, Ben and Krishnan, Dilip and Schmid, Cordelia and Isola, Phillip},
  journal={Advances in neural information processing systems},
  volume={33},
  pages={6827--6839},
  year={2020}
}

@inproceedings{wang2022rethinking,
  title={Rethinking minimal sufficient representation in contrastive learning},
  author={Wang, Haoqing and Guo, Xun and Deng, Zhi-Hong and Lu, Yan},
  booktitle={Proceedings of the IEEE/CVF Conference on Computer Vision and Pattern Recognition},
  pages={16041--16050},
  year={2022}
}

@article{shwartz2023compress,
  title={To Compress or Not to Compress--Self-Supervised Learning and Information Theory: A Review},
  author={Shwartz-Ziv, Ravid and LeCun, Yann},
  journal={arXiv preprint arXiv:2304.09355},
  year={2023}
}

@article{tolooshams2021stable,
  title={Stable and Interpretable Unrolled Dictionary Learning},
  author={Tolooshams, Bahareh and Ba, Demba},
  journal={arXiv preprint arXiv:2106.00058},
  year={2021}
}

@book{Wright-Ma-2022, 
author = {John Wright and Yi Ma},
title = {High-Dimensional Data Analysis with Low-Dimensional Models: Principles, Computation, and Applications}, 
publisher = {Cambridge University Press},
year = {2022} 
}

@ARTICLE{Tolstikhin2021-yh,
  title         = "{MLP-Mixer}: An {all-MLP} Architecture for Vision",
  author        = "Tolstikhin, Ilya and Houlsby, Neil and Kolesnikov, Alexander
                   and Beyer, Lucas and Zhai, Xiaohua and Unterthiner, Thomas
                   and Yung, Jessica and Steiner, Andreas and Keysers, Daniel
                   and Uszkoreit, Jakob and Lucic, Mario and Dosovitskiy,
                   Alexey",
  month         =  may,
  year          =  2021,
  url           = "http://arxiv.org/abs/2105.01601",
  archivePrefix = "arXiv",
  eprint        = "2105.01601",
  primaryClass  = "cs.CV",
  arxivid       = "2105.01601"
}

@ARTICLE{Trockman2022-po,
  title         = "Understanding the Covariance Structure of Convolutional
                   Filters",
  author        = "Trockman, Asher and Willmott, Devin and Zico Kolter, J",
  month         =  oct,
  year          =  2022,
  url           = "http://arxiv.org/abs/2210.03651",
  archivePrefix = "arXiv",
  eprint        = "2210.03651",
  primaryClass  = "cs.CV",
  arxivid       = "2210.03651"
}

@article{song2023consistency,
  title={Consistency models},
  author={Song, Yang and Dhariwal, Prafulla and Chen, Mark and Sutskever, Ilya},
  journal={arXiv preprint arXiv:2303.01469},
  year={2023}
}

@inproceedings{deng2009imagenet,
  title={Imagenet: A large-scale hierarchical image database},
  author={Deng, Jia and Dong, Wei and Socher, Richard and Li, Li-Jia and Li, Kai and Fei-Fei, Li},
  booktitle={2009 IEEE conference on computer vision and pattern recognition},
  pages={248--255},
  year={2009},
  organization={Ieee}
}

@article{chen2023symbolic,
  title={Symbolic discovery of optimization algorithms},
  author={Chen, Xiangning and Liang, Chen and Huang, Da and Real, Esteban and Wang, Kaiyuan and Liu, Yao and Pham, Hieu and Dong, Xuanyi and Luong, Thang and Hsieh, Cho-Jui and others},
  journal={arXiv preprint arXiv:2302.06675},
  year={2023}
}

@INPROCEEDINGS{Li2023-ig,
  title     = "The Lazy Neuron Phenomenon: On Emergence of Activation Sparsity
               in Transformers",
  booktitle = "The Eleventh International Conference on Learning
               Representations",
  author    = "Li, Zonglin and You, Chong and Bhojanapalli, Srinadh and Li,
               Daliang and Rawat, Ankit Singh and Reddi, Sashank J and Ye, Ke
               and Chern, Felix and Yu, Felix and Guo, Ruiqi and Kumar, Sanjiv",
  year      =  2023,
  url       = "https://openreview.net/forum?id=TJ2nxciYCk-"
}

@article{he2022law,
  title={A law of data separation in deep learning},
  author={He, Hangfeng and Su, Weijie J},
  journal={arXiv preprint arXiv:2210.17020},
  year={2022}
}

@article{beck2009fast,
  title={A fast iterative shrinkage-thresholding algorithm for linear inverse problems},
  author={Beck, Amir and Teboulle, Marc},
  journal={SIAM journal on imaging sciences},
  volume={2},
  number={1},
  pages={183--202},
  year={2009},
  publisher={SIAM}
}

@ARTICLE{Saharia2022-na,
  title         = "Photorealistic {Text-to-Image} Diffusion Models with Deep
                   Language Understanding",
  author        = "Saharia, Chitwan and Chan, William and Saxena, Saurabh and
                   Li, Lala and Whang, Jay and Denton, Emily and Ghasemipour,
                   Seyed Kamyar Seyed and Ayan, Burcu Karagol and Sara Mahdavi,
                   S and Lopes, Rapha Gontijo and Salimans, Tim and Ho,
                   Jonathan and Fleet, David J and Norouzi, Mohammad",
  month         =  may,
  year          =  2022,
  url           = "http://arxiv.org/abs/2205.11487",
  archivePrefix = "arXiv",
  eprint        = "2205.11487",
  primaryClass  = "cs.CV",
  arxivid       = "2205.11487"
}

@ARTICLE{Karras2018-si,
  title         = "A {Style-Based} Generator Architecture for Generative
                   Adversarial Networks",
  author        = "Karras, Tero and Laine, Samuli and Aila, Timo",
  month         =  dec,
  year          =  2018,
  url           = "http://arxiv.org/abs/1812.04948",
  archivePrefix = "arXiv",
  eprint        = "1812.04948",
  primaryClass  = "cs.NE",
  arxivid       = "1812.04948"
}

@INPROCEEDINGS{He2016-lc,
  title     = "Deep Residual Learning for Image Recognition",
  booktitle = "2016 {IEEE} Conference on Computer Vision and Pattern
               Recognition ({CVPR})",
  author    = "He, Kaiming and Zhang, Xiangyu and Ren, Shaoqing and Sun, Jian",
  pages     = "770--778",
  month     =  jun,
  year      =  2016,
  url       = "http://dx.doi.org/10.1109/CVPR.2016.90",
  issn      = "1063-6919",
  doi       = "10.1109/CVPR.2016.90"
}

@ARTICLE{Carion2020-fm,
  title         = "{End-to-End} Object Detection with Transformers",
  author        = "Carion, Nicolas and Massa, Francisco and Synnaeve, Gabriel
                   and Usunier, Nicolas and Kirillov, Alexander and Zagoruyko,
                   Sergey",
  month         =  may,
  year          =  2020,
  url           = "http://arxiv.org/abs/2005.12872",
  archivePrefix = "arXiv",
  eprint        = "2005.12872",
  primaryClass  = "cs.CV",
  arxivid       = "2005.12872"
}

@ARTICLE{He2017-be,
  title         = "Mask {R-CNN}",
  author        = "He, Kaiming and Gkioxari, Georgia and Doll{\'a}r, Piotr and
                   Girshick, Ross",
  month         =  mar,
  year          =  2017,
  url           = "http://arxiv.org/abs/1703.06870",
  archivePrefix = "arXiv",
  eprint        = "1703.06870",
  primaryClass  = "cs.CV",
  arxivid       = "1703.06870"
}

@ARTICLE{Kirillov2023-pm,
  title         = "Segment Anything",
  author        = "Kirillov, Alexander and Mintun, Eric and Ravi, Nikhila and
                   Mao, Hanzi and Rolland, Chloe and Gustafson, Laura and Xiao,
                   Tete and Whitehead, Spencer and Berg, Alexander C and Lo,
                   Wan-Yen and Doll{\'a}r, Piotr and Girshick, Ross",
  month         =  apr,
  year          =  2023,
  url           = "http://arxiv.org/abs/2304.02643",
  archivePrefix = "arXiv",
  eprint        = "2304.02643",
  primaryClass  = "cs.CV",
  arxivid       = "2304.02643"
}

@INPROCEEDINGS{Radford2021-ir,
  title     = "Learning Transferable Visual Models From Natural Language
               Supervision",
  booktitle = "Proceedings of the 38th International Conference on Machine
               Learning",
  author    = "Radford, Alec and Kim, Jong Wook and Hallacy, Chris and Ramesh,
               Aditya and Goh, Gabriel and Agarwal, Sandhini and Sastry, Girish
               and Askell, Amanda and Mishkin, Pamela and Clark, Jack and
               Krueger, Gretchen and Sutskever, Ilya",
  editor    = "Meila, Marina and Zhang, Tong",
  publisher = "PMLR",
  volume    =  139,
  pages     = "8748--8763",
  series    = "Proceedings of Machine Learning Research",
  year      =  2021,
  url       = "https://proceedings.mlr.press/v139/radford21a.html"
}

@INPROCEEDINGS{Chen2020-ha,
  title     = "A Simple Framework for Contrastive Learning of Visual
               Representations",
  booktitle = "Proceedings of the 37th International Conference on Machine
               Learning",
  author    = "Chen, Ting and Kornblith, Simon and Norouzi, Mohammad and
               Hinton, Geoffrey",
  editor    = "Iii, Hal Daum{\'e} and Singh, Aarti",
  publisher = "PMLR",
  volume    =  119,
  pages     = "1597--1607",
  series    = "Proceedings of Machine Learning Research",
  year      =  2020,
  url       = "https://proceedings.mlr.press/v119/chen20j.html"
}

@ARTICLE{He2021-lb,
  title         = "Masked Autoencoders Are Scalable Vision Learners",
  author        = "He, Kaiming and Chen, Xinlei and Xie, Saining and Li,
                   Yanghao and Doll{\'a}r, Piotr and Girshick, Ross",
  month         =  nov,
  year          =  2021,
  url           = "http://arxiv.org/abs/2111.06377",
  archivePrefix = "arXiv",
  eprint        = "2111.06377",
  primaryClass  = "cs.CV",
  arxivid       = "2111.06377"
}

@ARTICLE{Donoho2005-ag,
  title   = "Image Manifolds which are Isometric to Euclidean Space",
  author  = "Donoho, David L and Grimes, Carrie",
  journal = "Journal of mathematical imaging and vision",
  volume  =  23,
  number  =  1,
  pages   = "5--24",
  month   =  jul,
  year    =  2005,
  url     = "https://doi.org/10.1007/s10851-005-4965-4",
  issn    = "0924-9907, 1573-7683",
  doi     = "10.1007/s10851-005-4965-4"
}

@ARTICLE{Koehler2022-ed,
  title         = "Statistical Efficiency of Score Matching: The View from
                   Isoperimetry",
  author        = "Koehler, Frederic and Heckett, Alexander and Risteski,
                   Andrej",
  month         =  oct,
  year          =  2022,
  url           = "http://arxiv.org/abs/2210.00726",
  archivePrefix = "arXiv",
  eprint        = "2210.00726",
  primaryClass  = "cs.LG",
  arxivid       = "2210.00726"
}

@ARTICLE{Gribonval2014-zr,
  title         = "Sparse and spurious: dictionary learning with noise and
                   outliers",
  author        = "Gribonval, R{\'e}mi and Jenatton, Rodolphe and Bach, Francis",
  month         =  jul,
  year          =  2014,
  url           = "http://arxiv.org/abs/1407.5155",
  archivePrefix = "arXiv",
  eprint        = "1407.5155",
  primaryClass  = "cs.LG",
  arxivid       = "1407.5155"
}

@ARTICLE{Spielman2012-le,
  title         = "Exact Recovery of {Sparsely-Used} Dictionaries",
  author        = "Spielman, Daniel A and Wang, Huan and Wright, John",
  month         =  jun,
  year          =  2012,
  url           = "http://arxiv.org/abs/1206.5882",
  archivePrefix = "arXiv",
  eprint        = "1206.5882",
  primaryClass  = "cs.LG",
  arxivid       = "1206.5882"
}

@INPROCEEDINGS{Venkatakrishnan2013-rt,
  title     = "{Plug-and-Play} priors for model based reconstruction",
  booktitle = "2013 {IEEE} Global Conference on Signal and Information
               Processing",
  author    = "Venkatakrishnan, Singanallur V and Bouman, Charles A and
               Wohlberg, Brendt",
  pages     = "945--948",
  month     =  dec,
  year      =  2013,
  url       = "http://dx.doi.org/10.1109/GlobalSIP.2013.6737048",
  doi       = "10.1109/GlobalSIP.2013.6737048"
}

@ARTICLE{Papyan2018-qc,
  title    = "Theoretical Foundations of Deep Learning via Sparse
              Representations: A Multilayer Sparse Model and Its Connection to
              Convolutional Neural Networks",
  author   = "Papyan, Vardan and Romano, Yaniv and Sulam, Jeremias and Elad,
              Michael",
  journal  = "IEEE Signal Processing Magazine",
  volume   =  35,
  number   =  4,
  pages    = "72--89",
  month    =  jul,
  year     =  2018,
  url      = "http://dx.doi.org/10.1109/MSP.2018.2820224",
  issn     = "1558-0792",
  doi      = "10.1109/MSP.2018.2820224"
}

@ARTICLE{Bruna2013-on,
  title    = "Invariant scattering convolution networks",
  author   = "Bruna, Joan and Mallat, St{\'e}phane",
  journal  = "IEEE transactions on pattern analysis and machine intelligence",
  volume   =  35,
  number   =  8,
  pages    = "1872--1886",
  month    =  aug,
  year     =  2013,
  url      = "http://dx.doi.org/10.1109/TPAMI.2012.230",
  language = "en",
  issn     = "0162-8828",
  pmid     = "23787341",
  doi      = "10.1109/TPAMI.2012.230"
}

@ARTICLE{Romano2017-rp,
  title     = "The Little Engine That Could: Regularization by Denoising
               ({RED})",
  author    = "Romano, Yaniv and Elad, Michael and Milanfar, Peyman",
  journal   = "SIAM journal on imaging sciences",
  publisher = "Society for Industrial and Applied Mathematics",
  volume    =  10,
  number    =  4,
  pages     = "1804--1844",
  month     =  jan,
  year      =  2017,
  url       = "https://doi.org/10.1137/16M1102884",
  doi       = "10.1137/16M1102884"
}

\clearpage
\appendix

\begin{center}
    \LARGE \textbf{Appendix}
\end{center}

\section{Technical Details from \Cref{sec:approach}}

\subsection{Companion to \Cref{sub:denoising}} \label{app:proofs-denoising}

We first wish to re-iterate the core contributions of our approach in \Cref{sub:denoising} at a slightly more technical level. Connections between denoising and score matching are well-understood \cite{karras2022elucidating}, and computing the optimal denoising function (i.e., the conditional expectation) against a mixture-of-Gaussians model is a rather simple computation giving existing tools such as Tweedie's formula \cite{Efron2011-wn}. These are not our main contributions. Instead, the main contributions of \Cref{sub:denoising} are two-fold:
\begin{itemize}
    \item First, we demonstrate a mechanism to learn representations via denoising within a idealized mixture of Gaussian data model for a single token (i.e., with sequence length \(N = 1\)).
    \item Second, we illustrate the similarities between a such-derived representation learning scheme and existing self-attention layers within the transformer (with sequence length \(1\)), thus demonstrating an interpretation of the self-attention layer as a generalized mechanism to denoise against a mixture-of-Gaussian-marginal model for a set of tokens.
\end{itemize}

Now we produce the proofs alluded to in \Cref{sub:denoising}, which mostly form the technical aspects of the first listed contribution. To simplify the proofs, we use the following notation correspondences: \(\x \mapsto \z^{\ell}\), \(\z \mapsto \z^{\ell + 1}\), and \(\sigma \mapsto \sigma^{\ell}\).

\begin{proposition}\label{thm:score_multi_subspaces}
    Let \(\vu_{1}, \dots, \vu_{K} \in \bR^{d}\) be independent and have distribution \(\vu_{k} \sim \mathcal{N}(\Zero, \vSigma_{k})\) for \(\vSigma_{k} \succeq \Zero \),
    and let \(\z\) take value \(\vu_{k}\) with probability \(\pi_{k} > 0\). Let \(\w \sim \mathcal{N}(\Zero, \I_{d})\) be independent of \(\z\). Let \(\x \doteq \z + \sigma \w\). Let \(\x \mapsto q(\x)\) be the density of \(\x\). We define
    \begin{equation}
        \vM_{k} \doteq (\vSigma_{k} + \sigma^{2}\I_{d})^{-1/2}
    \end{equation}
    and assume that \(\pi_{i}\det(\vM_{i}) = \pi_{j}\det(\vM_{j})\) for all \(1 \leq i \leq j \leq K\). Then we have
    \begin{align}
        &\nabla_{\x}\log q(\x) \\
        &= -\mat{\M_{1}, \cdots, \M_{K}}\left[\diaglr{\softmax{-\frac{1}{2}\mat{\norm{\M_{1}\adj\x}_{2}^{2} \\ \vdots \\ \norm{\M_{K}\adj\x}_{2}^{2}}}} \otimes \I_{d}\right]\mat{\M_{1}\adj\x \\ \vdots \\ \M_{K}\adj\x},
    \end{align}
    where \(\otimes\) denotes the Kronecker product, i.e., the block matrix defined by
    \begin{equation}
        \bm{A} \otimes \bm{B} = \mat{A_{11}\bm{B} & \cdots & A_{1n}\bm{B} \\ \vdots & \ddots & \vdots \\ A_{m1}\bm{B} & \cdots & A_{mn}\bm{B}}
    \end{equation}
\end{proposition}
\begin{proof}
    Let \(u\) be the multinomial random variable such that \(\z = \z_{u}\), so that \(u\) has probability mass function \(\pi\). Then by the law of total probability, we have
    \begin{align}
        \nabla_{\x}\log q(\x) 
        &= \nabla_{\x} \log \sum_{k = 1}^{K}q(\x \mid k)\pi_{k} \\
        &= \frac{\sum_{k = 1}^{K}\pi_{k}\nabla_{\x}q(\x \mid k)}{\sum_{k = 1}^{K}q(\x \mid k)\pi_{k}}
    \end{align}
    where \(q(\x \mid k)\) is the conditional density of \(\x\) given the event \(\{u = k\}\). To compute this quantity, note that \textit{conditional on the value of \(u\)}, we have
    \begin{equation}
        \x = \z_{u} + \sigma \w \sim \mathcal{N}(\Zero, \vSigma_{u} + \sigma^{2}\I_{d}).
    \end{equation}
    Thus we have
    \begin{equation}
        q(\x \mid k) = \frac{1}{\sqrt{(2\pi)^{d}\det(\vSigma_{k} + \sigma^{2}\I_{d})}}\explr{-\frac{1}{2}\x\adj(\vSigma_{k} + \sigma^{2}\I_{d})^{-1}\x},
    \end{equation}
    This gives
    \begin{equation}
        \nabla_{\x}q(\x \mid k) = -q(\x \mid k) \cdot (\vSigma_{k} + \sigma^{2}\I_{d})^{-1}\x.
    \end{equation}
    Putting this all together, we get
    \begin{align}
        &\nabla_{\x}\log q(\x) \\
        &= -\frac{\sum_{k = 1}^{K}q(\x \mid k)\pi_{k} \cdot (\vSigma_{k} + \sigma^{2}\I_{d})^{-1}\x}{\sum_{k = 1}^{K}q(\x \mid k)\pi_{k}} \\
        &= -\frac{\sum_{k = 1}^{K}\pi_{k}\det(\vSigma_{k} + \sigma^{2}\I_{d})^{-1/2}\explr{-\frac{1}{2}\x\adj(\vSigma_{k} + \sigma^{2}\I_{d})^{-1}\x}\cdot (\vSigma_{k} + \sigma^{2}\I_{d})^{-1}\x}{\sum_{k = 1}^{K}\pi_{k}\det(\vSigma_{k} + \sigma^{2}\I_{d})^{-1/2}\explr{-\frac{1}{2}\x\adj(\vSigma_{k} + \sigma^{2}\I_{d})^{-1}\x}}.
    \end{align}
    Now define \(\vM_{k} \doteq (\vSigma_{k} + \sigma^{2}\I_{d})^{-1/2}\). With this notation, we have
    \begin{align}
        \nabla_{\x}\log q(\x) 
        &= -\frac{\sum_{k = 1}^{K}\pi_{k}\det(\vM_{k})\explr{-\frac{1}{2}\x\adj\vM_{k}\vM_{k}\adj\x}\cdot \vM_{k}\vM_{k}\adj\x}{\sum_{k = 1}^{K}\pi_{k}\det(\vM_{k})\explr{-\frac{1}{2}\x\adj\vM_{k}\vM_{k}\adj\x}} \\
        &= -\frac{\sum_{k = 1}^{K}\pi_{k}\det(\vM_{k})\explr{-\frac{1}{2}\norm{\vM_{k}\adj\x}_{2}^{2}}\cdot \vM_{k}\vM_{k}\adj\x}{\sum_{k = 1}^{K}\pi_{k}\det(\vM_{k})\explr{-\frac{1}{2}\x\adj\vM_{k}\vM_{k}\adj\x}}.
    \end{align}
    Given our assumption that 
    each \(\pi_{k}\det(\M_{k})\) is the same, we have
    \begin{align}
        &\nabla_{\x}\log q(\x) \\
        &= -\frac{\sum_{k = 1}^{K}\pi_k\det(\vM_{k})\explr{-\frac{1}{2}\norm{\vM_{k}\adj\x}_{2}^{2}}\cdot \vM_{k}\vM_{k}\adj\x}{\sum_{k = 1}^{K}\pi_k\det(\vM_{k})\explr{-\frac{1}{2}\norm{\vM_{k}\adj\x}_{2}^{2}}} \\
        &= -\frac{\sum_{k = 1}^{K}\explr{-\frac{1}{2}\norm{\vM_{k}\adj\x}_{2}^{2}}\cdot \vM_{k}\vM_{k}\adj\x}{\sum_{k = 1}^{K}\explr{-\frac{1}{2}\norm{\vM_{k}\adj\x}_{2}^{2}}} \\
        &= -\sum_{k = 1}^{K}\e_{k}\adj\softmax{-\frac{1}{2}\mat{\norm{\vM_{1}\adj\x}_{2}^{2} \\ \vdots \\ \norm{\vM_{K}\adj\x}_{2}^{2}}}\vM_{k}\vM_{k}\adj\x \\
        &= -\mat{\vM_{1}, \dots, \vM_{K}}\left[\diaglr{\softmax{-\frac{1}{2}\mat{\norm{\vM_{1}\adj\x}_{2}^{2} \\ \vdots \\ \norm{\vM_{K}\adj\x}_{2}^{2}}}} \otimes \I_{d}\right]\mat{\vM_{1}\adj\x \\ \vdots \\ \vM_{K}\adj\x}.
    \end{align}
\end{proof}

Now we provide a final justification for the result cited in \Cref{sub:denoising}.

\begin{approximation}\label{thm:opt_denoiser_multi_subspaces}
    In the setting of \Cref{thm:score_multi_subspaces}, 
    diagonalize \(\vSigma_{k} = \vU_{k}\vLambda_{k}\vU_{k}\adj\) where \(\vU_{k} \in \bR^{d \times p}\) is orthogonal and \(\vLambda_{k} \succ \Zero \in \bR^{p \times p}\) is diagonal.\footnote{This assumption can be easily relaxed to \(\vLambda_{k} \succeq \Zero\) for all \(k\), but requires some more notation to handle, and the form of the solution does not change. Thus we handle the case where all matrices are full rank for simplicity.} Then we have the approximation
    \begin{equation}
        \mathbb{E}[\z \mid \x] \approx \mat{\vU_{1}, \dots, \vU_{K}}\left[\diaglr{\softmax{\frac{1}{2\sigma^{2}}\mat{\norm{\vU_{1}\adj\x}_{2}^{2} \\ \vdots \\ \norm{\vU_{K}\adj\x}_{2}^{2}}}} \otimes \I_{p}\right]\mat{\vU_{1}\adj\x \\ \vdots \\ \vU_{K}\adj\x}.
    \end{equation}
\end{approximation}
\begin{proof}
    We have
    \begin{align}
        \nabla_{\x}\log q(\x)
        &= -\sum_{k = 1}^{K}\e_{k}\adj\softmax{-\frac{1}{2}\mat{\norm{\vM_{1}\adj\x}_{2}^{2} \\ \vdots \\ \norm{\vM_{K}\adj\x}_{2}^{2}}}\vM_{k}\vM_{k}\adj\x \\
        &= -\sum_{k = 1}^{K}\e_{k}\adj\softmax{-\frac{1}{2\sigma^{2}}\mat{\norm{\sigma \vM_{1}\adj\x}_{2}^{2} \\ \vdots \\ \norm{\sigma \vM_{K}\adj\x}_{2}^{2}}}\vM_{k}\vM_{k}\adj\x \\
        &= -\sum_{k = 1}^{K}\e_{k}\adj\softmax{\frac{1}{2\sigma^{2}}\mat{\norm{\x}_{2}^{2} - \norm{\sigma \vM_{1}\adj\x}_{2}^{2} \\ \vdots \\ \norm{\x}_{2}^{2} - \norm{\sigma \vM_{K}\adj\x}_{2}^{2}}}\vM_{k}\vM_{k}\adj\x.
    \end{align}
    Now define \(\P_{k} \doteq \I_{d} - \sigma \M_{k}\), and let \(\vU_{k}^{\perp} \in \bR^{d \times (d - p)}\) be an orthogonal complement of \(\vU_{k}\). Then we have
    \begin{align}
        \P_{k}
        &= \I_{d} - \sigma \M_{k} \\
        &= \I_{d} - \sigma \left(\vSigma_{k} + \sigma^{2}\I_{d}\right)^{-1/2} \\
        &= \I_{d} - \sigma \left(\mat{\vU_{k} & \vU_{k}^{\perp}}\mat{\vLambda_{k} & \Zero \\ \Zero & \Zero}\mat{\vU_{k}\adj \\ (\vU_{k}^{\perp})\adj} + \sigma^{2}\I_{d}\right)^{-1/2} \\ 
        &= \I_{d} - \sigma \left(\mat{\vU_{k} & \vU_{k}^{\perp}}\mat{\vLambda_{k} + \sigma^{2}\I_{p} & \Zero \\ \Zero & \sigma^{2}\I_{d - p}}\mat{\vU_{k}\adj \\ (\vU_{k}^{\perp})\adj}\right)^{-1/2} \\ 
        &= \I_{d} - \mat{\vU_{k} & \vU_{k}^{\perp}} \mat{\sigma(\vLambda_{k} + \sigma^{2}\I_{p})^{-1/2} & \Zero \\ \Zero & \sigma \cdot (\sigma^{2})^{-1/2}\I_{d - p}}\mat{\vU_{k}\adj \\ (\vU_{k}^{\perp})\adj} \\ 
        &= \I_{d} - \mat{\vU_{k} & \vU_{k}^{\perp}} \mat{(\sigma^{-2}\vLambda_{k} + \I_{p})^{-1/2} & \Zero \\ \Zero & \I_{d - p}}\mat{\vU_{k}\adj \\ (\vU_{k}^{\perp})\adj} \\ 
        &= \mat{\vU_{k} & \vU_{k}^{\perp}} \mat{\I_{p} - (\sigma^{-2}\vLambda_{k} + \I_{p})^{-1/2} & \Zero \\ \Zero & \Zero}\mat{\vU_{k}\adj \\ (\vU_{k}^{\perp})\adj} \\
        &\approx \mat{\vU_{k} & \vU_{k}^{\perp}} \mat{\I_{p} & \Zero \\ \Zero & \Zero}\mat{\vU_{k}\adj \\ (\vU_{k}^{\perp})\adj} \\
        &= \vU_{k}\vU_{k}\adj.
    \end{align}
    Thus \(\vP_{k}\) is approximately a projection when \(\sigma\) is small. Under this algebraic relation, we have
    \begin{align}
        &\nabla_{\x}\log q(\x) \\
        &= -\sum_{k = 1}^{K}\e_{k}\adj\softmax{\frac{1}{2\sigma^{2}}\mat{\norm{\x}_{2}^{2} - \norm{\sigma \vM_{1}\adj\x}_{2}^{2} \\ \vdots \\ \norm{\x}_{2}^{2} - \norm{\sigma \vM_{K}\adj\x}_{2}^{2}}}\vM_{k}\vM_{k}\adj\x \\
        &= -\frac{1}{\sigma^{2}}\sum_{k = 1}^{K}\e_{k}\adj\softmax{\frac{1}{2\sigma^{2}}\mat{\norm{\x}_{2}^{2} - \norm{(\I_{d} - \vP_{1})\adj\x}_{2}^{2} \\ \vdots \\ \norm{\x}_{2}^{2} - \norm{(\I_{d} - \vP_{K})\adj\x}_{2}^{2}}}(\I_{d} - \vP_{k})(\I_{d} - \vP_{k})\adj\x \\
        &\approx -\frac{1}{\sigma^{2}}\sum_{k = 1}^{K}\e_{k}\adj\softmax{\frac{1}{2\sigma^{2}}\mat{\norm{\vP_{1}\adj\x}_{2}^{2} \\ \vdots \\ \norm{\vP_{K}\adj\x}_{2}^{2}}}(\I_{d} - \vP_{k})(\I_{d} - \vP_{k})\adj\x \\
        &\approx -\frac{1}{\sigma^{2}}\sum_{k = 1}^{K}\e_{k}\adj\softmax{\frac{1}{2\sigma^{2}}\mat{\norm{\P_{1}\adj\x}_{2}^{2} \\ \vdots \\ \norm{\P_{K}\adj\x}_{2}^{2}}}(\I_{d} - \P_{k})^{*}\x \\
        &= -\frac{\x}{\sigma^{2}}\sum_{k = 1}^{K}\e_{k}\adj\softmax{\frac{1}{2\sigma^{2}}\mat{\norm{\P_{1}\adj\x}_{2}^{2} \\ \vdots \\ \norm{\P_{K}\adj\x}_{2}^{2}}} + \frac{1}{\sigma^{2}}\sum_{k = 1}^{K}\e_{k}\adj\softmax{\frac{1}{2\sigma^{2}}\mat{\norm{\P_{1}\adj\x}_{2}^{2} \\ \vdots \\ \norm{\P_{K}\adj\x}_{2}^{2}}}\P_{k}\adj\x \\
        &= -\frac{1}{\sigma^{2}}\x + \frac{1}{\sigma^{2}}\sum_{k = 1}^{K}\e_{k}\adj\softmax{\frac{1}{2\sigma^{2}}\mat{\norm{\P_{1}\adj\x}_{2}^{2} \\ \vdots \\ \norm{\P_{K}\adj\x}_{2}^{2}}}\P_{k}\adj\x \\
        &\approx -\frac{1}{\sigma^{2}}\x + \frac{1}{\sigma^{2}}\sum_{k = 1}^{K}\e_{k}\adj\softmax{\frac{1}{2\sigma^{2}}\mat{\norm{\vU_{1}\adj\x}_{2}^{2} \\ \vdots \\ \norm{\vU_{K}\adj\x}_{2}^{2}}}\vU_{k}\vU_{k}\adj\x \\
        &= -\frac{1}{\sigma^{2}}\x + \frac{1}{\sigma^{2}}\mat{\vU_{1}, \cdots, \vU_{K}}\left[\diaglr{\softmax{\frac{1}{2\sigma^{2}}\mat{\norm{\vU_{1}\adj\x}_{2}^{2} \\ \vdots \\ \norm{\vU_{K}\adj\x}_{2}^{2}}}} \otimes \I_{p}\right]\mat{\vU_{1}\adj\x \\ \vdots \\ \vU_{K}\adj\x}.
    \end{align}
    Plugging this into Tweedie's formula, we have
    \begin{equation}
        \mathbb{E}[\z \mid \x] \approx \mat{\vU_{1}, \cdots, \vU_{K}}\left[\diaglr{\softmax{\frac{1}{2\sigma^{2}}\mat{\norm{\vU_{1}\adj\x}_{2}^{2} \\ \vdots \\ \norm{\vU_{K}\adj\x}_{2}^{2}}}} \otimes \I_{p}\right]\mat{\vU_{1}\adj\x \\ \vdots \\ \vU_{K}\adj\x}.
    \end{equation}
\end{proof}

\begin{remark}
   Although \Cref{thm:opt_denoiser_multi_subspaces} is stated as an approximation rather than as a proposition, we believe it should be possible without too much extra work to convert it into a statement of asymptotic equivalence as $\sigma \to 0$ (in particular, holding for $\sigma$ below the smallest (nonzero) eigenvalue of any $\vSigma_k$. Most approximations taken in the derivation of \Cref{thm:opt_denoiser_multi_subspaces} can immediately be turned into asymptotic claims; the only slightly delicate point is treating the softmax, which can be accomplished using standard ``high temperature'' convergence behavior of the softmax function %
   (in particular, as $\sigma \to 0$ in our expressions, the softmax concentrates on the ``best head'').
\end{remark}

\subsection{Companion to \Cref{sub:compression}} \label{app:proofs-compression}

We again wish to re-iterate the core contribution of our approach in \Cref{sub:compression}. The application of a compression perspective to representation learning has been discussed before, for example in the line of maximal coding rate reduction works \cite{OriginalMCR2}. In \Cref{sub:compression}, we provide the following contributions and developments to this perspective:
\begin{itemize}
    \item We propose a generalized coding rate function \(R^{c}(\cdot; \vU_{[K]})\) which measures the coding rate with respect to a set of subspaces \(\vU_{[K]}\) as opposed to a set of classes (as in \cite{OriginalMCR2,chan2021redunet}), making the underlying formulation unsupervised.
    \item We then show how if we adopt the framework of alternating minimization of the sparse rate reduction objective, then unrolling the first alternating step --- gradient descent on this coding rate objective --- nearly exactly recovers the common multi-head attention mechanism found in transformer networks (except that the query/key/value operators are all the same operation \(\vU_{k}\adj\) now, which we interpret as projection onto a single subspace).
\end{itemize}
In the process of the second contribution, and in the following proofs, we make some simple approximations and technical assumptions. The validity of these assumptions may be explored, and the approximations refined, altogether providing a more complex (and possibly more performant) resulting self-attention like operator. For the sake of technical clarity and simplicity in this work, we make perhaps the \textit{simplest possible choices}. As a result, we \textit{do not} claim that our network is optimally designed, but rather that the principles we develop in this work (compression, denoising, sparsification, unrolled optimization) can provide the backbone for far superior and more interpretable network architectures in the future on sundry tasks. As it is, with our straightforward, simple, and interpretable design, we still obtain meaningful conceptual results and very solid empirical performance.

We now give the derivation of the approximation alluded to in \Cref{sub:compression}.

\begin{approximation}\label{thm:grad_codingrate_approx} 
    Let \(\Z \in \bR^{d \times N}\) have unit-norm columns, and \(\vU_{[K]} = (\vU_1, \dots, \vU_K) \) such that each \(\vU_{k} \in \bR^{d \times p}\) is an orthogonal matrix, the \((\vU_{k})_{k = 1}^{K}\) are incoherent, and the columns of \(\Z\) approximately lie on \(\bigcup_{k = 1}^{K}\mathrm{Span}(\vU_{k})\). Let \(\gamma = \frac{p}{N\eps^{2}}\). Let \(\kappa > 0\). Then
    \begin{equation}
        \Z - \kappa \nabla_{\Z}R^{c}(\Z \mid \vU_{[K]}) \approx (1 - \kappa\gamma)\Z + \kappa\gamma\MSSA{\Z | \vU_{[K]}},
    \end{equation}
    where as in \Cref{sub:compression} we have
    \begin{align}
        \SSA{\Z | \vU_{k}} 
        &= (\vU_{k}\adj\Z)\softmax{(\vU_{k}\adj\Z)\adj(\vU_{k}\adj\Z)}, \\
        \MSSA{\Z | \vU_{[K]}} 
        &= \gamma\mat{\vU_{1}, \dots, \vU_{K}}\mat{\SSA{\Z | \vU_{1}} \\ \vdots \\ \SSA{\Z | \vU_{K}}},
    \end{align}
    where \(\softmax{\cdot}\) is the softmax operator (applied to each column of an input matrix), i.e.,
    \begin{align}
        \softmax{\vv} 
        &= \frac{1}{\sum_{i = 1}^{n}e^{v_{i}}}\mat{e^{v_{1}} \\ \vdots \\ e^{v_{n}}}, \\
        \softmax{\mat{\vv_{1}, \dots, \vv_{K}}} 
        &= \mat{\softmax{\vv_{1}}, \dots, \softmax{\vv_{K}}}.
    \end{align}
\end{approximation}

\begin{proof}
    According to \Cref{eq:rate-gradient}, the gradient $\nabla_{\Z}{R}^c(\Z; \vU_{[K]})$ is
    \begin{equation}
     \nabla_{\Z}{R}^c(\Z; \vU_{[K]}) 
         = \gamma \sum_{k=1}^K \vU_k\vU_k\adj\Z\bp{\I +
    \gamma(\vU_k\adj\Z)\adj(\vU_k\adj\Z)}^{-1}.
    \end{equation}
    Notice that according to \cite{chan2021redunet}, the gradient is precisely the residual of a ridge regression for each (projected) token $\vU_k\adj\z_{i}$ using other projected tokens \(\vU_{k}\adj\z_{j}\) as the regressors, hence being the residual of an auto-regression. 
    
    However, as we have seen in the work of ReduNet \cite{chan2021redunet}, computing the inverse $\bp{\I +
    \gamma(\vU_k^*\Z)^*(\vU_k^*\Z)}^{-1}$ can be expensive. Hence for computational efficiency, we may approximate it with the first order term of its von Neumann expansion: 
    \begin{align}
    \nabla_{\Z}{R}^c(\Z; \vU_{[K]}) 
        & = \gamma \sum_{k=1}^K \vU_k\vU_k^*\Z\Big({\I +
    \gamma(\vU_k^*\Z)^*(\vU_k^*\Z)}\Big)^{-1} \\
     & \approx \gamma \sum_{k=1}^K \vU_k\vU_k^*\Z\Big({\I - 
    \gamma(\vU_k^*\Z)^*(\vU_k^*\Z)}\Big) \\
        &=  \gamma\sum_{k=1}^{K}\vU_k\Big({\vU^*_k\Z - \gamma\vU^*_k\Z [(\vU_{k}\adj\Z)\adj(\vU_{k}\adj\Z)]}\Big)
        \label{eqn:approximation-1}
    \end{align}
    Notice that the term \((\vU_{k}\adj\Z)\adj(\vU_{k}\adj\Z)\) is the auto-correlation among the projected tokens. As the tokens $\Z$ may be from different subspaces, we would prefer to use only tokens that belong to the \textit{same} subspace to regress and compress themselves. Hence we may convert the above correlation term into a subspace-membership indicator with a softmax operation, whence \eqref{eqn:approximation-1} becomes
    \begin{eqnarray}
        \nabla_{\Z}{R}^c(\Z; \vU_{[K]}) 
        &\approx & \gamma\sum_{k=1}^{K}\vU_k\Big({\vU^*_k\Z - \gamma\vU^*_k\Z [(\vU_{k}\adj\Z)\adj(\vU_k\adj\Z)]}\Big)\\
        &\approx & \gamma\sum_{k=1}^{K}\vU_k\vU_k^*\Z - \gamma^{2}\sum_{k=1}^{K}\vU_k\Big(\vU_k^*\Z \softmax{(\vU_{k}\adj\Z)\adj(\vU_{k}\adj\Z)}\Big)
        \label{eqn:approximation-2}
    \end{eqnarray}
    
    Then, we can rewrite the above approximation to the gradient of $R^c$ as: 
    \begin{align}\label{eq:appendix-ssa-derivation}
        \nabla_{\Z}{R}^c(\Z; \vU_{[K]}) 
        &\approx \gamma\sum_{k=1}^{K}\vU_k\vU_k^*\Z - \gamma^{2}\sum_{k=1}^{K}\vU_k\bp{\vU_k^*\Z \softmax{(\vU_{k}\adj\Z)\adj(\vU_k\adj\Z)}} \\
         &= \gamma\sum_{k=1}^{K}\vU_k\vU_k^*\Z - \gamma^{2}\sum_{k=1}^{K}\vU_k\SSA{\Z\given\vU_k} \\
         &= \underbrace{\bp{\gamma\sum_{k=1}^{K}\vU_k\vU_k^*}\Z}_{\approx \gamma \Z} -\gamma^{2}\mat{\vU_1, \cdots, \vU_K}
         \mat{\SSA{\Z\given\vU_1} \\ \vdots \\  \SSA{\Z\given\vU_K} } \\
         &\approx\gamma \Z - \gamma^{2}\mat{\vU_{1}, \cdots, \vU_{K}} \mat{\SSA{\Z \given \vU_{1}} \\ \vdots \\ \SSA{\Z \given \vU_{K}}}.
    \end{align}
    Thus the gradient descent step with learning rate \(\kappa > 0\) gives
    \begin{equation}\label{eq:appendix-mssa-block-variant}
        \Z - \kappa\nabla_{\Z}R^{c}(\Z \mid \vU_{[K]}) \approx (1 - \kappa\gamma)\Z + \kappa \gamma^{2}\mat{\vU_{1}, \dots, \vU_{K}}\mat{\SSA{\Z | \vU_{1}} \\ \vdots \\ \SSA{\Z | \vU_{K}}}.
    \end{equation}
\end{proof}

\subsection{Companion to \Cref{sub:sparse}} \label{app:proofs-sparse}

We again wish to re-iterate the core contribution of our approach in \Cref{sub:sparse}.
\begin{itemize}
    \item Within the framework of alternating minimization of the sparse rate reduction objective, we show that the second alternating step --- gradient descent on the overall coding rate plus a sparse regularization term --- has heuristic connections to a particular LASSO optimization.
    \item We show that the unrolling of the proximal gradient step to solve this LASSO optimization resembles the MLP which immediately follows the self-attention layer within transformer blocks.
\end{itemize}
In the main text, our connection between the second step of the alternating minimization and the LASSO optimization was high-level and heuristic. In some sense, the choice to pose the minimization step as a LASSO was a \textit{simple, reliable, and interpretable choice} which works well in practice, but is nonetheless not backed up by rigorous theoretical justification. In the following subsection, we provide a mathematical justification for a reformulation of the minimization step using a majorization-minimization framework. We further show that the associated unrolled optimization step bears a strong resemblance to the ISTA step. This confirms our earlier discussion --- we took the \textit{simplest possible choice} in designing \ours{}, but by more rigorous derivation we can uncover alternative operators which nonetheless have the same conceptual function and may perform better in practice.

\paragraph{Assumptions.} In this section, we present a rigorous optimization
analysis of an incremental minimization approach to the objective
\Cref{eq:whole-sparse}. We will show that under two simplifying assumptions,
namely
\begin{enumerate}
    \item The columns of $\vZ^{\ell + 1/2}$ are normalized, in the sense that
        $\diag((\vZ^{\ell+1/2})\adj \vZ^{\ell + 1/2}) = \bm{1}$;\footnote{This is
            a natural assumption in transformer-type architectures such as
            \ours{} due to the use of LayerNorm blocks---although these blocks
            (indeed, as we use them in \ours{}) include trainable mean and
            scale offsets as well as an additional mean subtraction operation
            \cite{phuong2022formal}, they are initialized to have zero mean and
            unit norm, hence this assumption corresponds to an analysis of the
        network at its initialization. }
    \item We have $d \geq N$,\footnote{This assumption is without loss
            of generality, as we will see in the analysis below. The reason is that
            $\vZ\adj \vZ$ and $\vZ\adj \vZ$ have the same nonzero eigenvalues
            regardless of the shape of $\vZ$, which implies that $\log\det(\vI + \alpha
            \vZ\adj \vZ) = \log\det(\vI + \alpha \vZ\vZ\adj)$. In particular,
            interpreting the norms appropriately (with a slight abuse of notation), we
            have $\varphi(\vZ) = \varphi(\vZ\adj)$, so for the purposes of analysis
            we can always proceed as though $\vZ$ is a tall matrix (as long as
            we do not use any special properties of  $\alpha$ in our
        derivation).} and the columns of $\vZ^{\ell
        + 1/2}$ are
        orthogonal, so that $(\vZ^{\ell+1/2})\adj \vZ^{\ell + 1/2} =
        \vI$.\footnote{This assumption is strictly stronger than the previous
            one, and strictly stronger than an assumption of incoherence on the
            columns. It corresponds to the representation $\vZ^{\ell+1/2}$ being non-collapsed, which we expect to hold at initialization due to the projections $\vU_{[K]}$ being random. %
        }
\end{enumerate}
the approach leads to an update iteration that is equal to a slightly
simplified version of the ISTA block \Cref{eq:ista-block}. We see this as a
justification for our derivation in \Cref{sub:sparse}, which obtained the ISTA
block by introducing an additional simplifying assumption on the distribution
of the data at layer $\ell$.

\paragraph{Analysis.} Following \Cref{eq:sparse-nonnegative}, we will consider
the natural relaxation of the $\ell{}_0$ ``norm'' to the $\ell^1$ norm, and
incorporate a nonnegativity constraint.  Consider the objective
\begin{equation}
    \varphi(\vZ) = \lambda \norm{\vZ}_1 + \chi_{\set{\vZ \geq \Zero}}(\vZ)
    - \underbrace{ \frac{1}{2} \log\det\left( \vI + \alpha \vZ\adj \vZ
    \right)}_{R(\vZ)},
\end{equation}
where $\vZ \in \bbR^{d \times N}$ and $\alpha = d / N \veps^2$, and
$\chi_{\set{\vZ \geq \Zero}}$ denotes the characteristic function for the set
of elementwise-nonnegative matrices $\vZ$. As in
\Cref{app:proofs-compression}, we calculate
\begin{equation}
    \nabla_{\vZ} R(\vZ) = \alpha \vZ \left( \vI + \alpha \vZ\adj \vZ \right)\inv.
\end{equation}
We consider an incremental optimization scheme for the highly nonlinear and
nonconvex objective $\varphi$. Following \Cref{sub:compression}, we optimize
locally at a ``post-compression'' iterate $\vZ^{\ell + 1/2}$. We follow the
standard proximal majorize-minimize framework \cite{Wright-Ma-2022} for
incremental/local optimization: this begins with the second-order Taylor
expansion for the smooth part of $\varphi$ in a neighborhood of the
current iterate $\vZ^{\ell + 1/2}$:
\begin{equation}
    \begin{split}
        R(\vZ)
        =
        R(\vZ^{\ell + 1/2}) 
        &+ \ip*{\nabla_{\vZ} R(\vZ^{\ell + 1/2})}{\vZ - \vZ^{\ell + 1/2}} \\
        &+ \int_0^1 (1 - t) \ip*{\vZ - \vZ^{\ell+1/2}}{\nabla^2 R(\vZ_t)\left(
                \vZ - \vZ^{\ell+1/2}
            \right)
        }\ \mathrm{d}t,
    \end{split}
    \label{eq:R-quadratic-taylor}
\end{equation}
where for any $\vZ \in \bbR^{d \times N}$, $\vZ_t = t \vZ^{\ell+1/2} + (1-t)
\vZ$. The proximal majorization-minimization approach alternates two steps to
minimize $\varphi$:
\begin{enumerate}
    \item First, use assumptions on $\vZ^{\ell + 1/2}$ to derive an upper bound
        on the operator norm of the Hessian $\nabla^2 R(\vZ)$ over the
        effective domain of the optimization problem. We will write $L$ for
        this (uniform) upper bound. This yields a quadratic upper bound for the
        smooth part of the objective $\varphi$.
    \item Then, alternately minimize the \textit{smooth part} of the quadratic
        upper bound as a function of $\vZ$, and take a \textit{proximal step}
        on the nonsmooth part. It can be shown \cite{Wright-Ma-2022} that
        corresponds to the iteration
        \begin{equation}
            \vZ^+ = \prox{\frac{\lambda}{L} (\norm{}_1 + \chi_{\set{\vZ \geq \Zero}} ) }\left(
                \vZ + \frac{1}{L} \nabla_{\vZ} R(\vZ)%
            \right)
            \label{eq:mm-iter}
        \end{equation}
        In the alternating minimization setting of this paper for optimizing
        \Cref{eq:sparse-rr}, we only take one such step, starting at
        $\vZ^{\ell+1/2}$.
\end{enumerate}
We will instantiate this program below, showing quantitative error bounds
related to our assumptions above as necessary. Rather than directly applying
the iteration \Cref{eq:mm-iter}, we will derive it below under our
aforementioned assumptions.

Starting at \Cref{eq:R-quadratic-taylor}, our first task is to upper bound the
quadratic residual. 
This corresponds to estimating
\begin{align}
    &\ip*{\vZ - \vZ^{\ell+1/2}}{\nabla^2 R(\vZ_t)\left(
            \vZ - \vZ^{\ell+1/2}
        \right)
    } \\
    &\qquad\leq
    \sup_{t \in [0,1]}
    \norm*{
        \nabla^2 R(\vZ_t)
    }_{\ell^2 \to \ell^2}
    \norm*{
            \vZ - \vZ^{\ell+1/2}
    }_{\frob}^2
\end{align}
with Cauchy-Schwarz. 
Using
\Cref{lem:logdet-hessian}, we can estimate the operator norm term in the
previous bound in terms of properties of $\vZ^{\ell + 1/2}$. We need to bound 
\begin{equation}
    \alpha \sup_{\norm{\vDelta}_{\frob} \leq 1}
    \norm*{
        \left(
            \vDelta
            - \alpha \vZ_t(\vI + \alpha \vZ_t\adj \vZ_t)\inv (\vZ_t\adj \vDelta +
            \vDelta\adj \vZ_t)
        \right)(\vI + \alpha \vZ_t\adj \vZ_t)\inv
    }_{\frob},
\end{equation}
and \Cref{lem:logdet-grad-lipschitz} gives that this term is no larger than
$9\alpha / 4$ for any $\vZ$ and any $t$.
With this
estimate and \Cref{eq:R-quadratic-taylor}, we have a quadratic upper bound for
$-R(\vZ)$:
\begin{equation}
    -R(\vZ)
    \leq
    -R(\vZ^{\ell + 1/2}) 
    + \ip*{-\nabla_{\vZ} R(\vZ^{\ell + 1/2})}{\vZ - \vZ^{\ell + 1/2}} \\
    + \frac{9\alpha}{8} \norm*{
        \vZ - \vZ^{\ell+1/2}
    }_{\frob}^2.
\end{equation}
Meanwhile, by our assumptions above, we have
\begin{equation}
    -\nabla_{\vZ} R(\vZ^{\ell + 1/2})
    = -\alpha \vZ^{\ell+1/2} \left( \vI + \alpha \vI \right)\inv
    = -\frac{\alpha}{1 + \alpha} \vZ^{\ell + 1/2}.
\end{equation}
We now minimize the preceding quadratic upper bound as a
function of $\vZ$. Differentiating, the minimizer $\vZ_{\mathrm{opt}}$ is
calculated as
\begin{equation}
    \vZ_{\mathrm{opt}}
    = \left(1 + \frac{4}{9(1 + \alpha)}\right) \vZ^{\ell+1/2},
\end{equation}
and it is well-known that the proximal operator of the sum of
$\chi_{\set{\vZ\geq \Zero}}$ and $\lambda \norm{}_{1}$ is simply the one-sided
soft-thresholding operator \cite{Wright-Ma-2022}
\begin{equation}
    \prox{\chi_{\set{\vZ\geq\Zero}} + \lambda \norm{}_1}\left( \vZ \right)
    =
    \max \set{
        \vZ - \lambda \One, \Zero
    },
\end{equation}
where the maximum is applied elementwise. As in \Cref{sub:sparse}, we may write
this elementwise maximum simply as $\operatorname{ReLU}$.  Thus, one step of
proximal majorization-minimization under our simplifying assumptions takes the
form
\begin{equation}
    \vZ^{\ell+1}
    =
    \operatorname{ReLU}\left(
        \left(1 + \frac{4}{9(1 + \alpha)}\right) \vZ^{\ell+1/2}
        - \frac{4\lambda}{9\alpha}\One
    \right).
\end{equation}
Finally, we point out one additional elaboration which introduces the
dictionary $\vD$ that appears in the ISTA block in \Cref{sub:sparse}. 
Notice that for any orthogonal $\vD$, one has $R(\vD \vZ) = R(\vZ)$ for every
$\vZ$. This symmetry implies equivariance properties of $\nabla_{\vZ} R(\vZ)$
and $\nabla^2_{\vZ} R(\vZ)$: for every $\vZ$ and every $\vDelta$ and every
orthogonal $\vD$,
\begin{align}
    \vD \nabla_{\vZ} R(\vZ) &= \nabla_{\vZ} R(\vD \vZ), \\
    \ip{\vD \vDelta}{\nabla^2_{\vZ} R(\vZ) \left(\vD \vDelta\right)}
    &= 
    \ip{\vDelta}{\nabla^2_{\vZ} R(\vD \vZ) \left(\vDelta\right)}.
\end{align}
Hence the quadratic Taylor expansion \Cref{eq:R-quadratic-taylor} can be
written equivalently as
\begin{equation}
    \begin{split}
        R(\vZ)
        =
        R(\vD\adj \vZ^{\ell + 1/2}) 
        &+ \ip*{\nabla_{\vZ} R(\vD\adj \vZ^{\ell + 1/2})}{\vZ - \vZ^{\ell + 1/2}} \\
        &+ \int_0^1 (1 - t) \ip*{\vZ - \vZ^{\ell+1/2}}{\nabla^2 R(\vD\adj\vZ_t)\left(
                \vZ - \vZ^{\ell+1/2}
            \right)
        }\ \mathrm{d}t,
    \end{split}
\end{equation}
for any orthogonal $\vD$. The significance of this is that we have obtained an
expression equivalent to \Cref{eq:R-quadratic-taylor}, but with
$\vZ^{\ell+1/2}$ replaced by $\vD\adj \vZ^{\ell + 1/2}$; moreover, because our
approximation arguments above are not affected by left-multiplication of
$\vZ^{\ell + 1/2}$ by an orthogonal matrix (this operation does not change the
norms of the columns of $\vZ^{\ell + 1/2}$, or their correlations, and hence the
matrix's incoherence), we can apply exactly the same line of reasoning above to
obtain that an equivalent proximal majorization-minimization iteration is given
by 
\begin{equation}\label{eq:prox_maj_min_iteration}
    \vZ^{\ell+1}
    =
    \operatorname{ReLU}\left(
        \left(1 + \frac{4}{9(1 + \alpha)}\right) \vD\adj \vZ^{\ell+1/2}
        - \frac{4\lambda}{9\alpha} \One
    \right),
\end{equation}
for any orthogonal dictionary $\vD$. This gives an update quite similar to the
ISTA block \Cref{eq:ista-block} in the case where the dictionary used in
\Cref{sub:sparse} is orthogonal, but without a skip connection.

We thus obtain a natural
white-box version of this part of the architecture, along with the natural
interpretation \textit{that its purpose is to sparsify the compressed tokens
$\vZ^{\ell+1/2}$ in a (learnable) dictionary}, which accords with recent
empirical studies \cite{Li2023-ig}.

\paragraph{Other architectures?} As we mentioned at the start of this section,
the preceding derivation is performed in the most elementary possible setting
in order to demonstrate the majorization-minimization approach for layer
design. More precise approximations or assumptions may lead to superior layer
designs that better optimize the target objective \Cref{eq:sparse-rr} (and in
particular \Cref{eq:whole-sparse}). We mention two here:
\begin{enumerate}
    \item \textbf{Beyond exactly-incoherent features}: our derivations above assumed that the
        incoming representations $\vZ^{\ell+1/2}$ were already maximal for the
        expansion term $R$ in \Cref{eq:whole-sparse}. It is desirable to obtain
        a `perturbative' derivation, which applies in cases where
        $\vZ^{\ell+1/2}$ is not fully orthogonal, but instead near-orthogonal,
        in particular \textit{incoherent} \cite{Wright-Ma-2022}. The
        derivations above can be adapted to this setting; the perturbation
        bounds become slightly more delicate, and the ultimate layer
        \Cref{eq:prox_maj_min_iteration} changes to involve additional
        normalization.
    \item \textbf{Beyond orthogonal dictionaries}: The symmetries of the
        expansion term $R$ in \Cref{eq:whole-sparse} may be followed to lead to
        a pair of dictionaries $\vD$ and $\vD'$ and an objective that
        sparsifies $\vD \vZ \vD'$. This type of transformation is suggestive of
        popular architectures that mix over tokens
        \cite{Tolstikhin2021-yh,Trockman2022-po}, however we consider the
        simpler form $\vD \vZ$ in this work. In addition, we have
        focused for simplicity on orthogonal dictionaries $\vD$; as in the
        previous bullet, one may consider in a similar way dictionaries $\vD$
        which are complete and near-orthogonal. Adapting the derivation to
        \textit{overcomplete dictionaries} is an interesting future direction
        that we expect to improve the scalability of \ours{}; one avenue to
        achieve this could be increasing the number of projections $\vU_{[K]}$
        and their embedding dimensions.
\end{enumerate}

\subsubsection{Auxiliary Lemmas}

\begin{lemma}
    \label{lem:logdet-hessian}
    Consider the function
    \begin{equation}
        R(\vZ) = \frac{1}{2} \log \det \left( \vI + \alpha \vZ\adj \vZ \right),
    \end{equation}
    where $\alpha > 0$ is a constant. Then we have
    \begin{equation}
        \nabla_{\vZ} R(\vZ) = \alpha \vZ \left( \vI + \alpha \vZ\adj \vZ\right)
        \inv,
    \end{equation}
    and the Hessian operator $\nabla_{\vZ}^{2}R(\vZ) \colon \mathbb{R}^{d \times N} \to \mathbb{R}^{d \times N}$ satisfies that for any $\vDelta \in \mathbb{R}^{d \times N}$, 
    \begin{align}
        &\nabla^2_{\vZ} R(\vZ)\left( \vDelta \right) \\
        &=
        \alpha\vDelta \left( \vI + \alpha \vZ\adj \vZ \right)\inv
        - \alpha^2 \vZ 
        \left( \vI + \alpha \vZ\adj \vZ \right)\inv
        \left( \vZ\adj \vDelta + \vDelta\adj \vZ \right)
        \left( \vI + \alpha \vZ\adj \vZ \right)\inv.
    \end{align}
    \begin{proof}
        The gradient calculation follows from \cite{OriginalMCR2}, for example.
        For the Hessian, we use the usual approach to calculating derivatives:
        if $\vDelta$ is any matrix with the same shape as $\vZ$ and $t > 0$,
        \begin{equation}
            \nabla_{\vZ}^2 R(\vZ)\left(
                \vDelta
            \right) = 
            \dac\left[t\mapsto \nabla_{\vZ} R(\vZ + t
            \vDelta)\right],
        \end{equation}
        valid since $R$ is smooth. We have
        \begin{align*}
            &\nabla_{\vZ} R(\vZ + t \vDelta)\\
            =&
            \alpha(\vZ + t \vDelta) \left( \vI + \alpha (\vZ +
            t\vDelta)\adj(\vZ + t\vDelta) \right)\inv \\
            =&
            \alpha(\vZ + t \vDelta) \left( \vI + \alpha \vZ\adj \vZ
                + \alpha t\left[
                    \vZ\adj \vDelta + \vDelta\adj \vZ + t \vDelta\adj \vDelta
                \right]
            \right)\inv \\
            =&
            \alpha(\vZ + t \vDelta) \left( 
                \vI + 
                \alpha t
                \left(\vI + \alpha \vZ\adj \vZ\right)\inv
                \left[
                    \vZ\adj \vDelta + \vDelta\adj \vZ + t \vDelta\adj \vDelta
                \right]
            \right)\inv
            \left( \vI + \alpha \vZ\adj \vZ \right)\inv
            \\
            =&
            \alpha(\vZ + t \vDelta) \left( 
                \sum_{k = 0}^\infty
                (- \alpha t)^k
                \left(
                    \left(\vI + \alpha \vZ\adj \vZ\right)\inv
                    \left[
                        \vZ\adj \vDelta + \vDelta\adj \vZ + t \vDelta\adj \vDelta
                    \right]
                \right)^k
            \right)
            \left( \vI + \alpha \vZ\adj \vZ \right)\inv,
        \end{align*}
        where in the fourth line we require that $t$ is sufficiently close to
        $0$ in order to invoke the Neumann series. First, notice that the term
        involving $\vDelta\adj \vDelta$ does not play a role in the final
        expression: after we differentiate with respect to $t$ and take a limit
        $t \to 0$, terms arising due to differentiation of $t \mapsto t
        \vDelta\adj \vDelta$ go to zero, because whenever the summation index
        $k > 0$ we have a term $(-\alpha t)^k$ that goes to zero as $t \to 0$.
        We thus obtain with the product rule
        \begin{align}
            &\dac\left[t\mapsto \nabla_{\vZ} R(\vZ + t
            \vDelta)\right]\\
            =\,\, 
            &\alpha\vDelta \left( \vI + \alpha \vZ\adj \vZ \right)\inv 
             - \alpha^2 \vZ 
            \left( \vI + \alpha \vZ\adj \vZ \right)\inv
            \left( \vZ\adj \vDelta + \vDelta\adj \vZ \right)
            \left( \vI + \alpha \vZ\adj \vZ \right)\inv.
        \end{align}
    \end{proof}
\end{lemma}

\begin{lemma}
    \label{lem:logdet-grad-lipschitz}
    One has
    \begin{equation}
        \sup_{\norm{\vDelta}_{\frob} \leq 1}
        \norm*{
            \left(
                \vDelta
                - \alpha \vZ_t(\vI + \alpha \vZ_t\adj \vZ_t)\inv (\vZ_t\adj \vDelta +
                \vDelta\adj \vZ_t)
            \right)(\vI + \alpha \vZ_t\adj \vZ_t)\inv
        }_{\frob}
        \leq
        \frac{9}{4}.
    \end{equation}
    \begin{proof}
        Fix $\vDelta$ satisfying $\norm{\vDelta}_{\frob}\leq 1$. By the
        triangle inequality,
        \begin{align}
            &\norm*{
                \left(
                    \vDelta
                    - \alpha \vZ_t(\vI + \alpha \vZ_t\adj \vZ_t)\inv (\vZ_t\adj \vDelta +
                    \vDelta\adj \vZ_t)
                \right)(\vI + \alpha \vZ_t\adj \vZ_t)\inv
            }_{\frob}
            \\
            &\leq
            \norm*{
                \vDelta (\vI + \alpha \vZ_t\adj \vZ_t)\inv
            }_{\frob}
            +
            \alpha \norm*{
                \vZ_t(\vI + \alpha \vZ_t\adj \vZ_t)\inv (\vZ_t\adj \vDelta +
                \vDelta\adj \vZ_t)
                (\vI + \alpha \vZ_t\adj \vZ_t)\inv
            }_{\frob}.
        \end{align}
        For the first term, we note that
        \begin{equation}
            \norm*{
                \vDelta (\vI + \alpha \vZ_t\adj \vZ_t)\inv
            }_{\frob}
            =
            \norm*{
                \left(
                    (\vI + \alpha \vZ_t\adj \vZ_t)\inv
                    \kron \vI
                \right) \vec(\vDelta)
            }_{\frob},
        \end{equation}
        and since $(\vI + \alpha \vZ_t\adj \vZ_t)\inv \preceq \vI$,
        we obtain from Cauchy-Schwarz\footnote{Recall that the eigenvalues of a Kronecker product
            of symmetric matrices are the tensor product of the eigenvalues
        (with multiplicity).}
        \begin{equation}
            \norm*{
                \vDelta (\vI + \alpha \vZ_t\adj \vZ_t)\inv
            }_{\frob}
            \leq \norm{\vDelta}_{\frob}.
        \end{equation}
        We can use a similar idea to control the second term. We have from the
        triangle inequality
        \begin{align}
            &\norm*{
                \vZ_t(\vI + \alpha \vZ_t\adj \vZ_t)\inv (\vZ_t\adj \vDelta +
                \vDelta\adj \vZ_t)
                (\vI + \alpha \vZ_t\adj \vZ_t)\inv
            }_{\frob} \\
            &\quad\leq
            \norm*{
                \vZ_t(\vI + \alpha \vZ_t\adj \vZ_t)\inv 
                \vZ_t\adj \vDelta
                (\vI + \alpha \vZ_t\adj \vZ_t)\inv
            }_{\frob} \\
            &\qquad+
            \norm*{
                (\vI + \alpha \vZ_t\adj \vZ_t)\inv \vZ_t\adj
                \vDelta 
                (\vI + \alpha \vZ_t\adj \vZ_t)\inv
                \vZ_t\adj
            }_{\frob}.
        \end{align}
        For the first term, we have
        \begin{align}
            &\norm*{
                \vZ_t(\vI + \alpha \vZ_t\adj \vZ_t)\inv 
                \vZ_t\adj \vDelta
                (\vI + \alpha \vZ_t\adj \vZ_t)\inv
            }_{\frob} \\
            &\quad=
            \norm*{
                \left(
                    (\vI + \alpha \vZ_t\adj \vZ_t)\inv
                    \kron
                    \vZ_t(\vI + \alpha \vZ_t\adj \vZ_t)\inv \vZ_t\adj
                \right)
                \vec(\vDelta)
            }_{\frob} \\
            &\quad\leq
            \sigma_{\max}\left(
                (\vI + \alpha \vZ_t\adj \vZ_t)\inv
            \right)
            \sigma_{\max}\left(
                \vZ_t(\vI + \alpha \vZ_t\adj \vZ_t)\inv \vZ_t\adj
            \right)
            \norm{\vDelta}_{\frob} \\
            &\quad\leq \frac{1}{\alpha} \norm{\vDelta}_{\frob}.
        \end{align}
        The last estimate follows from a computation using the SVD of $\vZ_t$.
        Meanwhile, we have for the second term by a similar argument (using the
        fact that the singular values of $\vA$ and $\vA\adj$ are identical for
        any matrix $\vA$)
        \begin{align}
            \norm*{
                (\vI + \alpha \vZ_t\adj \vZ_t)\inv \vZ_t\adj
                \vDelta 
                (\vI + \alpha \vZ_t\adj \vZ_t)\inv
                \vZ_t\adj
            }_{\frob}
            &\leq
            \sigma_{\max}\left( (\vI + \alpha \vZ_t\adj \vZ_t)\inv \vZ_t\adj
            \right)^2
            \norm{\vDelta}_{\frob} \\
            &\leq
            \frac{1}{4\alpha} \norm{\vDelta}_{\frob},
        \end{align}
        where once again the estimate follows from a computation involving the
        SVD of $\vZ_t$ (together with the fact that the function $\sigma
        \mapsto \sigma / (1 + \alpha \sigma^2)$ is bounded on $\sigma \geq 0$
        by $1/(2 \sqrt{\alpha})$). Putting it together, we have obtained
        \begin{equation}
            \norm*{
                \left(
                    \vDelta
                    - \alpha \vZ_t(\vI + \alpha \vZ_t\adj \vZ_t)\inv (\vZ_t\adj \vDelta +
                    \vDelta\adj \vZ_t)
                \right)(\vI + \alpha \vZ_t\adj \vZ_t)\inv
            }_{\frob}
            \leq \frac{9}{4} \norm{\vDelta}_{\frob},
        \end{equation}
        which gives the claim after taking suprema.

    \end{proof}
\end{lemma}


\newpage
\section{Additional Experiments and Details}\label{sec:appendix-exp}

In this section, we provide details about our experiments, and report the results of additional experiments that were not covered in the main text. \ours{} takes arguably the most basic design choices possible, and so we do \textit{not} attempt to directly compete with state-of-the-art performance from heavily engineered and empirically designed transformers. 
The results of our experiments are meant to convey a few core messages:
\begin{itemize}
    \item \textit{Despite not being engineered to compete with the state-of-the-art, \ours{} performs strongly on large-scale real-world datasets}, including classification on ImageNet-1K. \ours{} also achieves strong transfer learning performance.
    \item \textit{Because our model is designed through unrolled optimization of a well-understood objective, each layer is interpretable}. In particular, we can analyze the performance of \ours{}, as well as design network modifications, on a \textit{layer-wise basis}. This is powered by an arguably unparalleled level of insight into the role of each operator in our network.
    \item \textit{We make the simplest possible choices during the design of \ours{}, but these can be changed easily while keeping the same framework}. We study a few modifications later in this section (\Cref{subsec:appendix-arch-variants}) and show that they do not significantly hurt empirical performance, but emphasize here that there is significant potential for improvement with different architecture choices (and in particular a different theoretical analysis).
\end{itemize}

\subsection{Implementation details}\label{sub:appendix-exp-details}
In this subsection, we provide more details for implementing \ours{} on vision tasks.

\subsubsection{Architecture of \ours{}}
\paragraph{Architectural modifications.} Compared to the conceptual architecture proposed in \Cref{sub:architecture,sec:exp}, we make the following change for the sake of implementation simplicity:
\begin{itemize}
    \item In the compression step, replace the term \(\frac{p}{N\eps^{2}}\mat{\vU_{1},\dots,\vU_{K}}\) in the \(\texttt{MSSA}\) operator with another trainable parameter \(\vW \in \bR^{d \times pK}\). Thus the \(\texttt{MSSA}\) block becomes
    \begin{equation}\label{eq:mssa_trainable_w}
        \MSSA{\Z \mid \vU_{[K]}, \W} \doteq \vW\mat{\SSA{\Z \mid \vU_{1}} \\ \vdots \\ \SSA{\Z \mid \vU_{K}}}.
    \end{equation}
\end{itemize}

\textbf{\texttt{PyTorch} code for \ours{}.} 
We provide \texttt{PyTorch}-style code for implementing our proposed network architecture. 
\Cref{algo:pseudocode_crate} defines the overall architecture, \Cref{algo:pseudo code} and \Cref{algo:pseudocode_others} contain details for the transformer block, self-attention block (\texttt{MSSA}-block), and MLP block (\texttt{ISTA}-block).

\subsubsection{Training Setup}

\paragraph{Pre-training on ImageNet-1K.} 
We apply the Lion optimizer~\cite{chen2023symbolic} for pre-training both \ours{} and ViT models. 
We configure the learning rate as $2.4 \times 10^{-4}$, weight decay as 0.5, and batch size as 2,048. 
We incorporate a warm-up strategy with a linear increase over 5 epochs, followed by training the models for a total of 150 epochs with cosine decay. 
For data augmentation, we only apply the standard techniques, random cropping and random horizontal flipping, on the ImageNet-1K dataset. 
We apply label smoothing with smoothing parameter $0.1$. 
One training epoch of $\ours{-Base}$ takes around 240 seconds using 16 A100 40GB GPUs.

\paragraph{Fine-tuning.} 
We fine-tune our pre-trained \ours{} and ViT models on the following target datasets: CIFAR10/CIFAR100~\cite{krizhevsky2009learning}, Oxford Flowers-102~\cite{nilsback2008automated}, Oxford-IIIT-Pets~\cite{parkhi2012cats}. 
We also evaluate our pre-trained models on the commonly used ImageNet Real~\cite{beyer2020we} benchmark. For each fine-tuning task, we use the AdamW optimizer~\cite{loshchilov2017decoupled}. 
We configure the learning rate as $5 \times 10^{-5}$, weight decay as 0.01, and batch size to be 512. 
To allow transfer learning, we first resize our input data to 224. For data augmentations, we also adopt several standard techniques: random cropping, random horizontal flipping, and random augmentation (with number of transformations $n=2$ and magnitude of transformations $m=14$).\footnote{\url{https://github.com/huggingface/pytorch-image-models/blob/main/timm/data/auto_augment.py}}

\begin{algorithm}[ht]
\SetAlgoLined
  \PyComment{Class ViT\_dictionary definition} \\
\PyCode{CRATE:} \\
\Indp
    \PyComment{initialization} \\
    \PyCode{def init(self, image\_size, patch\_size, num\_classes, dim, depth, heads, mlp\_dim, pool = 'cls', channels = 3, dim\_head = 64, dropout = 0., emb\_dropout = 0.):} \\
    \Indp
        \PyComment{define patch, image dimensions and number of patches}\\
        \PyCode{image\_height, image\_width = pair(image\_size)}\\
        \PyCode{patch\_height, patch\_width = pair(patch\_size)}\\
       \PyCode{num\_patches = (image\_height // patch\_height) * (image\_width // patch\_width)}\\
        \PyCode{patch\_dim = channels * patch\_height * patch\_width}\\
        \PyCode{}\\
        \PyComment{define patch embedding, positional embedding, dropout, and transformer}\\
        \PyCode{self.to\_patch\_embedding = Sequential(Rearrange, LayerNorm(patch\_dim), Linear(patch\_dim, dim), LayerNorm(dim))}\\
        \PyCode{self.pos\_embedding = Parameter(random(1, num\_patches + 1, dim))}\\
        \PyCode{self.cls\_token = Parameter(random(1, 1, dim))}\\
        \PyCode{self.dropout = Dropout(emb\_dropout)}\\
        \PyCode{self.transformer = Transformer(dim, depth, heads, dim\_head, mlp\_dim, dropout)}\\
        \PyCode{}\\
        \PyComment{define pooling, latent layer, and MLP head}\\
        \PyCode{self.pool = pool}\\
        \PyCode{self.to\_latent = Identity()}\\
        \PyCode{self.mlp\_head = Sequential(LayerNorm(dim), Linear(dim, num\_classes))}\\
    \Indm
    \PyCode{}\\
    \PyComment{forward pass}\\
    \PyCode{def forward(self, img):}\\
    \Indp
        \PyCode{x = self.to\_patch\_embedding(img)}\\
        \PyCode{b, n, \_ = shape(x)}\\
        \PyCode{cls\_tokens = repeat(self.cls\_token, '1 1 d -> b 1 d', b = b)}\\
        \PyCode{x = concatenate((cls\_tokens, x), dim=1)}\\
        \PyCode{x += self.pos\_embedding[:, :(n + 1)]}\\
        \PyCode{x = self.dropout(x)}\\
        \PyCode{x = self.transformer(x)}\\
        \PyCode{x = mean(x, dim = 1) if self.pool == 'mean' else x[:, 0]}\\
        \PyCode{x = self.to\_latent(x)}\\
        \PyCode{return self.mlp\_head(x)}\\
    \Indm
\Indm

\caption{PyTorch-style pseudocode for \ours Network}
\label{algo:pseudocode_crate}
\end{algorithm}

\begin{algorithm}[ht]
\SetAlgoLined

    \PyComment{Class Transformer definition} \\
    \PyCode{class Transformer:} \\
    \Indp
        \PyComment{initialization} \\
        \PyCode{def init(self, dim, depth, heads, dim\_head, mlp\_dim, dropout = 0.):} \\
        \Indp
            \PyComment{define layers}\\
            \PyCode{self.layers = []}\\
            \PyCode{self.depth = depth}\\
            \PyCode{for \_ in range(depth):} \\
            \Indp
                \PyCode{self.layers.append([LayerNorm(dim, Attention(dim, heads, dim\_head, dropout))])}\\
                \PyCode{self.layers.append([LayerNorm(dim, FeedForward(dim, mlp\_dim, dropout))])}\\
            \Indm
            
        \PyCode{}\\
        \Indm
        \PyComment{forward pass}\\
        \PyCode{def forward(self, x):}\\
        \Indp
        \PyCode{for attn, ff in self.layers:} \\
            \Indp
                \PyCode{x\_ = attn(x) + x}\\
                \PyCode{x = ff(x\_)}\\
            \Indm
            \PyCode{return x}\\
        \Indm
        \PyCode{}\\
        \Indm
    \Indm
\caption{Pytorch Style Pseudocode for Transformer Block in \ours{}}
\label{algo:pseudo code}
\end{algorithm}

\begin{algorithm}[h]
\SetAlgoLined

    \PyComment{Class FeedForward definition} \\
    \PyCode{class FeedForward:} \\
    \Indp
        \PyComment{initialization} \\
        \PyCode{def init(self, dim, hidden\_dim, dropout = 0., step\_size=0.1, lambd=0.1):} \\
        \Indp
            \PyCode{self.weight = Parameter(Tensor(dim, dim))}\\
            \PyCode{init.kaiming\_uniform\_(self.weight)}\\
            \PyCode{self.step\_size = step\_size}\\
            \PyCode{self.lambd = lambd}\\
        \Indm
        \PyComment{forward pass}\\
        \PyCode{def forward(self, x):}\\
        \Indp
            \PyCode{x1 = linear(x, self.weight, bias=None)}\\
            \PyCode{grad\_1 = linear(x1, self.weight.t(), bias=None)}\\
            \PyCode{grad\_2 = linear(x, self.weight.t(), bias=None)}\\
            \PyCode{grad\_update = self.step\_size * (grad\_2 - grad\_1) - self.step\_size * self.lambd}\\
            \PyCode{output = relu(x + grad\_update)}\\
            \PyCode{return output}\\
        \Indm
    \Indm

    \PyComment{Class Attention definition} \\
    \PyCode{class Attention:} \\
    \Indp
        \PyComment{initialization} \\
        \PyCode{def init(self, dim, heads = 8, dim\_head = 64, dropout = 0.):} \\
        \Indp
            \PyCode{inner\_dim = dim\_head *  heads}\\
            \PyCode{project\_out = not (heads == 1 and dim\_head == dim)}\\
            \PyCode{self.heads = heads}\\
            \PyCode{self.scale = dim\_head ** -0.5}\\
            \PyCode{self.attend = Softmax(dim = -1)}\\
            \PyCode{self.dropout = Dropout(dropout)}\\
            \PyCode{self.qkv = Linear(dim, inner\_dim, bias=False)}\\
            \PyCode{self.to\_out = Sequential(Linear(inner\_dim, dim), Dropout(dropout)) if project\_out else nn.Identity()}
        \Indm
        \PyCode{}\\
        \PyComment{forward pass}\\
        \PyCode{def forward(self, x):}\\
        \Indp
            \PyCode{w = rearrange(self.qkv(x), 'b n (h d) -> b h n d', h = self.heads)}\\
            \PyCode{dots = matmul(w, w.transpose(-1, -2)) * self.scale}\\
            \PyCode{attn = self.attend(dots)}\\
            \PyCode{attn = self.dropout(attn)}\\
            \PyCode{out = matmul(attn, w)}\\
            \PyCode{out = rearrange(out, 'b h n d -> b n (h d)')}\\
            \PyCode{return self.to\_out(out)}\\
        \Indm
        \Indm
    \Indm
\caption{Pseudocode for Attention and FeedForward}
\label{algo:pseudocode_others}
\end{algorithm}

\clearpage
\subsection{Experimental Results}\label{subsec:appendix-exp-results}
In this subsection, we provide additional experimental results on \ours{}, including layer-wise measurements, visualizations, as well as ablation studies.

\subsubsection{Layer-wise Evaluation and Visualization}

\paragraph{Layer-wise evaluation of compression and sparsity.} 
Similar to \Cref{fig:exp-rc-sparisty-small}, we conduct the layer-wise evaluation of compression term and sparsity for \ours{-Tiny}, \ours{-Base}, and \ours{-Large}. 
We observe similar behavior as mentioned in \Cref{subsec:exp-in-depth-analysis}: both the compression term and the sparsity term improves as the layer index increases.

\begin{figure}[ht]
     \centering
     \begin{subfigure}[b]{0.47\textwidth}
         \centering
    \includegraphics[width=\textwidth]{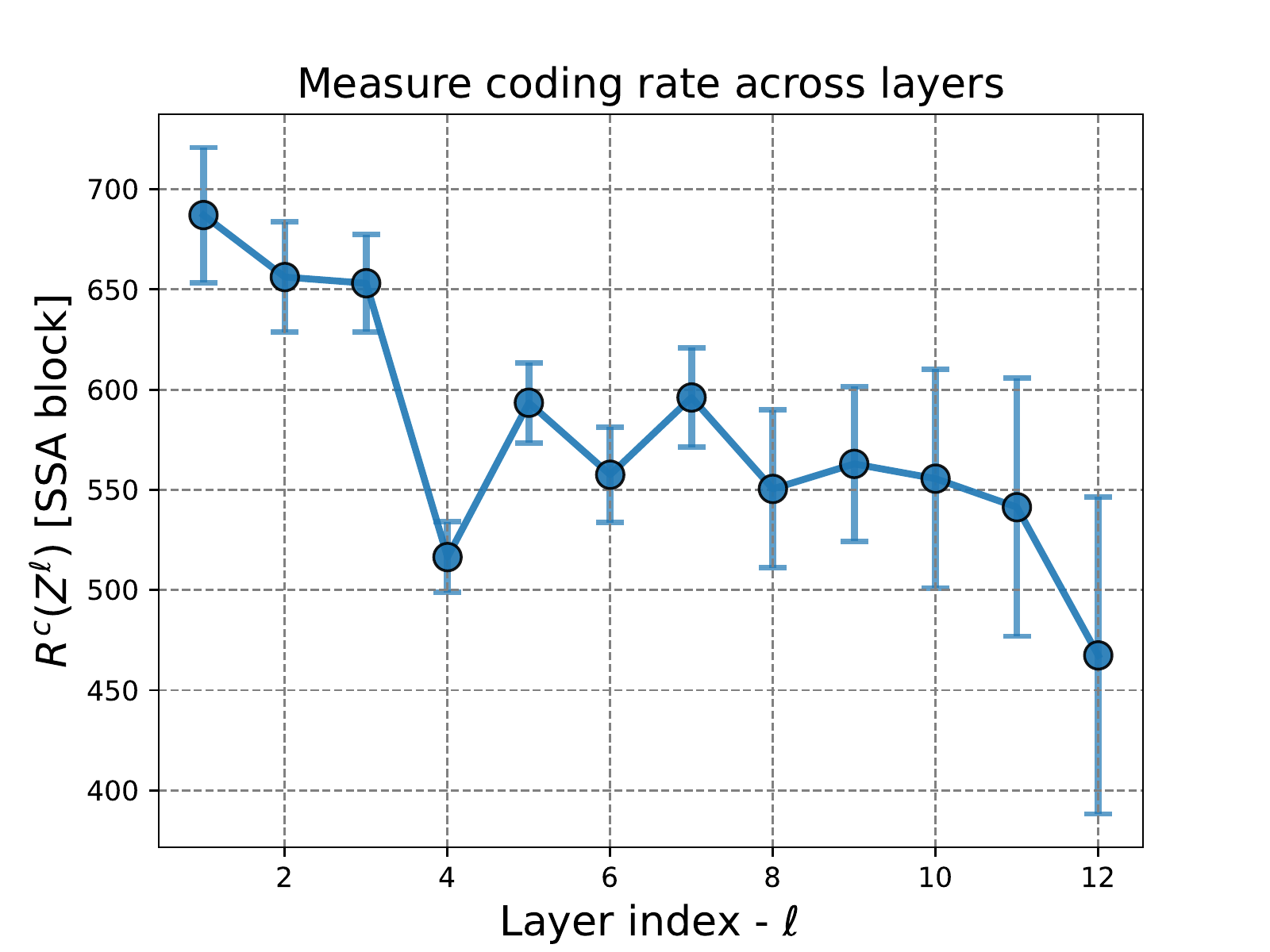}
         \caption{Compression (Model: \ours{-Tiny}).}
     \end{subfigure}
     \begin{subfigure}[b]{0.482\textwidth}
         \centering
    \includegraphics[width=\textwidth]{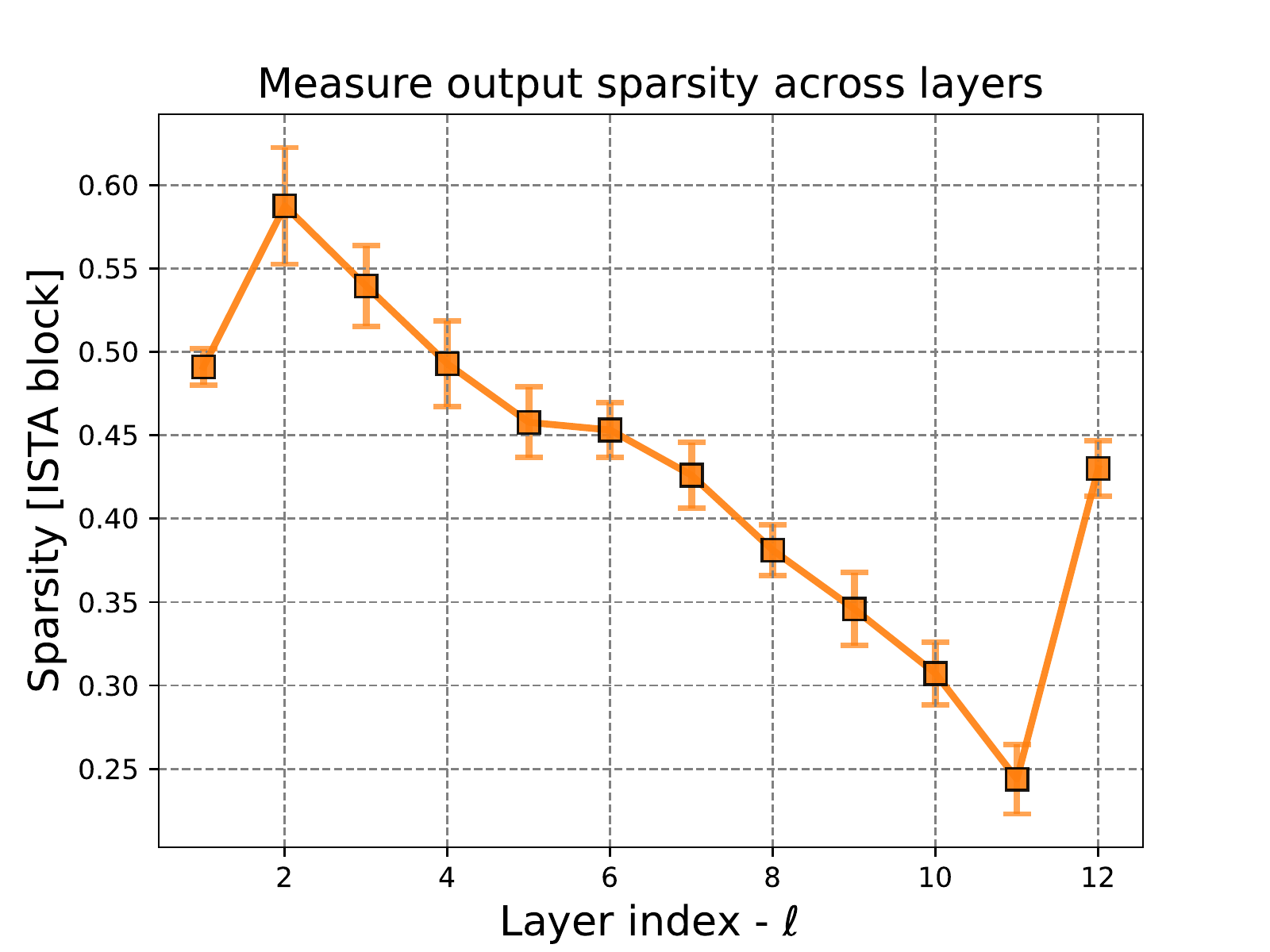}
         \caption{Sparsity (Model: \ours{-Tiny}).}
     \end{subfigure}
     \begin{subfigure}[b]{0.47\textwidth}
         \centering
    \includegraphics[width=\textwidth]{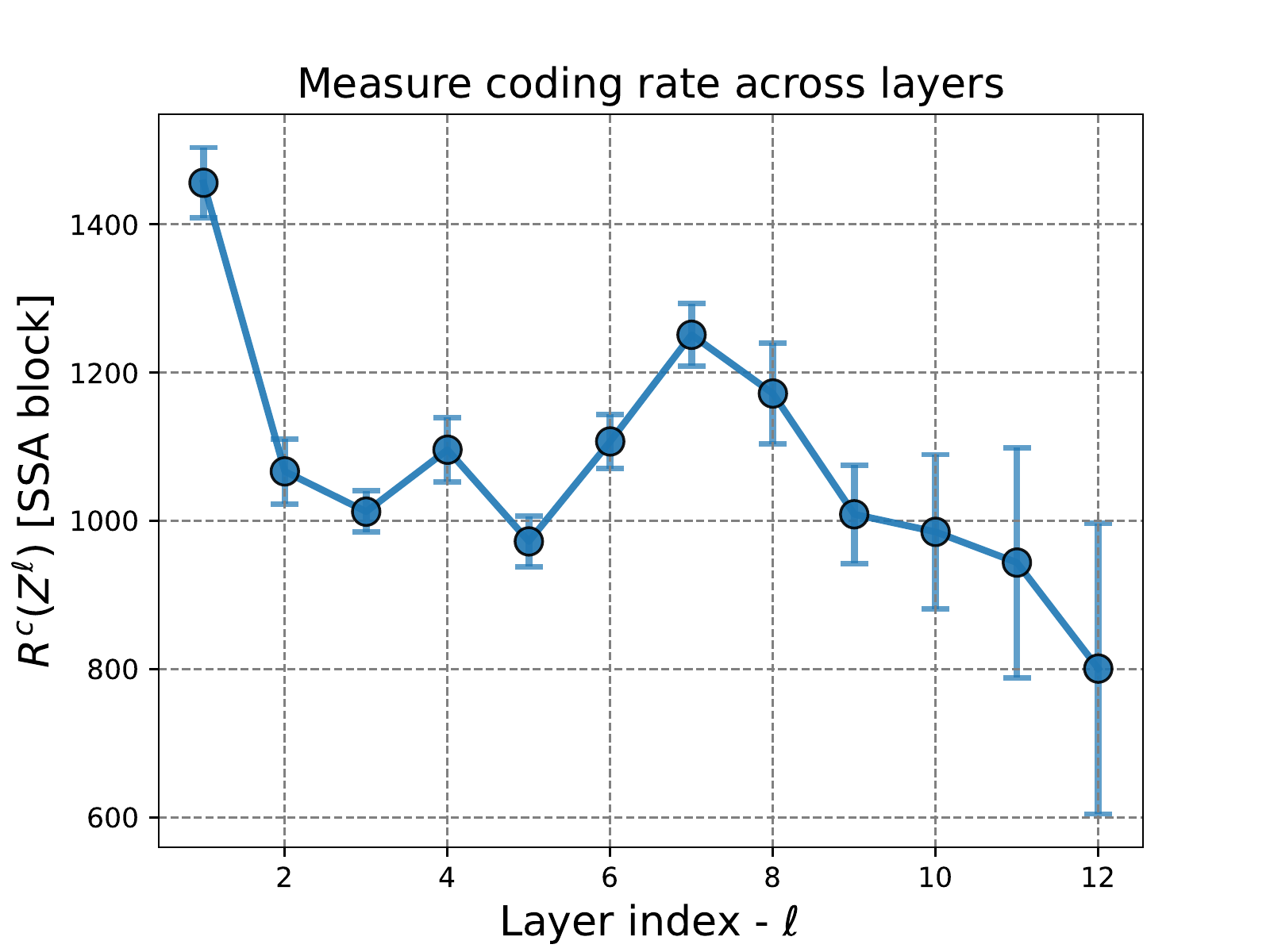}
         \caption{Compression (Model: \ours{-Base}).}
     \end{subfigure}
     \begin{subfigure}[b]{0.482\textwidth}
         \centering
    \includegraphics[width=\textwidth]{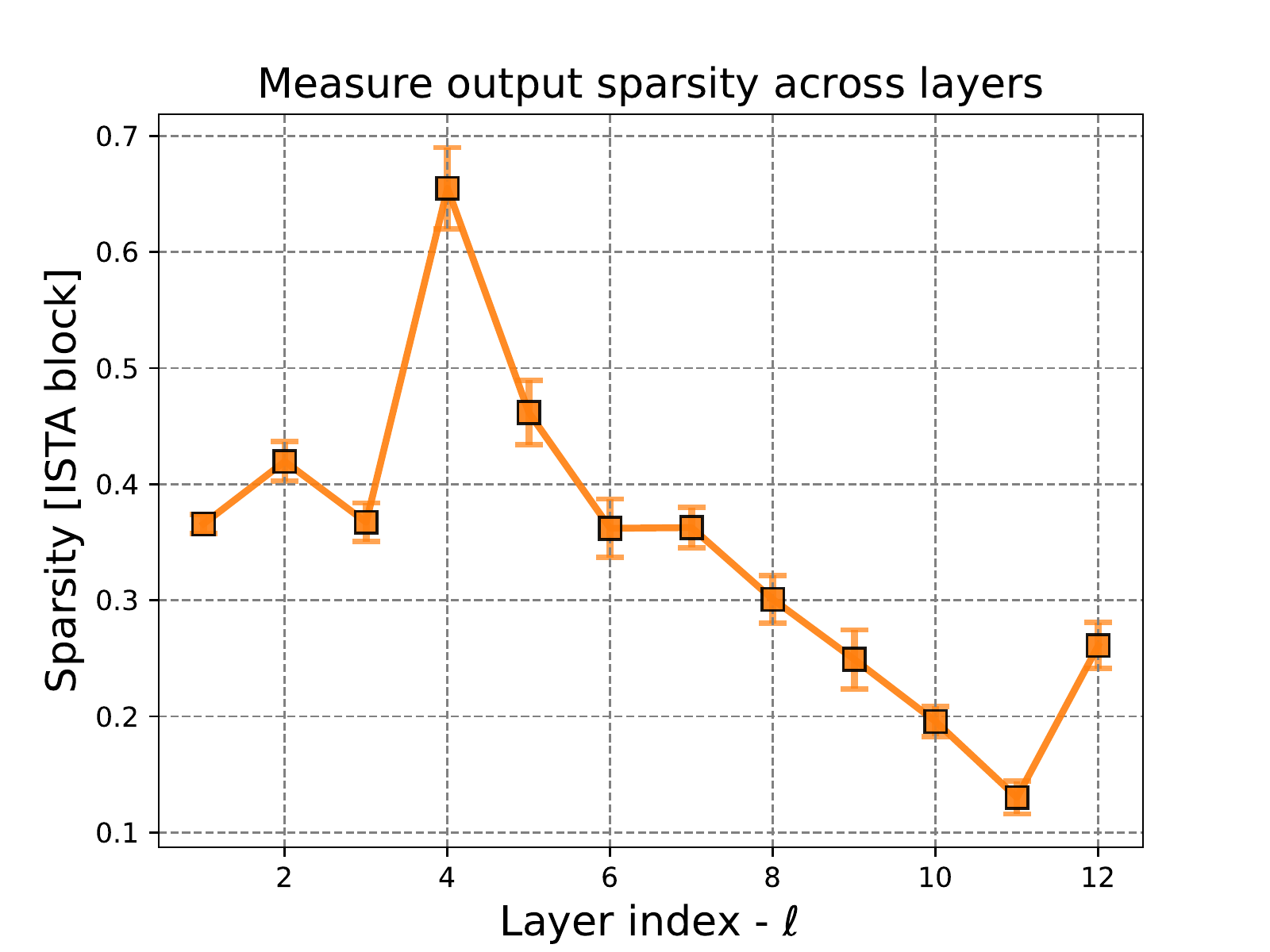}
         \caption{Sparsity (Model: \ours{-Base}).}
     \end{subfigure}
     \begin{subfigure}[b]{0.47\textwidth}
         \centering
    \includegraphics[width=\textwidth]{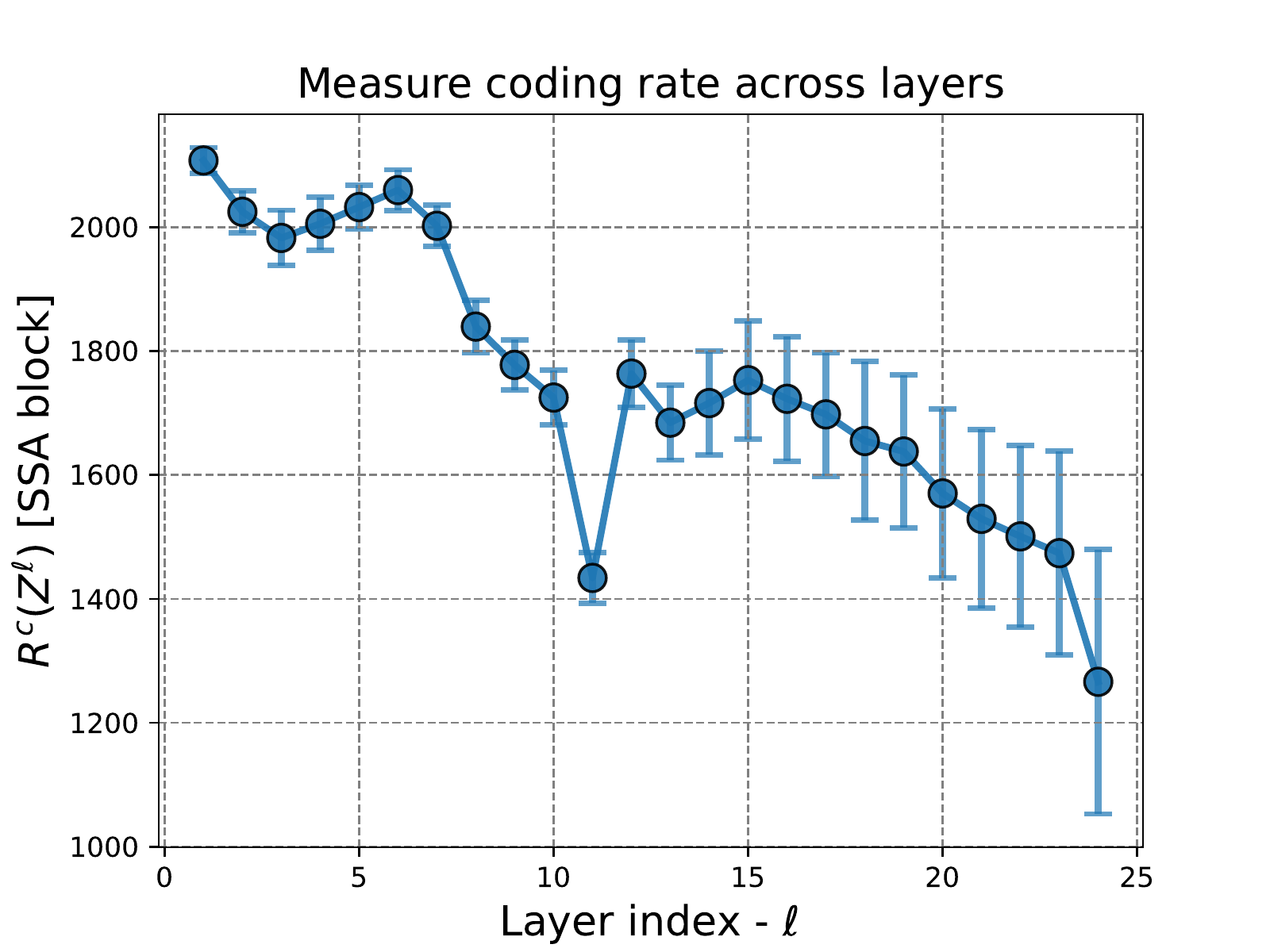}
         \caption{Compression (Model: \ours{-Large}).}
     \end{subfigure}
     \begin{subfigure}[b]{0.482\textwidth}
         \centering
    \includegraphics[width=\textwidth]{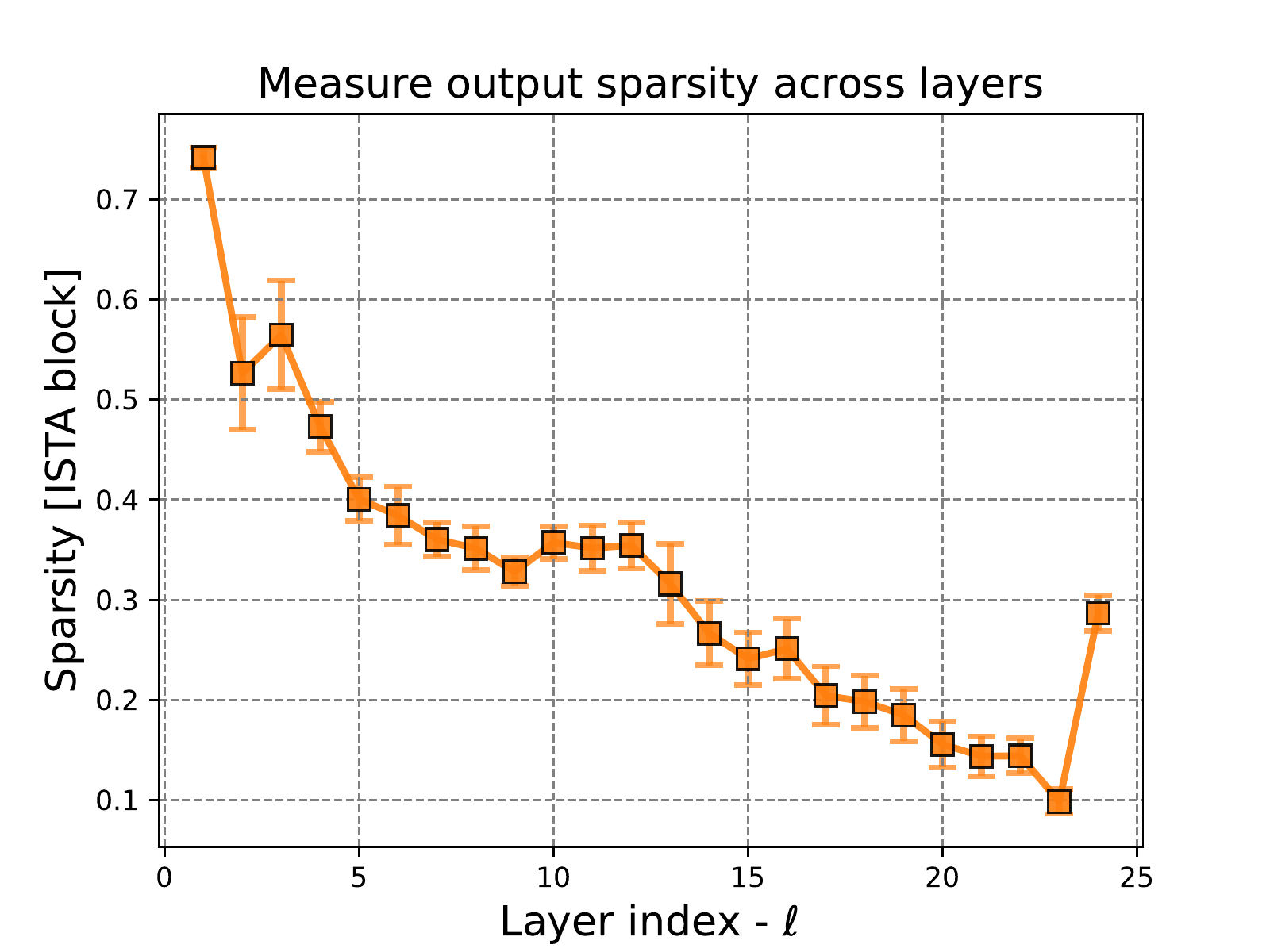}
         \caption{Sparsity (Model: \ours{-Large}).}
     \end{subfigure}
        \caption{\small \textit{Left}: The compression term $R^{c}(\Z^{\ell+1/2})$ of the \texttt{MSSA} outputs at different layers. \textit{Right}: the sparsity of the \texttt{ISTA} output block, $\|\Z^{\ell+1}\|_0 / (d\cdot N)$, at different layers.}
        \label{fig:appendix-exp-rc-sparisty-all-model-size}
        \vspace{-0.1in}
\end{figure}

\paragraph{Visualizing layer-wise token representations.}  
In \Cref{fig:appendix-exp-ista-sparsity-heatmap}, we visualize the token representations $\Z^{\ell}$ at different layers $\ell \in \{1, \dots, 12\}$. 
We provide more results evaluated on other samples in \Cref{subsec:appendix-token-representation-visualize}.

\paragraph{Visualizing layer-wise subspaces in multi-head self-attention.}
We provide the visualization of $\vU_{[K]}^{\ell}$ in \Cref{fig:appendix-exp-visualize-UiUj}.

\begin{figure}[ht]
     \centering
     \begin{subfigure}[b]{0.24\textwidth}
         \centering
    \includegraphics[width=\textwidth]{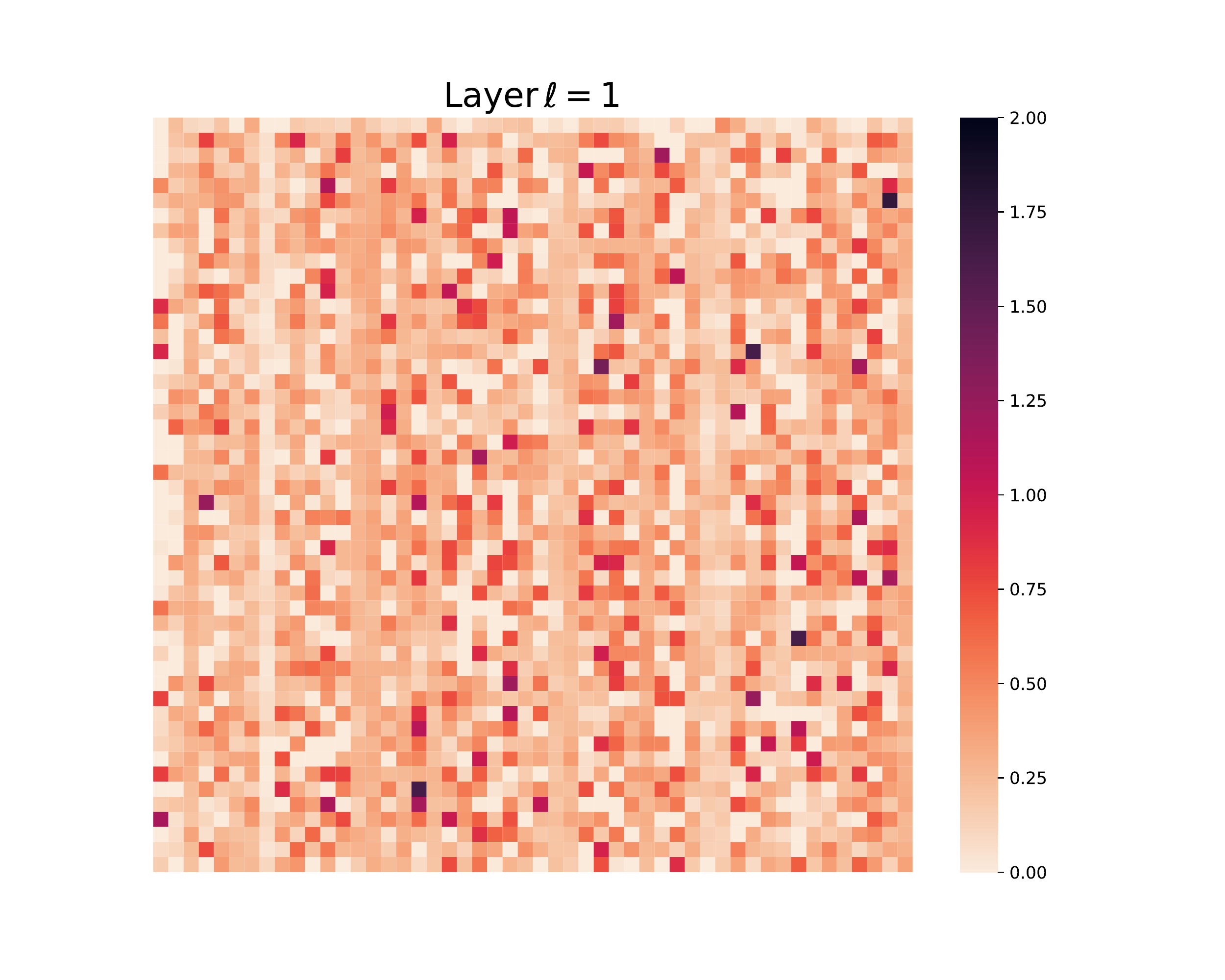}
         \caption{$\ell=1$.}
     \end{subfigure}
     \begin{subfigure}[b]{0.24\textwidth}
         \centering
    \includegraphics[width=\textwidth]{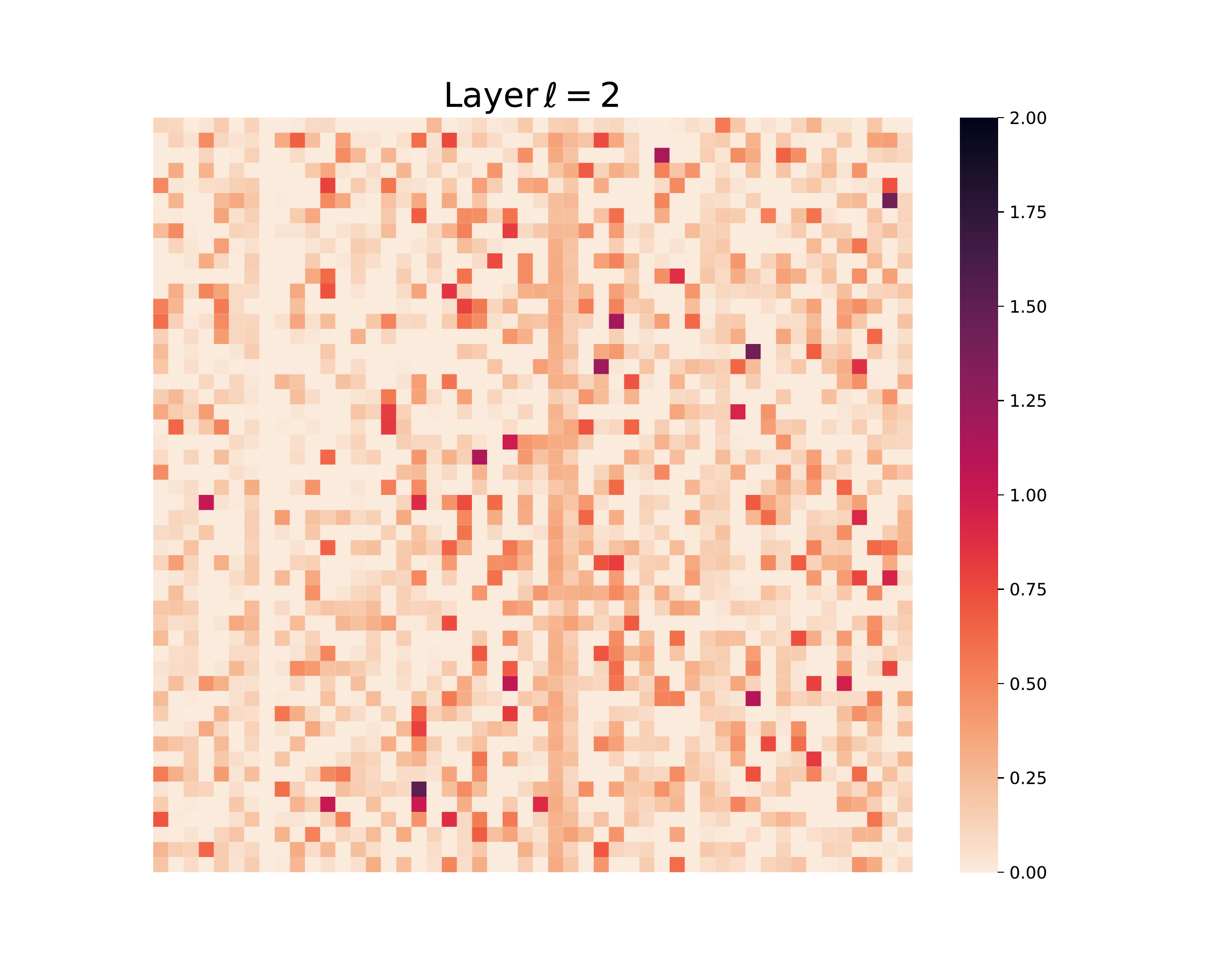}
         \caption{$\ell=2$.}
     \end{subfigure}
     \begin{subfigure}[b]{0.24\textwidth}
         \centering
    \includegraphics[width=\textwidth]{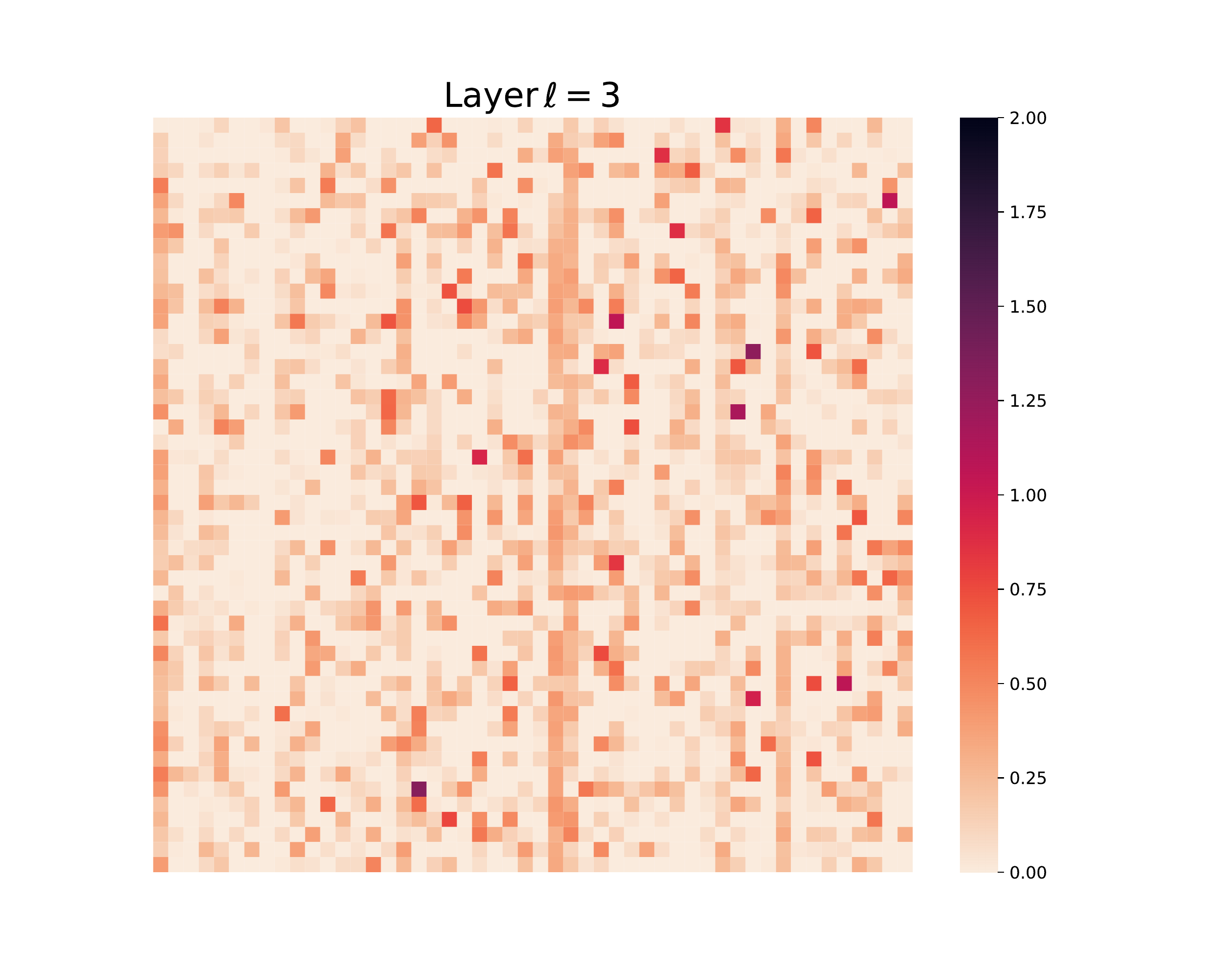}
         \caption{$\ell=3$.}
     \end{subfigure}
     \begin{subfigure}[b]{0.24\textwidth}
         \centering
    \includegraphics[width=\textwidth]{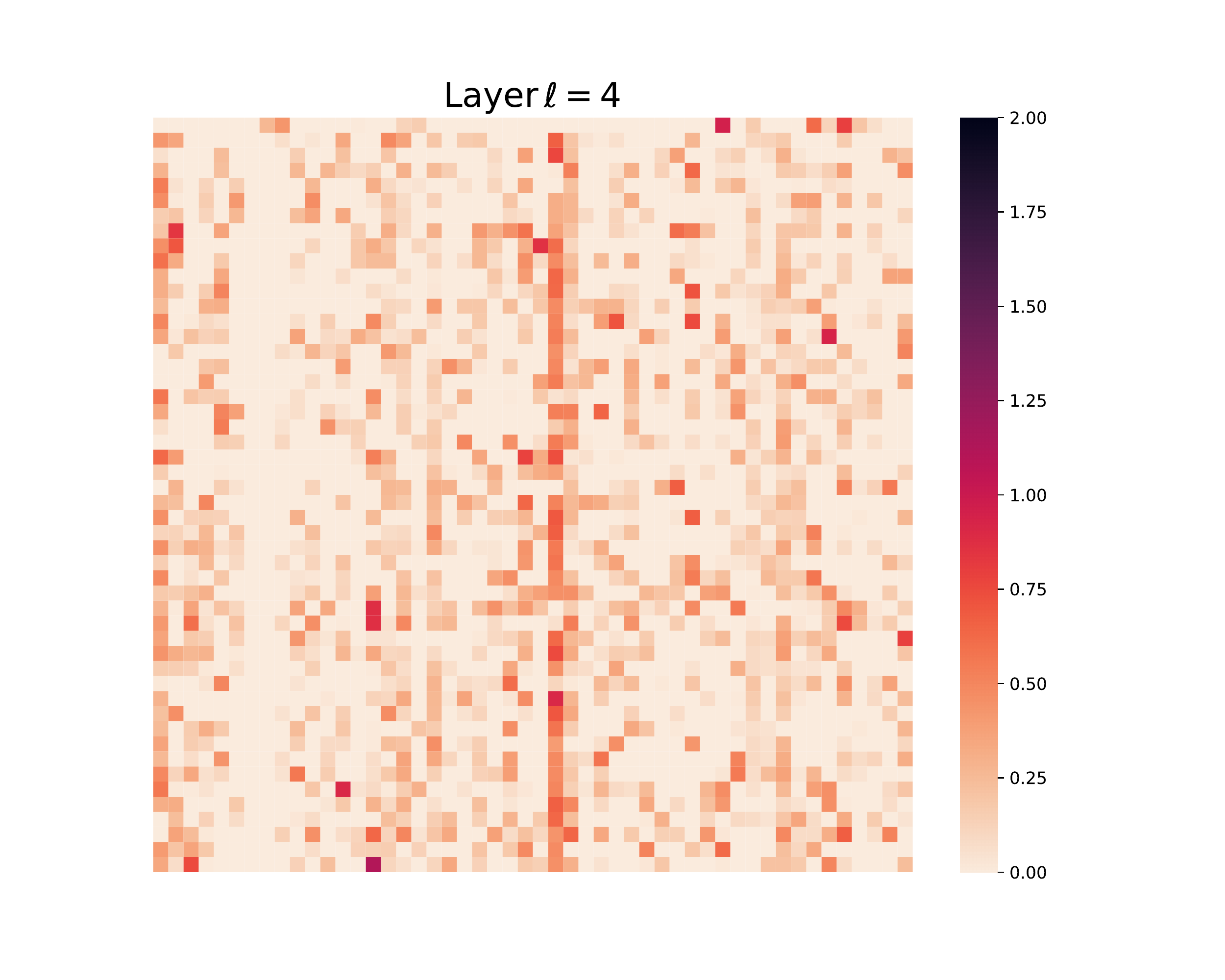}
         \caption{$\ell=4$.}
     \end{subfigure}
     \begin{subfigure}[b]{0.24\textwidth}
         \centering
    \includegraphics[width=\textwidth]{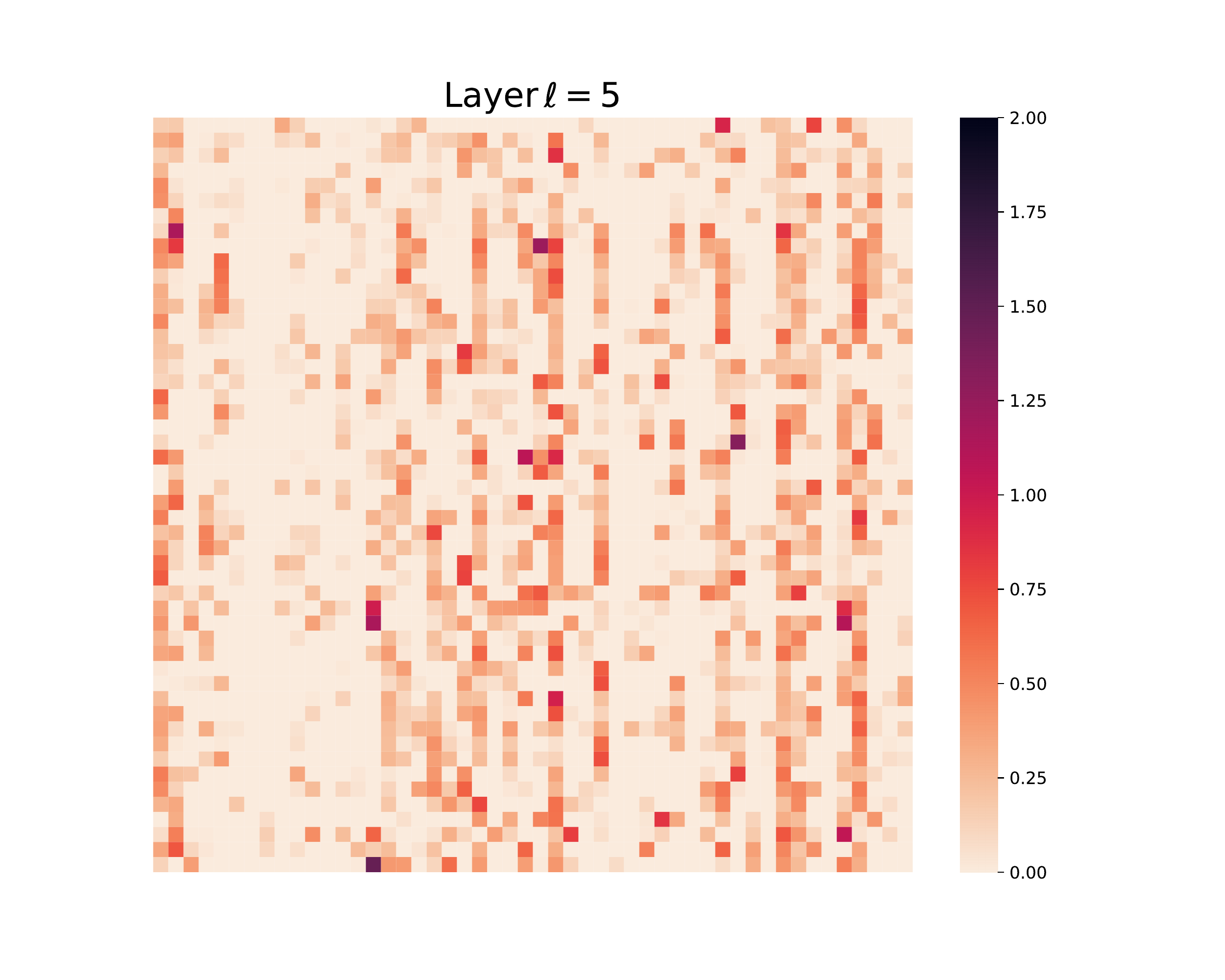}
         \caption{$\ell=5$.}
     \end{subfigure}
     \begin{subfigure}[b]{0.24\textwidth}
         \centering
    \includegraphics[width=\textwidth]{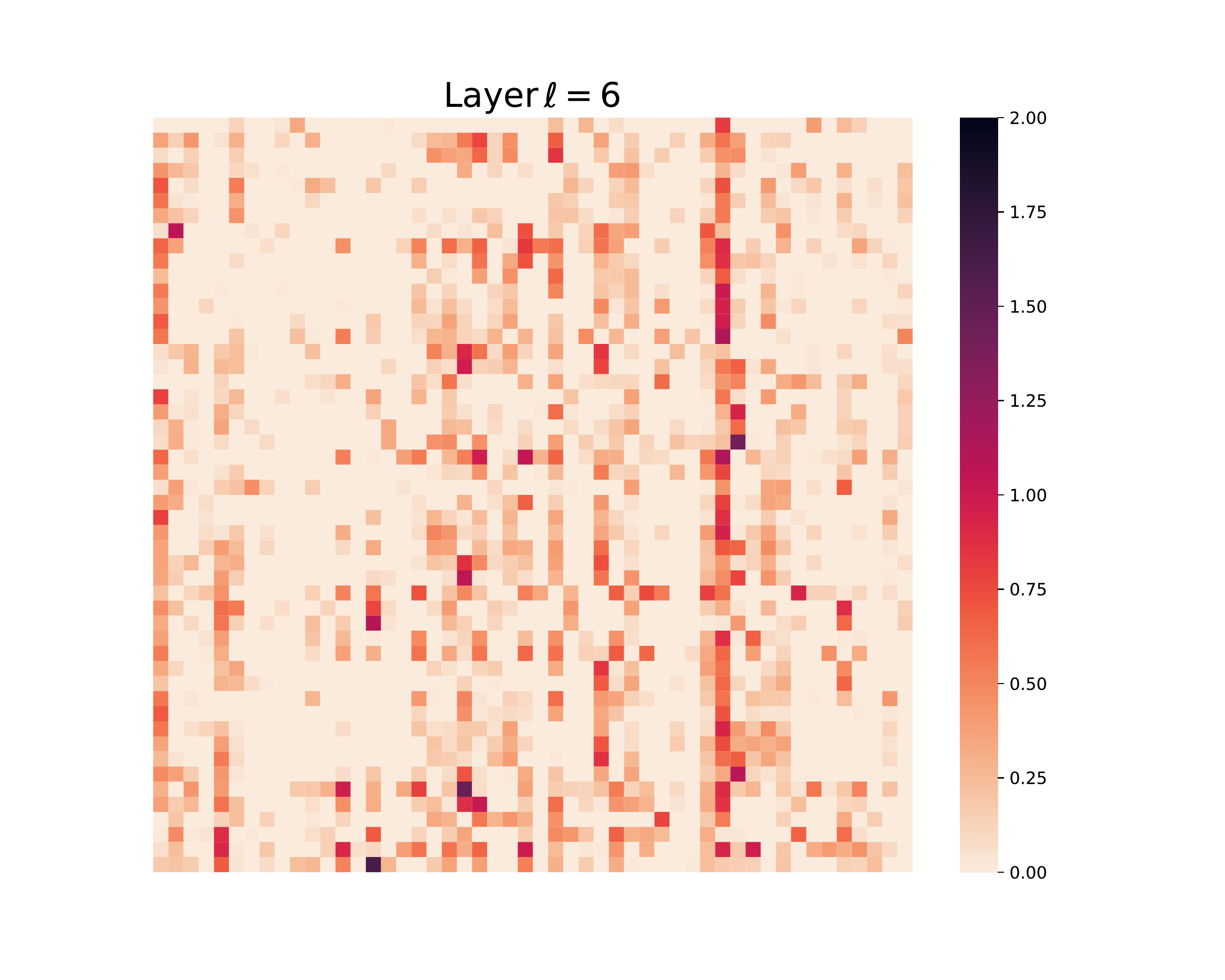}
         \caption{$\ell=6$.}
     \end{subfigure}
     \begin{subfigure}[b]{0.24\textwidth}
         \centering
    \includegraphics[width=\textwidth]{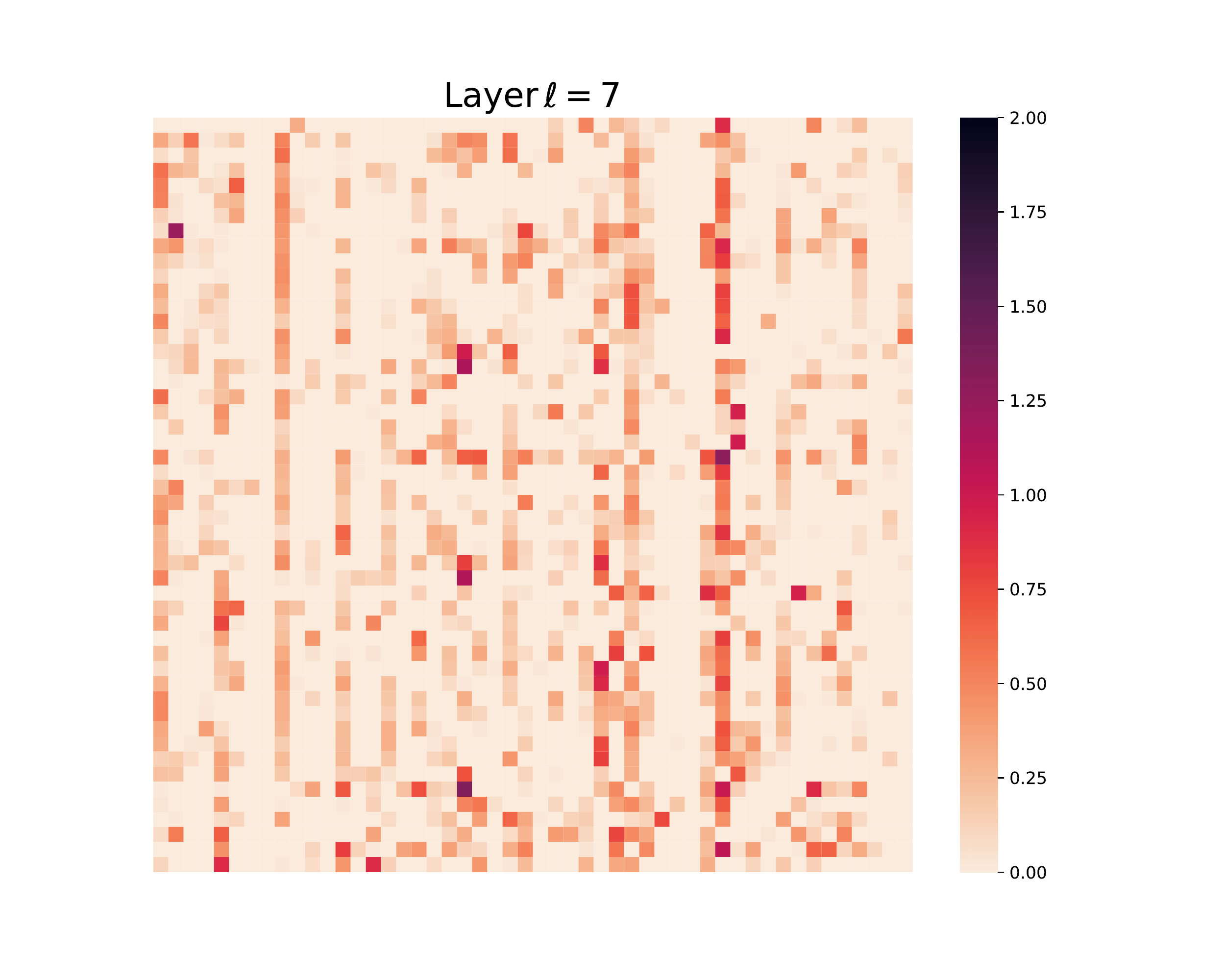}
         \caption{$\ell=7$.}
     \end{subfigure}
     \begin{subfigure}[b]{0.24\textwidth}
         \centering
    \includegraphics[width=\textwidth]{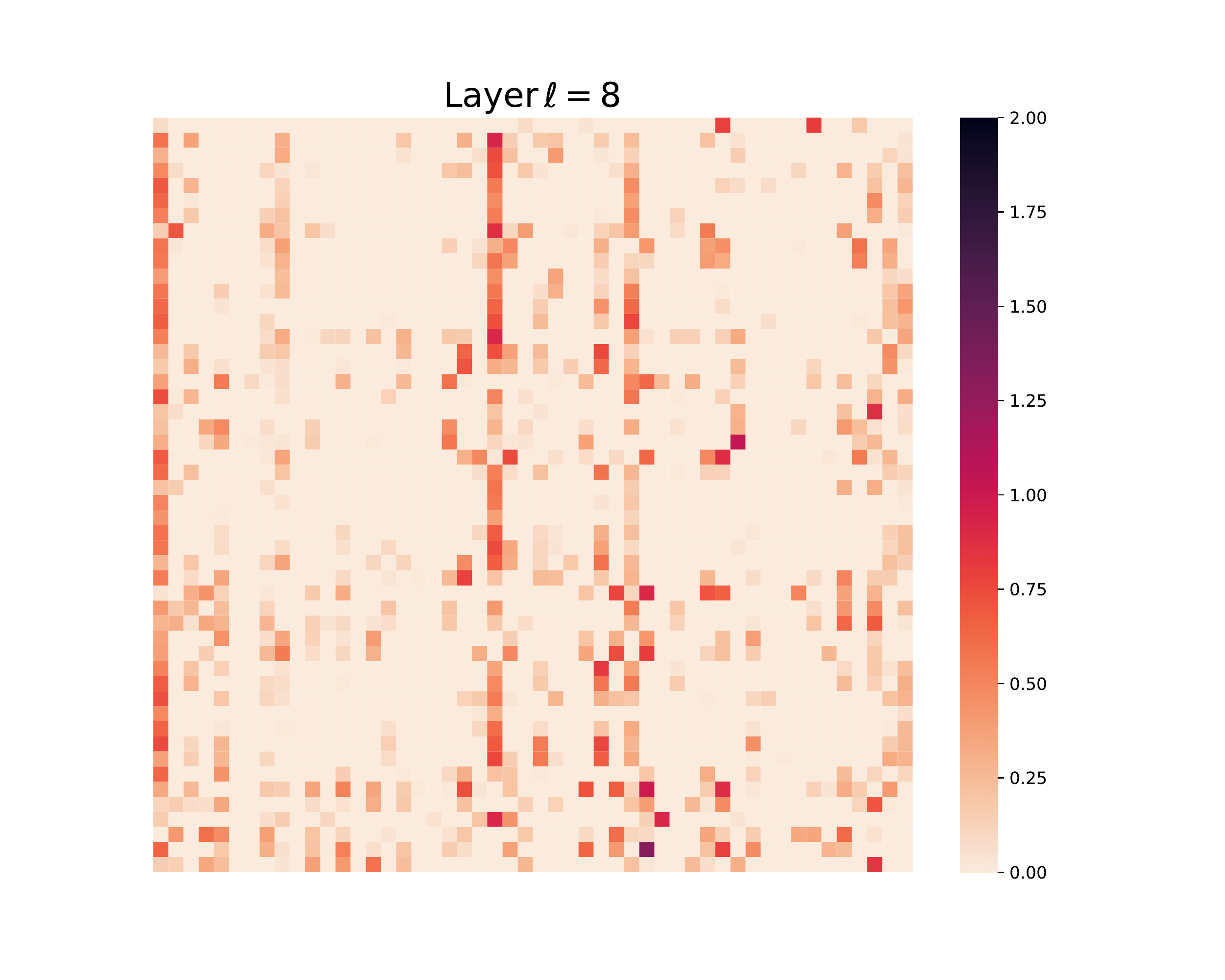}
         \caption{$\ell=8$.}
     \end{subfigure}
     \begin{subfigure}[b]{0.24\textwidth}
         \centering
    \includegraphics[width=\textwidth]{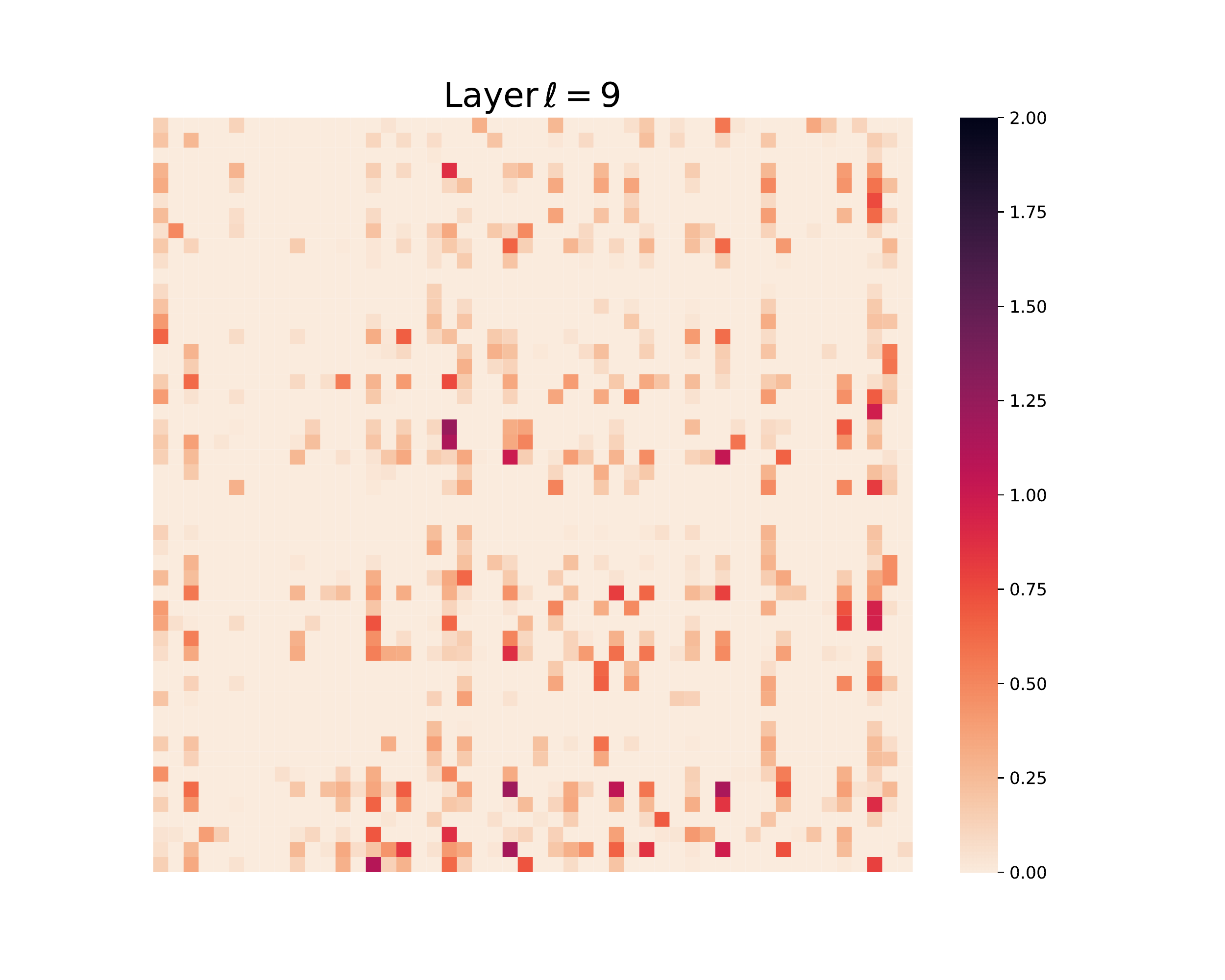}
         \caption{$\ell=9$.}
     \end{subfigure}
     \begin{subfigure}[b]{0.24\textwidth}
         \centering
    \includegraphics[width=\textwidth]{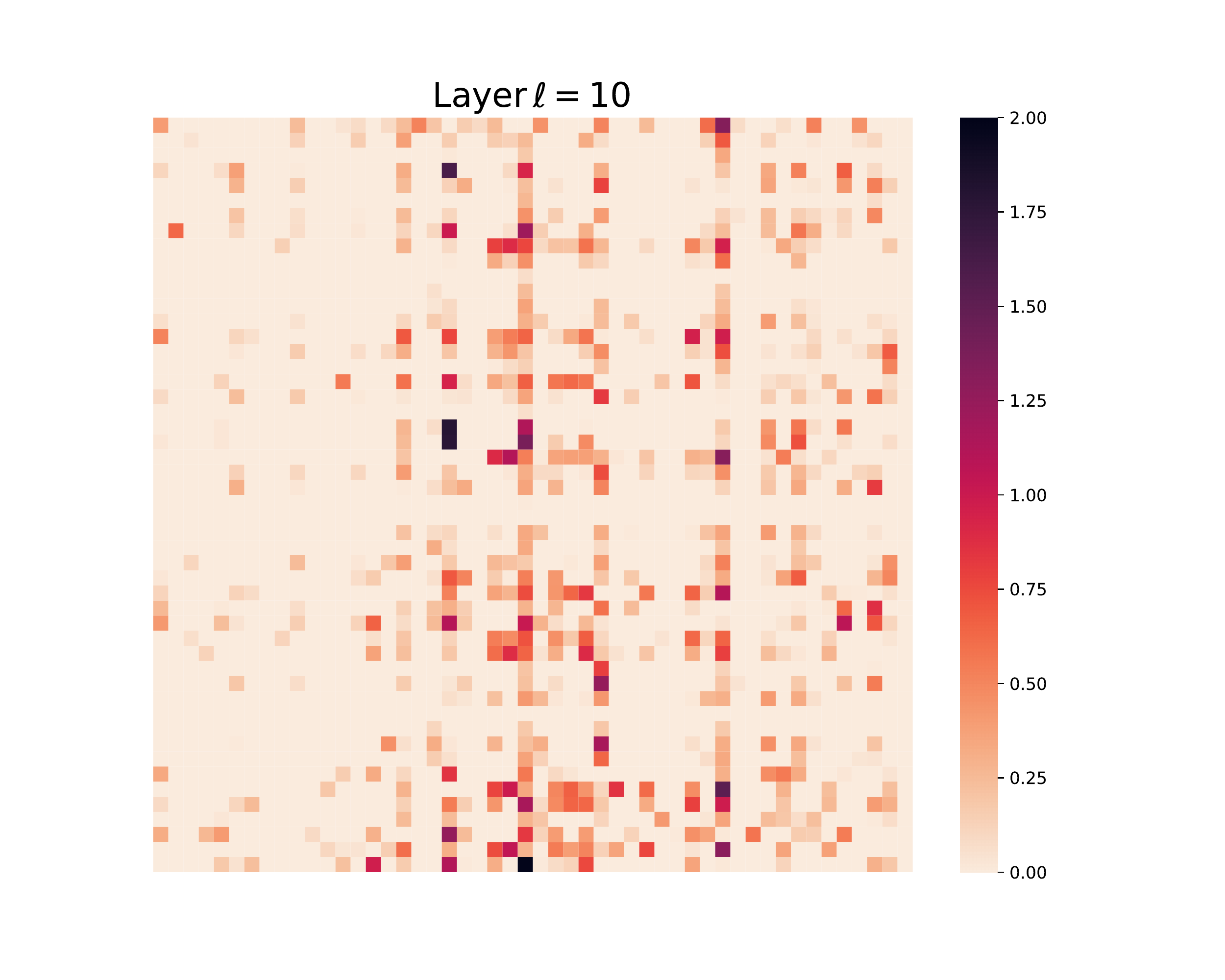}
         \caption{$\ell=10$.}
     \end{subfigure}
     \begin{subfigure}[b]{0.24\textwidth}
         \centering
    \includegraphics[width=\textwidth]{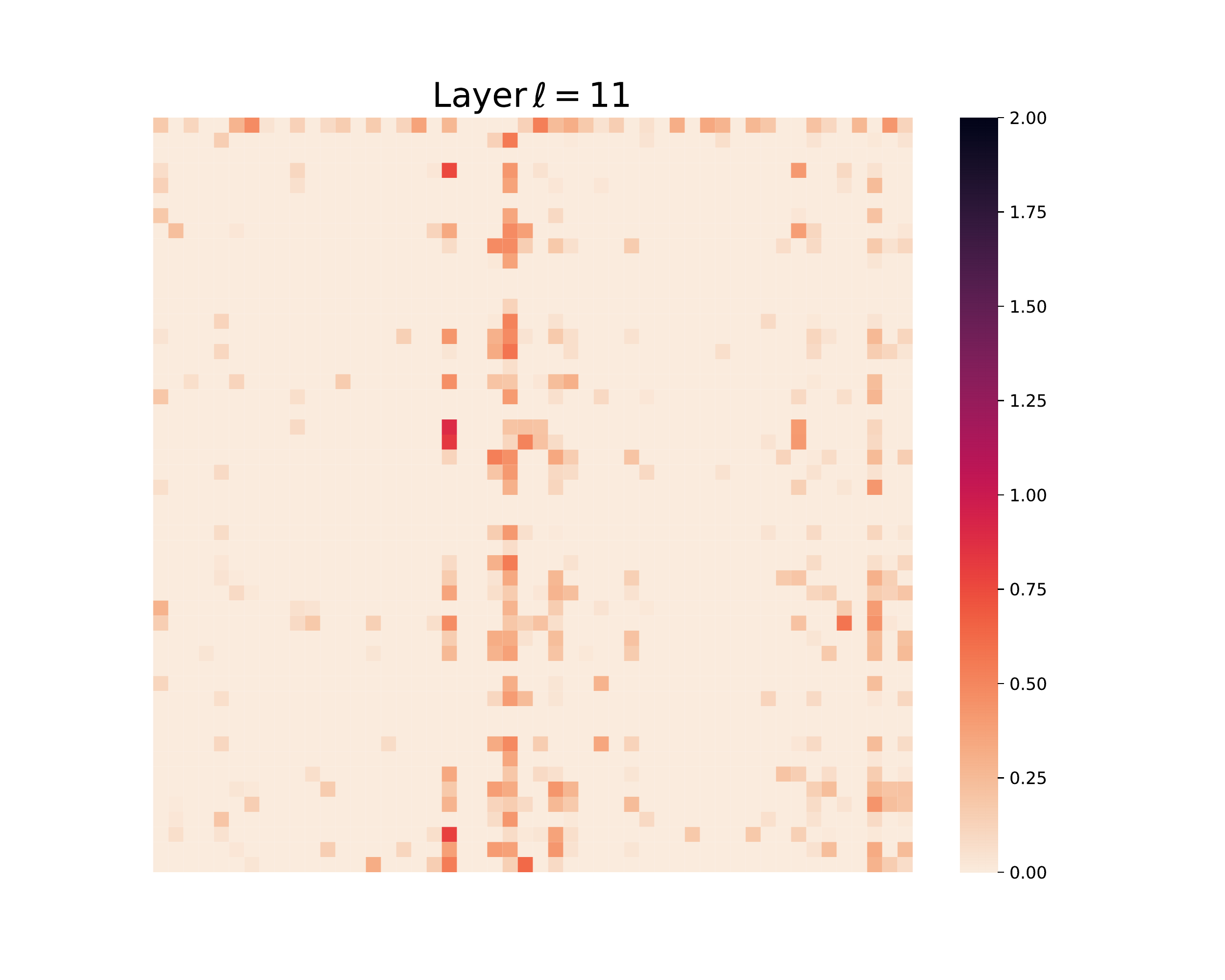}
         \caption{$\ell=11$.}
     \end{subfigure}
     \begin{subfigure}[b]{0.24\textwidth}
         \centering
    \includegraphics[width=\textwidth]{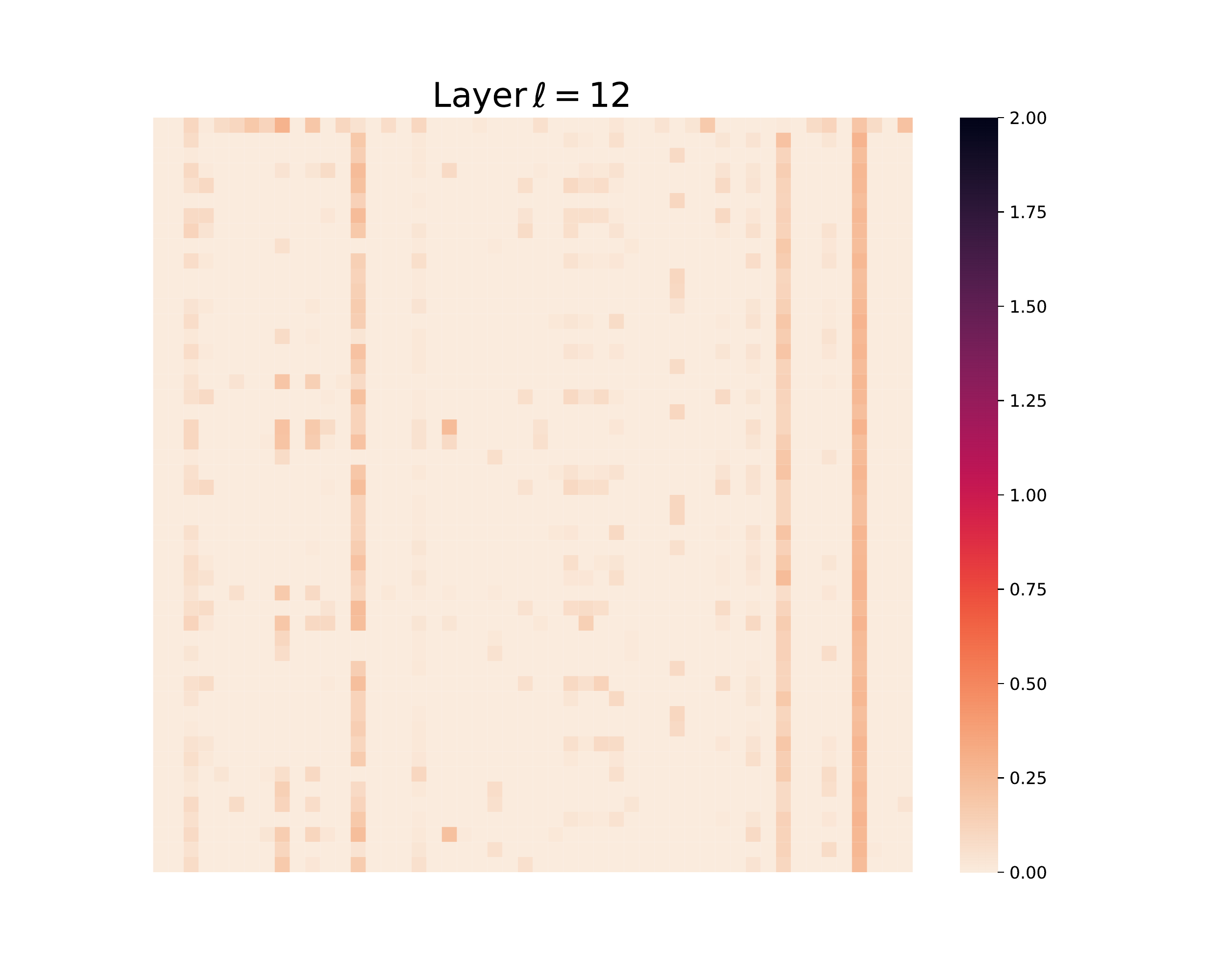}
         \caption{$\ell=12$.}
     \end{subfigure}
        \caption{Visualizing layer-wise token $\Z^{\ell}$ representations at each layer $\ell$. To enhance the visual clarity, we randomly extract a 50$\times$50 sub-matrix from $\Z^{\ell}$ for display purposes. (Model: \ours{-Tiny})}
        \label{fig:appendix-exp-ista-sparsity-heatmap}
        \vspace{-0.1in}
\end{figure}

\begin{figure}[t!]
     \centering
     \begin{subfigure}[b]{0.24\textwidth}
         \centering
    \includegraphics[width=\textwidth]{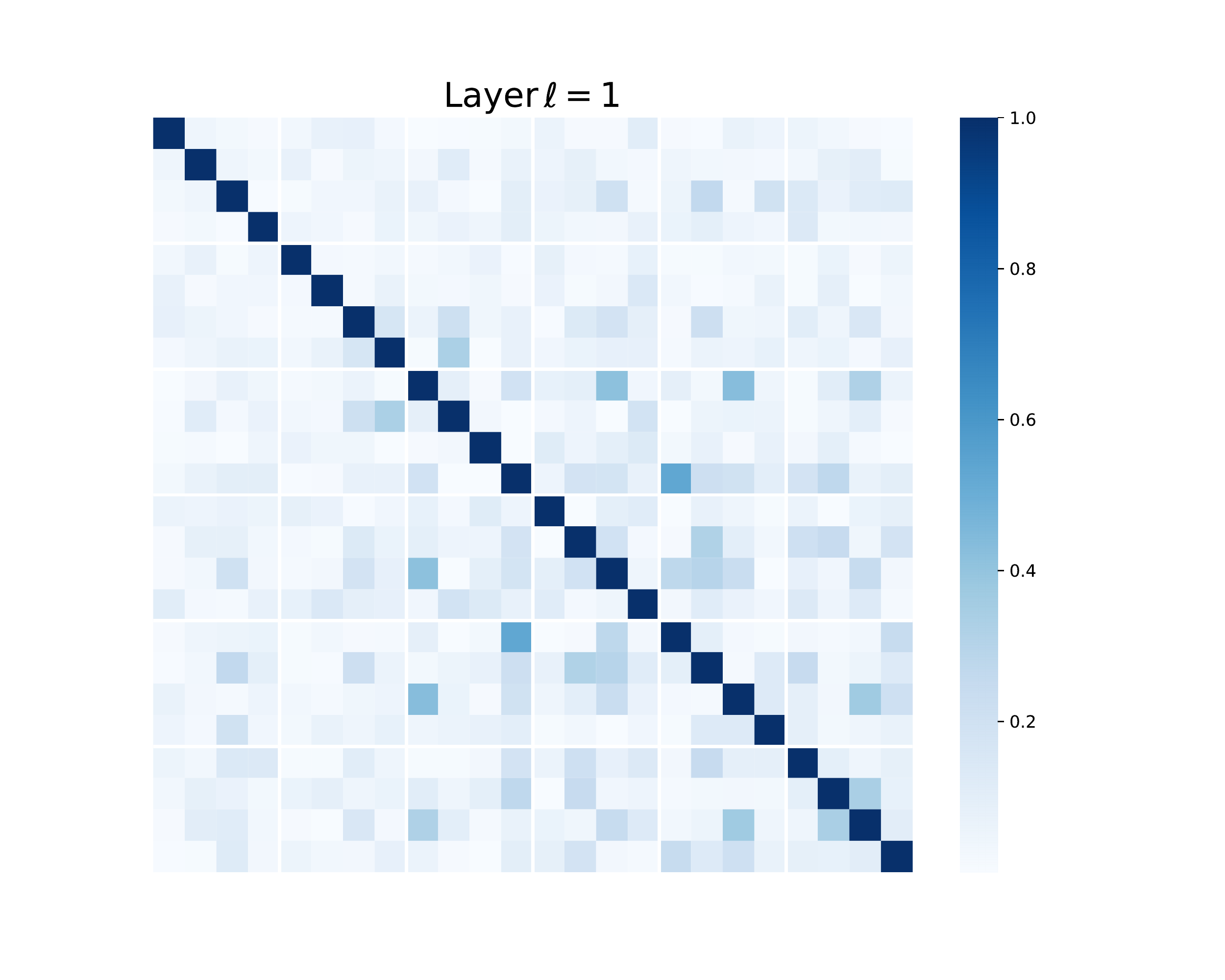}
         \caption{$\ell=1$.}
     \end{subfigure}
     \begin{subfigure}[b]{0.24\textwidth}
         \centering
    \includegraphics[width=\textwidth]{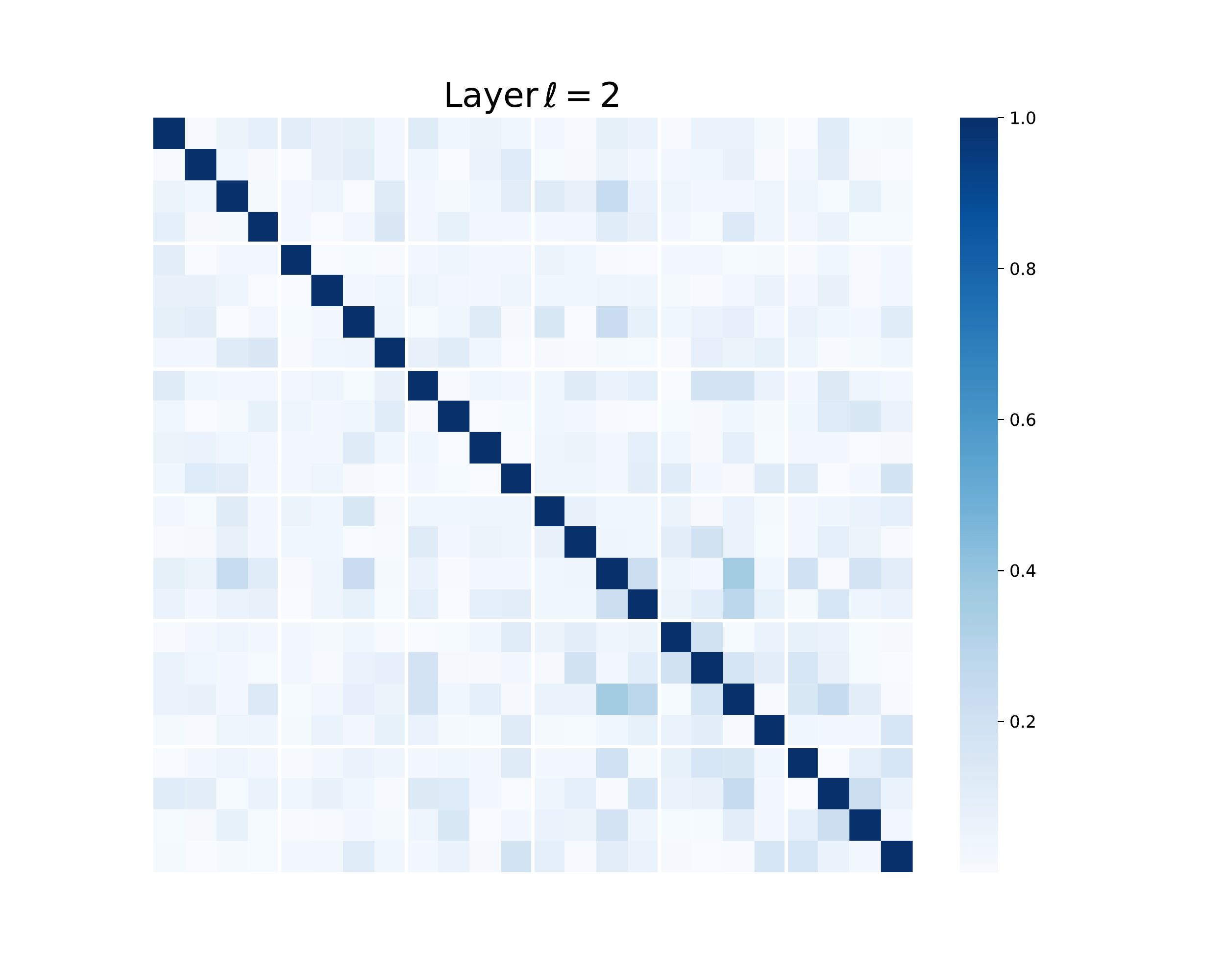}
         \caption{$\ell=2$.}
     \end{subfigure}
     \begin{subfigure}[b]{0.24\textwidth}
         \centering
    \includegraphics[width=\textwidth]{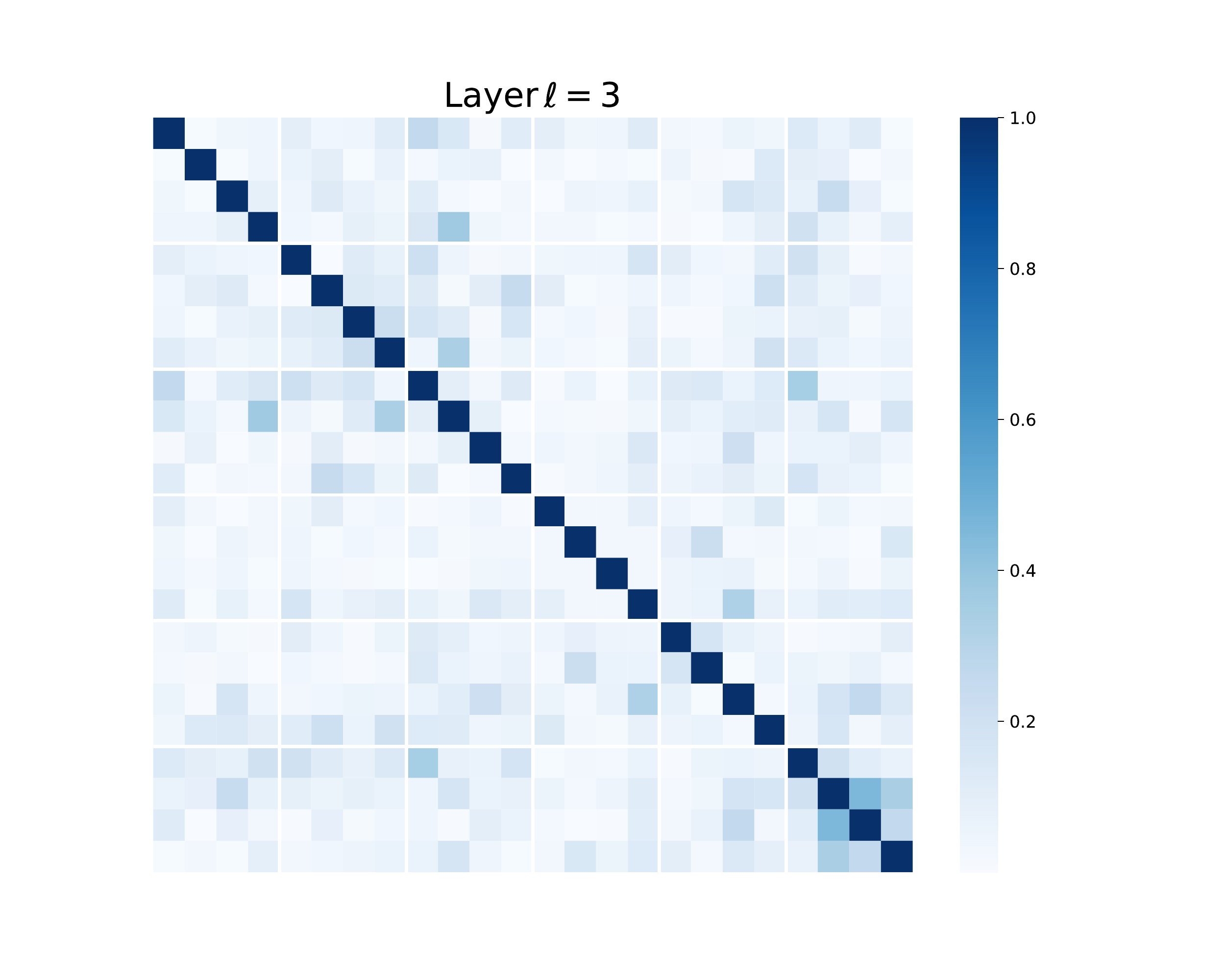}
         \caption{$\ell=3$.}
     \end{subfigure}
     \begin{subfigure}[b]{0.24\textwidth}
         \centering
    \includegraphics[width=\textwidth]{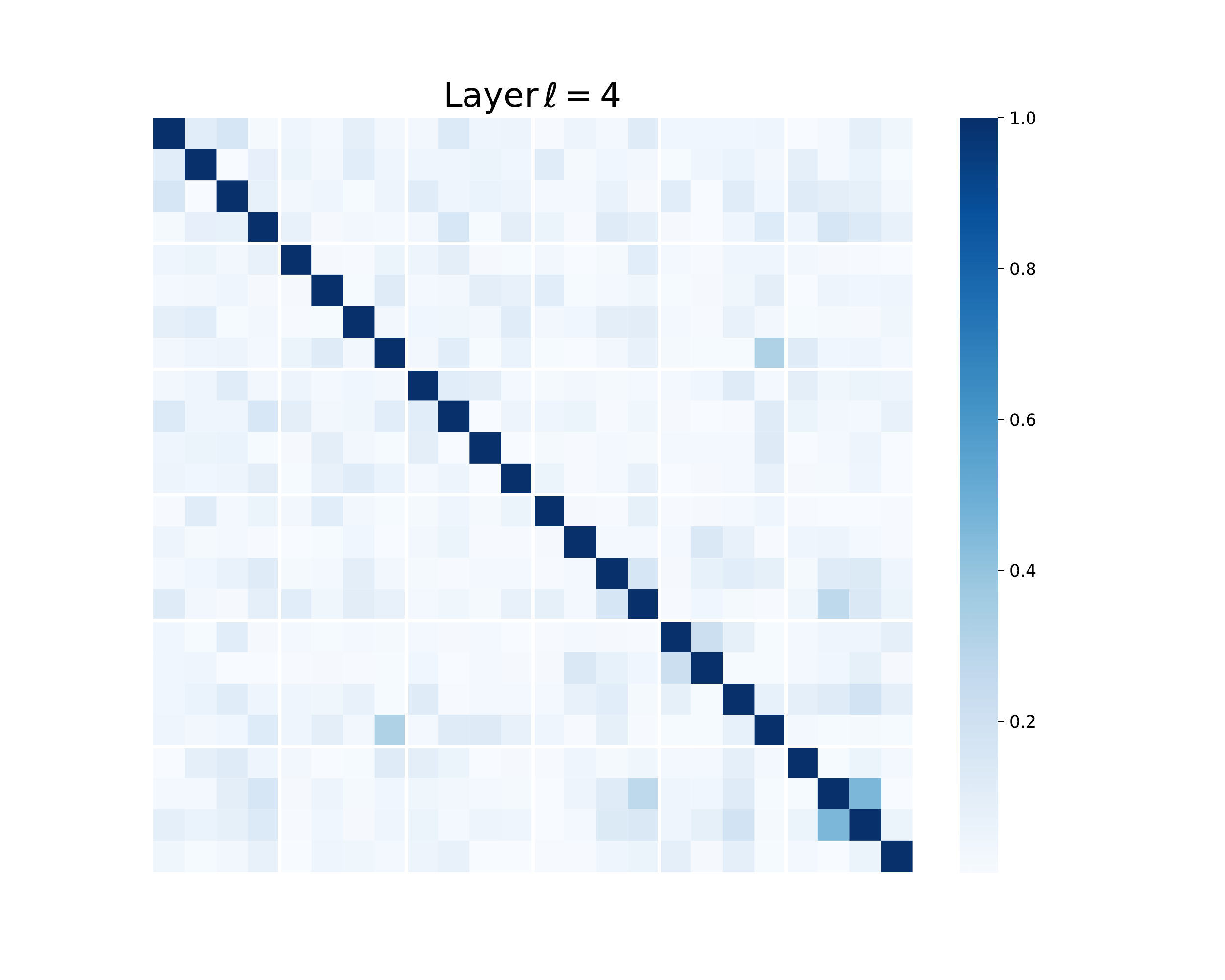}
         \caption{$\ell=4$.}
     \end{subfigure}
     \begin{subfigure}[b]{0.24\textwidth}
         \centering
    \includegraphics[width=\textwidth]{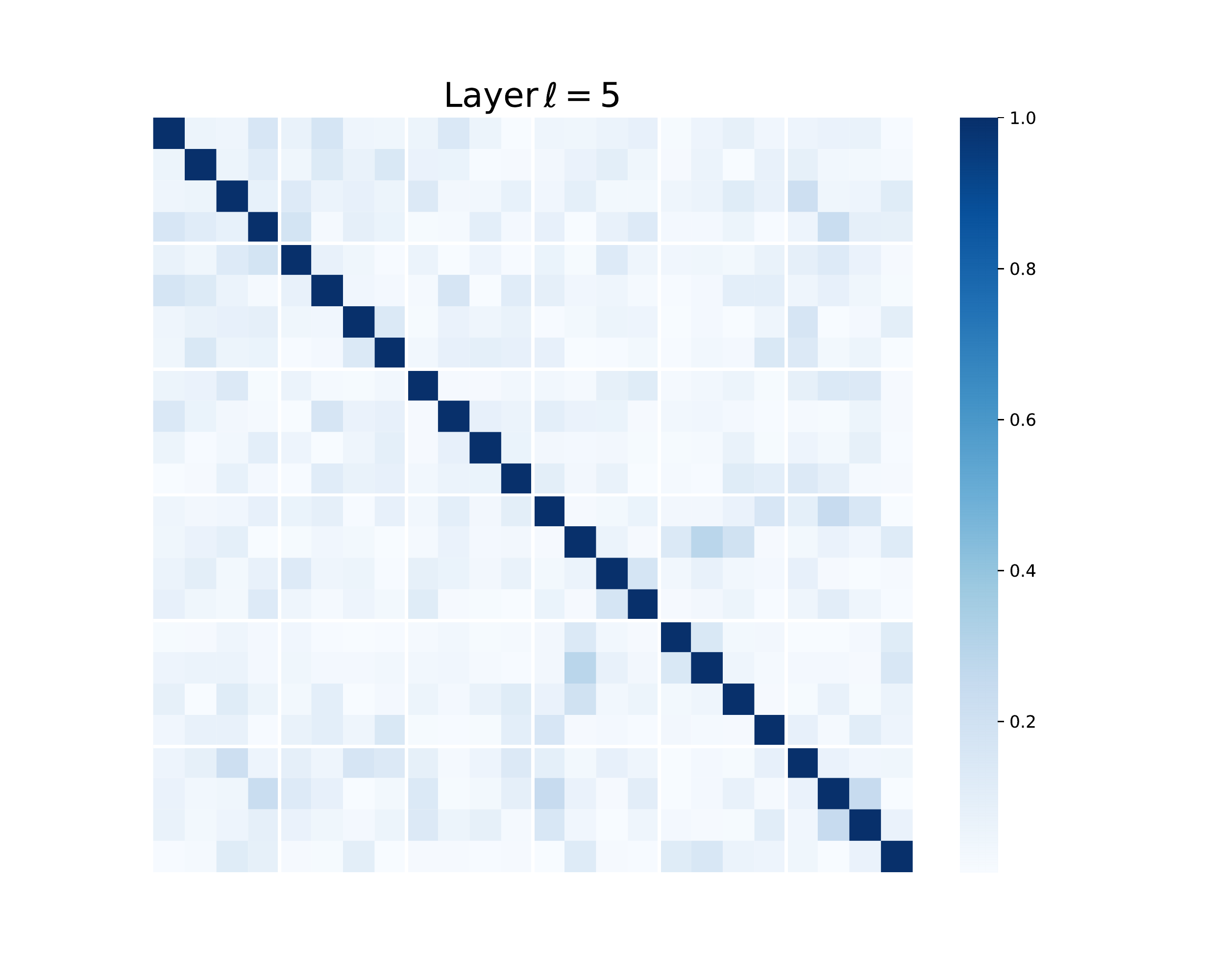}
         \caption{$\ell=5$.}
     \end{subfigure}
     \begin{subfigure}[b]{0.24\textwidth}
         \centering
    \includegraphics[width=\textwidth]{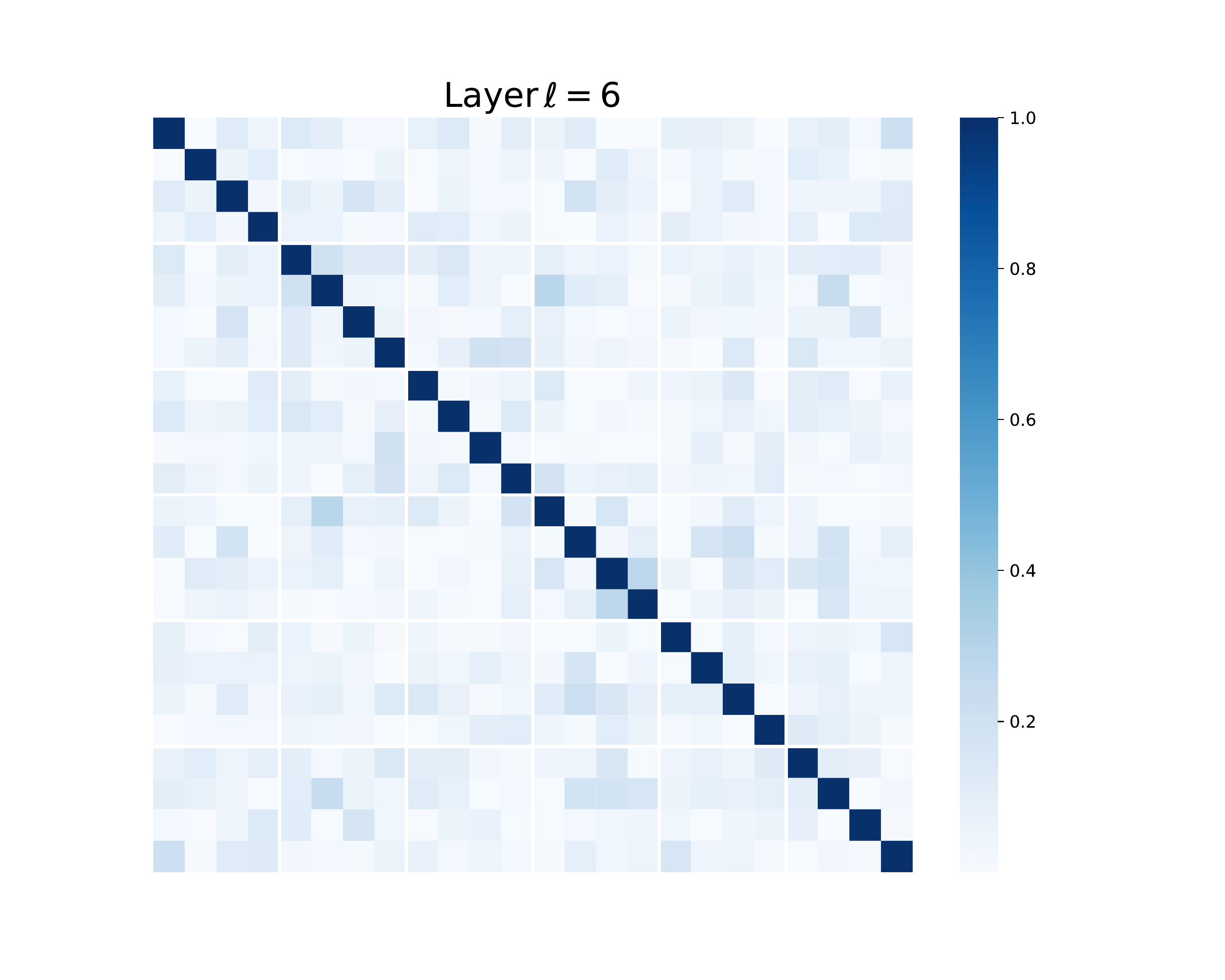}
         \caption{$\ell=6$.}
     \end{subfigure}
     \begin{subfigure}[b]{0.24\textwidth}
         \centering
    \includegraphics[width=\textwidth]{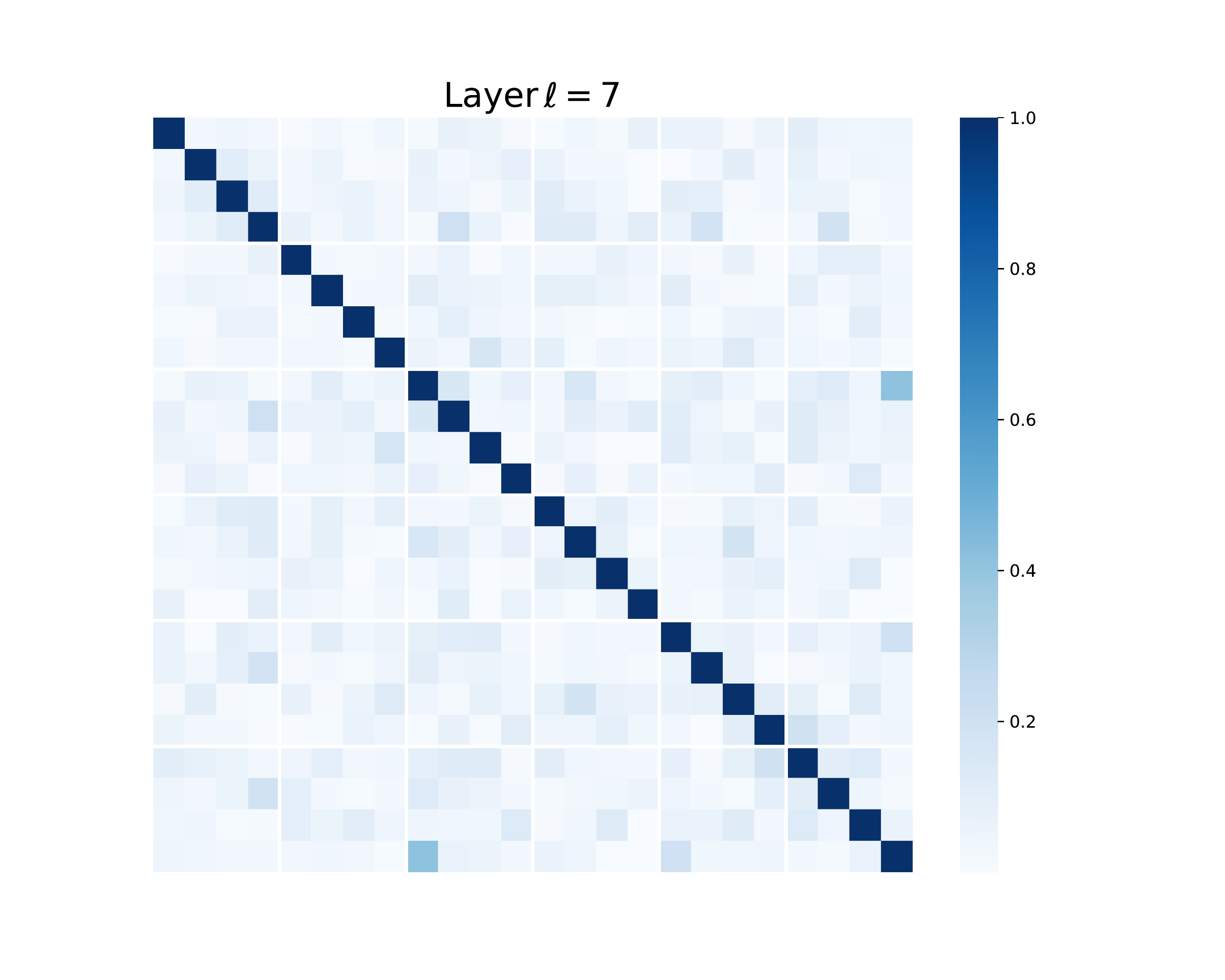}
         \caption{$\ell=7$.}
     \end{subfigure}
     \begin{subfigure}[b]{0.24\textwidth}
         \centering
    \includegraphics[width=\textwidth]{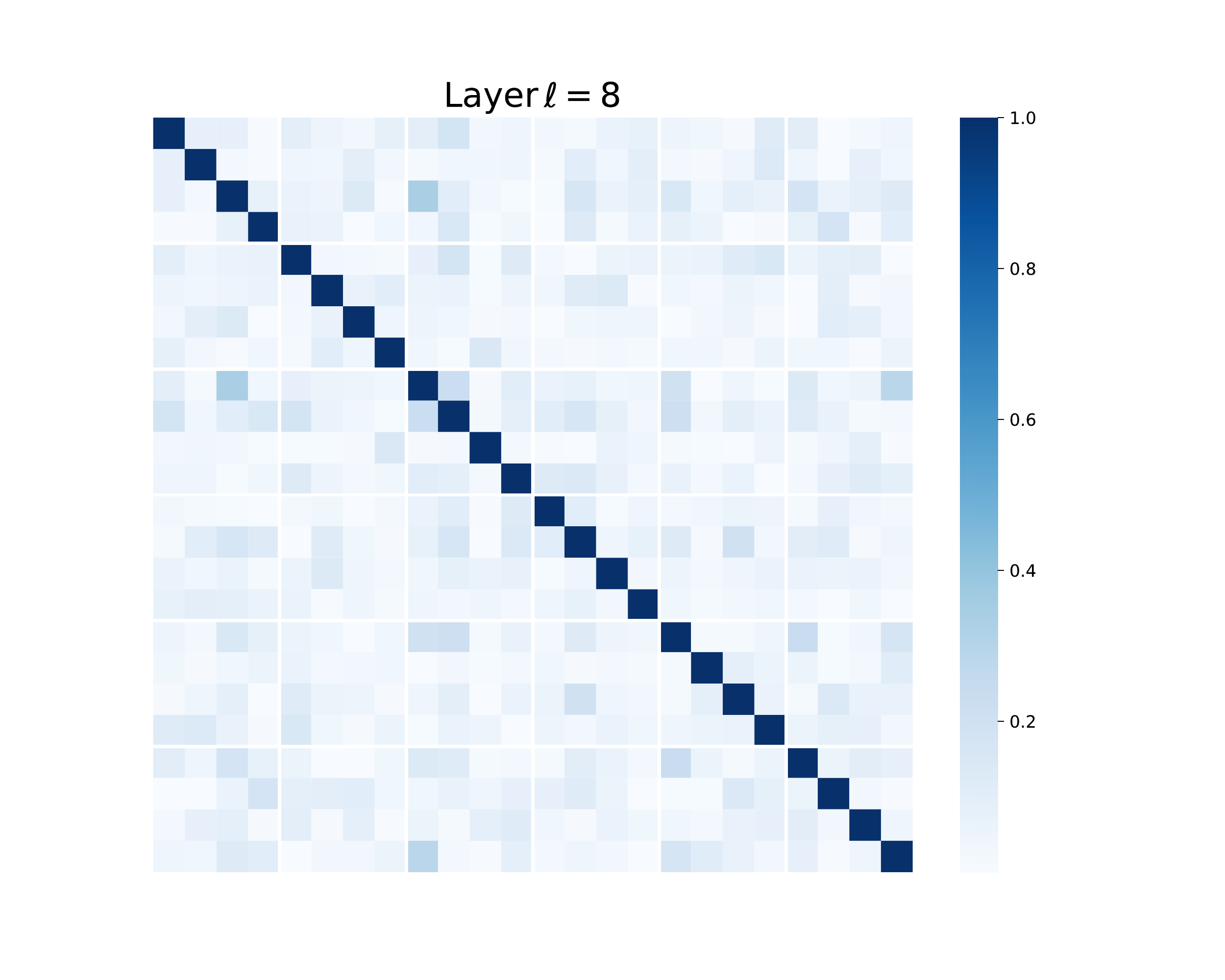}
         \caption{$\ell=8$.}
     \end{subfigure}
     \begin{subfigure}[b]{0.24\textwidth}
         \centering
    \includegraphics[width=\textwidth]{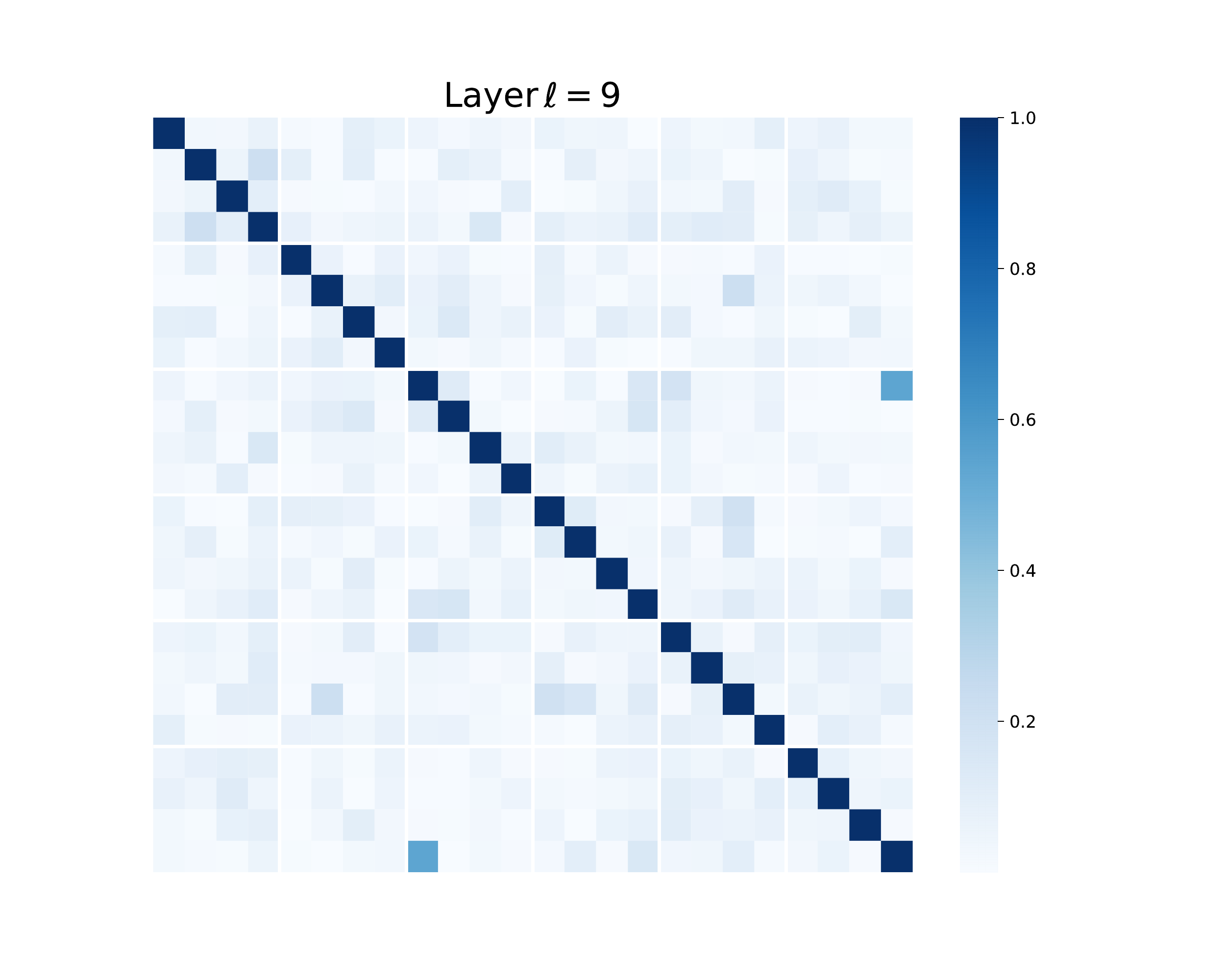}
         \caption{$\ell=9$.}
     \end{subfigure}
     \begin{subfigure}[b]{0.24\textwidth}
         \centering
    \includegraphics[width=\textwidth]{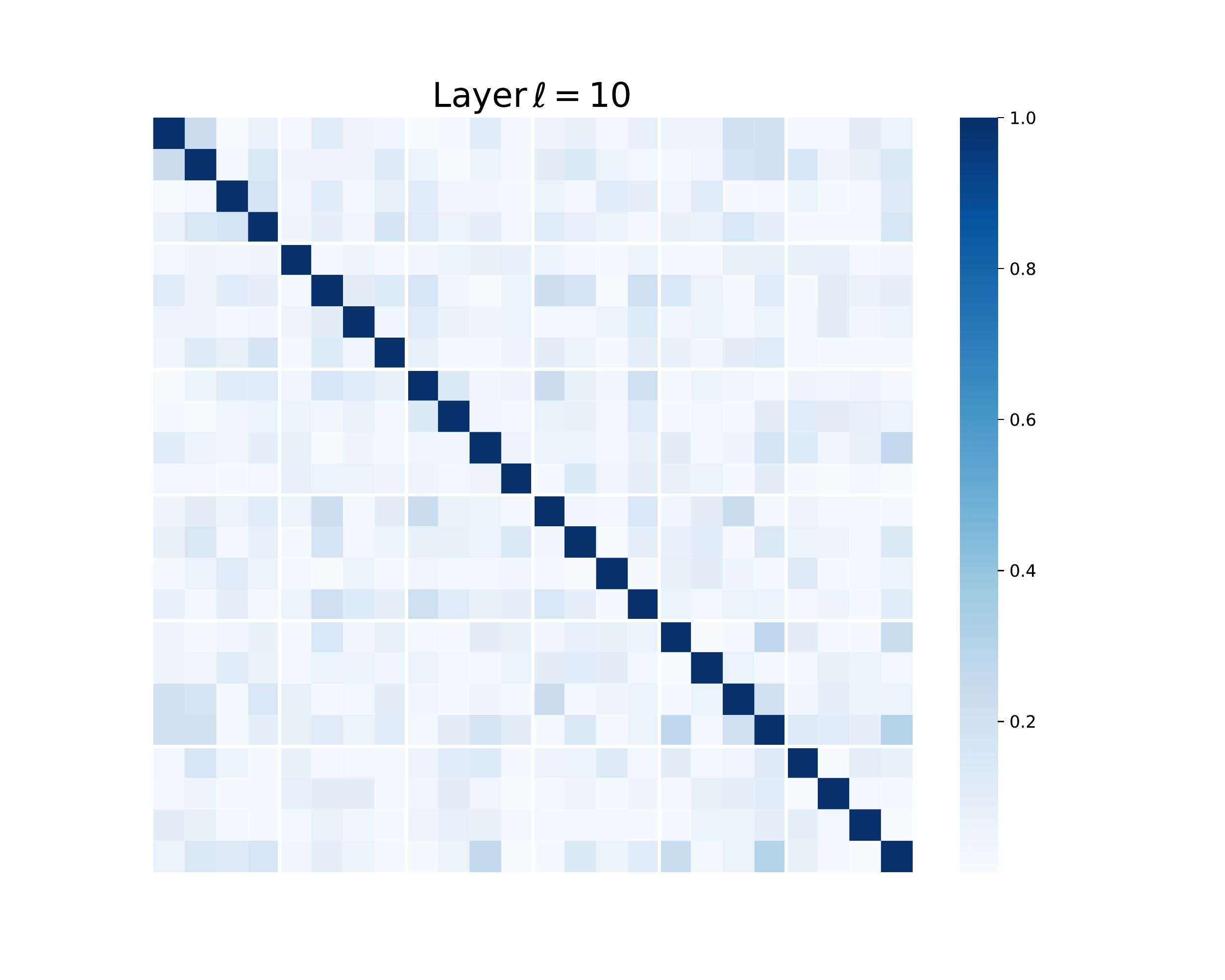}
         \caption{$\ell=10$.}
     \end{subfigure}
     \begin{subfigure}[b]{0.24\textwidth}
         \centering
    \includegraphics[width=\textwidth]{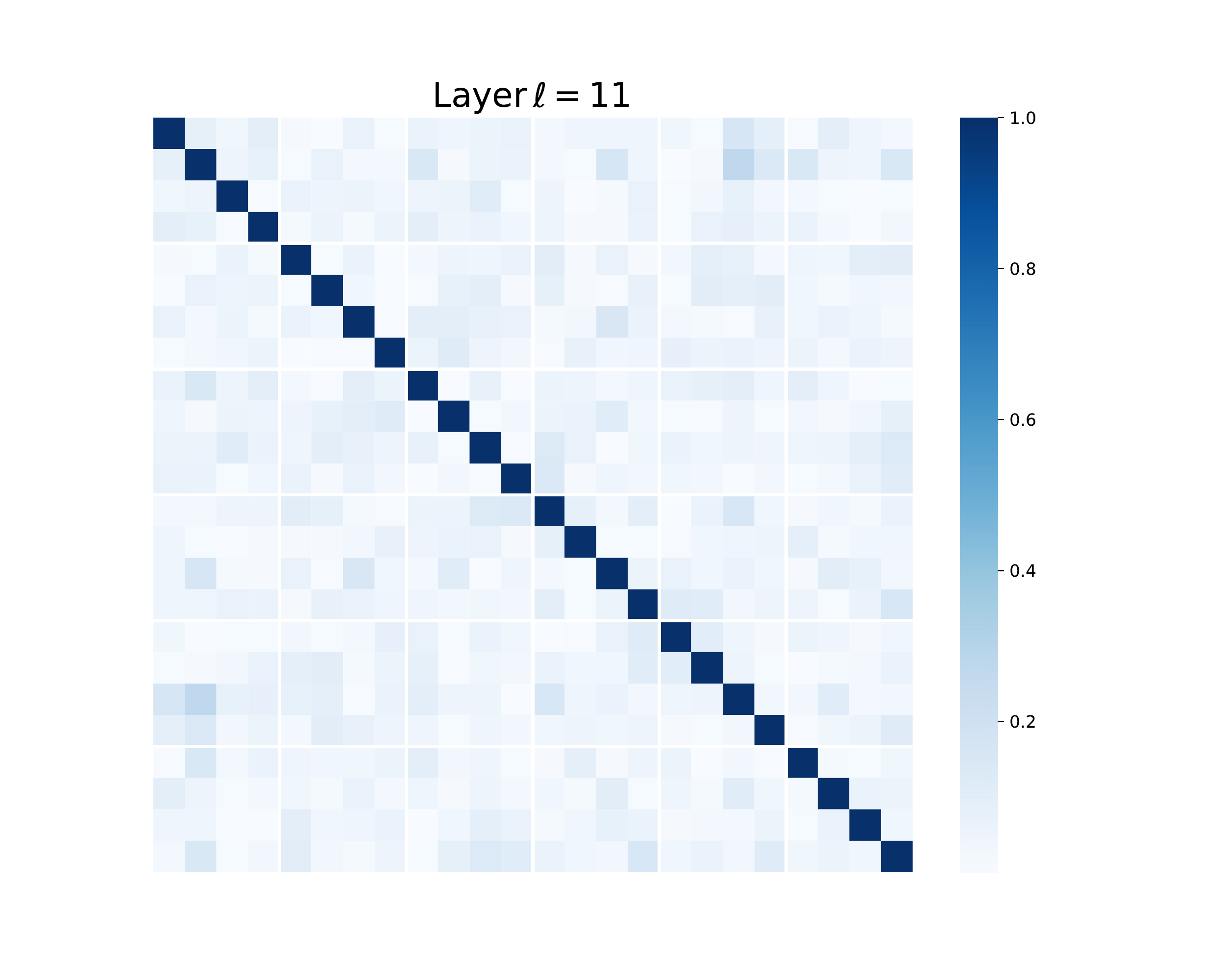}
         \caption{$\ell=11$.}
     \end{subfigure}
     \begin{subfigure}[b]{0.24\textwidth}
         \centering
    \includegraphics[width=\textwidth]{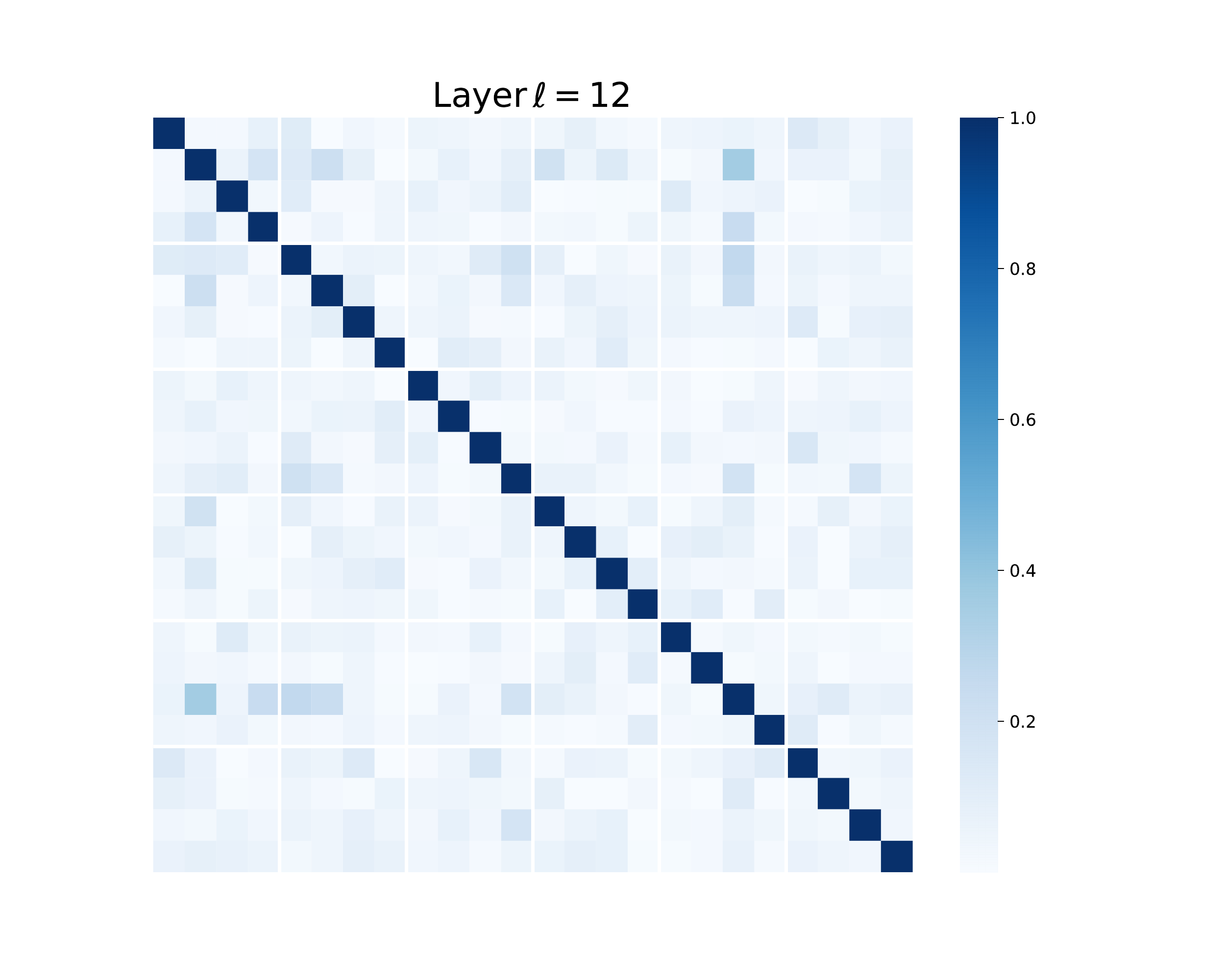}
         \caption{$\ell=12$.}
     \end{subfigure}
        \caption{We visualize the $[\vU_1^{\ell}, \dots, \vU_K^{\ell}]^{\adj}[\vU_1^{\ell}, \dots, \vU_K^{\ell}] \in \bR^{pK \times pK}$ at different layers.  
        The $(i, j)$-th block in each sub-figure corresponds to $(\vU_i^{\ell})\adj\vU_j^{\ell}$ for $i, j \in [K]$ at a particular layer $\ell$. To enhance the visual clarity, for each subspace $\vU_i$, we randomly pick 4 directions for display purposes. (Model: \ours{-Tiny})}
        \label{fig:appendix-exp-visualize-UiUj}
        \vspace{-0.1in}
\end{figure}

\clearpage
\subsubsection{Additional Layer-wise Visualization}\label{subsec:appendix-token-representation-visualize}
We provide more results of the layer-wise token representation visualization on other samples in \Cref{fig:appendix-exp-ista-sparsity-heatmap-sample1}, \Cref{fig:appendix-exp-ista-sparsity-heatmap-sample2}, \Cref{fig:appendix-exp-ista-sparsity-heatmap-sample3},  and \Cref{fig:appendix-exp-ista-sparsity-heatmap-sample4} (Model: \ours{-Base}).

\begin{figure}[ht]
     \centering
     \begin{subfigure}[b]{0.22\textwidth}
         \centering
    \includegraphics[width=\textwidth]{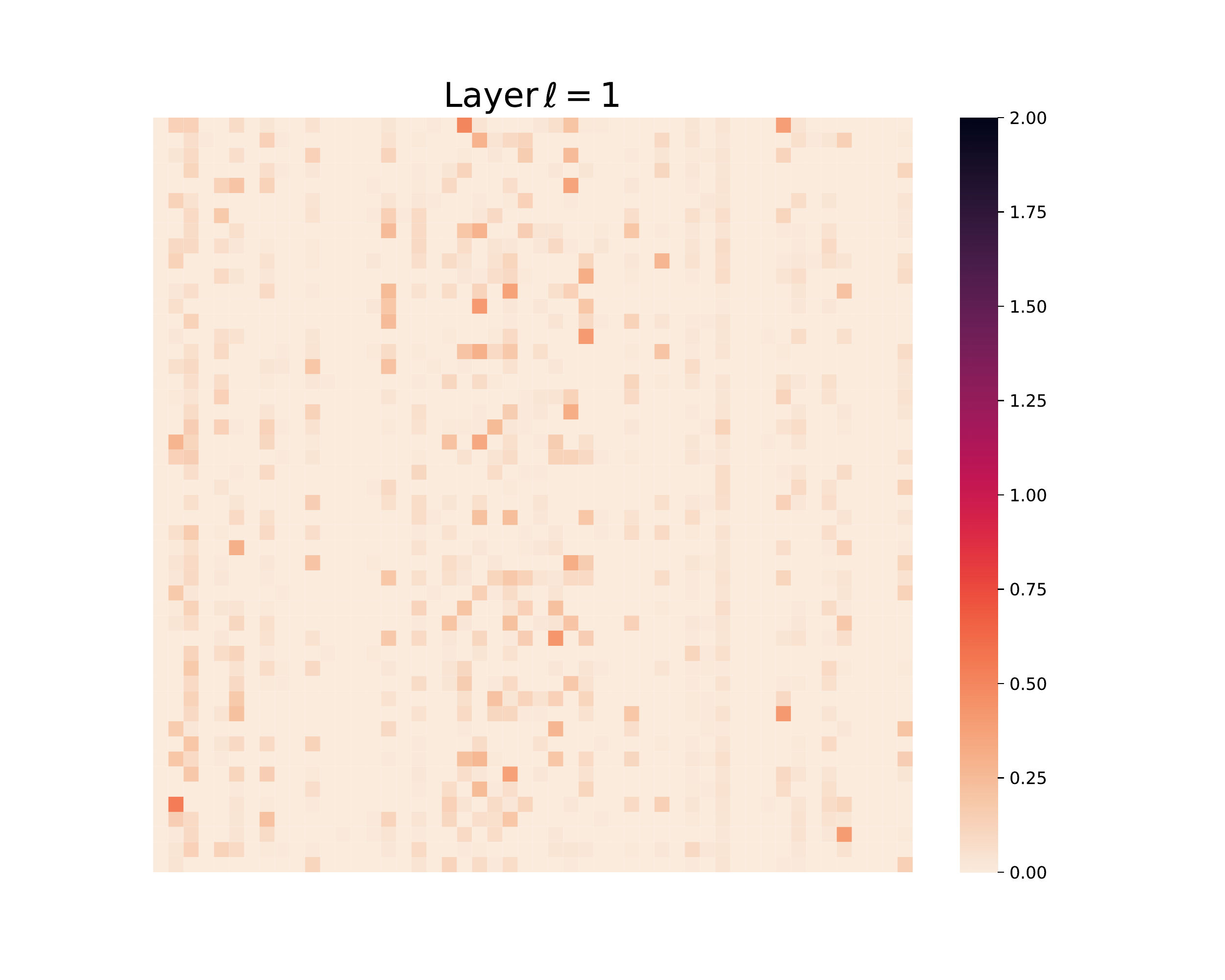}
     \end{subfigure}
     \begin{subfigure}[b]{0.22\textwidth}
         \centering
    \includegraphics[width=\textwidth]{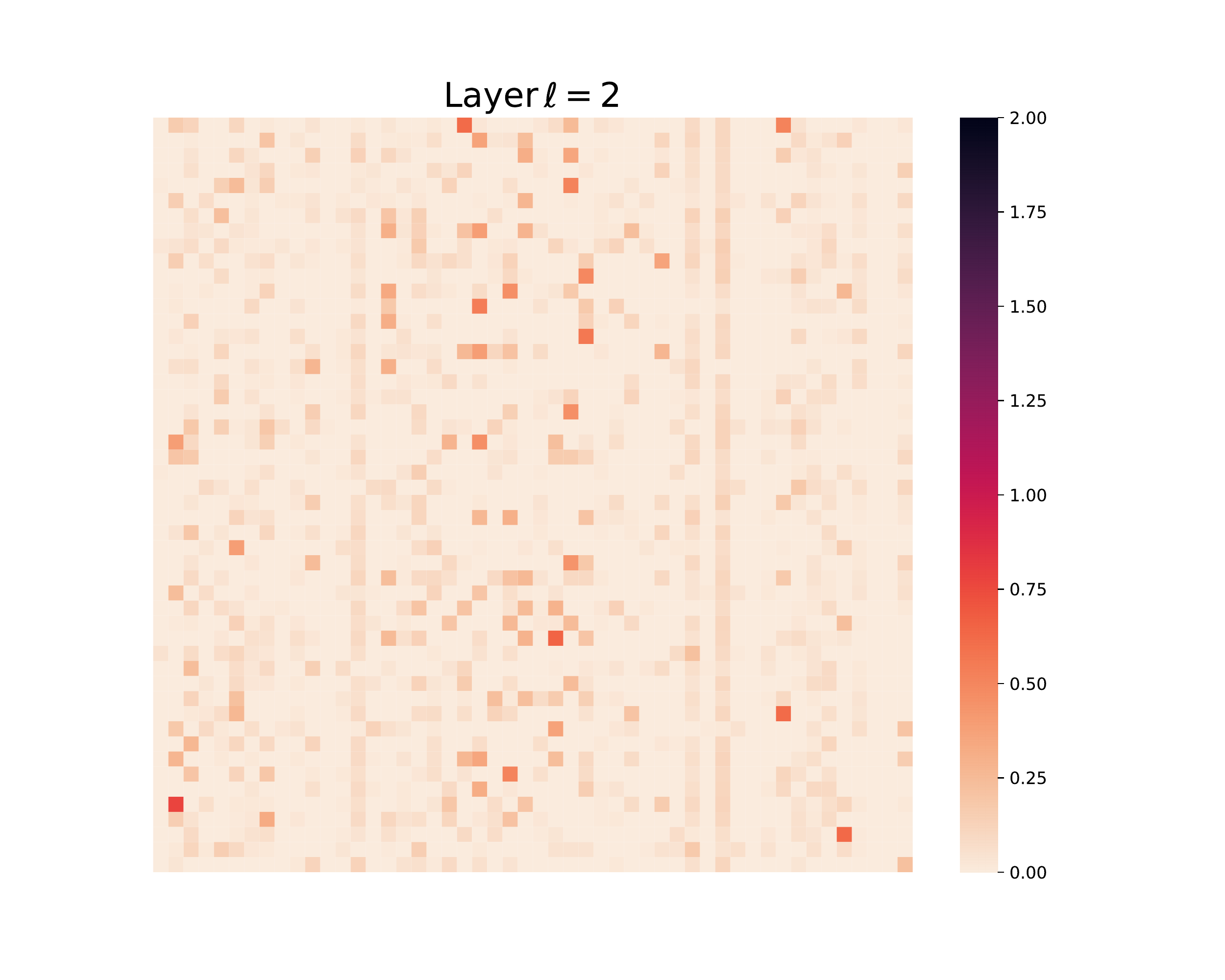}
     \end{subfigure}
     \begin{subfigure}[b]{0.22\textwidth}
         \centering
    \includegraphics[width=\textwidth]{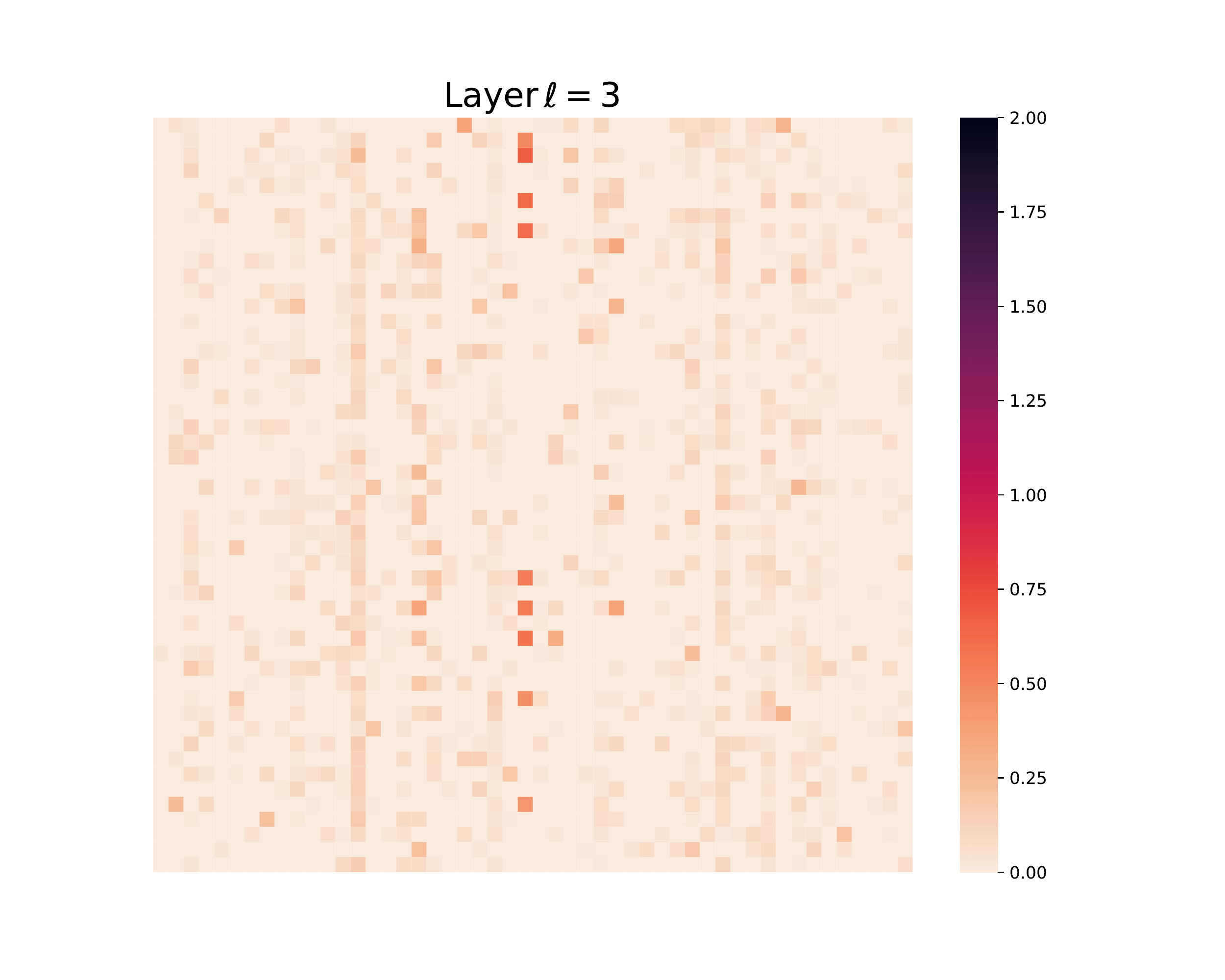}
     \end{subfigure}
     \begin{subfigure}[b]{0.22\textwidth}
         \centering
    \includegraphics[width=\textwidth]{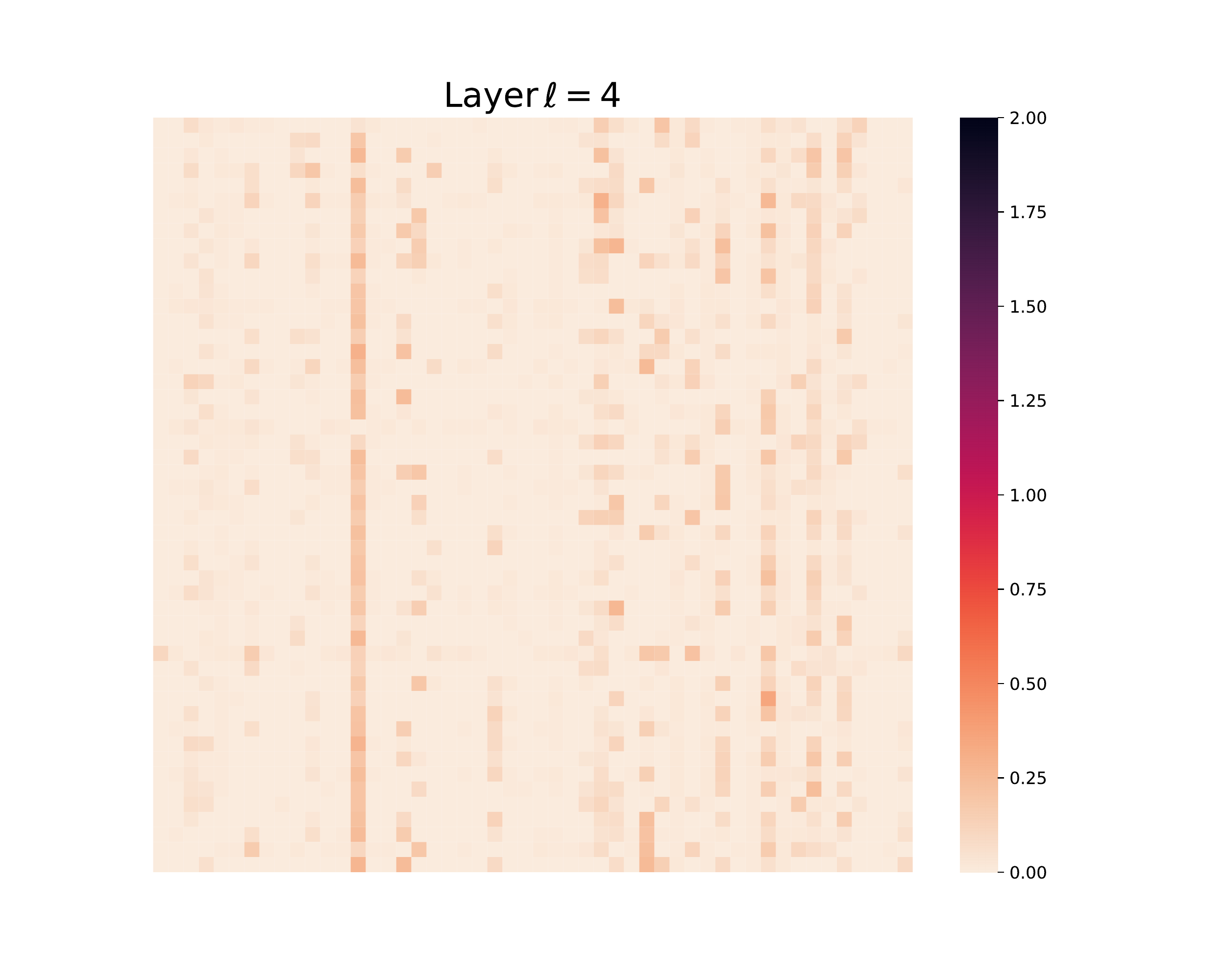}
     \end{subfigure}
     \begin{subfigure}[b]{0.22\textwidth}
         \centering
    \includegraphics[width=\textwidth]{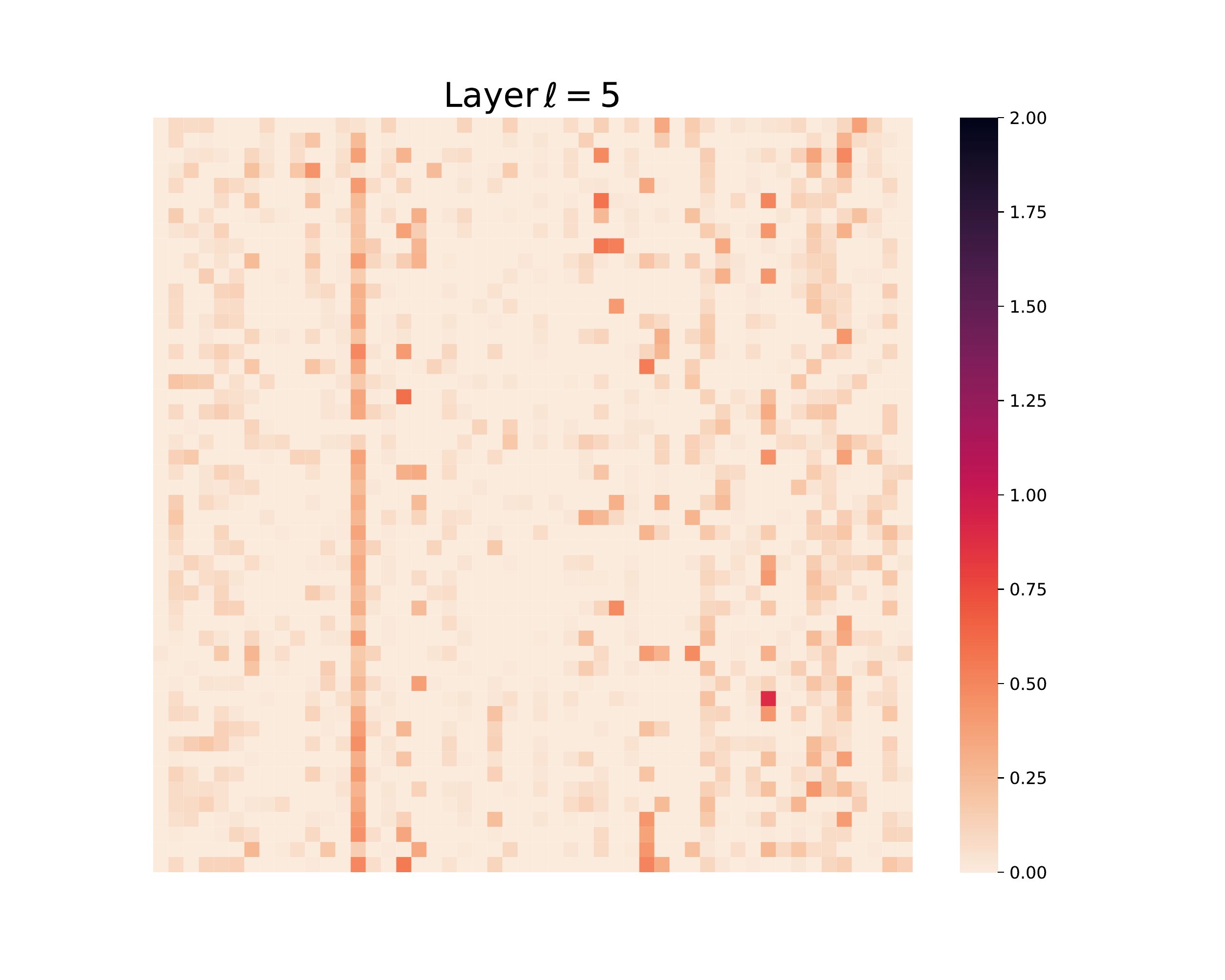}
     \end{subfigure}
     \begin{subfigure}[b]{0.22\textwidth}
         \centering
    \includegraphics[width=\textwidth]{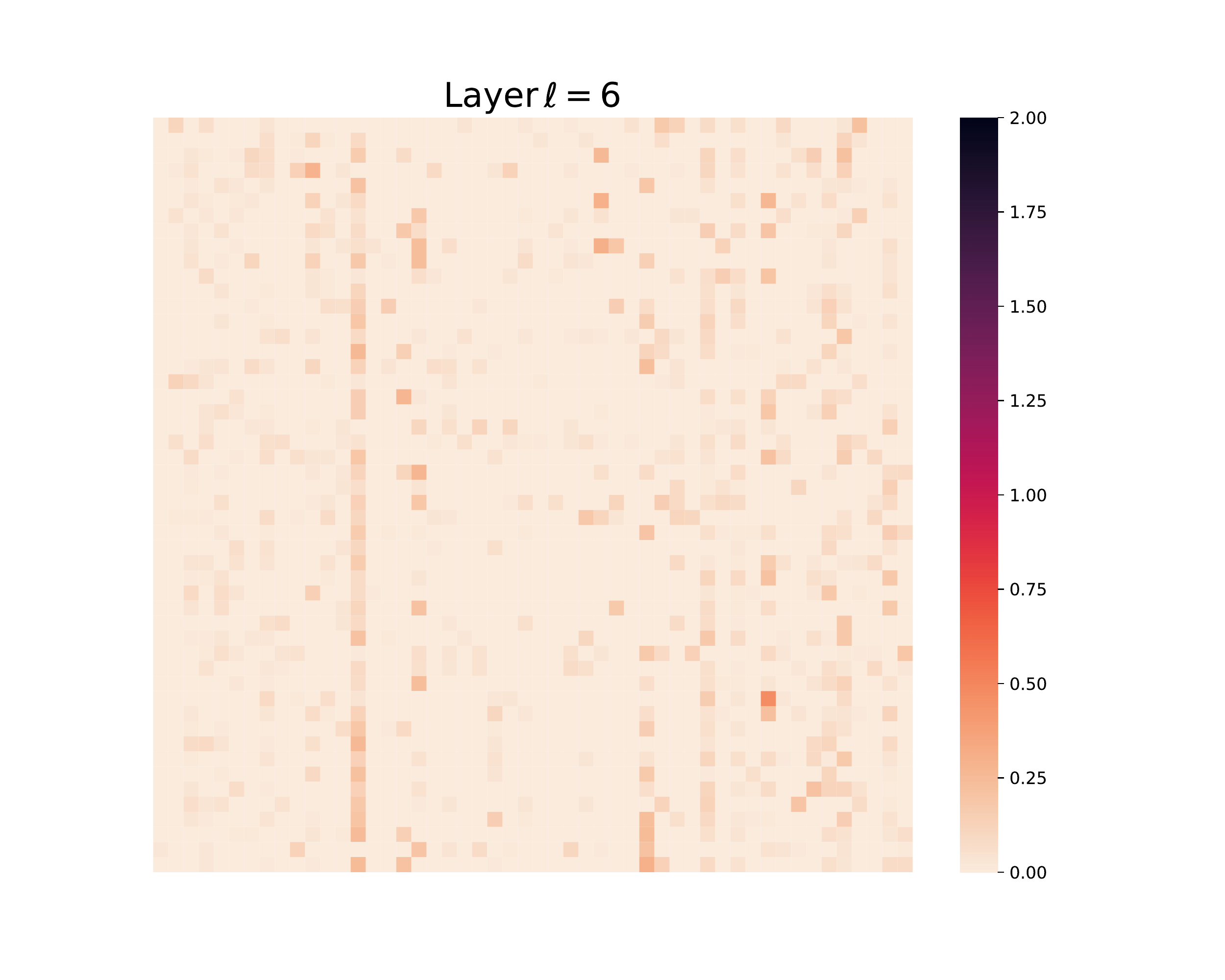}
     \end{subfigure}
     \begin{subfigure}[b]{0.22\textwidth}
         \centering
    \includegraphics[width=\textwidth]{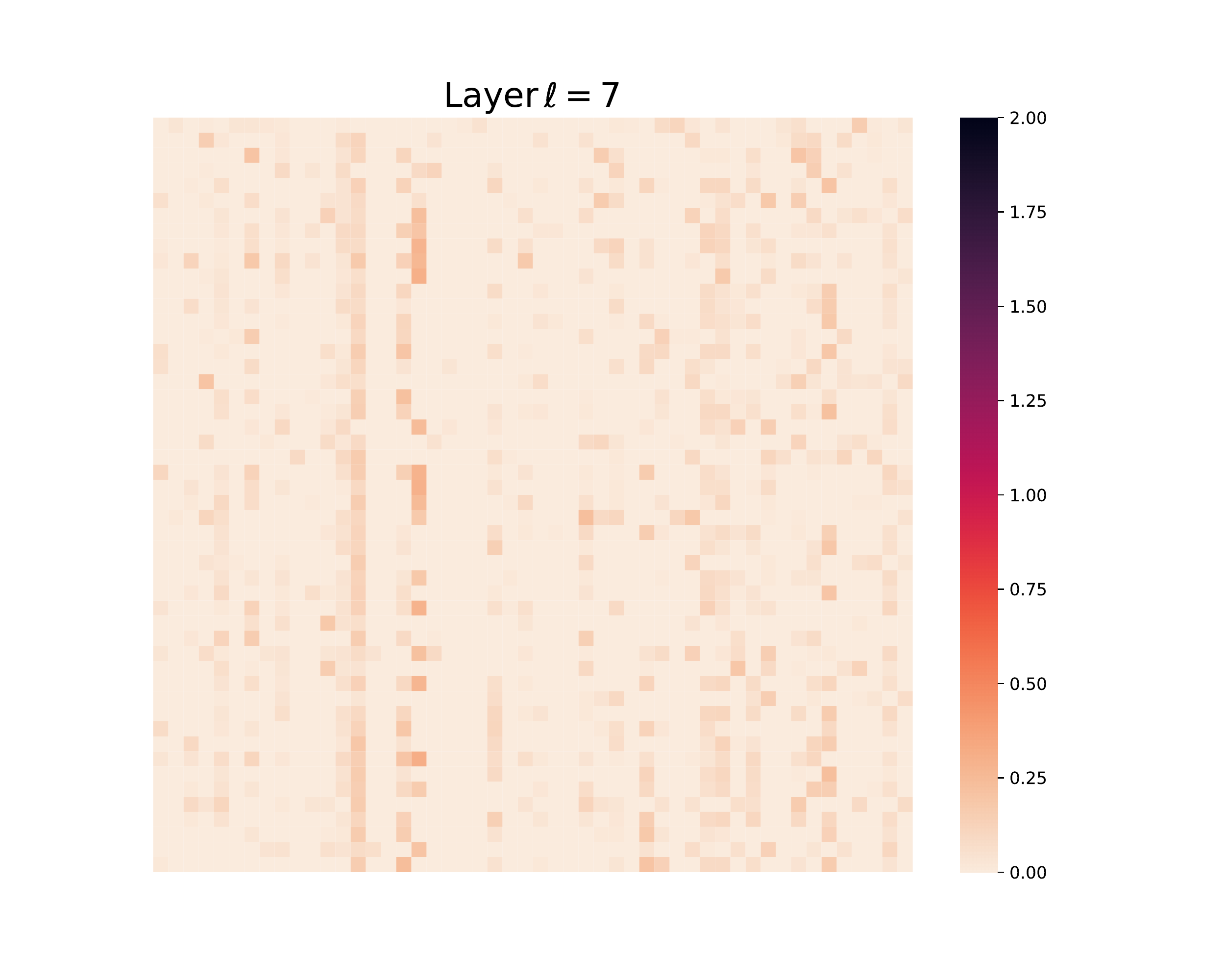}
     \end{subfigure}
     \begin{subfigure}[b]{0.22\textwidth}
         \centering
    \includegraphics[width=\textwidth]{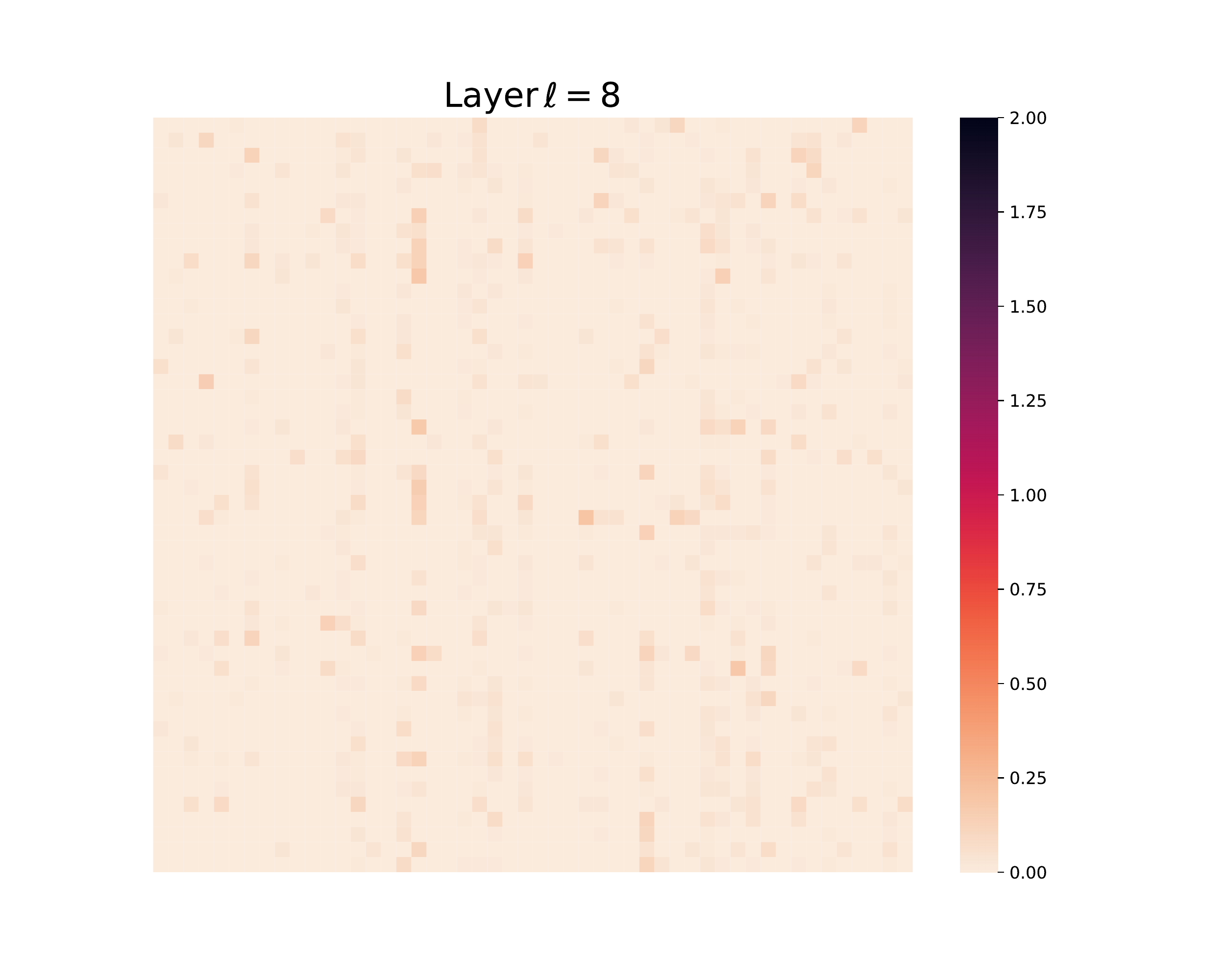}
     \end{subfigure}
     \begin{subfigure}[b]{0.22\textwidth}
         \centering
    \includegraphics[width=\textwidth]{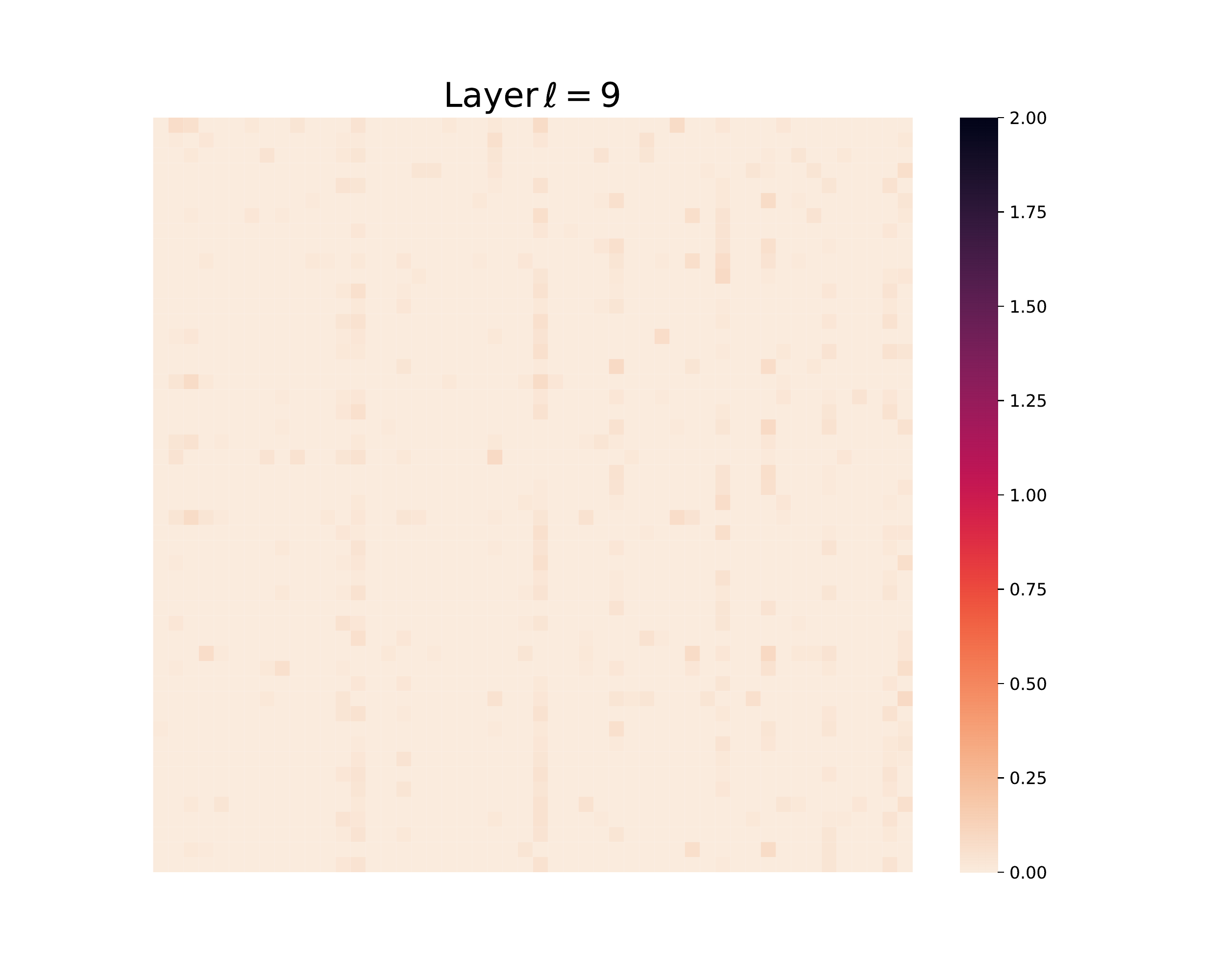}
     \end{subfigure}
     \begin{subfigure}[b]{0.22\textwidth}
         \centering
    \includegraphics[width=\textwidth]{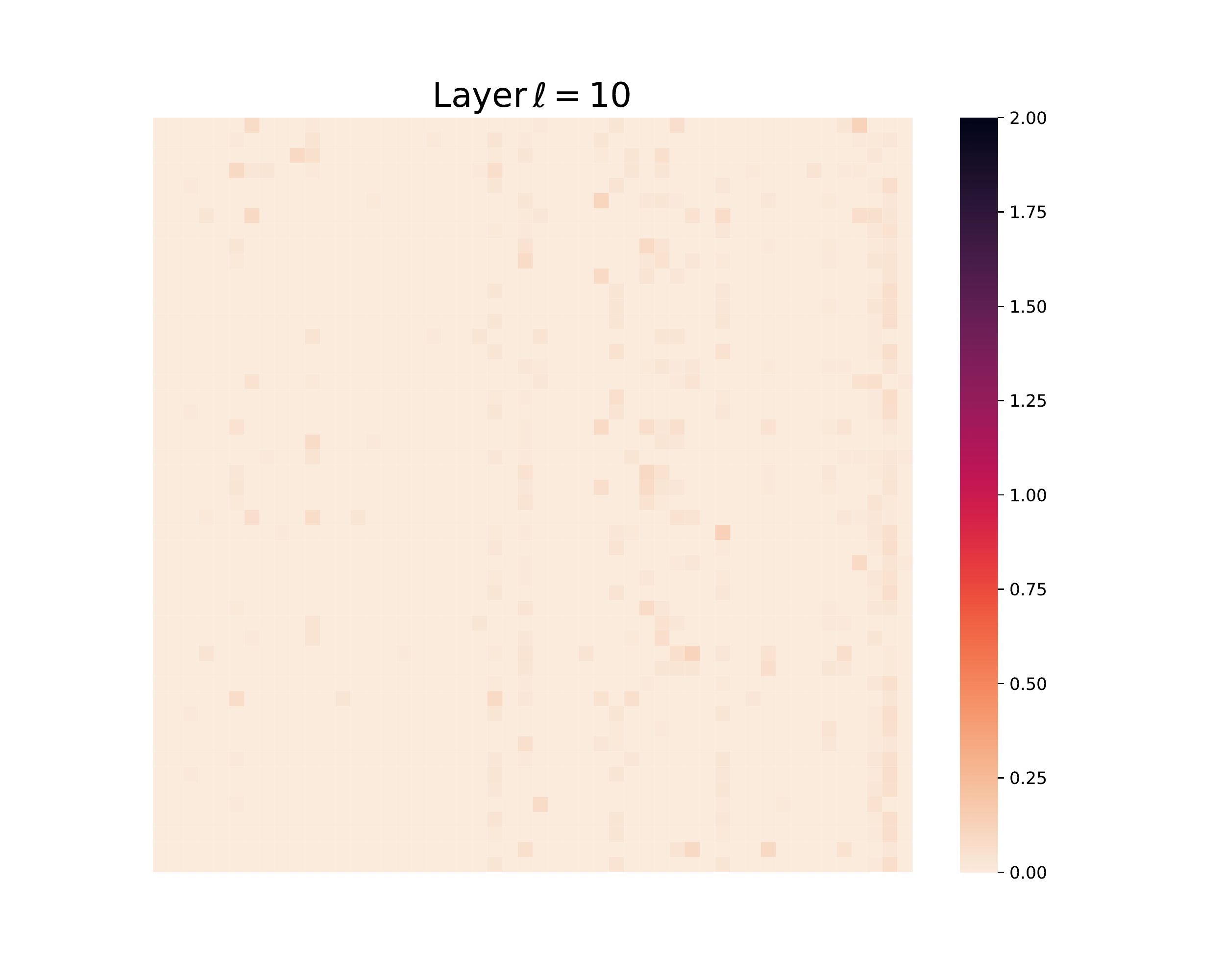}
     \end{subfigure}
     \begin{subfigure}[b]{0.22\textwidth}
         \centering
    \includegraphics[width=\textwidth]{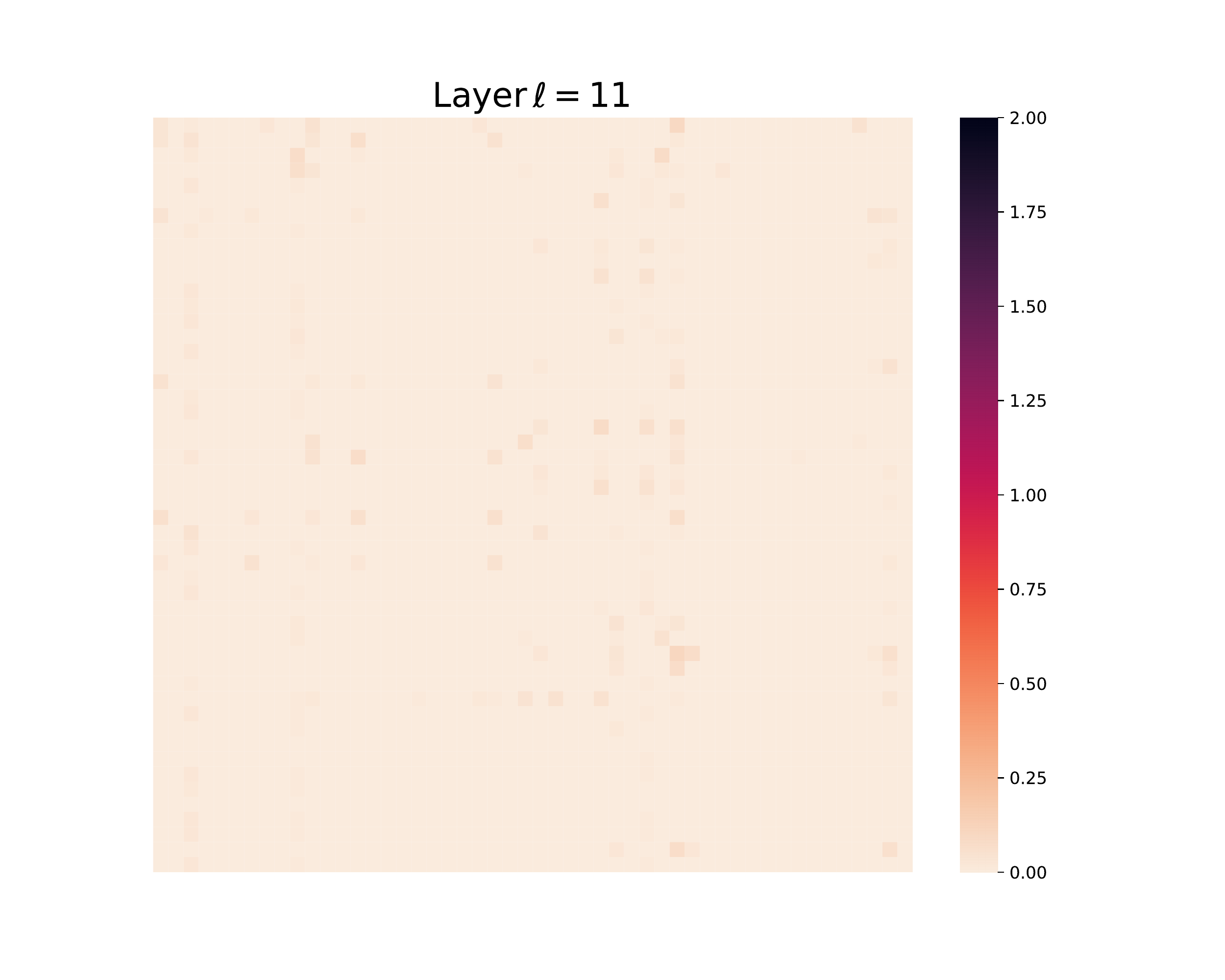}
     \end{subfigure}
     \begin{subfigure}[b]{0.22\textwidth}
         \centering
    \includegraphics[width=\textwidth]{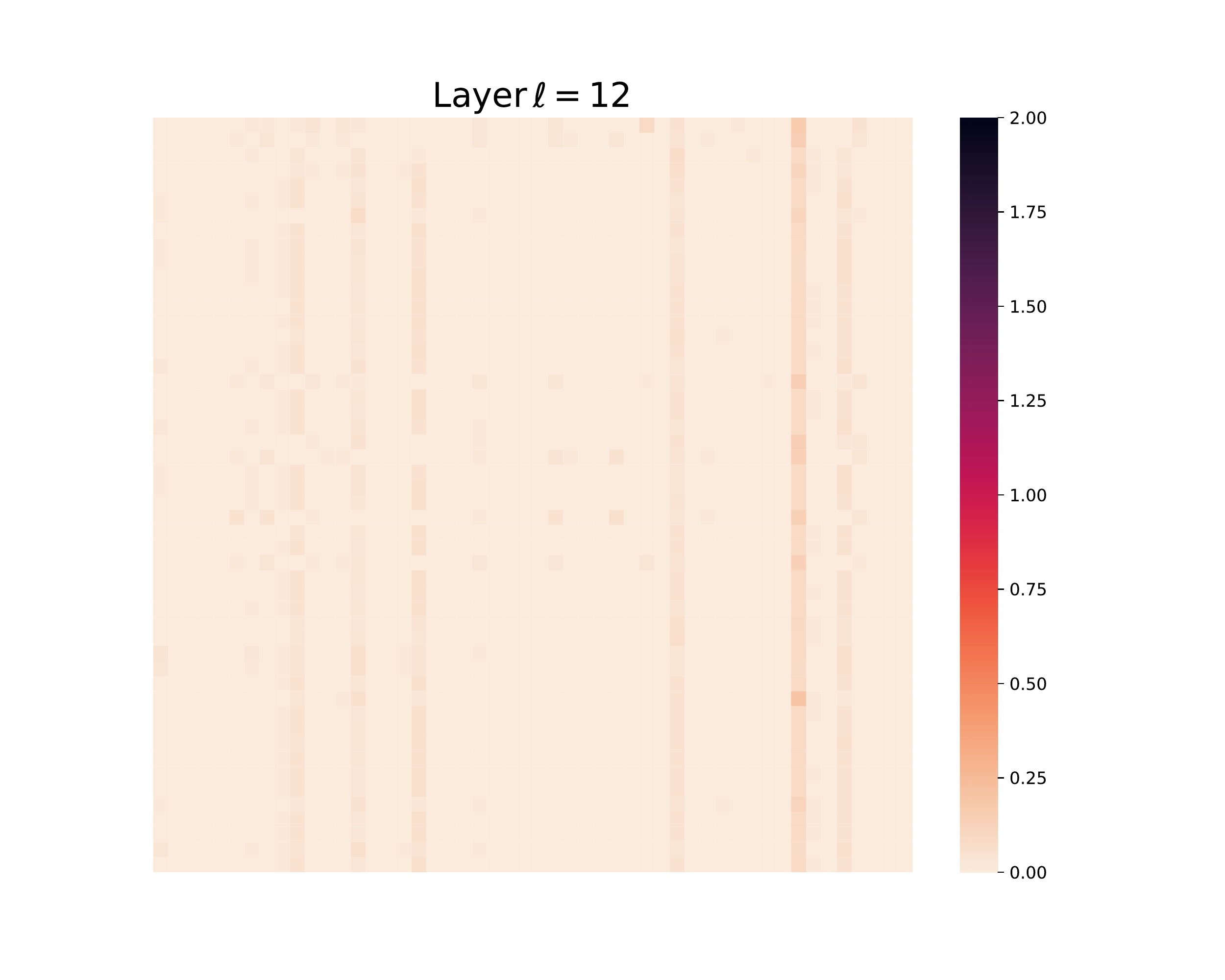}
     \end{subfigure}
        \caption{Visualizing layer-wise token $\Z^{\ell}$ representations at each layer $\ell$. To enhance the visual clarity, we randomly extract a 50$\times$50 sub-matrix from $\Z^{\ell}$ for display purposes. (\textit{Sample 1})}
        \label{fig:appendix-exp-ista-sparsity-heatmap-sample1}
        \vspace{-0.1in}
\end{figure}

\begin{figure}[ht]
     \centering
     \begin{subfigure}[b]{0.22\textwidth}
         \centering
    \includegraphics[width=\textwidth]{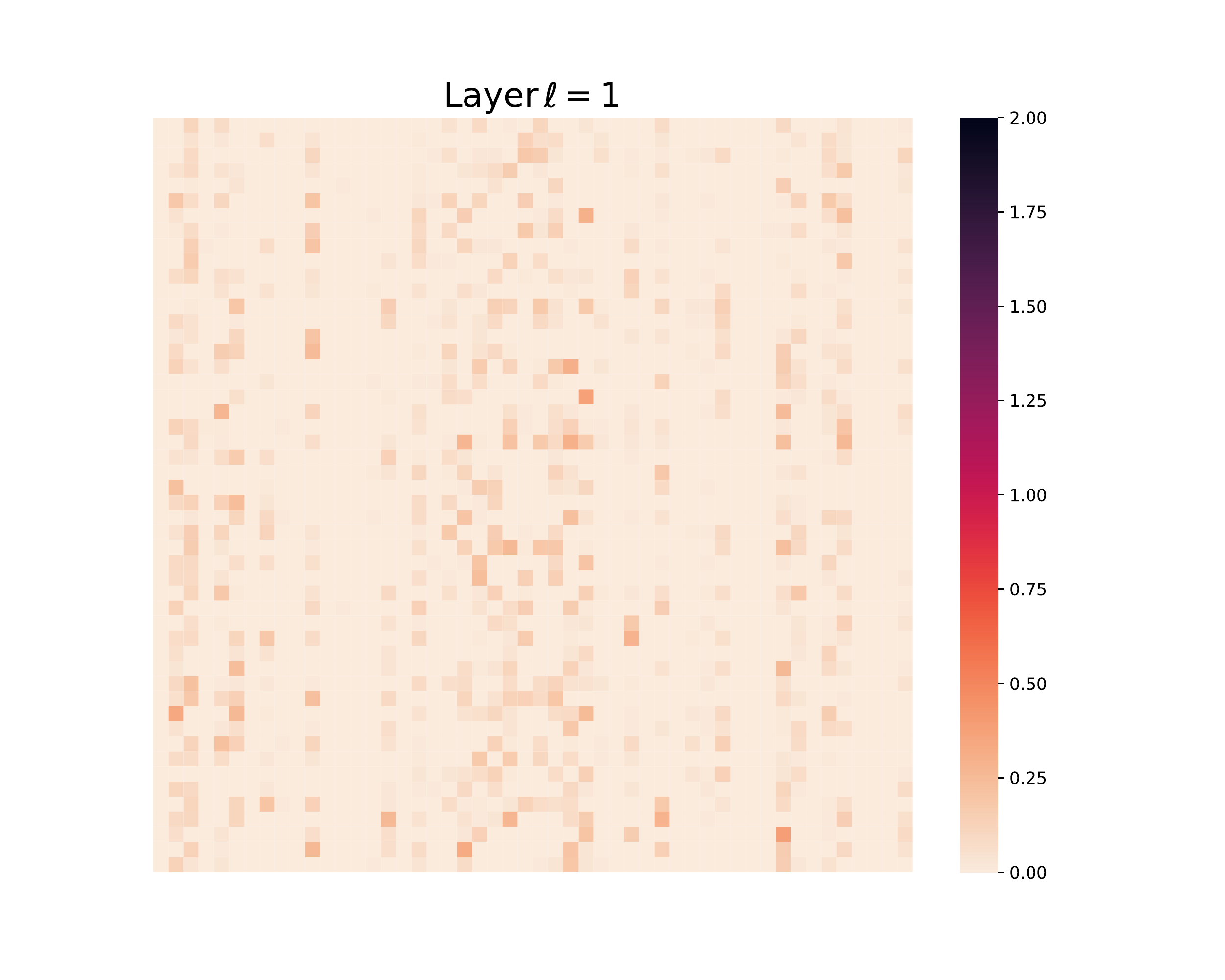}
     \end{subfigure}
     \begin{subfigure}[b]{0.22\textwidth}
         \centering
    \includegraphics[width=\textwidth]{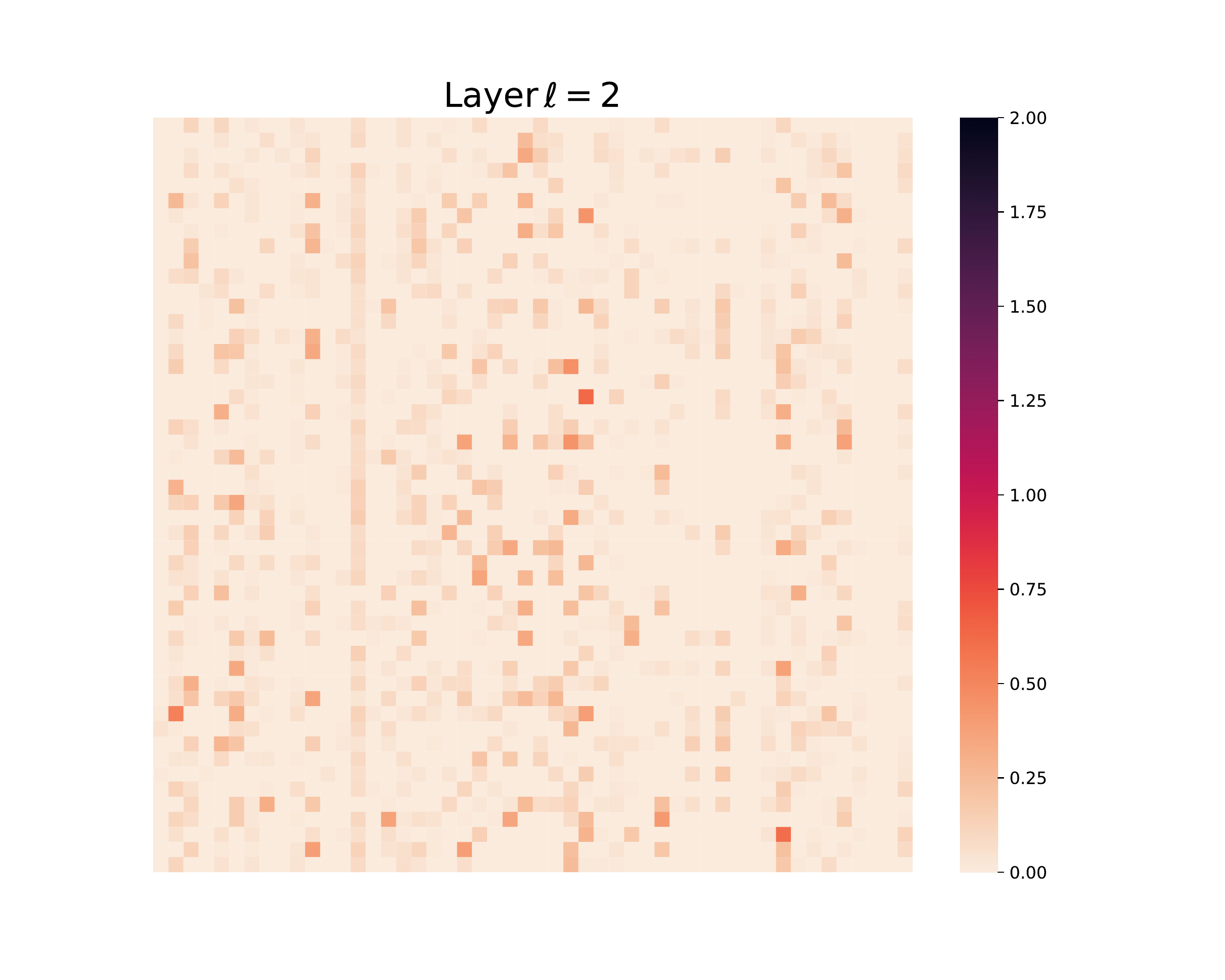}
     \end{subfigure}
     \begin{subfigure}[b]{0.22\textwidth}
         \centering
    \includegraphics[width=\textwidth]{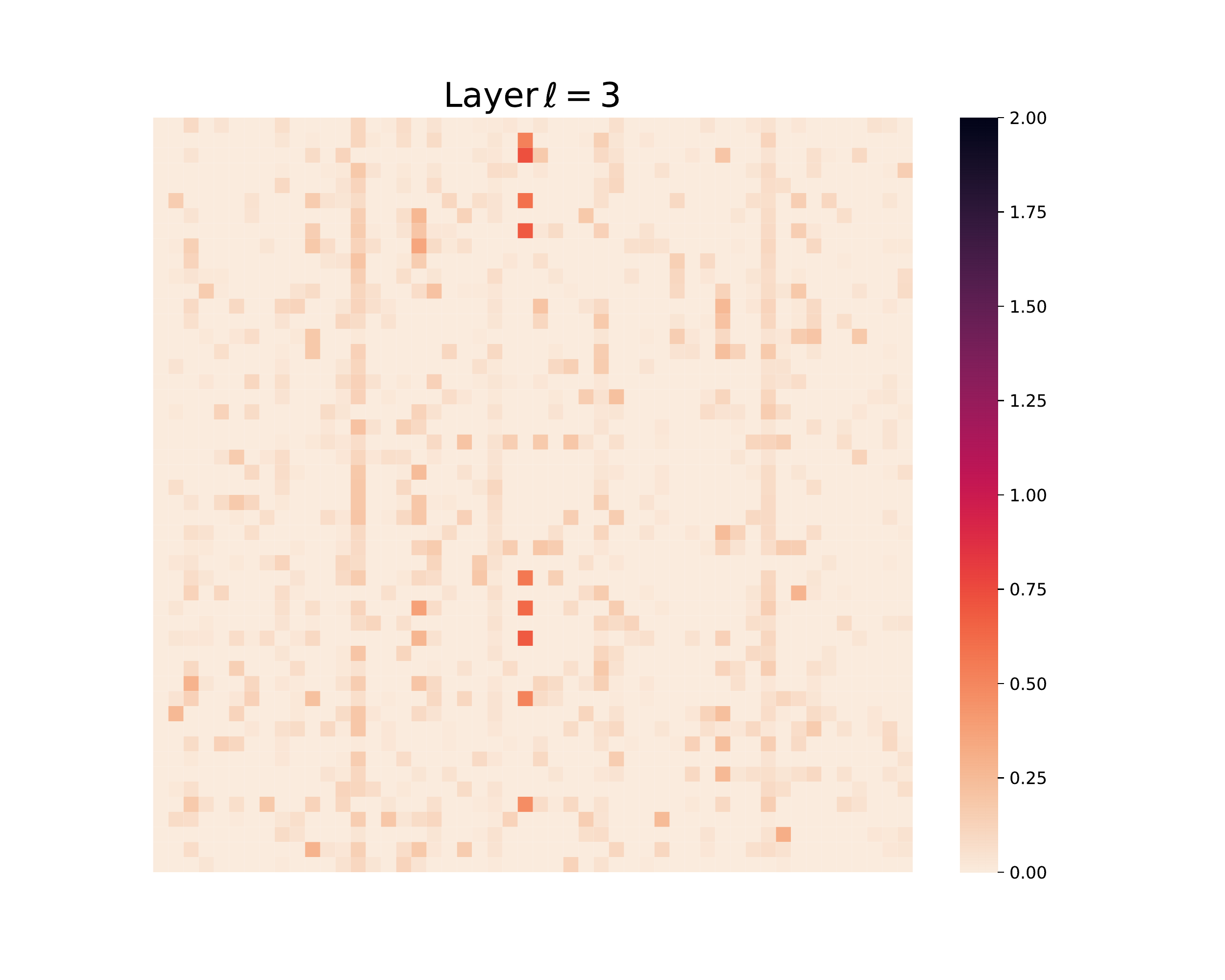}
     \end{subfigure}
     \begin{subfigure}[b]{0.22\textwidth}
         \centering
    \includegraphics[width=\textwidth]{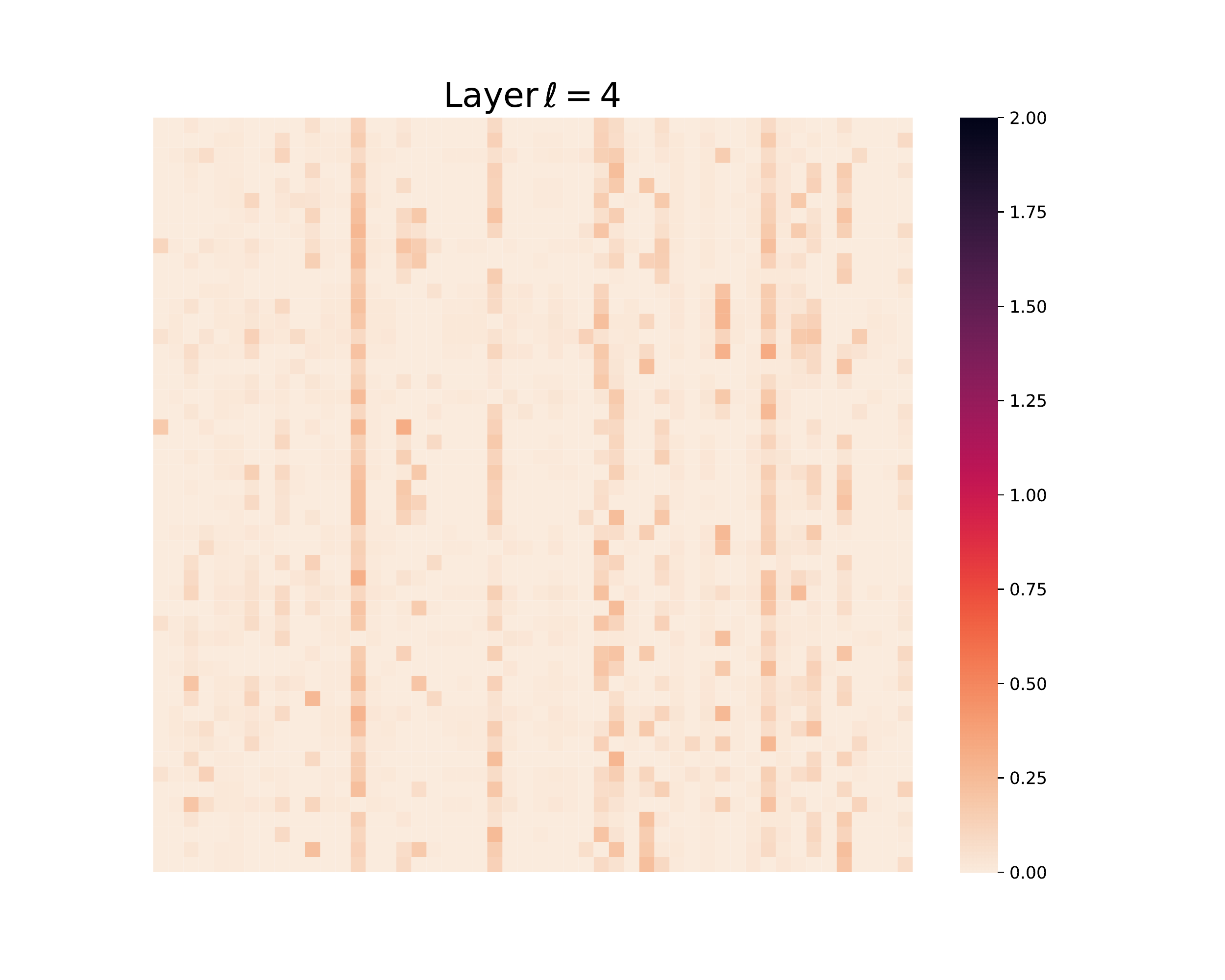}
     \end{subfigure}
     \begin{subfigure}[b]{0.22\textwidth}
         \centering
    \includegraphics[width=\textwidth]{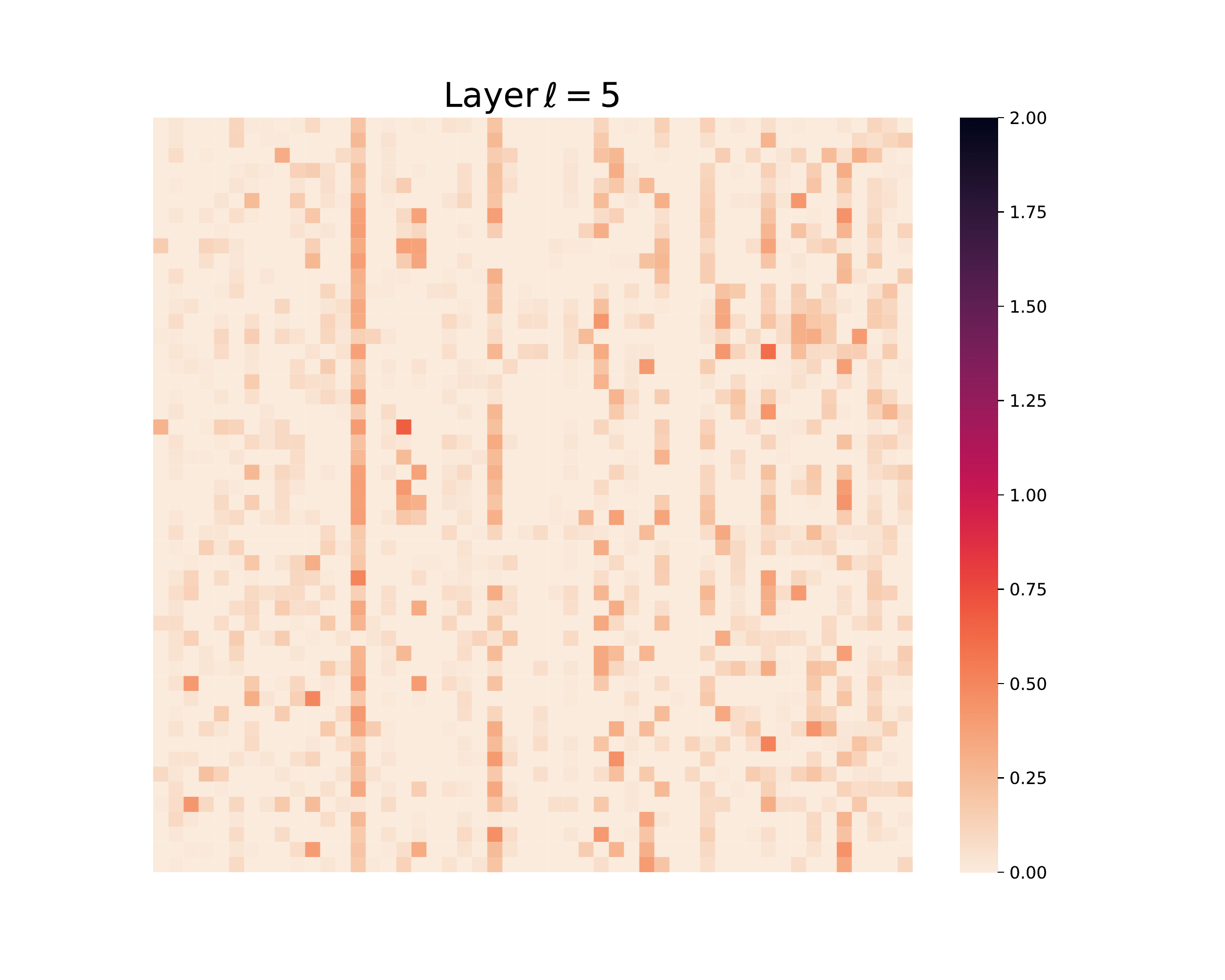}
     \end{subfigure}
     \begin{subfigure}[b]{0.22\textwidth}
         \centering
    \includegraphics[width=\textwidth]{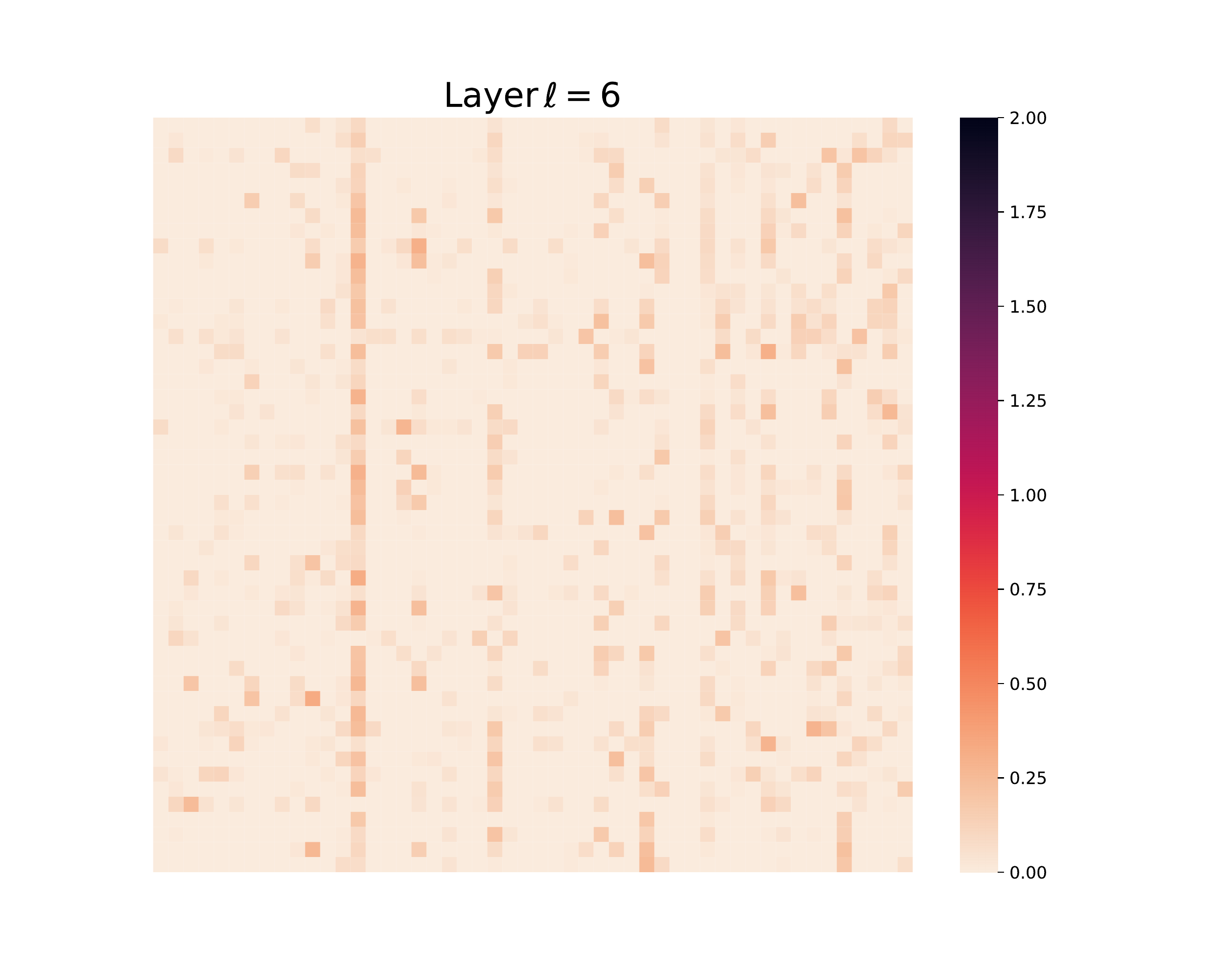}
     \end{subfigure}
     \begin{subfigure}[b]{0.22\textwidth}
         \centering
    \includegraphics[width=\textwidth]{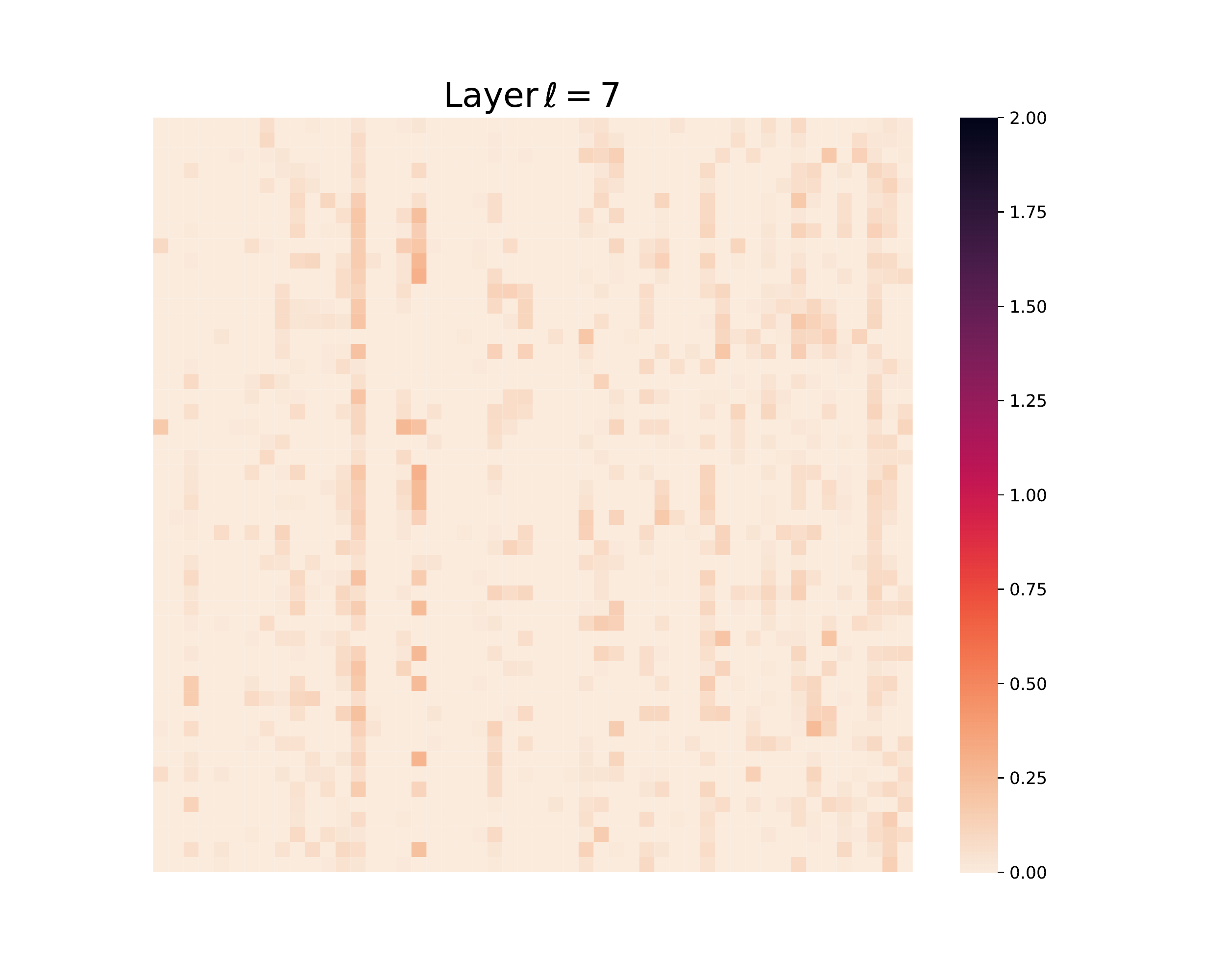}
     \end{subfigure}
     \begin{subfigure}[b]{0.22\textwidth}
         \centering
    \includegraphics[width=\textwidth]{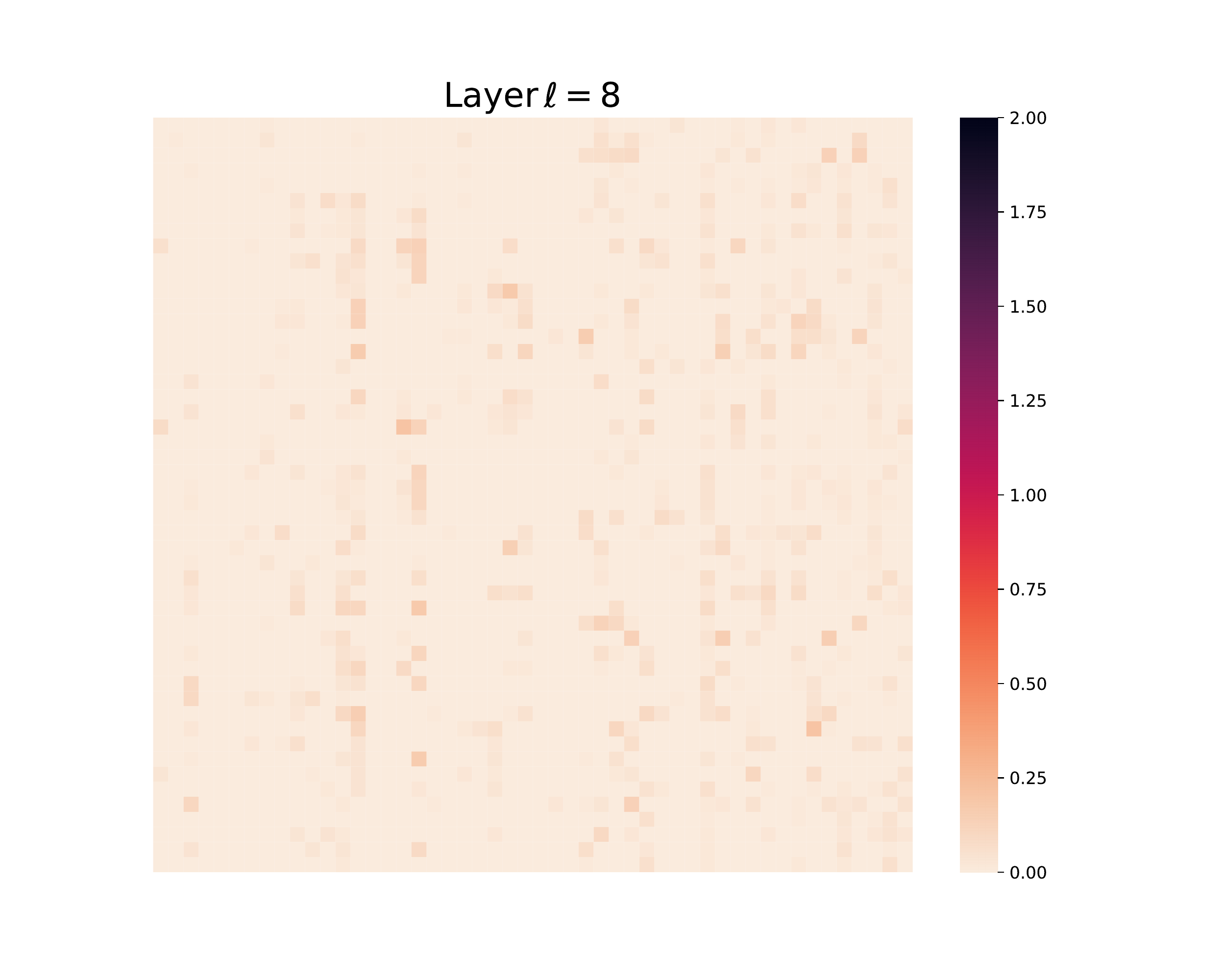}
     \end{subfigure}
     \begin{subfigure}[b]{0.22\textwidth}
         \centering
    \includegraphics[width=\textwidth]{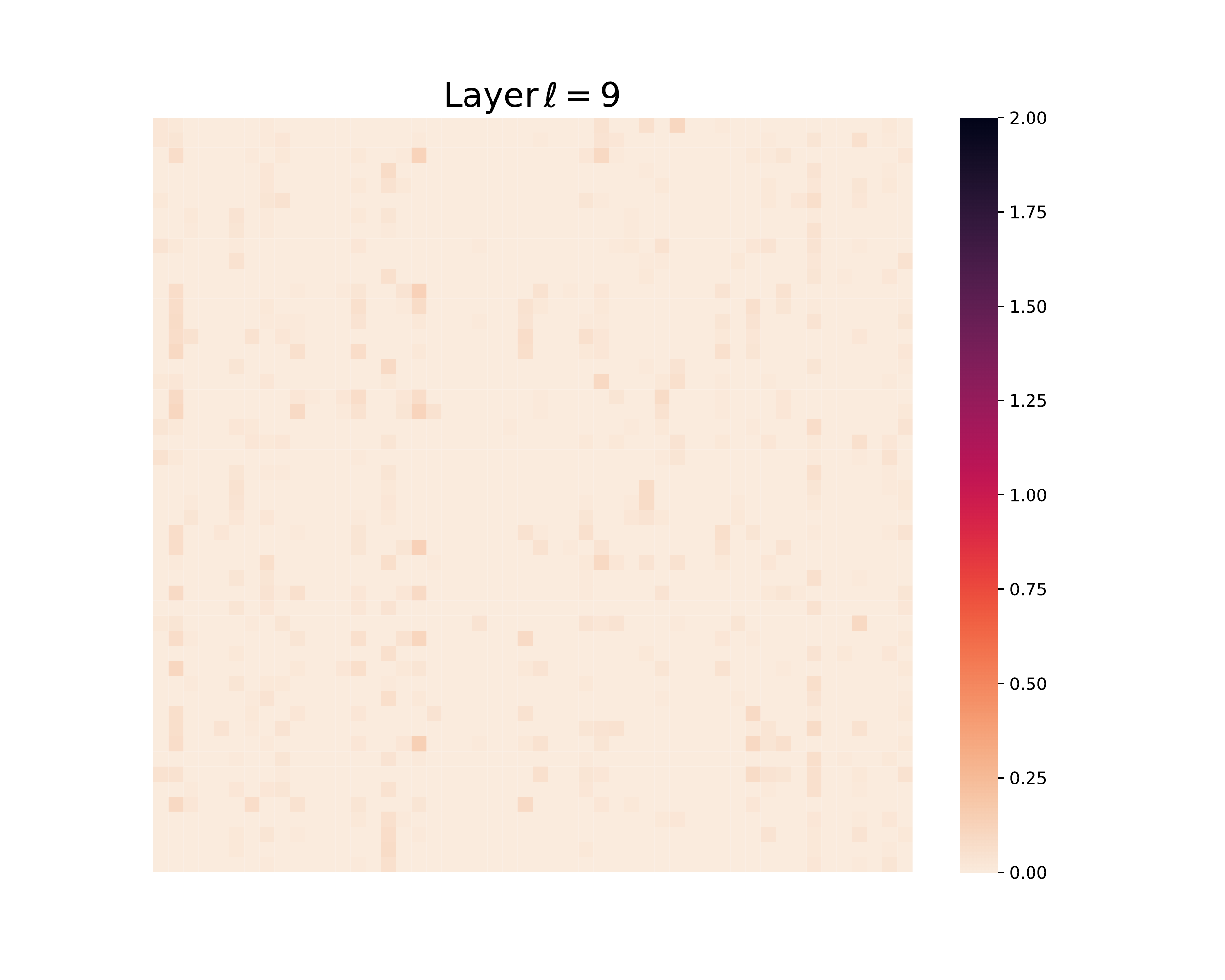}
     \end{subfigure}
     \begin{subfigure}[b]{0.22\textwidth}
         \centering
    \includegraphics[width=\textwidth]{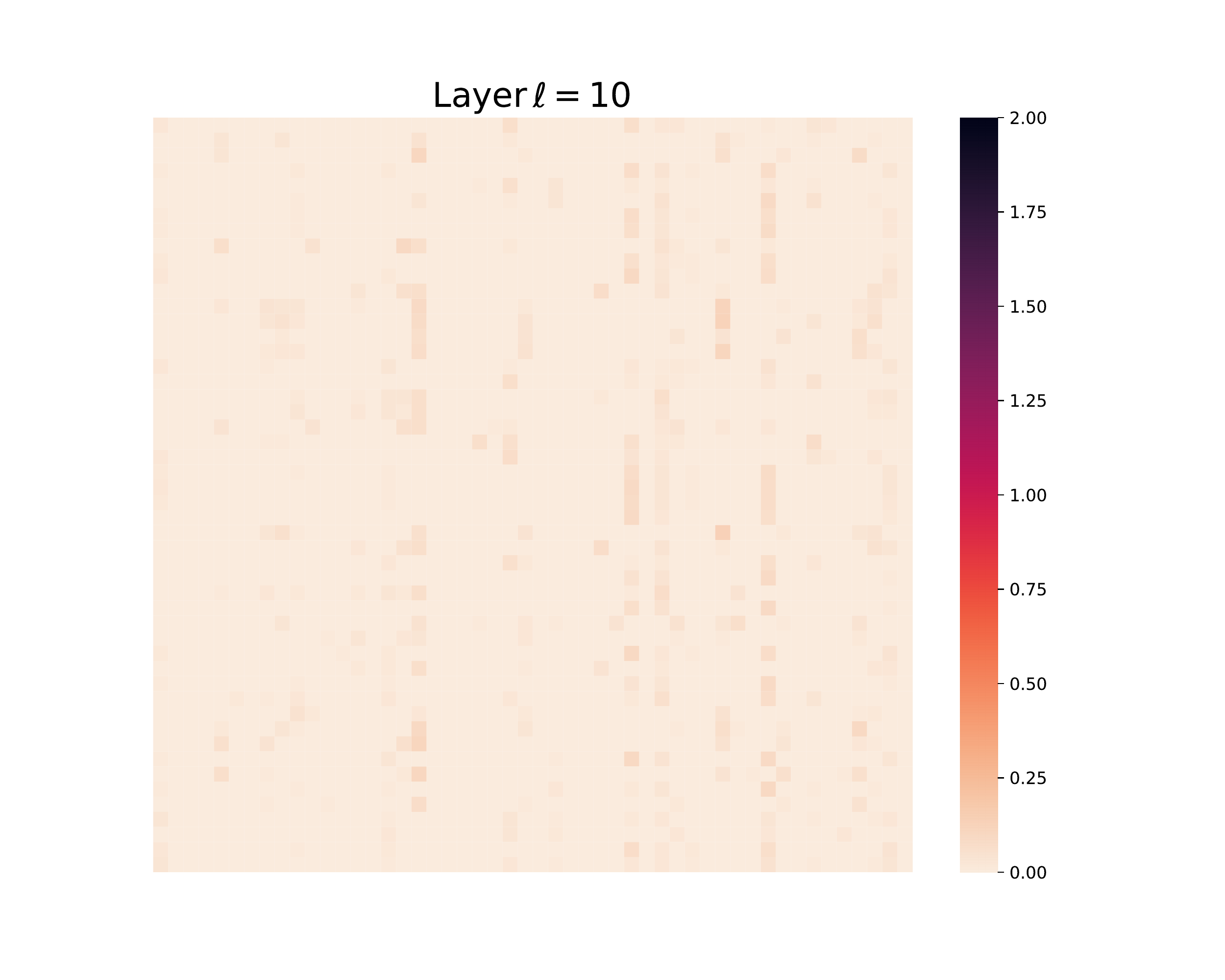}
     \end{subfigure}
     \begin{subfigure}[b]{0.22\textwidth}
         \centering
    \includegraphics[width=\textwidth]{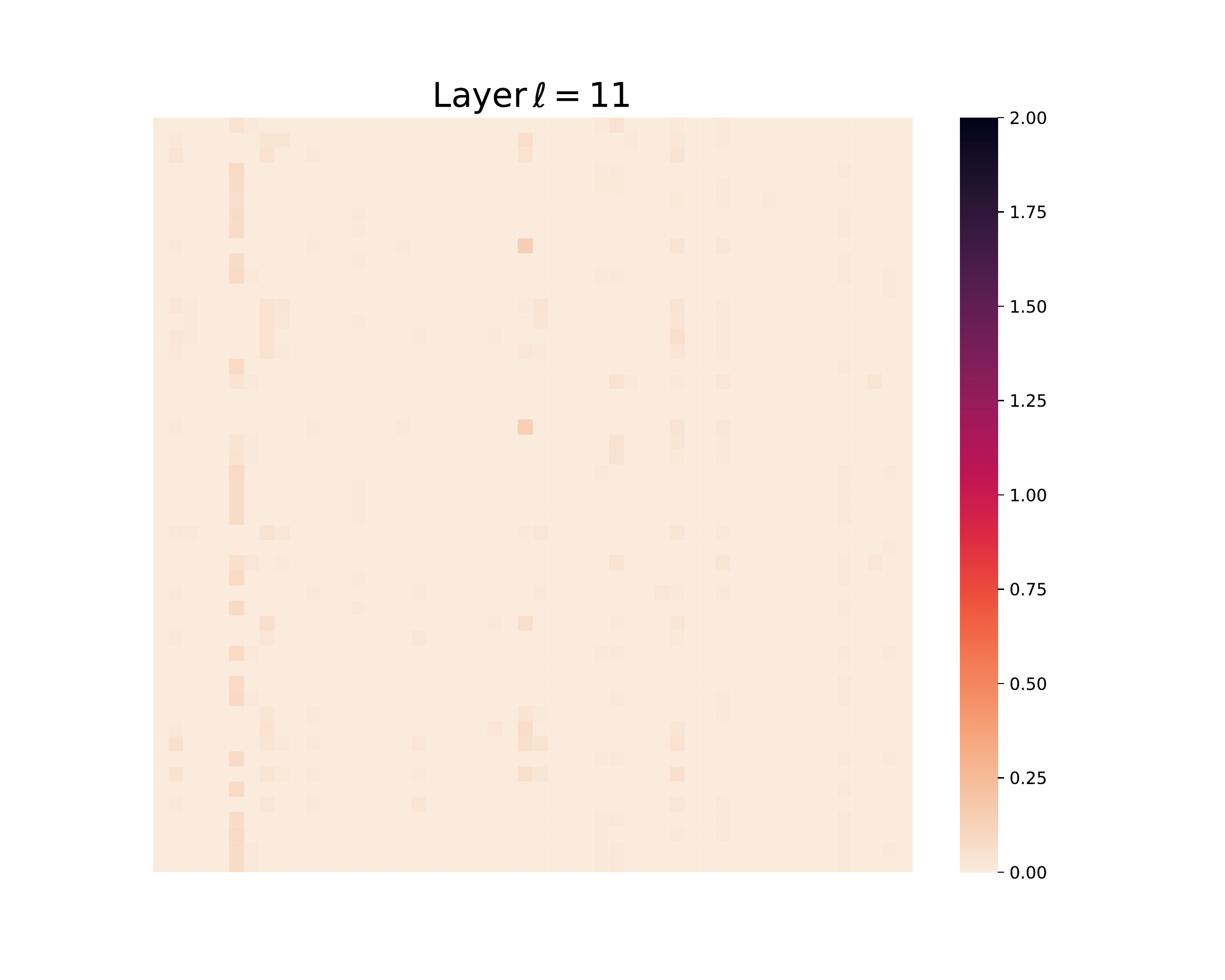}
     \end{subfigure}
     \begin{subfigure}[b]{0.22\textwidth}
         \centering
    \includegraphics[width=\textwidth]{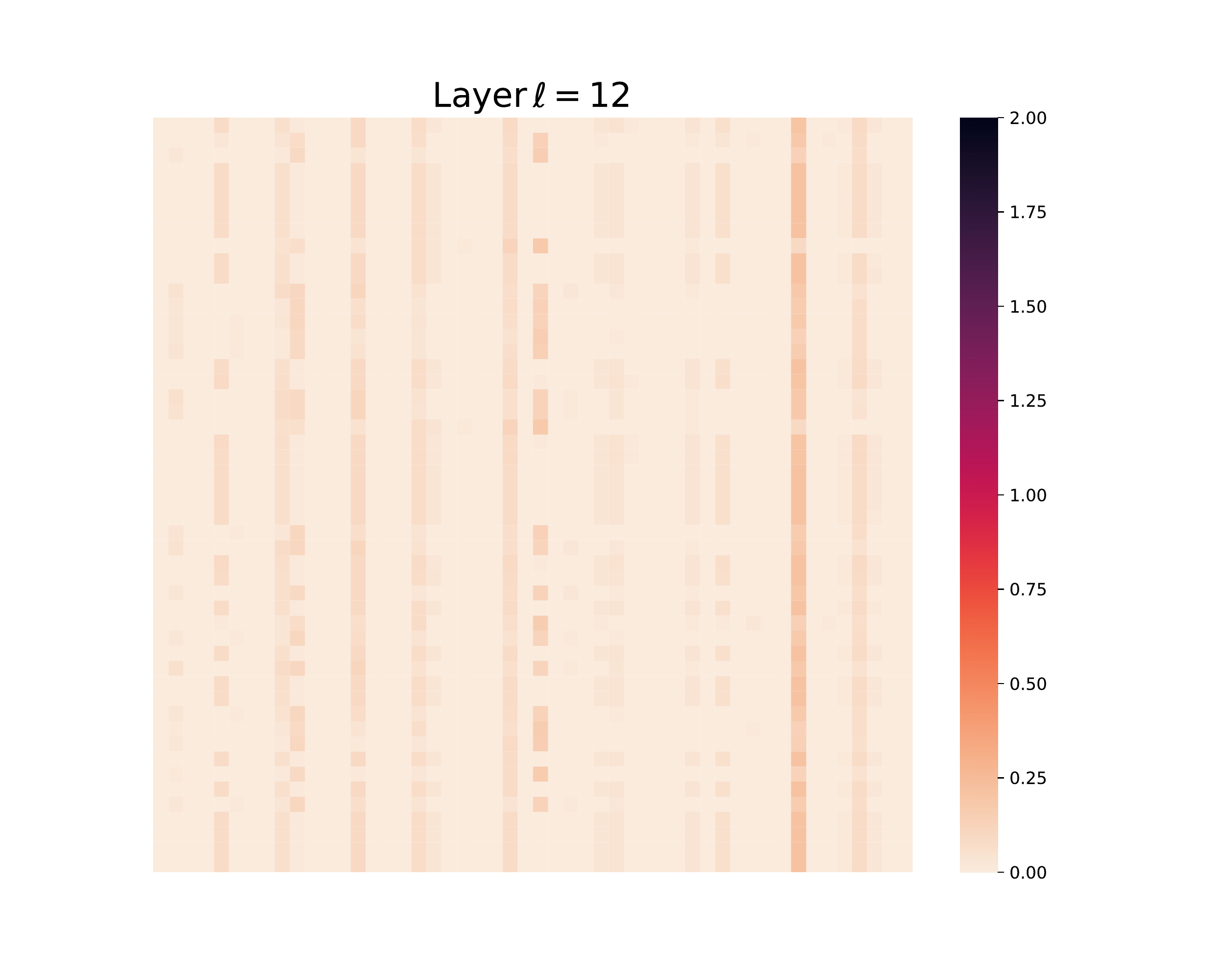}
     \end{subfigure}
        \caption{Visualizing layer-wise token $\Z^{\ell}$ representations at each layer $\ell$. To enhance the visual clarity, we randomly extract a 50$\times$50 sub-matrix from $\Z^{\ell}$ for display purposes. (\textit{Sample 2})}
        \label{fig:appendix-exp-ista-sparsity-heatmap-sample2}
        \vspace{-0.1in}
\end{figure}

\begin{figure}[ht]
     \centering
     \begin{subfigure}[b]{0.22\textwidth}
         \centering
    \includegraphics[width=\textwidth]{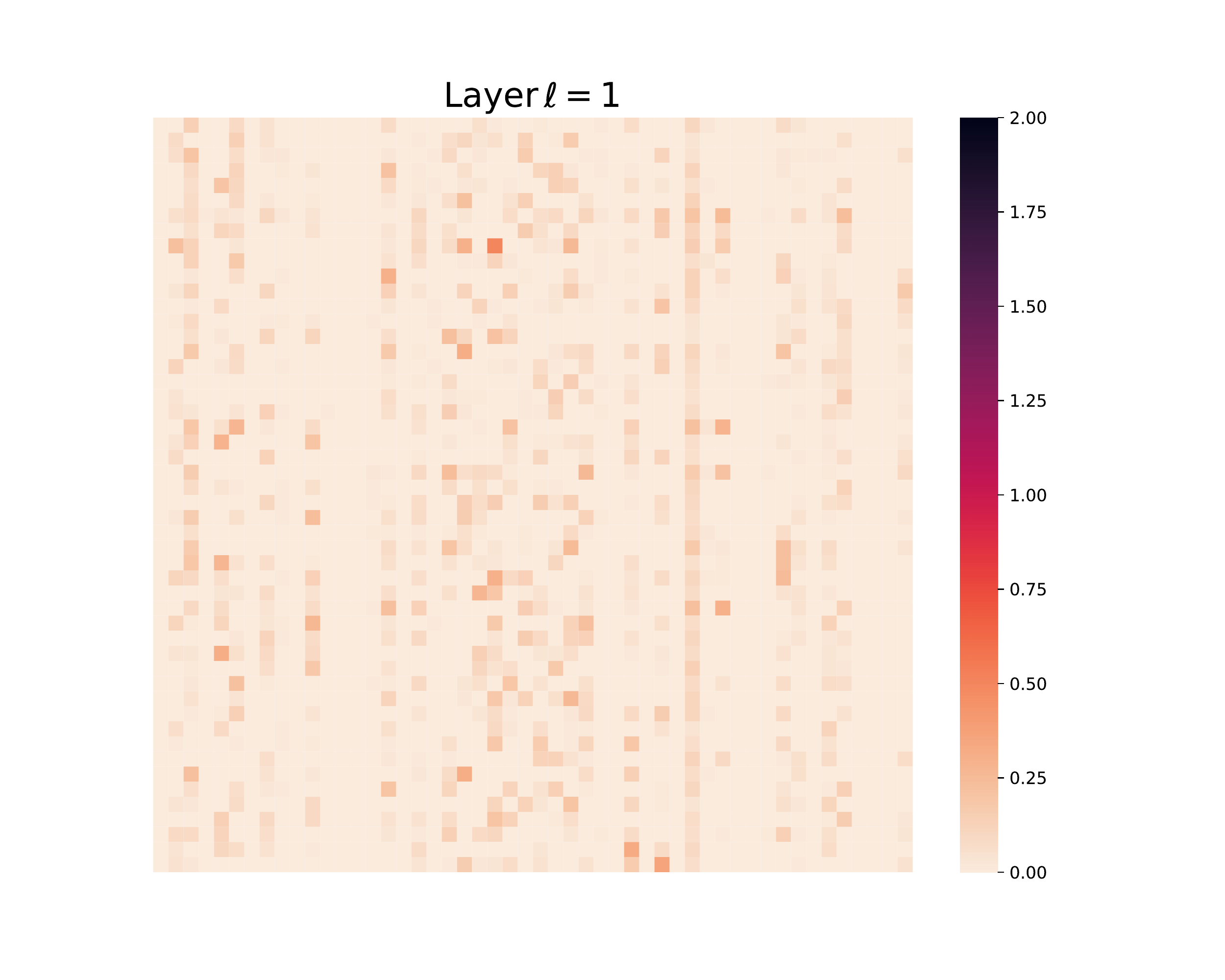}
     \end{subfigure}
     \begin{subfigure}[b]{0.22\textwidth}
         \centering
    \includegraphics[width=\textwidth]{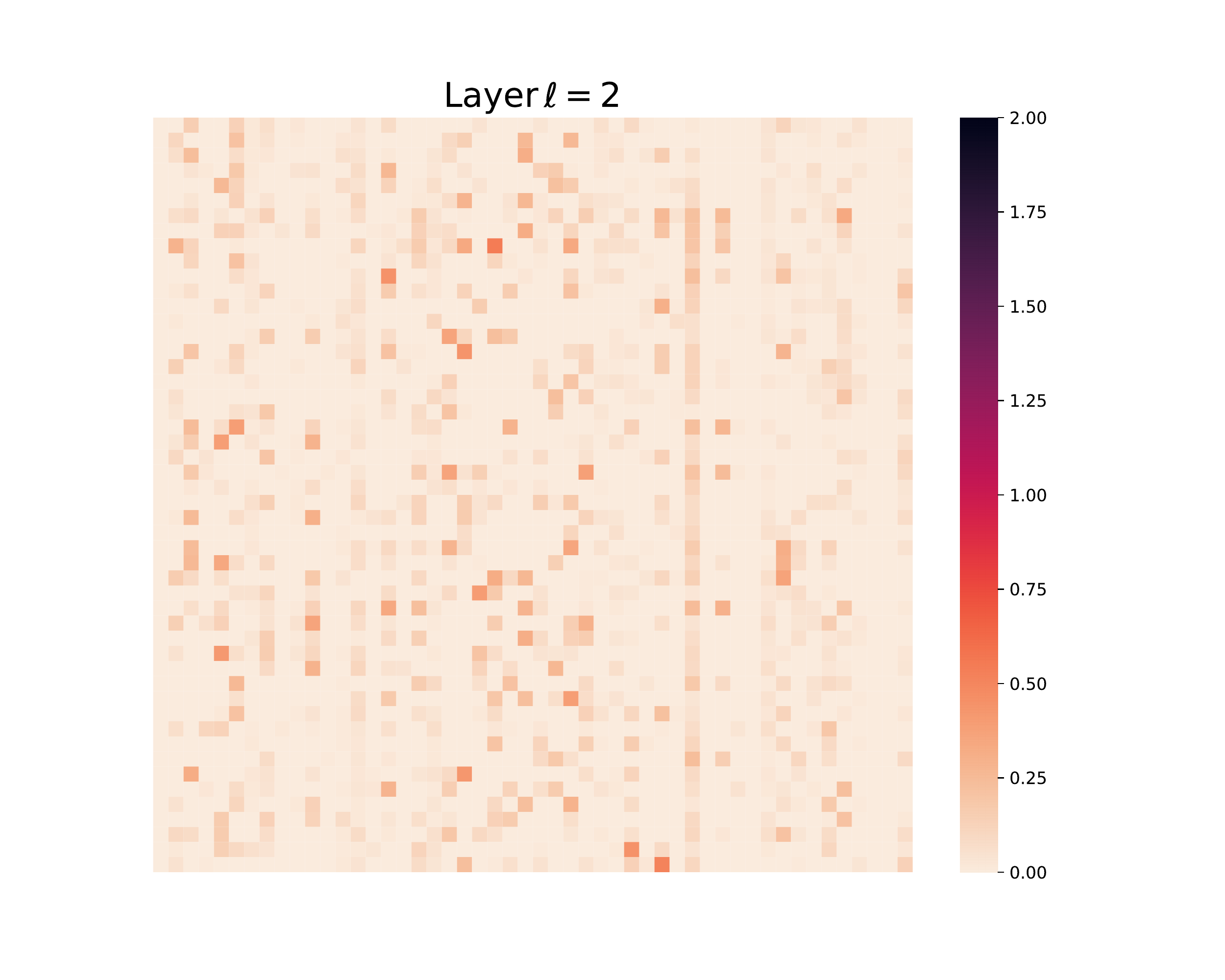}
     \end{subfigure}
     \begin{subfigure}[b]{0.22\textwidth}
         \centering
    \includegraphics[width=\textwidth]{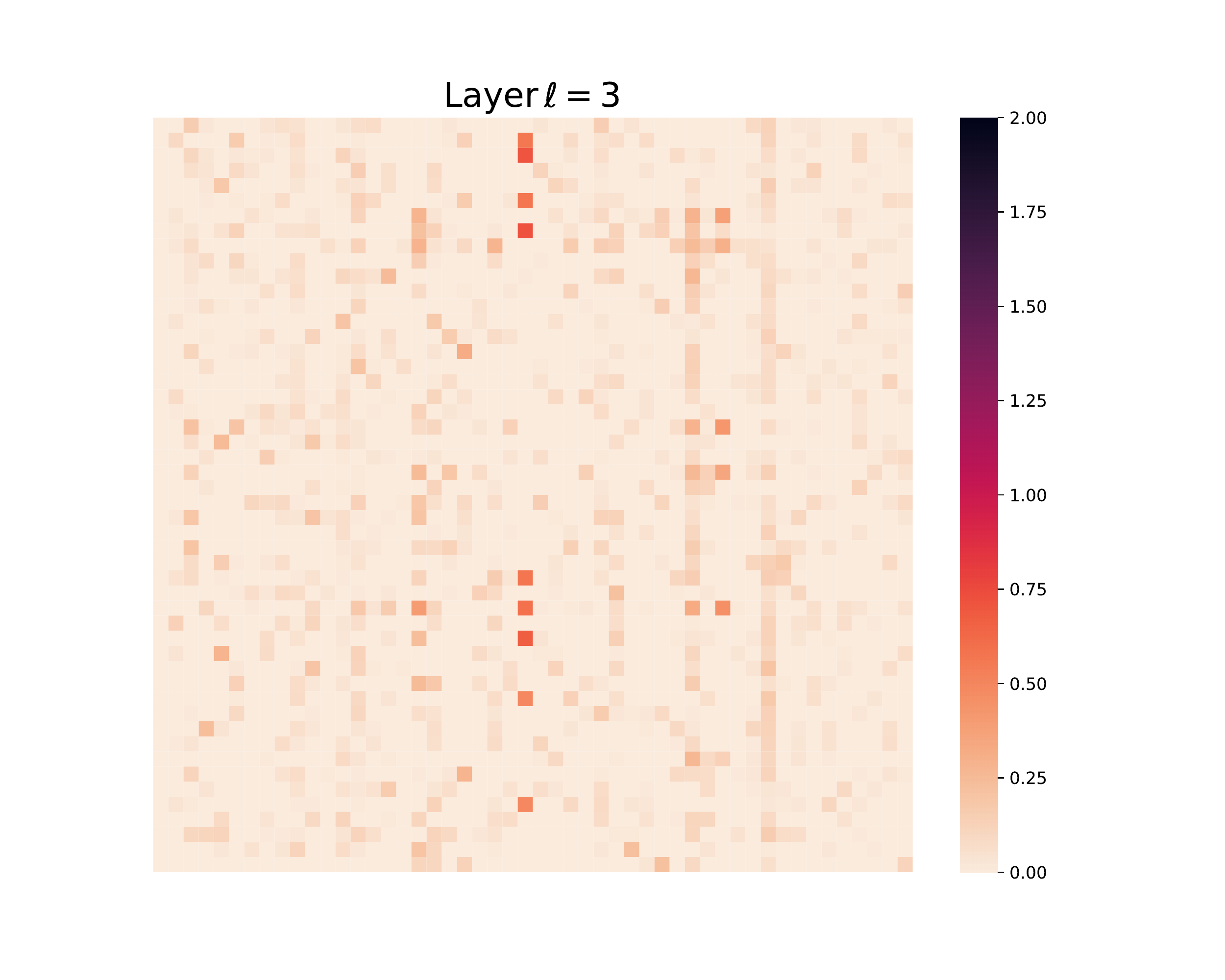}
     \end{subfigure}
     \begin{subfigure}[b]{0.22\textwidth}
         \centering
    \includegraphics[width=\textwidth]{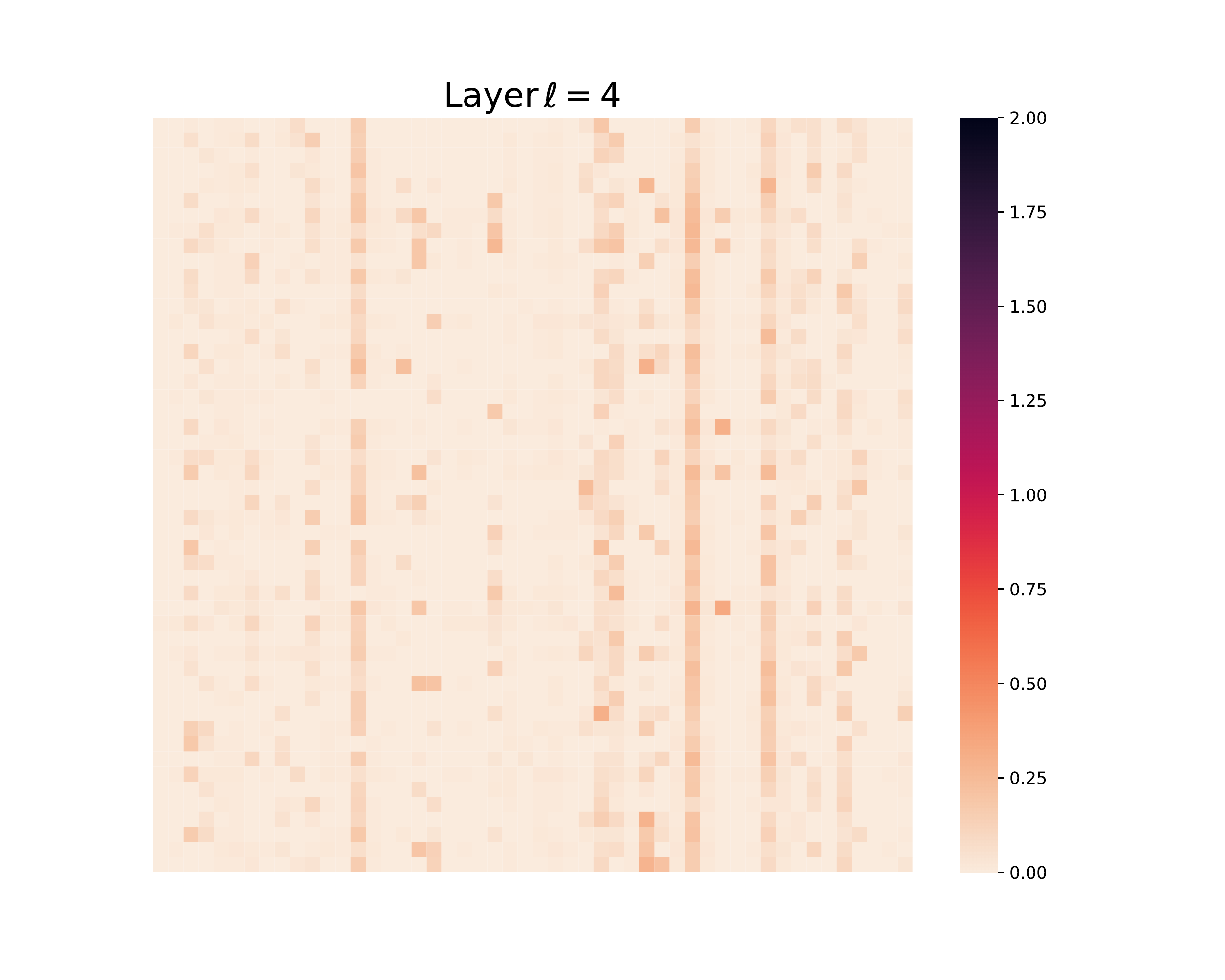}
     \end{subfigure}
     \begin{subfigure}[b]{0.22\textwidth}
         \centering
    \includegraphics[width=\textwidth]{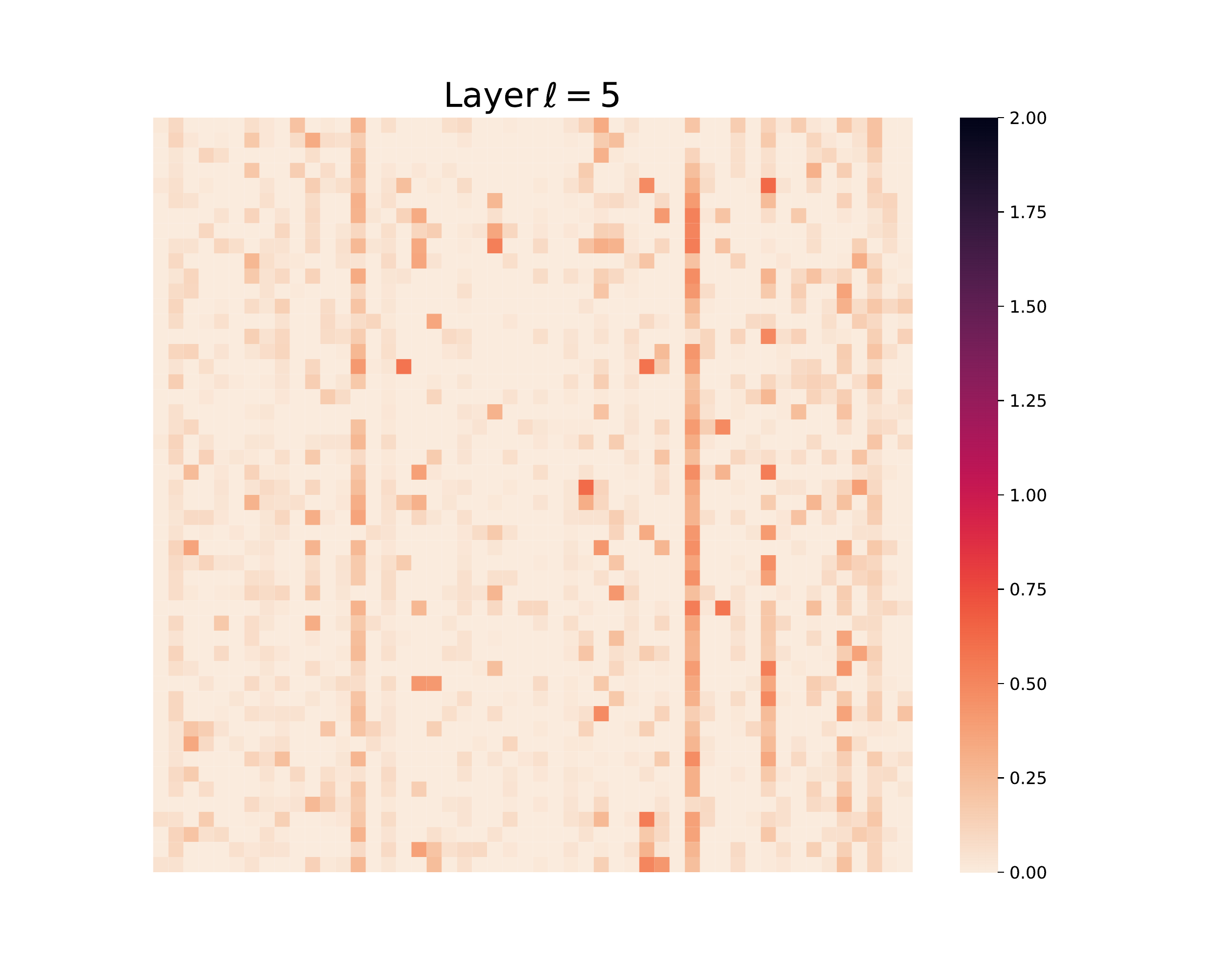}
     \end{subfigure}
     \begin{subfigure}[b]{0.22\textwidth}
         \centering
    \includegraphics[width=\textwidth]{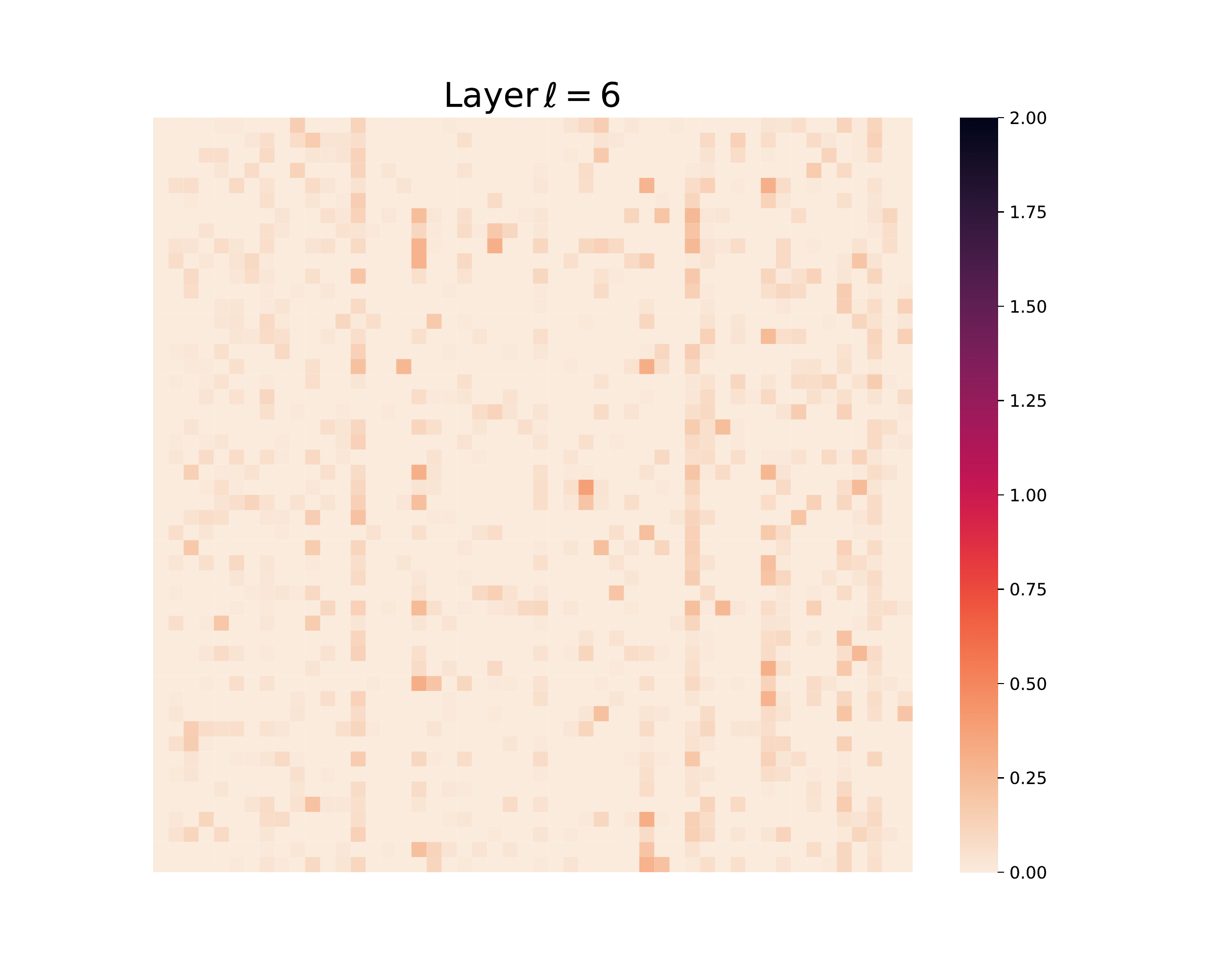}
     \end{subfigure}
     \begin{subfigure}[b]{0.22\textwidth}
         \centering
    \includegraphics[width=\textwidth]{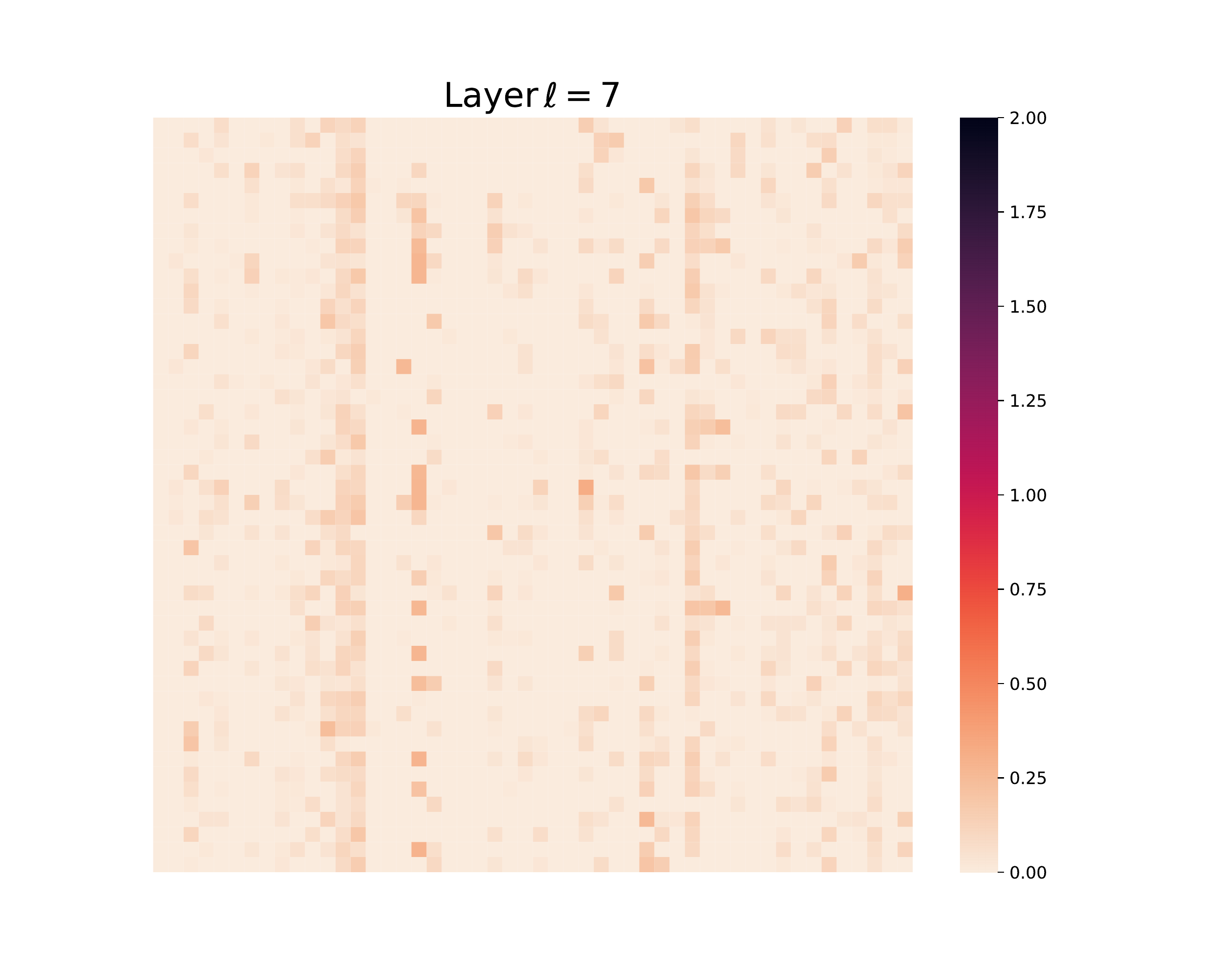}
     \end{subfigure}
     \begin{subfigure}[b]{0.22\textwidth}
         \centering
    \includegraphics[width=\textwidth]{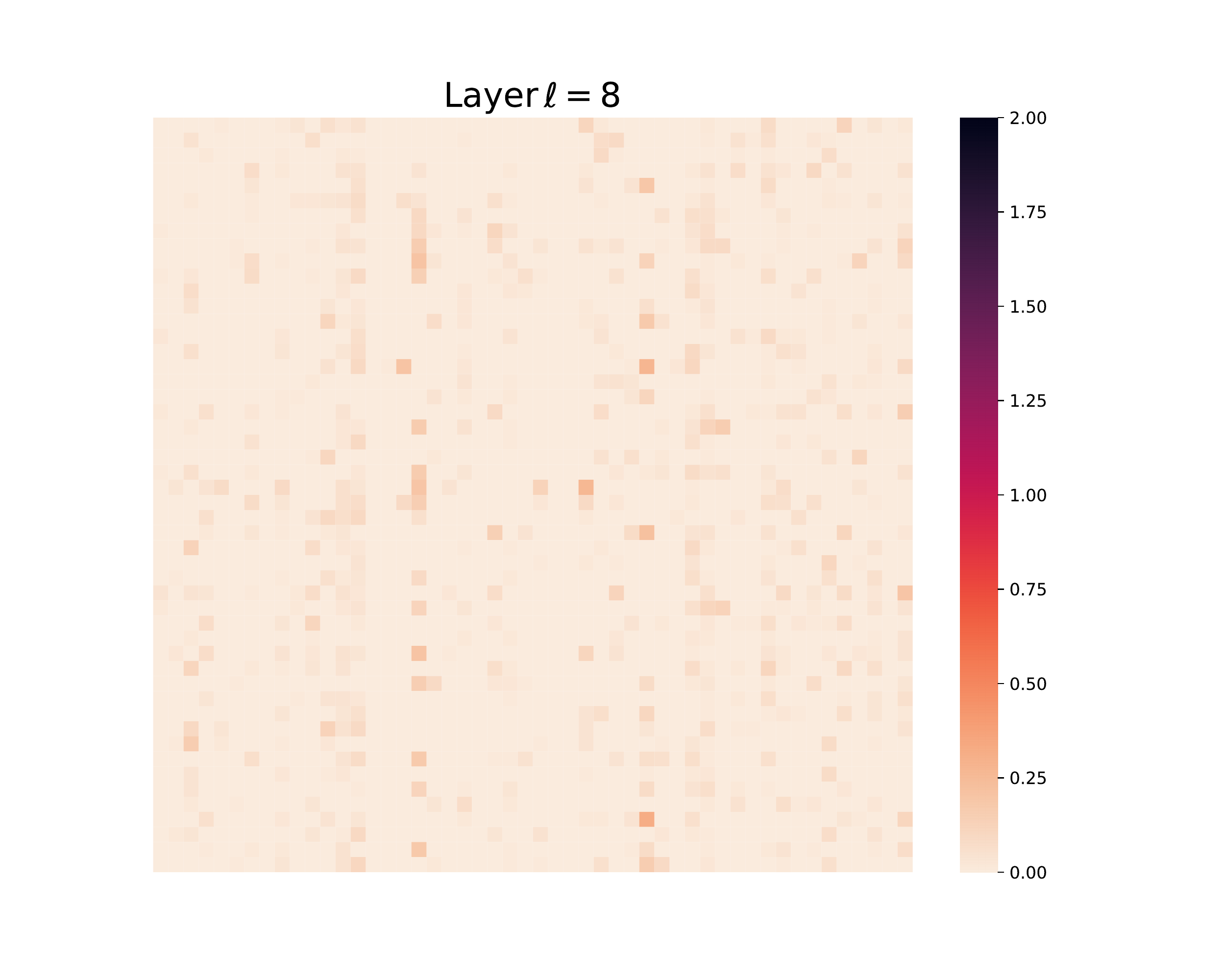}
     \end{subfigure}
     \begin{subfigure}[b]{0.22\textwidth}
         \centering
    \includegraphics[width=\textwidth]{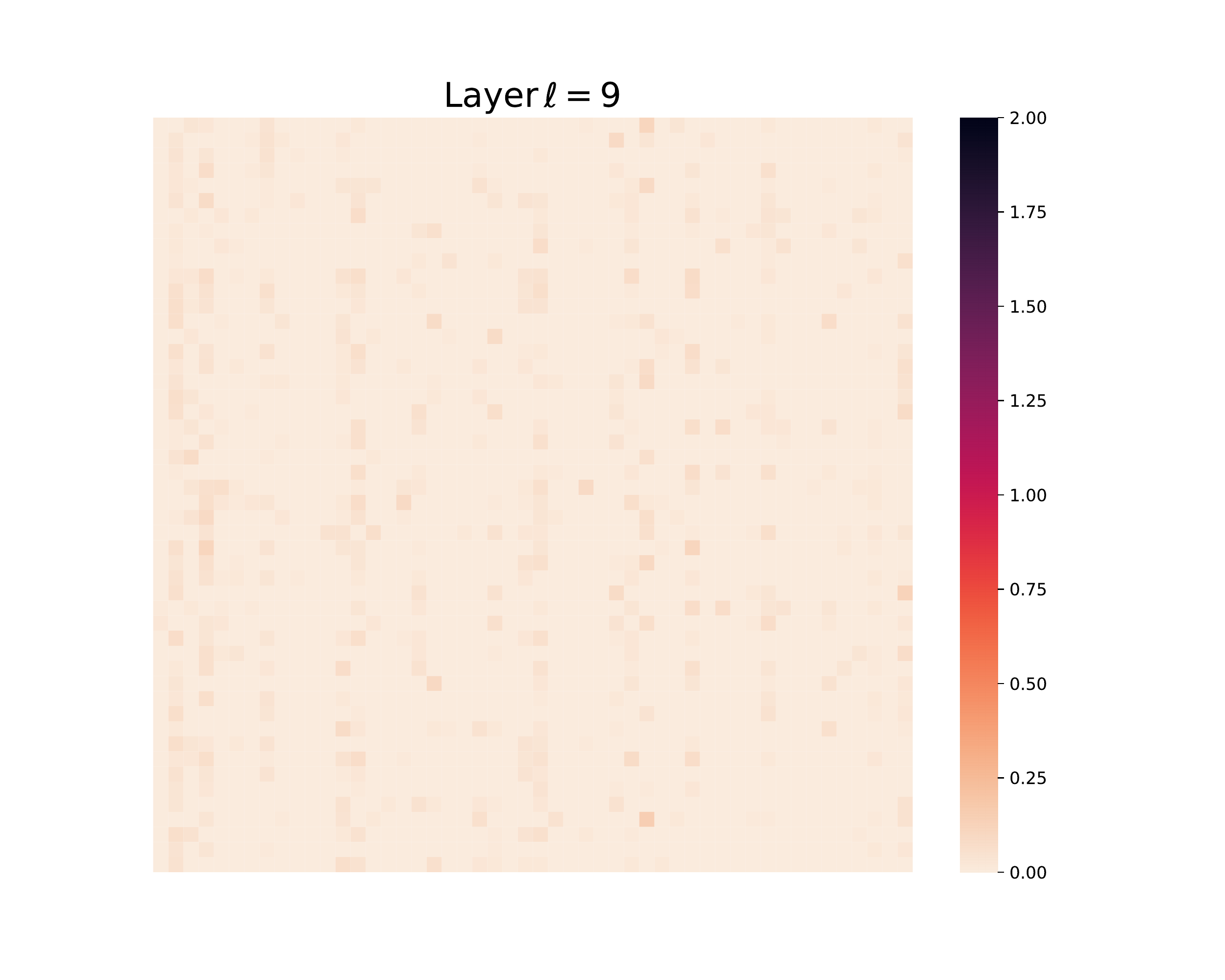}
     \end{subfigure}
     \begin{subfigure}[b]{0.22\textwidth}
         \centering
    \includegraphics[width=\textwidth]{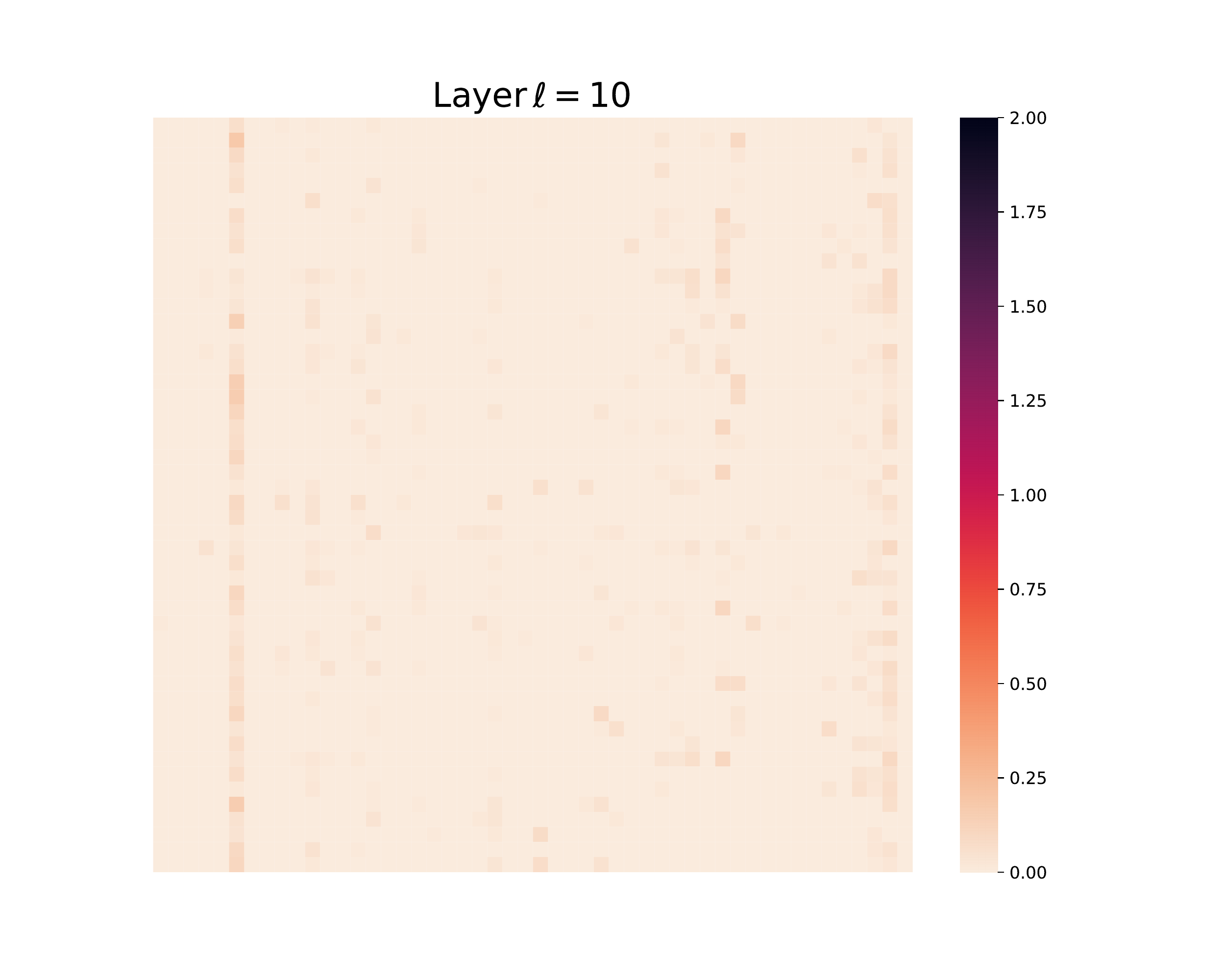}
     \end{subfigure}
     \begin{subfigure}[b]{0.22\textwidth}
         \centering
    \includegraphics[width=\textwidth]{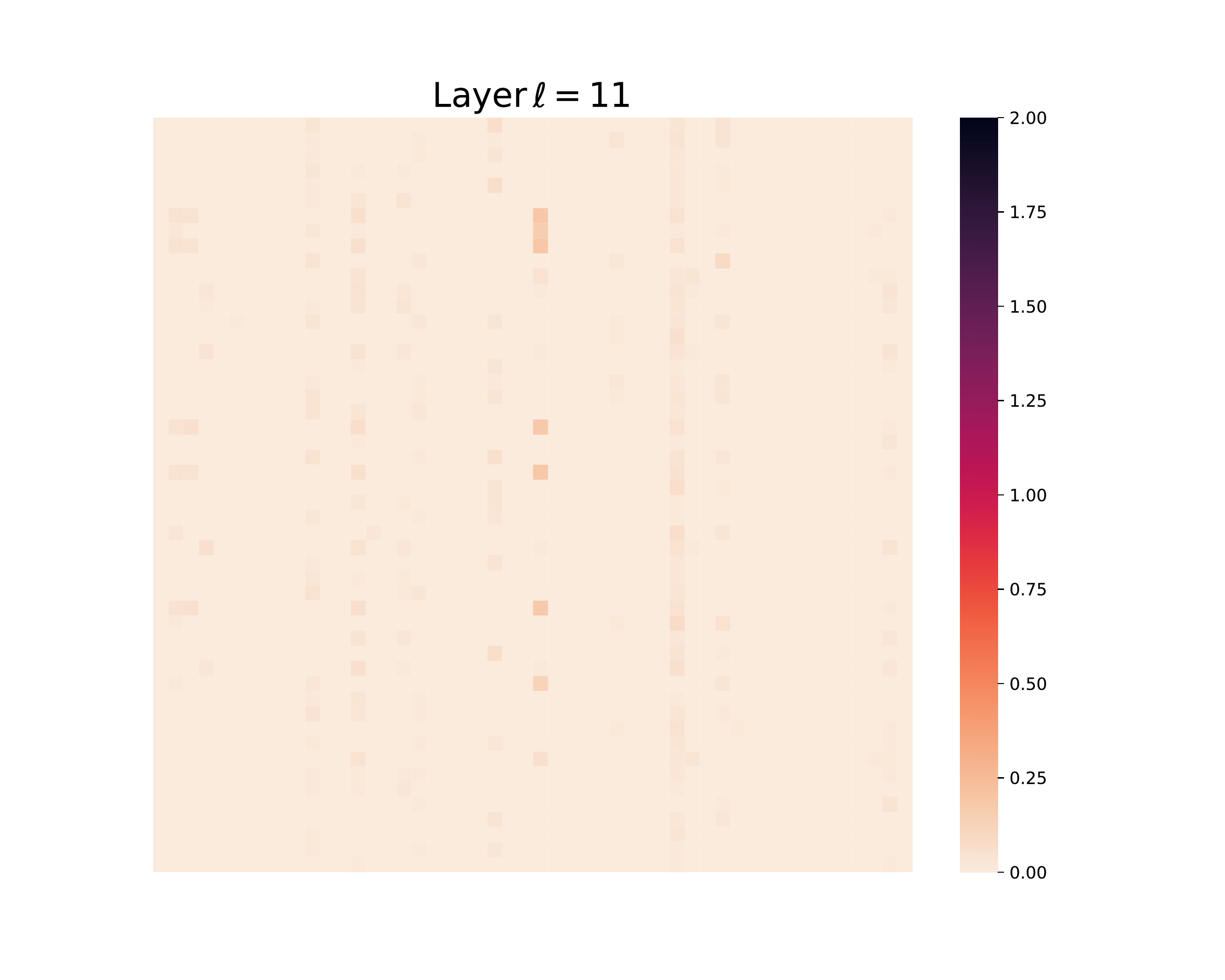}
     \end{subfigure}
     \begin{subfigure}[b]{0.22\textwidth}
         \centering
    \includegraphics[width=\textwidth]{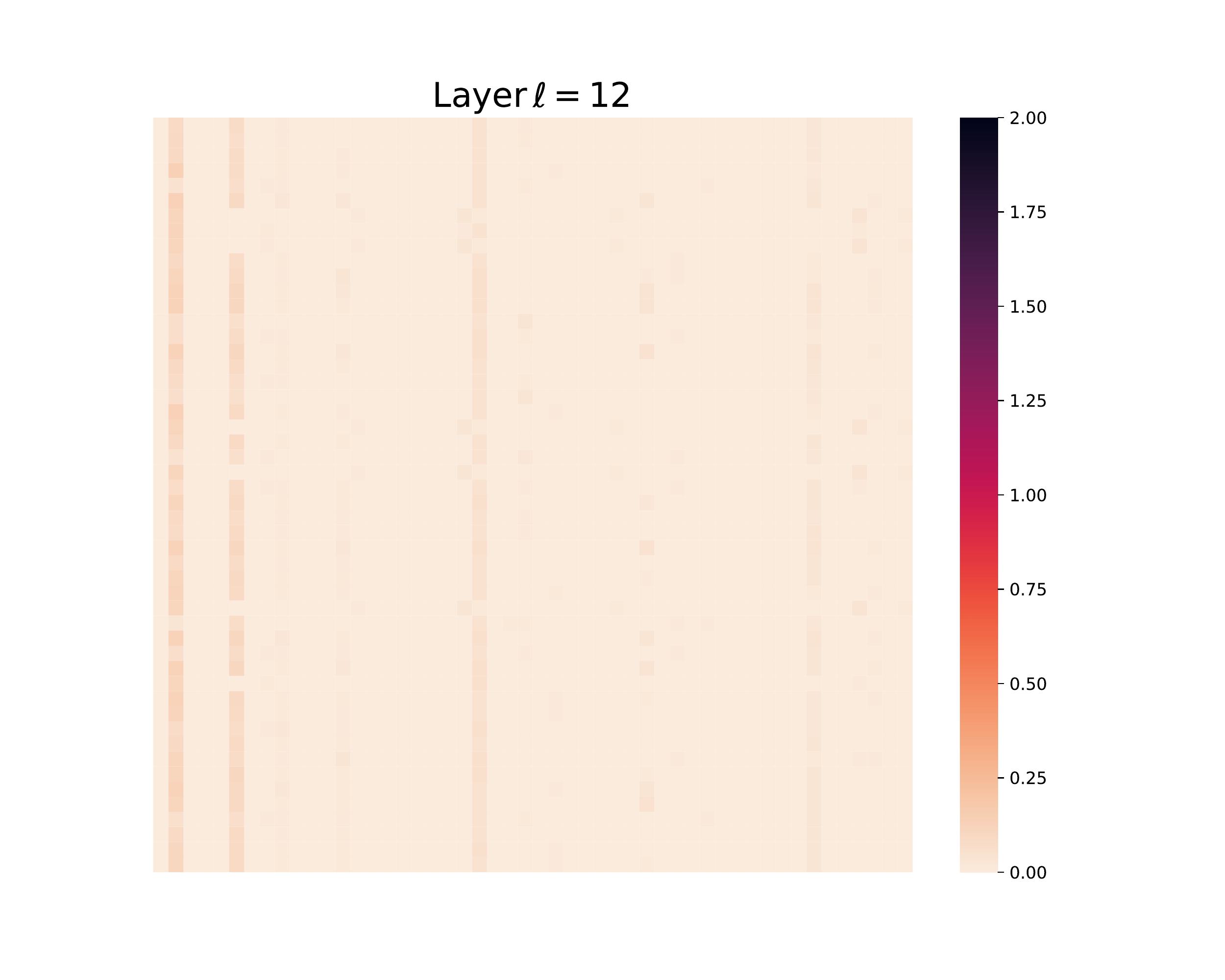}
     \end{subfigure}
        \caption{Visualizing layer-wise token $\Z^{\ell}$ representations at each layer $\ell$. To enhance the visual clarity, we randomly extract a 50$\times$50 sub-matrix from $\Z^{\ell}$ for display purposes. (\textit{Sample 3})}
        \label{fig:appendix-exp-ista-sparsity-heatmap-sample3}
        \vspace{-0.1in}
\end{figure}

\begin{figure}[ht]
     \centering
     \begin{subfigure}[b]{0.22\textwidth}
         \centering
    \includegraphics[width=\textwidth]{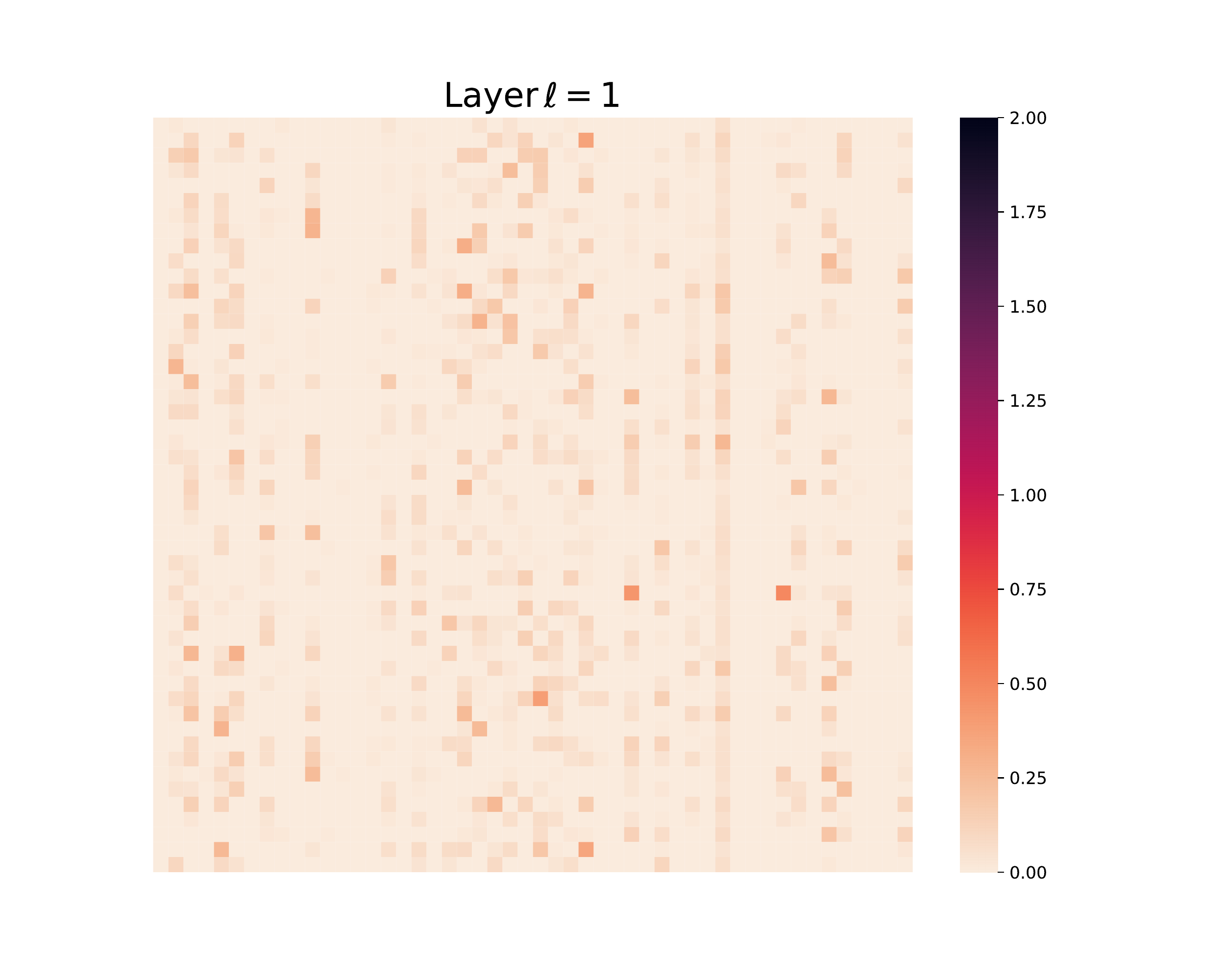}
     \end{subfigure}
     \begin{subfigure}[b]{0.22\textwidth}
         \centering
    \includegraphics[width=\textwidth]{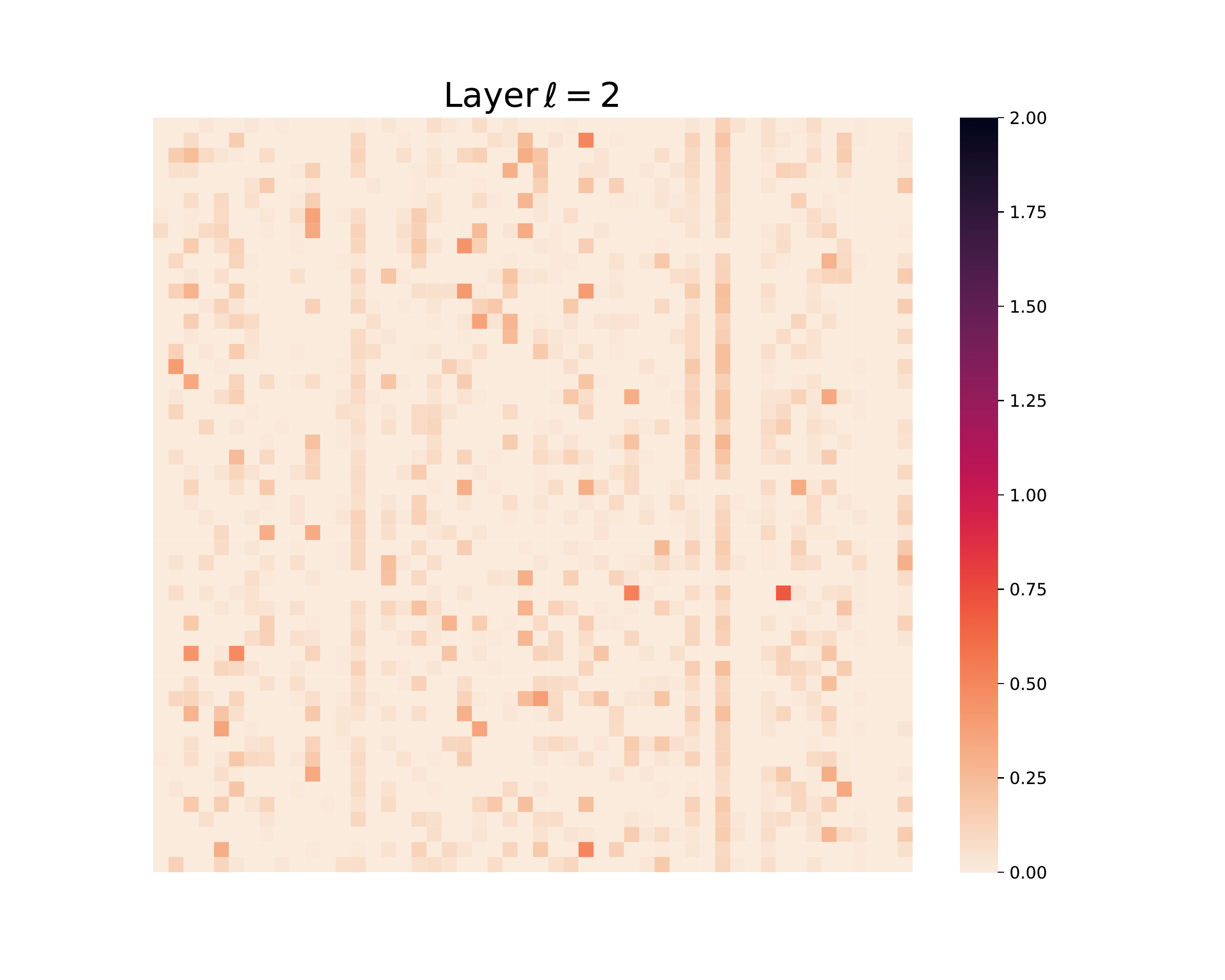}
     \end{subfigure}
     \begin{subfigure}[b]{0.22\textwidth}
         \centering
    \includegraphics[width=\textwidth]{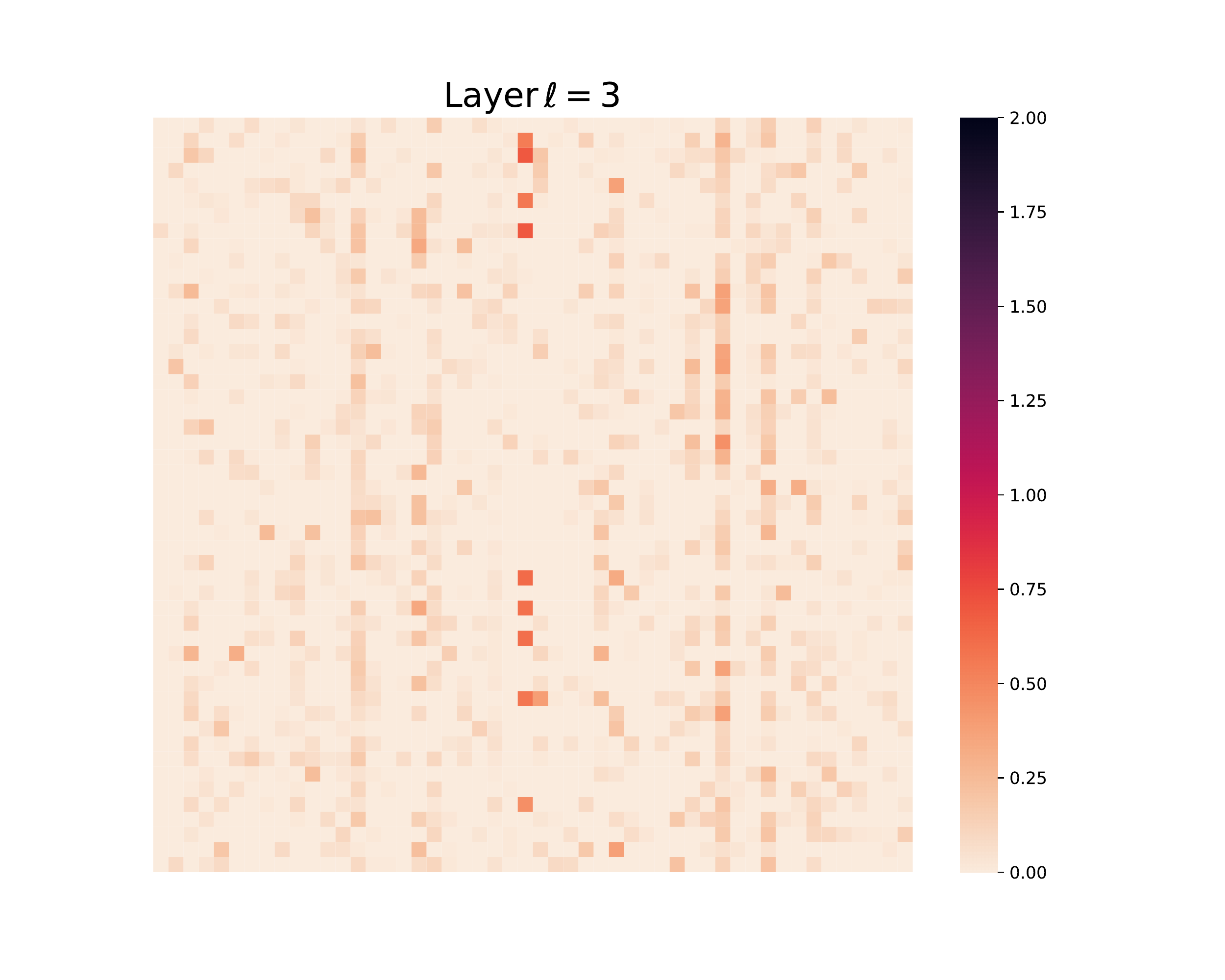}
     \end{subfigure}
     \begin{subfigure}[b]{0.22\textwidth}
         \centering
    \includegraphics[width=\textwidth]{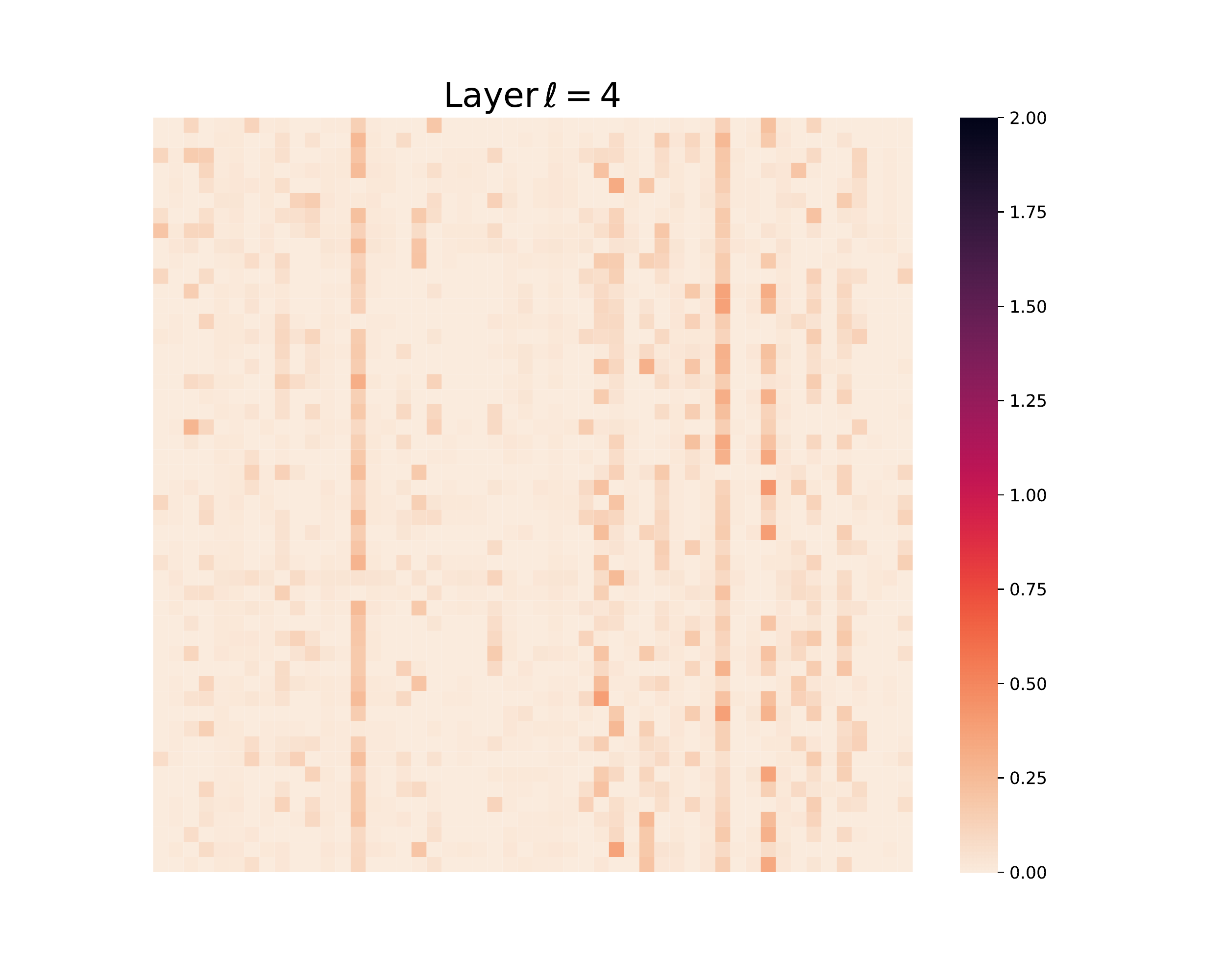}
     \end{subfigure}
     \begin{subfigure}[b]{0.22\textwidth}
         \centering
    \includegraphics[width=\textwidth]{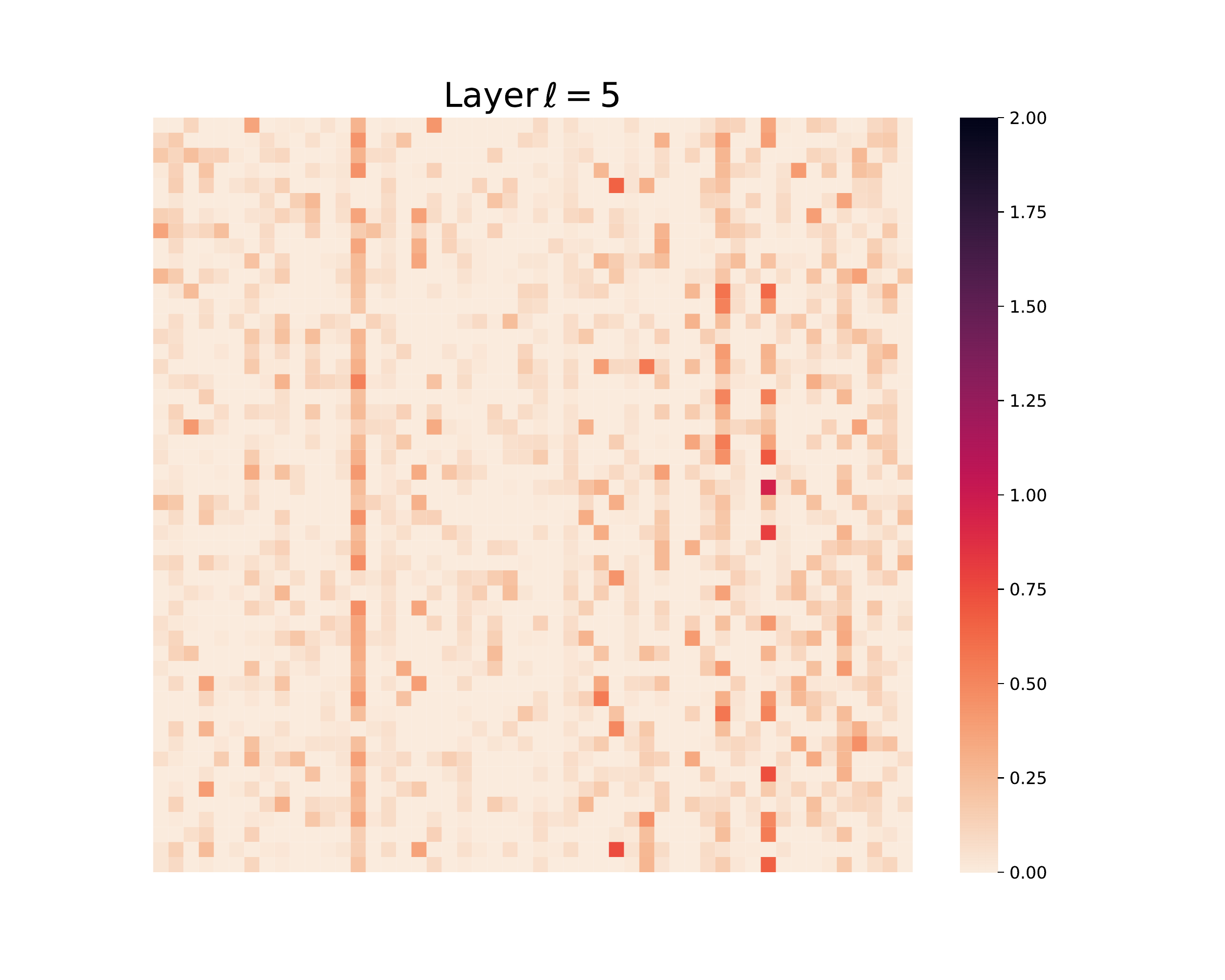}
     \end{subfigure}
     \begin{subfigure}[b]{0.22\textwidth}
         \centering
    \includegraphics[width=\textwidth]{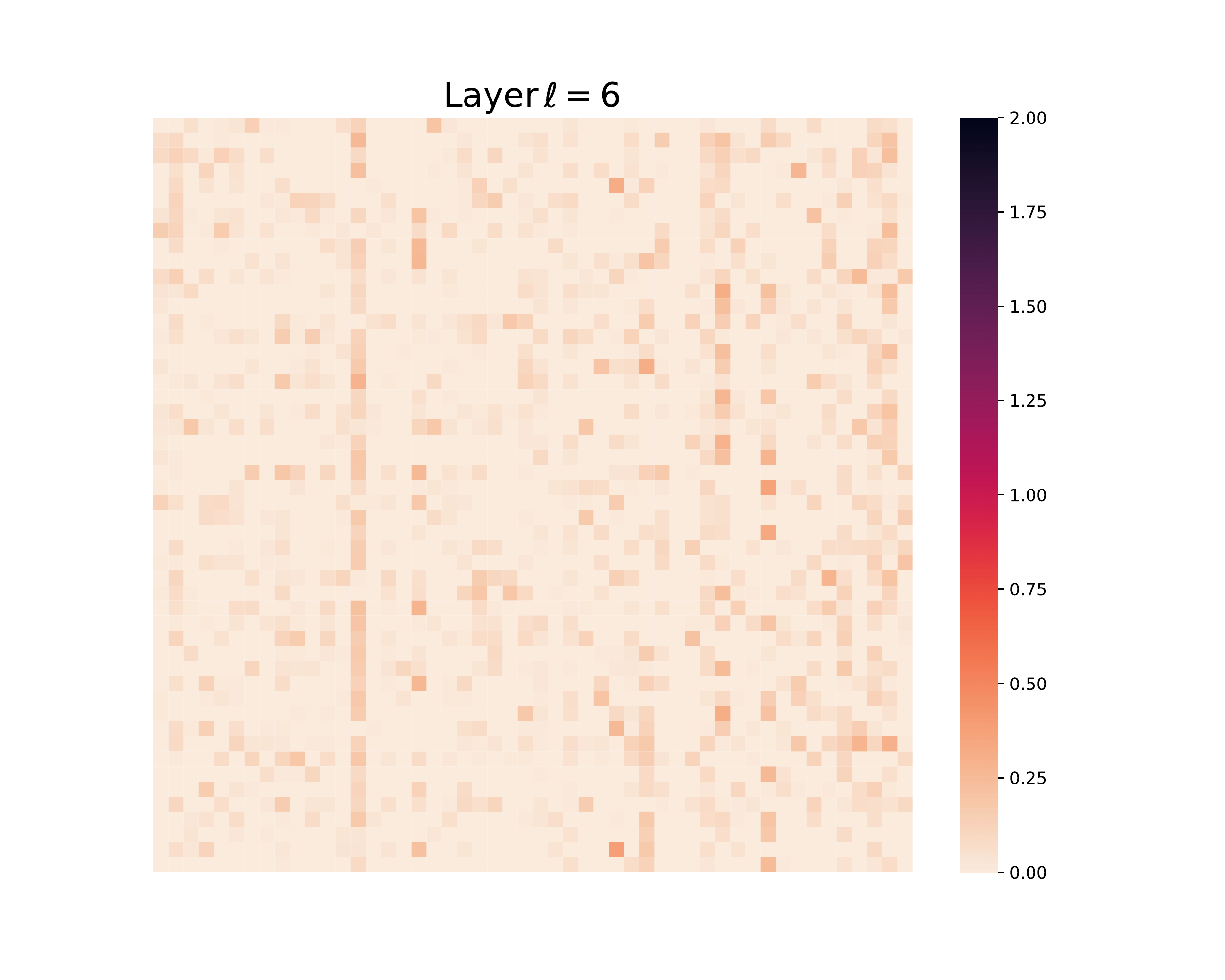}
     \end{subfigure}
     \begin{subfigure}[b]{0.22\textwidth}
         \centering
    \includegraphics[width=\textwidth]{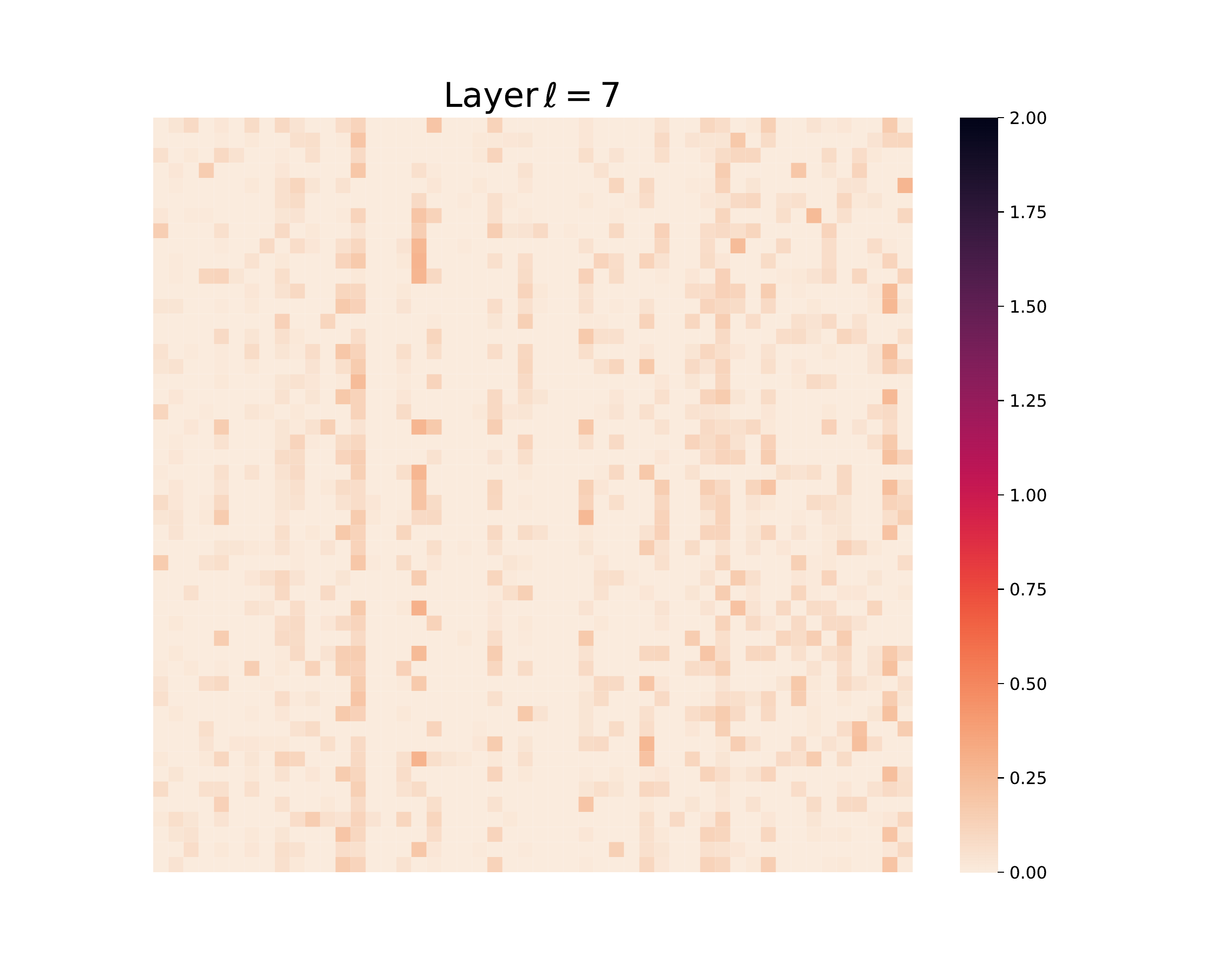}
     \end{subfigure}
     \begin{subfigure}[b]{0.22\textwidth}
         \centering
    \includegraphics[width=\textwidth]{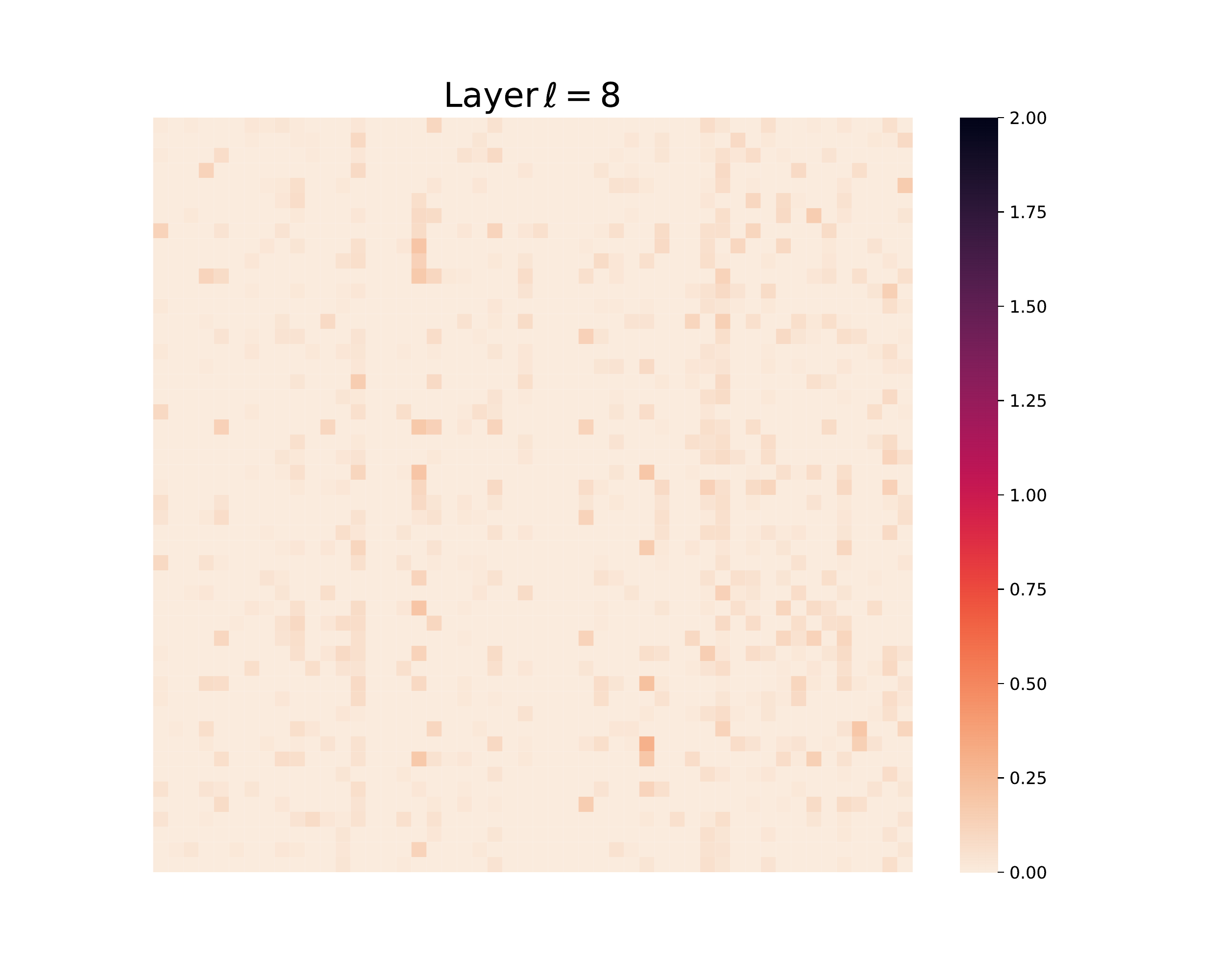}
     \end{subfigure}
     \begin{subfigure}[b]{0.22\textwidth}
         \centering
    \includegraphics[width=\textwidth]{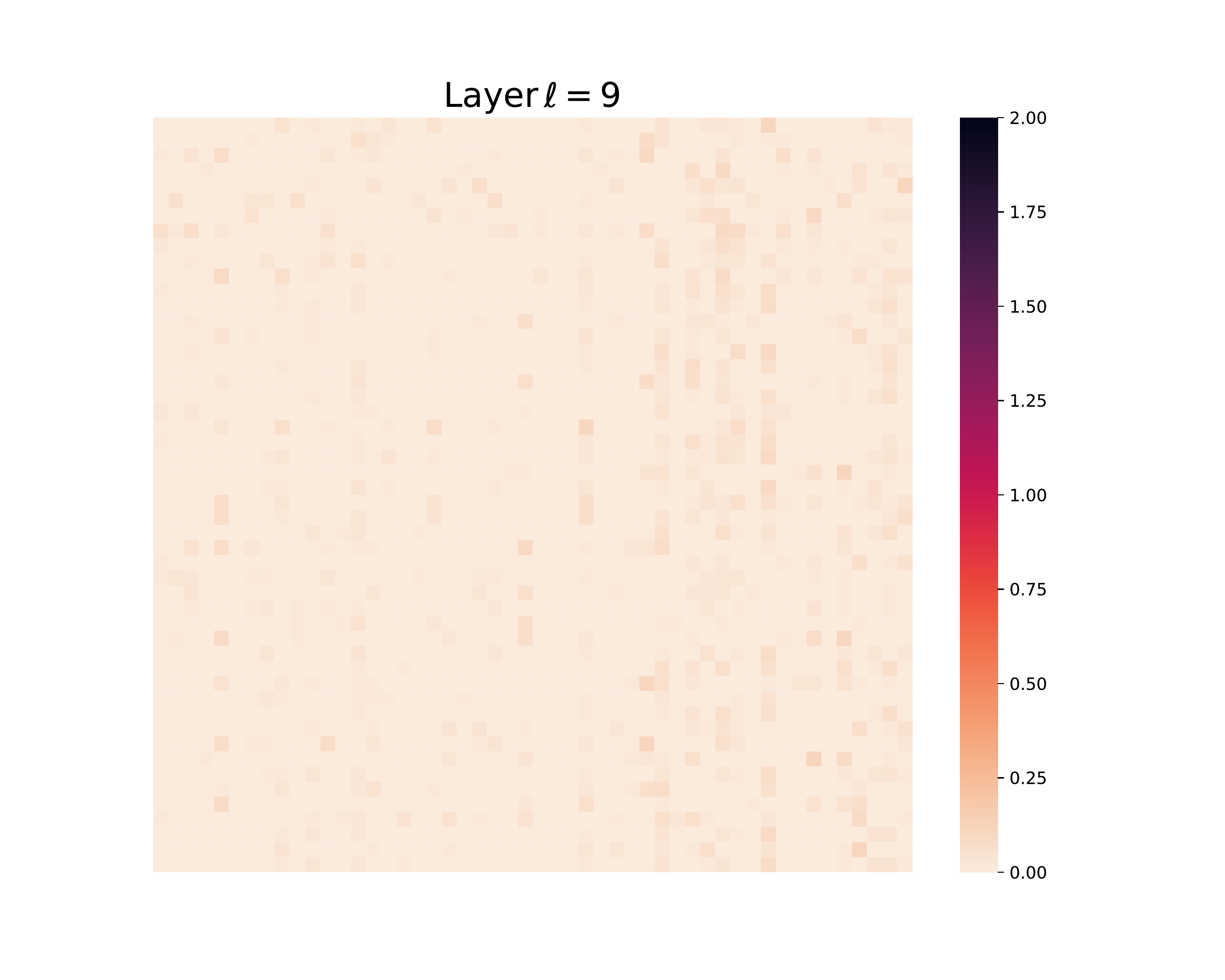}
     \end{subfigure}
     \begin{subfigure}[b]{0.22\textwidth}
         \centering
    \includegraphics[width=\textwidth]{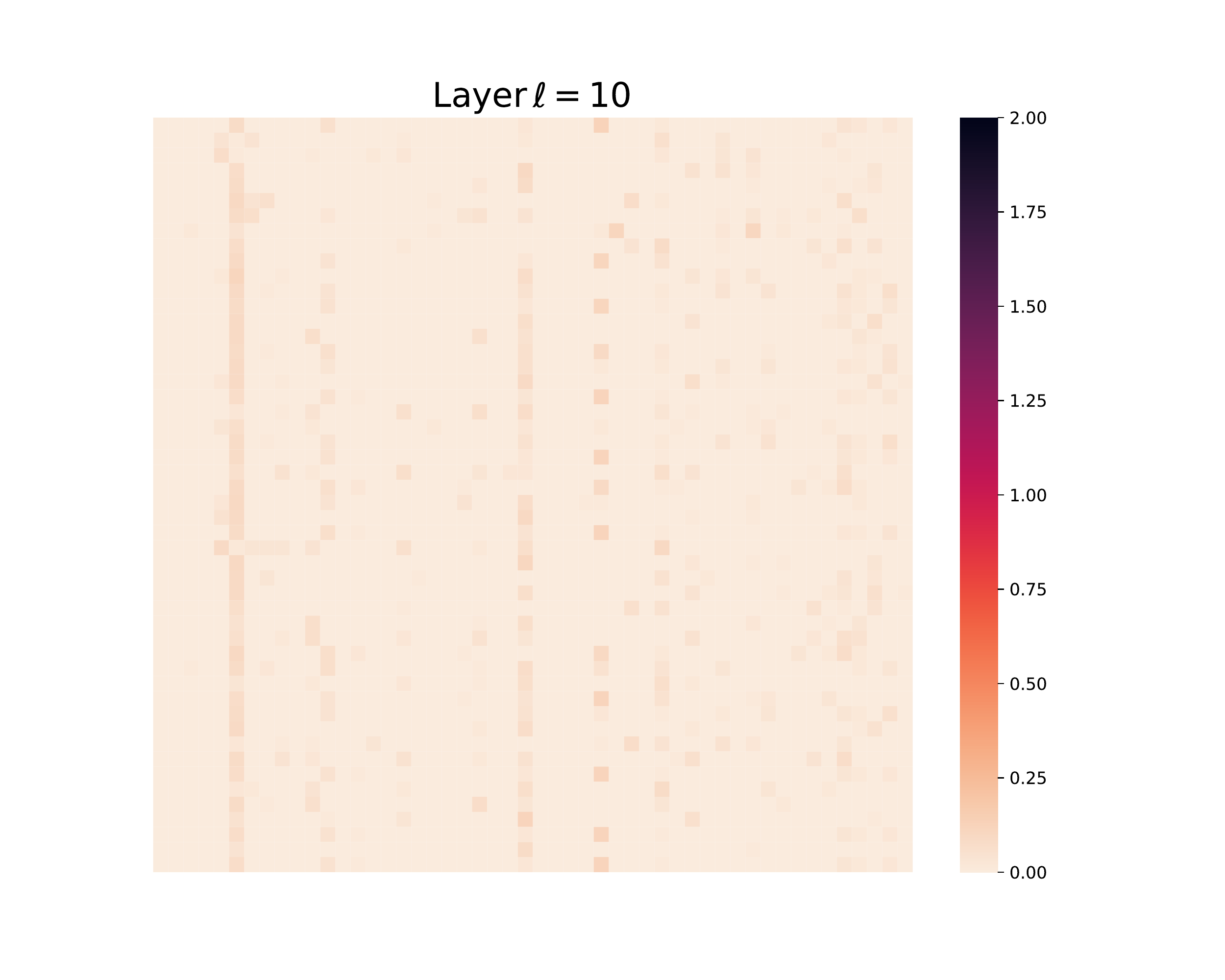}
     \end{subfigure}
     \begin{subfigure}[b]{0.22\textwidth}
         \centering
    \includegraphics[width=\textwidth]{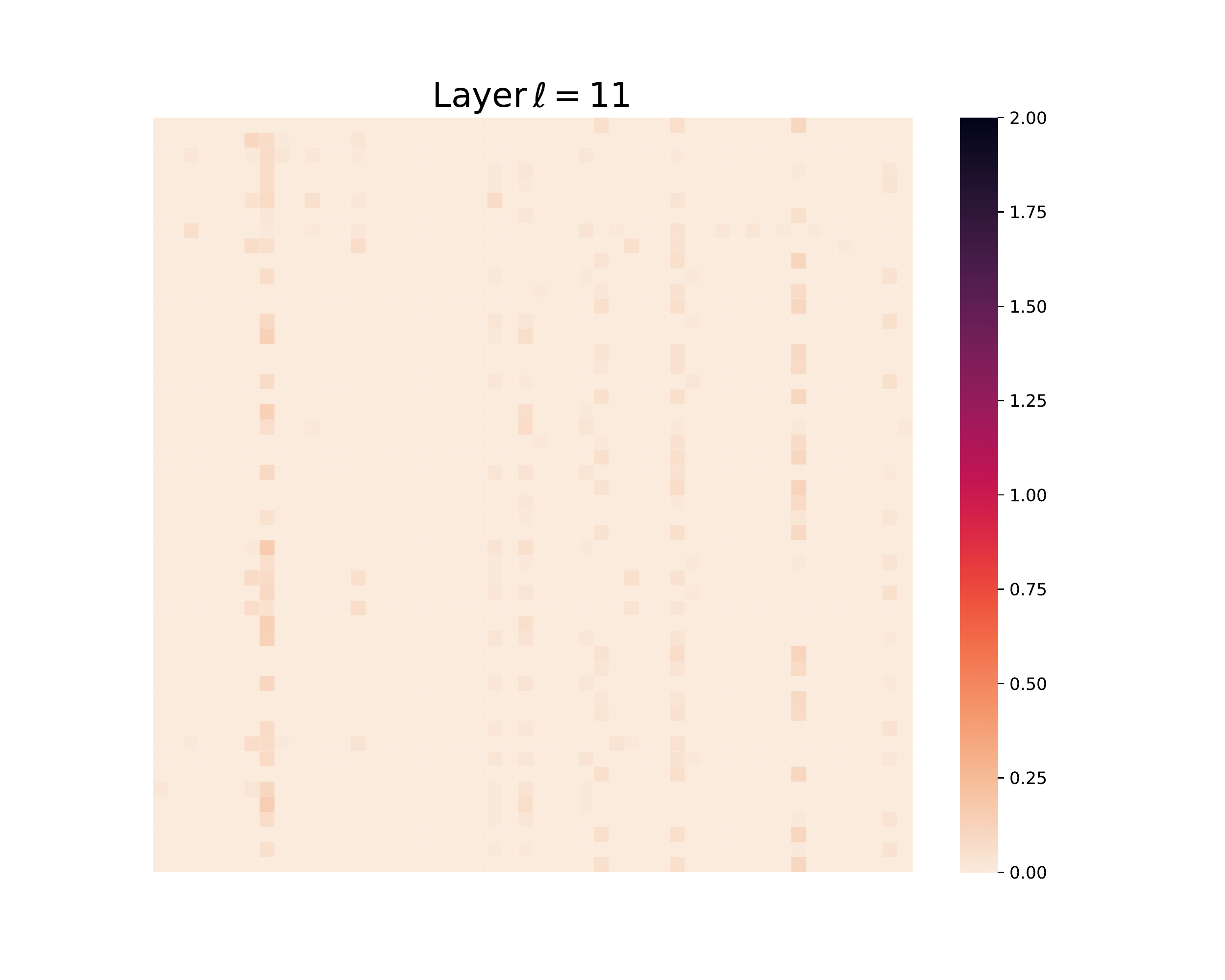}
     \end{subfigure}
     \begin{subfigure}[b]{0.22\textwidth}
         \centering
    \includegraphics[width=\textwidth]{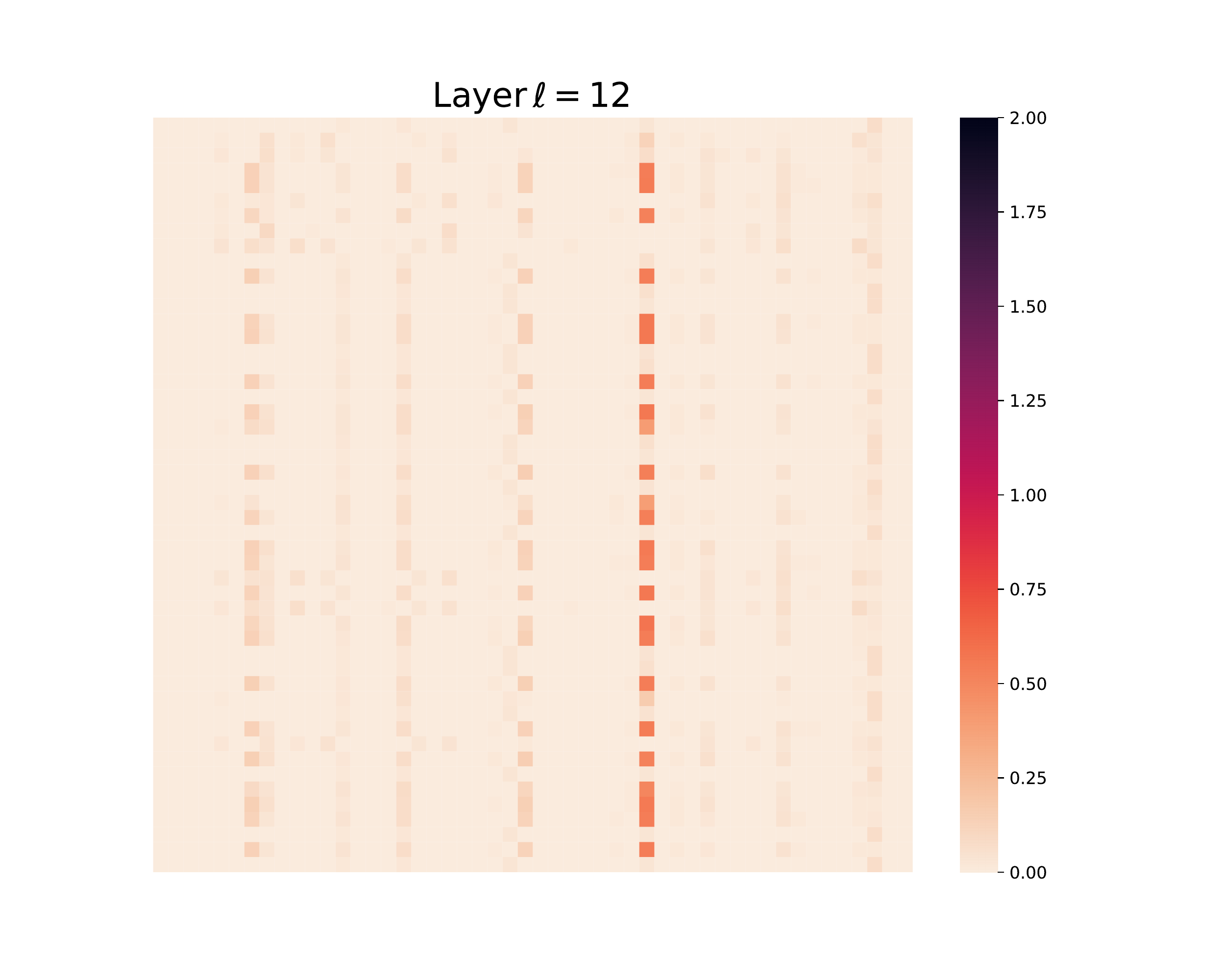}
     \end{subfigure}
        \caption{Visualizing layer-wise token $\Z^{\ell}$ representations at each layer $\ell$. To enhance the visual clarity, we randomly extract a 50$\times$50 sub-matrix from $\Z^{\ell}$ for display purposes. (\textit{Sample 4})}
        \label{fig:appendix-exp-ista-sparsity-heatmap-sample4}
        \vspace{-0.1in}
\end{figure}

\clearpage
\subsection{\ours{} Ablation}

\paragraph{Hyperparameters of \ours{}.} 
In \Cref{tab:ablation-parameters}, we present evaluation of \ours{} trained with various parameters.
More specifically, we investigate the effect of number of epochs, weight decay, learning rate, step size $(\eta)$ and the regularization term $(\lambda)$ in \texttt{ISTA} block. 
As shown in \Cref{tab:ablation-parameters}, \ours{} demonstrates consistently satisfactory performance across a diverse range of hyperparameters.

\begin{table*}[ht]
\centering
\caption{\small Top 1 accuracy of \ours{} on various datasets with different architecture design variants 
when trained on ImageNet. }
\label{tab:ablation-parameters}
\vspace{-1mm}
\small
    \setlength{\tabcolsep}{13.6pt}
\resizebox{0.98\textwidth}{!}{%
\begin{tabular}{@{}l|ccc|cc|cc@{}}
\toprule
\textbf{Model} & \texttt{epoch} & \texttt{weight decay}  &   \texttt{lr} &   $\eta$ (\texttt{ISTA}) &   $\lambda$ (\texttt{ISTA}) & ImageNet  \\ 
\midrule
\midrule
 \ours{-B} & 150 (default) & 0.5 (default) & $2.4\times 10^{-4}$ & 0.1 & 0.1 & 70.8 \\
 \midrule
 \midrule
 \ours{-B} & 150 & 0.5 & $2.4\times 10^{-4}$ & \textit{\color{gray} 0.02} & 0.1 & 70.7 \\
 \midrule
 \ours{-B} & 150 & 0.5 & $2.4\times 10^{-4}$ & \textit{\color{gray}0.5} & 0.1 & 66.7 \\
 \midrule
 \ours{-B} & 150 & 0.5 & $2.4\times 10^{-4}$ & 0.1 & \textit{\color{gray}0.02} & 70.8 \\
 \midrule
 \ours{-B} & 150 & 0.5 & $2.4\times 10^{-4}$ & 0.1 & \textit{\color{gray}0.5} & 70.5 \\
 \midrule
 \midrule
 \ours{-B} & \textit{\color{gray}90} & 0.5 & $2.4\times 10^{-4}$ & 0.1 & 0.1 & 69.5 \\
 \midrule
 \ours{-B} & \textit{\color{gray}300} & 0.5 & $2.4\times 10^{-4}$ & 0.1 & 0.1 & 70.9 \\
 \midrule
 \ours{-B} & 150 & \textit{\color{gray}1.0} & $2.4\times 10^{-4}$ & 0.1 & 0.1 & 70.3 \\
 \midrule
 \ours{-B} & 150 & \textit{\color{gray}0.05} & $2.4\times 10^{-4}$ & 0.1 & 0.1 & 70.2 \\
 \midrule
 \ours{-B} & 150 & 0.5 & \textit{\color{gray}$4.8\times 10^{-4}$} & 0.1 & 0.1 & 70.2 \\
 \midrule
 \ours{-B} & 150 & 0.5 & \textit{\color{gray}$1.2\times 10^{-4}$} & 0.1 & 0.1 & 70.3 \\
 \bottomrule
\end{tabular}%
}
\vspace{-0.1in}
\end{table*}

\subsection{Exploring Architecture Variants}\label{subsec:appendix-arch-variants}

In this section, we explore the two following alternative architectures. One architecture involves a modification to the attention mechanism, while the other involves a modification to the sparsification mechanism. Again, we re-emphasize that these choices, although principled, are entirely modular and the choices we make here still lead to very simple architectures. A more sophisticated analysis may lead to different, more complicated architectures that perform better in practice. The architectures we experiment with are:
\begin{itemize}
    \item Compression-inspired attention mechanism: revert the change in \Cref{eq:mssa_trainable_w}. That is, the attention mechanism implements \Cref{eq:SSA,eq:Multi-Head-SSA} directly.
    \item Majorization-minimization proximal step sparsification: instead of \Cref{eq:ista-block}, implement \Cref{eq:prox_maj_min_iteration}.
\end{itemize}

We obtain the following classification results in \Cref{tab:ablation-arch-variants}. 
After conducting additional simplifications to the network architecture (i.e., imposing additional constraints to the network architecture design), we discover that \ours{} maintains reasonable performance on ImageNet-1K.

\begin{table*}[ht]
\centering
\caption{\small Top 1 accuracy of \ours{} on various datasets with different architecture design variants 
when trained on ImageNet. }
\label{tab:ablation-arch-variants}
\vspace{-1mm}
\small
    \setlength{\tabcolsep}{13.6pt}
\resizebox{0.7\textwidth}{!}{%
\begin{tabular}{@{}lcc|cc|cc@{}}
\toprule
\textbf{Model} & \texttt{MSSA}-block  &   \texttt{ISTA}-block& ImageNet  \\ 
\midrule
\midrule
 \ours{-B} & default & default & 70.8 \\
 \midrule
 \ours{-B} & Eq.~\Cref{eq:SSA,eq:Multi-Head-SSA} & default & 63.3 \\
 \ours{-B} & default & Eq.~\Cref{eq:prox_maj_min_iteration} & 68.6 \\
 \bottomrule
\end{tabular}%
}
\vspace{-0.2in}
\end{table*}

\end{document}